%% file: main.tex
\documentclass[11pt,letterpaper]{article}
\usepackage[margin=1.25in]{geometry}
\usepackage{amsmath,amssymb,amsthm,mathtools,bm}
\usepackage{microtype,graphicx,booktabs,array,tabularx,float,placeins}

\usepackage{tikz,needspace}
\usetikzlibrary{arrows.meta,calc,positioning,shapes.geometric,decorations.pathreplacing}
\definecolor{figbase}{gray}{0.12}
\definecolor{figred}{RGB}{180,0,0}
\definecolor{figgray}{gray}{0.48}
\tikzset{
  >={Latex[length=1.8mm,width=1.2mm]},
  every node/.style={font=\small,inner sep=2pt},
  edge/.style={draw=figgray,line width=.65pt},
  signal/.style={draw=figbase,line width=.9pt,->},
  dot/.style={circle,fill=figbase,draw=figbase,inner sep=1.5pt},
  open/.style={circle,draw=figgray,fill=white,line width=.7pt,inner sep=1.6pt},
  stage/.style={circle,draw=figbase,fill=white,line width=.8pt,minimum size=5.6mm,inner sep=1pt}
}

\usepackage{hyperref}
\usepackage[backend=biber,style=alphabetic,natbib=true,sorting=nyt,
            maxbibnames=99,maxcitenames=3,maxalphanames=4,
            doi=false,isbn=false]{biblatex}
\usepackage[capitalize,noabbrev]{cleveref}

\newtheorem{theorem}{Theorem}[section]
\newtheorem{lemma}[theorem]{Lemma}
\newtheorem{proposition}[theorem]{Proposition}
\newtheorem{corollary}[theorem]{Corollary}
\theoremstyle{definition}

\theoremstyle{remark}
\newtheorem{remark}[theorem]{Remark}

\crefname{theorem}{theorem}{theorems}
\Crefname{theorem}{Theorem}{Theorems}
\crefname{lemma}{lemma}{lemmas}
\Crefname{lemma}{Lemma}{Lemmas}
\crefname{corollary}{corollary}{corollaries}
\Crefname{corollary}{Corollary}{Corollaries}
\crefname{appendix}{appendix}{appendices}
\Crefname{appendix}{Appendix}{Appendices}
\hypersetup{hidelinks,
  pdftitle={On Global Joint Lower Bounds in Zeroth-Order Nonconvex Optimization},
  pdfauthor={Yuxuan Xia and Xiao-Tong Yuan}}
\newcommand{\resultlink}[1]{\hyperref[#1]{\Cref*{#1}}}
\newcommand{\R}{\mathbb R}
\newcommand{\E}{\mathbb E}
\newcommand{\PP}{\mathbb P}
\newcommand{\sph}{\mathbb S}

\newcommand{\norm}[1]{\lVert #1\rVert}
\newcommand{\Lip}{\operatorname{Lip}}
\newcommand{\conv}{\operatorname{conv}}

\newcommand{\cH}{\mathcal H}

\newcommand{\ind}{\mathbf 1}
\allowdisplaybreaks[1]

\title{On Global Joint Lower Bounds in Zeroth-Order Nonconvex Optimization}
\author{Yuxuan Xia\textsuperscript{1,2}\qquad Xiao-Tong Yuan\textsuperscript{1,2}\\[.55em]
  {\small \textsuperscript{1}School of Intelligence Science and Technology, Nanjing University}\\
  {\small Suzhou 215163, China}\\
  {\small \textsuperscript{2}State Key Laboratory for Novel Software Technology, Nanjing University}\\
  {\small Nanjing 210023, China}\\[.25em]
  {\small \texttt{xiayxcn@gmail.com} \qquad \texttt{xtyuan@nju.edu.cn}}}
\date{}

\begin{document}
\maketitle

\begin{abstract}
Must the dimension and accuracy costs of
zeroth-order optimization multiply?
Global upper bounds for smooth stochastic nonconvex
optimization scale as $O(d\epsilon^{-4})$,
with both factors separately necessary.
We show that their product is unavoidable. We establish an $\Omega(d\epsilon^{-4})$ lower bound for smooth stochastic nonconvex optimization, yielding an $\Omega(d\delta^{-1}\epsilon^{-3})$ bound for nonsmooth Goldstein stationarity. Both match the best-known full-space upper bounds, establishing joint minimax optimality in dimension and accuracy. These bounds hold against fully adaptive randomized algorithms with unrestricted queries and outputs, even under arbitrary sample reuse. As a byproduct, we establish noiseless exact-value lower bounds for both smooth and Goldstein stationarity, even against randomized algorithms.
\end{abstract}

\input{sections/main_body}
\label{end:main}

\section{Open Problems}

Recent advances have brought upper and lower bounds for
zeroth-order nonconvex optimization into agreement in several
important regimes. Nevertheless, a complete minimax
characterization remains open, particularly concerning the
initial objective gap, noiseless exact-value oracles, and
the role of nonsmooth geometry.

\paragraph{Optimal dependence on the initial gap.}
What is the optimal dependence on the initial global
objective gap $\Delta$? Our basic lower bound exhibits
a capped dependence, while extensions based on recursive
constructions yield logarithmic growth beyond this range
and enlarge the linear-gap regime under additional parameter
restrictions. Neither closes the gap to general-purpose
upper bounds with linear dependence on $\Delta$.
Intriguingly, noiseless value-oracle results exhibit
logarithmic gap dependence in one dimension and square-root
dependence in two dimensions
\citep{chewi2023complexity,hollender2025computational}.
These contrasting behaviors suggest a possible dimensional
transition. Does linear dependence on $\Delta$ become
minimax-optimal above some dimension threshold?

\paragraph{Complexity of noiseless exact-value nonsmooth optimization.}
The minimax complexity of noiseless zeroth-order nonsmooth nonconvex optimization remains incompletely understood, even when algorithms are randomized and may query anywhere in $\mathbb{R}^d$. Our exact-value lower bound for Goldstein stationarity does not match known upper bounds throughout the parameter range. Is this gap due to the limitations of existing algorithms, or can stronger lower bounds be established against unrestricted randomized scalar queries?

\paragraph{Hardness beyond smooth reductions.}
A separate challenge concerns the intrinsic difficulty
of nonsmooth Goldstein stationarity. Smooth constructions
provide joint zeroth-order lower bounds through a
smooth-to-Goldstein reduction, but the resulting truncation
prevents matching the $\mathcal{O}(d\delta^{-1}\epsilon^{-3})$
upper bound throughout the parameter range.
A similar restriction appears in the smooth-based
first-order lower bound of \citet{cutkosky2023}.
Can genuinely nonsmooth constructions overcome this
restriction while preserving the multiplicative
dependence on dimension and accuracy?

\clearpage
\appendix
\section*{Proof guide}
Each construction or calculation has one primary proof below.
The following references identify the result needed at each step.
\begin{center}
\begin{tabularx}{\linewidth}{@{}lX@{}}
\toprule
Question & Result and proof \\
\midrule
Global stationary localization &
\Cref{lem:base-chain,lem:global-geometry} \\
Fixed Gaussian samples and features &
\Cref{lem:field-regularity,lem:tensor-shift} \\
Adaptive reuse and posterior information &
\Cref{lem:canonical-adaptive-kernel,lem:posterior-information} \\
Global sample Lipschitz budget & \Cref{thm:noise-moment} \\
Sequential cost and a static finite oracle &
\Cref{lem:prefix-simulator,lem:stage-charging,lem:static-coupling} \\
Physical parameters and expected cost &
\Cref{lem:uniform-parameters,lem:expected-cost} \\
Goldstein and exact-value reductions &
\Cref{lem:goldstein-comparison,lem:noiseless-list} \\
Amplification beyond the basic gap window &
\Cref{lem:wide-core,prop:two-level-assembly,prop:fixed-depth-assembly,prop:exact-nesting,prop:nested-entry} \\
\bottomrule
\end{tabularx}
\end{center}

\input{sections/external_tools}
\input{sections/geometry_proofs}
\input{sections/gaussian_proofs}
\input{sections/noise_proofs}
\input{sections/assembly_proofs}
\input{sections/scaling_and_noiseless_proofs}
\input{sections/gap_amplification_proofs}
\input{sections/fixed_depth_proofs}
\input{sections/logarithmic_gap_proofs}
\label{end:appendices}
\clearpage

\printbibliography
\end{document}

%% file: sections/main_body.tex
\section{Introduction}\label{sec:introduction}\label{sec:related-work}

Finding a global minimum of a nonconvex function is generally
intractable. A more attainable goal is stationarity, finding a
point at which first-order information indicates little room
for improvement. For a smooth objective $f:\R^d\to\R$, this
means finding $x$ such that $\|\nabla f(x)\|\le\epsilon$.
For nonsmooth objectives, where gradients may not exist,
Goldstein $(\delta,\epsilon)$-stationarity extends this
criterion by allowing convex combinations of subgradients
within a $\delta$-neighborhood of $x$
\citep{goldstein1977,zhang2020,tian2022,davis2022}.
Both criteria, however, rely on first-order information
that may be unavailable to the algorithm.

Zeroth-order optimization studies precisely this limitation,
the algorithm observes only scalar function values, yet must
find a point satisfying a first-order stationarity criterion.
Unlike a gradient, a single function value reveals no explicit
direction of improvement. Useful directional information must
instead be inferred from multiple evaluations. In stochastic
optimization, the oracle returns values $F(x;\xi)$ of random
sample functions, whose expectation defines the objective
$f(x)=\E_\xi F(x;\xi)$. The algorithm may also evaluate the
same sample function at different points, combining the
resulting information.

This information constraint introduces two distinct costs.
Recovering useful directions becomes harder as the dimension
$d$ grows, while reaching stationarity becomes harder as
the target accuracy $\epsilon$ decreases. Existing zeroth-order
upper bounds typically multiply these costs, but whether
this multiplication is unavoidable remains open. Separate
lower bounds on dimension and accuracy do not suffice,
since an algorithm might learn directional information once
and reuse it throughout optimization. We investigate whether
the product is necessary even for fully adaptive algorithms
with unrestricted query locations, sample reuse, and outputs.

Formally, let $\mathfrak I(d,H,L_0,\Delta)$ denote the class
of stochastic oracle instances whose population objectives
are globally $H$-smooth, whose initial objective gaps
$f(x_0)-\inf_{x\in\R^d}f(x)$ are at most $\Delta$, and whose
sample functions have Lipschitz constants $L(\xi)$ satisfying
$\E_\xi[L(\xi)^2]\le L_0^2$. We study the minimax
scalar-query complexity
\[
\mathcal C_{\mathrm{ZO}}(\epsilon)
:=\inf\left\{T:\exists A\in\mathcal A_T,
\sup_{I\in\mathfrak I}
\mathbb P_{A,I}\bigl(\|\nabla f_I(\widehat x_A)\|>\epsilon\bigr)
\le\tfrac13\right\},
\]
where $\mathcal A_T$ consists of randomized algorithms making
at most $T$ adaptive exact function-value queries, with
unrestricted query locations, sample reuse, and outputs.
The results in \Cref{sec:model} bound worst-case expected cost.
They also bound this fixed-budget complexity, since a budget of
$T$ queries implies expected cost at most $T$.

\paragraph{Development of zeroth-order upper bounds.}
The known upper bounds make the joint optimality question concrete. For smooth stochastic optimization,
\citet{ghadimi2013} established the $O(d\epsilon^{-4})$
benchmark, suppressing the remaining problem parameters.
For Goldstein stationarity, \citet{lin2022} related
randomized smoothing to the stationarity criterion and obtained
gradient-free guarantees. \citet{chen2023} improved the accuracy
dependence through recursive gradient estimation, and
\citet{kornowski2024} reduced the dimension dependence from
$d^{3/2}$ to $d$.
 With $L_0=\Delta=1$ and
$0<\delta,\epsilon\le1$, the corresponding bounds are
\[
 \underbrace{O\!\left(\frac{d^{3/2}}{\delta\epsilon^4}\right)}_
 {\text{Lin et al. (2022)}}
 \quad\longrightarrow\quad
 \underbrace{O\!\left(\frac{d^{3/2}}{\delta\epsilon^3}\right)}_
 {\text{Chen et al. (2023)}}
 \quad\longrightarrow\quad
 \underbrace{O\!\left(\frac{d}{\delta\epsilon^3}\right)}_
 {\text{Kornowski--Shamir (2024)}}.
\]
Restoring the Lipschitz and gap parameters, the last rate is
$O(dL_0^2\Delta/(\delta\epsilon^3))$ in the nontrivial accuracy
regime. 
More fundamentally, the stochastic function-value oracle
allows algorithms to extract information from repeated,
adaptive evaluations of the same sample function, beyond
what independent noisy observations can reveal.

\paragraph{From separate lower bounds to a joint bound.}
\label{subsec:comparison}
Existing lower bounds establish dimension and accuracy costs
in different optimization models. For stochastic convex
zeroth-order optimization, \citet{duchi2015} prove an
$\Omega(d\epsilon^{-2})$ query lower bound for
$\epsilon$-suboptimality under normalized Lipschitz and
domain-radius parameters. For smooth nonconvex stationarity,
\citet{carmon2019} establish an $\Omega(\epsilon^{-2})$
lower bound even with access to higher-order derivatives,
while \citet{arjevani2023} prove an $\Omega(\epsilon^{-4})$
lower bound for stochastic first-order oracles. For nonsmooth Goldstein stationarity, \citet{cutkosky2023}
establish an $\Omega(\Delta L_0^2/(\delta\epsilon^3))$
stochastic first-order lower bound in their parameter regime.
More recently, \citet{kornowski2026stationarity} showed
that the same rate remains necessary even with exact
first-order information, for randomized zero-respecting
algorithms.

These results do not imply a joint zeroth-order lower bound.
The convex result concerns objective suboptimality rather
than stationarity, while the remaining results address
first-order oracles. Even within a common model, separate
bounds of $\Omega(d)$ and $\Omega(\epsilon^{-4})$ imply only
$\Omega(\max\{d,\epsilon^{-4}\})$, not
$\Omega(d\epsilon^{-4})$. An algorithm might recover
directional information once and reuse it throughout
optimization. \citet{kornowski2024} explicitly identify
this gap and propose resolving it through a smooth joint
$\Omega(d\epsilon^{-4})$ zeroth-order lower bound.

Turning to joint zeroth-order lower bounds,
\citet{zhang2026} establish
$\Omega(dL_0^2\Delta/(\delta\epsilon^3))$ scalar
evaluations on Euclidean balls, allowing randomized
adaptive algorithms and repeated use of sampled functions.
Their instances contain a stationary point in the ball,
and the dimension requirement is only logarithmic in
the normalized query bound.

However, the known zeroth-order upper bounds apply globally
over $\R^d$, with the initial objective gap measured against
the global infimum. In contrast, \citet{zhang2026} restrict
both queries and outputs to a Euclidean ball and measure
the gap only within that ball. This distinction is
structural rather than merely spatial. Removing the
restriction simultaneously strengthens the algorithm
and imposes stricter requirements on the hard instance.
Queries outside the hard region may reveal hidden
information prematurely, while unrestricted outputs
may exploit unintended stationary points elsewhere.
Meanwhile, the construction must preserve global
smoothness, sample Lipschitz bounds, and a finite global
objective gap. These requirements are intertwined, as
modifying the function outside the ball may compromise
the information-hiding mechanism or introduce new
stationary points. Indeed, the linear carrier of
\citet{zhang2026} is unbounded below on $\R^d$,
precluding a direct extension of their hard instance.
Thus, matching rates on bounded domains do not settle
joint minimax optimality in the unrestricted global
setting. Whether the same multiplicative hardness
persists under these global requirements remains open.
\Cref{tab:prior-rates} compares the zeroth-order
bounds in their respective models.

\paragraph{Global joint lower bounds.}
We address this unrestricted problem for a finite family of static
sample functions. The algorithm may query anywhere in $\R^d$,
adaptively choose and reuse a sample label, and return any point.
For an $H$-smooth population objective, our basic lower bound is
\begin{equation}\label{eq:intro-main}
 \Omega\!\left(\frac{dL_0^2}{\epsilon^4}
                      \min\{\Delta H,L_0^2\}\right),
\end{equation}
under the dimension and accuracy conditions in \Cref{thm:main}.
The choice $H=\epsilon/\delta$ transfers this result to Goldstein
stationarity, giving the joint
$\Omega(dL_0^2\Delta/(\delta\epsilon^3))$ rate when
$\Delta\epsilon/\delta\le L_0^2$. In this range the dimension and
accuracy dependence of the known upper bound is necessary even
with unrestricted queries and outputs.

We further investigate two extensions aimed at improving the dependence on the initial objective gap. Recursive outer chains
enlarge the linear-gap range at each fixed depth; exact nesting
gives logarithmic growth for arbitrarily large gaps, subject to
its dimension condition. Their ranges and constants are compared
in \Cref{tab:gap-rates}. We also obtain a separate exact-value
lower bound against randomized algorithms by removing the noise
and analyzing the resulting scalar measurements.


\paragraph{Organization.}
\Cref{sec:model} states the model and the basic theorem.
\Cref{sec:proof-overview} explains why the stage and direction costs
multiply. \Cref{subsec:larger-gaps} explains the gap extensions, and
\Cref{sec:consequences} gives the Goldstein and exact-value results.
The appendices follow these dependencies, beginning with formal
statements of the external probabilistic tools
(see \Cref{thm:tool-gaussian,thm:tool-lsi,thm:tool-sphere}).

\section{Oracle Model and Main Results}
\label{sec:model}\label{sec:main-theorem}

Let $\Xi$ be finite and
$f=|\Xi|^{-1}\sum_{\xi\in\Xi}f_\xi$ on $\R^d$.
With initial point $0$, assume
\begin{equation}\label{eq:model-assumptions}
 \inf f>-\infty,\quad f(0)-\inf f\le\Delta,\quad
 \Lip(\nabla f)\le H,\quad
 \frac1{|\Xi|}\sum_{\xi\in\Xi}\Lip(f_\xi)^2\le L_0^2.
\end{equation}
Here $\Lip$ denotes the global Lipschitz constant.
An algorithm may randomize, adaptively query any
$(\xi,x)\in\Xi\times\R^d$, reuse labels, and output any point.
Each exact scalar evaluation costs one query.
Throughout, lower bounds concern worst-case expected query
cost for success probability at least $2/3$ on every
admissible instance. The finite family may depend on the
algorithm; the model also covers fresh-sample access.

\paragraph{Notation.}\label{sec:notation}\label{subsec:oracle}
Norms are Euclidean, matrix norms are operator norms,
logarithms are natural, and $c,C$ denote universal
positive constants unless specified otherwise. Set
\begin{equation}\label{eq:def-B}
 B=\min\{\Delta H,L_0^2\}.
\end{equation}

\begin{theorem}[Basic global stochastic lower bound]
\label{thm:main}
For universal $c,C,c_0>0$, suppose
\begin{equation}\label{eq:main-regime}
 \epsilon^2\le c_0B,\qquad
 d\ge C\frac{B}{\epsilon^2}
       \log\left(2+\frac{BL_0^2}{\epsilon^4}\right).
\end{equation}
Finding $\|\nabla f(\widehat x)\|\le\epsilon$
under \eqref{eq:model-assumptions} requires
\begin{equation}\label{eq:main-rate}
 N\ge c\frac{dL_0^2B}{\epsilon^4}.
\end{equation}
\end{theorem}

The rate is linear in $\Delta H$ when
$\Delta H\le L_0^2$; the cap reflects the gradient
budget of the basic construction and is extended later.
\Cref{sec:sequential} proves the bound. The parameter selection
satisfies all dimension and budget conditions in \eqref{eq:main-regime}
(see \Cref{lem:uniform-parameters}, Eq.~\eqref{eq:app-uniform-prefactor}).

\paragraph{Goldstein stationarity.}
For locally Lipschitz $f$, let $\partial_C f$ denote
the Clarke subdifferential and define
\begin{equation}\label{eq:goldstein-def}
 \partial_\delta^G f(x)
 =\conv\!\bigcup_{\|y-x\|\le\delta}\partial_C f(y),
 \qquad
 g_\delta(x)=\inf_{v\in\partial_\delta^Gf(x)}\|v\|.
\end{equation}
A Goldstein $(\delta,\epsilon)$-stationary point
satisfies $g_\delta(x)\le\epsilon$
\citep{goldstein1977,clarke1990}.

\begin{theorem}[Stochastic Goldstein lower bound]
\label{thm:goldstein}
Let $B_G=\min\{\Delta\epsilon/\delta,L_0^2\}$.
Under \eqref{eq:main-regime} with $B$ replaced by
$B_G$ (for suitable universal constants), finding a
Goldstein $(\delta,\epsilon)$-stationary point requires
\begin{equation}\label{eq:goldstein-rate}
 N=\Omega\left(
 dL_0^2\min\left\{
 \frac{\Delta}{\delta\epsilon^3},
 \frac{L_0^2}{\epsilon^4}\right\}\right).
\end{equation}
This holds for globally Lipschitz, possibly nonsmooth
population objectives with global gap at most $\Delta$
and sample Lipschitz moment at most $L_0^2$.
\end{theorem}

The implication \eqref{eq:goldstein-reduction} in
\Cref{subsec:goldstein} reduces this task to smooth stationarity;
substituting its tolerance gives the stated rate and hypotheses
(see \Cref{lem:goldstein-comparison} and the
\hyperref[proof:goldstein-bound]{proof of \Cref*{thm:goldstein}}).
The same reduction gives the larger-gap extensions
(see \Cref{cor:large-gap-goldstein}).

\paragraph{Exact-value lower bounds.}
Here the oracle evaluates a single fixed function;
randomization, queries, and outputs remain unrestricted.

\begin{theorem}[Exact-value lower bound]
\label{thm:noiseless}
Let $B$ be as in \eqref{eq:def-B}. If
\[
 \epsilon^2\le c_0B,\qquad
 d\ge C\frac{B}{\epsilon^2}
       \log\left(2+\frac{B}{\epsilon^2}\right),
\]
then finding an $\epsilon$-stationary point of a
globally $L_0$-Lipschitz, $H$-smooth function with
global initial gap at most $\Delta$ requires
\begin{equation}\label{eq:noiseless-rate}
 N=\Omega(dB/\epsilon^2)
\end{equation}
exact evaluations, even for randomized algorithms.
\end{theorem}

\begin{corollary}[Exact-value Goldstein lower bound]
\label[corollary]{cor:exact-goldstein}
Let $B_G=\min\{\Delta\epsilon/\delta,L_0^2\}$.
Under the accuracy and dimension conditions of
\Cref{thm:noiseless} with $B=B_G$, finding a Goldstein
$(\delta,\epsilon)$-stationary point of a globally
$L_0$-Lipschitz function with global gap at most
$\Delta$ requires
$\Omega(dB_G/\epsilon^2)$ exact evaluations.
\end{corollary}

In the unsaturated regime, the smooth and Goldstein
rates are respectively
$\Omega(d\Delta H/\epsilon^2)$ and
$\Omega(d\Delta/(\delta\epsilon))$.
The smooth rate matches the standard exact-value
upper bound. \Cref{subsec:noiseless} reduces scalar replies to
linear measurements. Its list estimate and scaling then give the
expected-cost bound (see \Cref{lem:noiseless-list} and the
\hyperref[proof:exact-bound]{proof of \Cref*{thm:noiseless} and
\Cref*{cor:exact-goldstein}}).

\paragraph{Gap amplification.}
The next results extend the dependence on $\Delta$
under the same oracle and success conventions.

\begin{theorem}[Two-level]\label{thm:extended-gap}
There are universal $c,c_e,C_*,C>0$ such that
\begin{equation}\label{eq:extended-gap-regime}
 \epsilon\le c_eL_0,\quad
 \Delta H\ge C_*L_0^2,\quad
 d\ge C\frac{L_0^2}{\epsilon^2}
       \log\left(2+\frac{L_0^2}{\epsilon^2}\right)
\end{equation}
imply
\begin{equation}\label{eq:extended-gap-rate}
 N\ge c\frac{dL_0^2}{\epsilon^4}
       \min\left\{\Delta H,\frac{L_0^4}{\epsilon^2}\right\}.
\end{equation}
\end{theorem}

\begin{corollary}[Fixed-depth]
\label[corollary]{cor:fixed-depth}
For every fixed integer $p\ge2$, there are positive
$c_p,c_{e,p},C_p,C_{*,p}$ such that
\[
 \epsilon\le c_{e,p}L_0,\quad
 \Delta H\ge C_{*,p}L_0^2,\quad
 d\ge C_p\frac{L_0^2}{\epsilon^2}
       \log\left(2+\frac{L_0^2}{\epsilon^2}\right)
\]
imply
\begin{equation}\label{eq:fixed-depth-rate}
 N\ge c_p\frac{dL_0^2}{\epsilon^4}
       \min\left\{\Delta H,\,
       L_0^2\left(\frac{L_0^2}{\epsilon^2}\right)^{p-1}
       \right\}.
\end{equation}
\end{corollary}

The two-level construction is the $p=2$ case.
Fixed depth enlarges the linear gap window, with
constants depending on $p$; exact nesting instead
gives a logarithmic dependence for larger gaps.

\begin{theorem}[Logarithmic gap]
\label{thm:logarithmic-gap}
There are universal positive $c,c_e,C_0,C$ and
$C_1>1$ such that, if
$\epsilon\le c_eL_0$, $\Delta H\ge C_0L_0^2$, and
\begin{equation}\label{eq:logarithmic-gap-regime}
 K=1+\left\lfloor
 \frac{\log(\Delta H/(C_0L_0^2))}{\log C_1}
 \right\rfloor,\qquad
 d\ge C\left[
 \frac{L_0^2}{\epsilon^2}
 \log\left(2+\frac{L_0^2}{\epsilon^2}\right)
 +\log(2K)\right],
\end{equation}
then
\begin{equation}\label{eq:logarithmic-gap-rate}
 N\ge c\frac{dL_0^4}{\epsilon^4}
       \left[1+\log\frac{\Delta H}{C_0L_0^2}\right].
\end{equation}
\end{theorem}

\Cref{subsec:larger-gaps} describes the common first-entry
interface. Two-level amplification combines the scaled budgets
with admissibility conditioning (see the
\hyperref[proof:two-level-bound]{two-level proof},
Eqs.~\eqref{eq:linear-scaled-budgets} and \eqref{eq:linear-conditioned-success}).
Fixed-depth amplification iterates the same interface
(see \Cref{prop:fixed-depth-assembly}, Eq.~\eqref{eq:depth-rate-calculation}).
Exact nesting instead combines geometric insertion with repeated
first-entry costs (see \Cref{prop:exact-nesting,prop:nested-entry} and
the \hyperref[proof:logarithmic-bound]{logarithmic-gap proof}).

\subsection{Comparison of the bounds}
Table~\ref{tab:prior-rates} compares dimension and accuracy dependence
in the respective oracle models. \Cref{tab:gap-rates} compares the
gap dependence together with each theorem's conditions.

\input{sections/comparison_table}

\input{sections/gap_table}

\paragraph{The basic bound and its extensions.}
\Cref{thm:main} covers the small-gap regime
$\epsilon^2\lesssim\Delta H\le L_0^2$ as well as the capped regime.
The two-level and fixed-depth theorems are stated for
$\Delta H\ge C_*L_0^2$, with the constant depending on the depth
in the latter case. They therefore do not replace all the basic
theorem's parameter range. On their common large-gap range, the
extensions improve the asymptotic bound after the basic cap is reached.
The strongest applicable lower bound is the one to use there.
The basic construction provides the geometric and
information-theoretic foundations for both extensions
and the exact-value lower bounds.

\paragraph{Depth and accuracy.}
At fixed depth $p$, the linear-gap window extends to
$\Delta H\le L_0^2(L_0^2/\epsilon^2)^{p-1}$.
The accuracy condition and lower-bound constant depend on $p$.
For the present recursion, enforcing a fixed Lipschitz budget while
letting $p$ grow eventually makes the public starting point
$\epsilon$-stationary. \Cref{subsec:larger-gaps} derives this restriction by comparing
the origin gradient with the Lipschitz budget after rescaling
(see \Cref{prop:depth-loss}, Eqs.~\eqref{eq:depth-origin-gradient}
and \eqref{eq:depth-lipschitz-lower}).

\paragraph{Logarithmic amplification without a gap cutoff.}
\Cref{thm:logarithmic-gap} instead applies to every
$\Delta H\ge C_0L_0^2$, provided the dimension condition in the last
row holds. Its constants are independent of the gap. At fixed
$L_0/\epsilon$, the extra dimension needed as the gap grows is
$O(\log\log(\Delta H/L_0^2))$. For each fixed $p$, this logarithmic
bound eventually exceeds the saturated fixed-depth rate, up to their
constant factors; in the unsaturated fixed-depth window, the latter
has the stronger linear gap dependence.

\subsection{Proof Overview}
\label{sec:difficulties}

Our construction takes a different geometric route from the bounded-domain approach of \citet{zhang2026}, requiring accompanying analytical tools that go beyond a straightforward adaptation of their methods. The basic lower bound combines sequential optimization
hardness with the information cost of recovering hidden
directions. Building on the smooth zero-chain of
\citet{carmon2019}, we construct a sequence of
$T=\Theta(B/\epsilon^2)$ stages, each associated with
a fresh hidden direction. The central challenge is to
ensure that these stages remain sequentially hard even
under unrestricted adaptive scalar queries.

We must rule out three shortcuts available
to fully adaptive algorithms.
First, it may bypass the intended stage order by
querying later stages or outputting an unqueried
stationary point.
Second, repeated or correlated sample evaluations
may cancel the noise and reveal hidden directions
at negligible cost.
Third, queries at earlier stages may disclose future
directions, allowing the algorithm to amortize
information costs across stages.

We rule out these shortcuts through three complementary
mechanisms.
Global geometric constraints ensure that approximate
stationarity anywhere in $\R^d$ certifies completion of
all stages, even for unqueried outputs
\eqref{eq:overview-certificate}.
A Gaussian feature channel matched to the nonlinear mean
prevents sample reuse from bypassing the one-stage
information bound (\Cref{thm:reuse}).
Finally, exact prefix-based simulation keeps future
directions hidden conditional on the past, allowing
the stagewise costs to accumulate
(\Cref{lem:prefix-simulator,lem:stage-charging}).
The construction also preserves the prescribed global
Lipschitz moment bound for the sample functions without low-cost cross-label noise cancellation.
Together, these mechanisms yield
\begin{equation}\label{eq:product-plan}
\underbrace{\text{scalar queries}}_{N}
\;\gtrsim\;
\underbrace{\text{hidden stages}}_{T}
\times
\underbrace{\text{queries per stage}}_{dL_0^2/\epsilon^2}
\;=\;
\Omega\!\left(\frac{dL_0^2B}{\epsilon^4}\right).
\end{equation}
The gap-amplification results reuse the mechanisms, while
the exact-value lower bounds follow from a separate
analysis of noiseless scalar measurements.

\FloatBarrier

\section{Construction and Proof of the Joint Bound}\label{sec:proof-overview}\label{subsec:roadmap}
We first construct a normalized hard instance and then rescale it.
Split the variables into chain coordinates $z\in\R^r$ and workspace
coordinates $x\in\R^m$, where $r+m=d$ and $r,m\asymp d$.
Independent 
Haar  orthonormal frames\footnote{A Haar orthonormal frame consists of uniformly random,
mutually orthogonal unit vectors, with a distribution invariant
under rotations. Given the preceding vectors, the next remains
uniformly distributed in their orthogonal complement.
In our setting, preservation of this property under adaptive
query histories is established later through the prefix simulation.}
$U=(u_1,\ldots,u_T)$ and $V=(v_1,\ldots,v_T)$ hide the chain
directions and the workspace directions, respectively. The first frame
prevents premature activation; the second supplies the scalar
estimation cost. Keeping their roles separate is essential for the
conditional law at the start of each stage.

The proof separates geometry from information. A small
gradient requires reaching the end of the chain and the final
workspace center. Informationally, reaching that center requires
learning a new $v_k$ before seeing the next workspace direction.
The geometric certificate alone does not bound how many centers
a query can discover. The information bound makes the total cost
proportional to the number of stages.

\subsection{Global geometry and stage isolation}\label{sec:geometry}

\paragraph{Key idea.}
The zero-chain provides $T$ stages, all of which must
be completed to reach a small-gradient point.
The bounded chain map \eqref{eq:soft} controls the hidden projections
regardless of query distance, while radial confinement
prevents stationary outputs from escaping to infinity.
A workspace well with a moving center introduces a hidden direction
at each stage, creating a separate information cost.
Gate correction suppresses successor information until
the current workspace gate is reached.

\paragraph{Smooth zero-chain and radial confinement.} We start with the smooth zero-chain of \citet{carmon2019},
which enforces sequential coordinate activation for
zero-respecting algorithms, each successor term remains
inactive until its preceding coordinate crosses a threshold.
The length-$T$ chain is defined by

\begin{equation}\label{eq:chain-profiles}
\Psi(t)=
\begin{cases}
0,&t\le1/2,\\
\exp(1-(2t-1)^{-2}),&t>1/2,
\end{cases}
\qquad
\Phi(t)=\sqrt e\int_{-\infty}^t e^{-u^2/2}\,du,
\end{equation}
\begin{equation}\label{eq:chain-terms}
\begin{gathered}
b_1(w)=-\Phi(w_1),\qquad  b_i(w)=\Psi(-w_{i-1})\Phi(-w_i)
-\Psi(w_{i-1})\Phi(w_i),\qquad
B_T(w)=\sum_{i=1}^T b_i(w).
\end{gathered}
\end{equation}
Here $\Psi$ vanishes on $(-\infty,1/2]$, and $\Phi$ is a smooth
increasing Gaussian primitive. Thus $|w_{i-1}|\le1/2$ makes the entire
successor term vanish. The derivatives also have a consistent sign (see Eq.~\eqref{eq:app-chain-sign}).
At the first coordinate with $|w_i|\le1$, after earlier coordinates
have passed that threshold, $\partial_iB_T\le-1$ (see Eq.~\eqref{eq:app-chain-barrier-proof}). This supplies
the chain's obstruction to stopping early.

In particular, the coordinate barrier implies
\[
 \|\nabla B_T(w)\|<1\quad\Longrightarrow\quad |w_i|>1
 \quad\text{for every }i.
\]
Together with sequential activation,
this suffices to enforce sequential progress for
zero-respecting algorithms. However, unrestricted
queries may activate future coordinates directly,
bypassing this sequential structure.
We hide the chain directions using the Haar frame $U$.
However, evaluating the chain directly at $U^\top z$
would allow unrestricted queries to amplify these
projections beyond the activation threshold simply
by increasing $\|z\|$. We prevent this by mapping
the chain-coordinate input into a bounded region:
\begin{equation}\label{eq:soft}
 R=64\sqrt T,\qquad
 \alpha=(1+\|z\|^2/R^2)^{-1/2},\qquad
 y=\alpha z,\qquad w=U^\top y.
\end{equation}

Now $\|y\|<R$, and
$w_i=\langle u_i,y\rangle=\alpha\langle u_i,z\rangle$
is the coordinate of the bounded input $y$ along
the hidden chain direction $u_i$. By \eqref{eq:chain-profiles}, $|w_i|\le1/2$ implies
$\Psi(\pm w_i)=0$ and hence $b_{i+1}(w)=0$. For each $i$, conditioned on the preceding Haar
columns $U_{<i}=(u_1,\ldots,u_{i-1})$, the direction
$u_i$ remains uniform in their orthogonal complement.
For a fixed bounded query $y$, spherical concentration
then ensures that $|w_i|=|\langle u_i,y\rangle|\le1/2$
with high probability, provided the remaining dimension
is sufficiently large (see \Cref{thm:tool-sphere,lem:static-coupling},
Eq.~\eqref{eq:app-pbad}). 
The extension from fixed to adaptive queries is
established later through the prefix simulation,
with the failure probability bounded in
\eqref{eq:bad-projection}.

The bounded chain map \eqref{eq:soft} alone is insufficient: its derivative
vanishes at large distances, potentially creating
spurious small-gradient points. We therefore add
radial confinement:
\begin{equation}\label{eq:confinement}
 P(z)=\frac{R^2}{5}(\alpha^{-1}-1),\qquad
 F_{\rm base}(z)=P(z)+B_T(w).
\end{equation}
Indeed, $\nabla P(z)=y/5$, whose norm stays bounded
away from zero as $\|z\|\to\infty$, preventing the
bounded chain map from introducing easy stationary outputs
at large distances.

These two modifications work together. For $\|z\|\ge R$,
the outward radial derivative of $P$ dominates the
chain contribution, excluding small-gradient points
outside the ball (see Eq.~\eqref{eq:app-base-radial-barrier}). For $\|z\|<R$, the first coordinate
with $|w_i|\le1$ satisfies $\partial_i B_T\le-1$ (see Eq.~\eqref{eq:app-chain-barrier-proof}),
and the radial correction cannot cancel this gradient (see Eq.~\eqref{eq:app-base-component-barrier}).
A further sign argument excludes $w_i<-1$ (see Eq.~\eqref{eq:app-base-negative-coordinate}).
Together these properties give the all-space localization
(see \Cref{lem:base-chain}):
\begin{equation}\label{eq:base-localization}
 \|\nabla F_{\rm base}(z)\|\le1/5
 \quad\Longrightarrow\quad w_i>1\quad(1\le i\le T).
\end{equation}

We need the stronger condition $w_i>1$, rather than
merely $w_i>1/2$: crossing $1/2$ allows the successor
term to become active, but does not overcome the
zero-chain gradient barrier. Indeed, as long as some
coordinate satisfies $|w_i|\le1$, the original
zero-chain cannot have a small gradient. Thus,
every small-gradient output must cross the positive
threshold $1$ in all $T$ coordinates. This establishes
the geometric necessity of completing the entire
chain; the information-theoretic analysis will
establish the query cost of doing so.

\begin{remark}
\itshape
The stronger conclusion $w_i>1$ is not needed merely
to activate the stage gates, which depend on $|w_i|$.
Its positive sign becomes useful in the later
gap-amplification constructions, where one-sided
switches and convexity on the region
$\{w:w_i\ge1,\ \forall i\}$ enable localization
near the terminal minimum.

\end{remark}

\begin{remark}[Global regularity of the base objective]
\itshape
The construction also preserves the global regularity
needed for the prescribed function class.
The gap and global derivative bounds are
(see \Cref{lem:base-chain}, Eqs.~\eqref{eq:app-base-gap}
and \eqref{eq:app-base-regularity-proof})
\begin{equation}\label{eq:base-regularity}
\begin{gathered}
F_{\rm base}(0)-\inf F_{\rm base}\le12T,\\
\sup_z\|\nabla F_{\rm base}(z)\|\le37\sqrt T,\qquad
\sup_z\|\nabla^2F_{\rm base}(z)\|\le161.
\end{gathered}
\end{equation}
Thus, the chain length increases the objective gap
only linearly, while the smoothness remains bounded
independently of $T$. These bounds provide the
geometric regularity required for the subsequent
rescaling.
\end{remark}

So far, \textbf{the zero-chain has provided $T$ stages}, all of which must
be completed to reach a small-gradient point. To obtain a
multiplicative lower bound, we must make each stage individually
costly \textbf{and} prove that these costs accumulate under unrestricted
queries. We next address the per-stage difficulty, setting aside
for now the challenge of multiplication, which is
addressed in \Cref{sec:sequential}.
To this end, we couple chain activation to a workspace well
whose center moves along a fresh hidden direction $v_k$.
Locating the new center requires estimating $v_k$,
thereby introducing a separate information cost at each stage,
as enforced by the gate construction below.

\paragraph{Smooth activation and the workspace well.} To move the well without discontinuities, we need smooth
activation functions that are exactly zero before a stage and exactly
one after activation, with vanishing first and second derivatives
at the endpoints. Let $S_5$ be the clamped quintic interpolant
$10u^3-15u^4+6u^5$ on $[0,1]$, extended as zero for
$u\le0$ and one for $u\ge1$. This choice is $C^2$ across
the joins. Define

\begin{equation}\label{eq:smooth-activation}
\Gamma(t)=S_5(4t-1),\qquad
g_i(w)=\Gamma(|w_i|).
\end{equation}
 Although $|w_i|$ is not differentiable at zero, $g_i$ is
identically zero in a neighborhood of zero; the flat endpoints of
$S_5$ ensure that $g_i$ is globally $C^2$. Its first and second
derivatives vanish in the fully inactive and active regions (see Eq.~\eqref{eq:app-activation-derivatives}).

Using the activations \eqref{eq:smooth-activation}, we define the workspace well
$h(x-a(w))$, where
\begin{equation}\label{eq:workspace-well}
h(q)=4(\sqrt{1+\|q\|^2}-1),\qquad
a(w)=A\sum_{i=1}^T g_i(w)v_i,\qquad
a_{<i}(w)=A\sum_{j<i}g_j(w)v_j,
\end{equation}
and $A>0$ is a displacement scale to be fixed below.
For fixed $w$, the well is minimized at $x=a(w)$, with
\[
\nabla_x h(x-a(w))
=\frac{4(x-a(w))}{\sqrt{1+\|x-a(w)\|^2}}.
\]
Its gradient is globally bounded, and
\begin{equation}\label{eq:well-localization}
\begin{aligned}
\|\nabla_x h(x-a(w))\|\le\eta<4
&\quad\Longrightarrow\quad
(16-\eta^2)\|x-a(w)\|^2\le\eta^2\\
&\quad\Longrightarrow\quad
\|x-a(w)\|\le\frac{\eta}{\sqrt{16-\eta^2}}.
\end{aligned}
\end{equation}
The middle inequality follows by squaring the displayed gradient
formula and multiplying by $1+\|x-a(w)\|^2$.
Thus, a small workspace gradient requires proximity to
the moving center $a(w)$.

By \eqref{eq:workspace-well}, as $g_k$ changes from
zero to one, the center shifts by $Av_k$.
The gate construction below makes locating this new center
necessary, thereby imposing a constant-angle estimation task
for the hidden direction $v_k$ at each stage.

\paragraph{Gate correction.}
The workspace well \eqref{eq:workspace-well} introduces a direction-estimation task,
but does not by itself ensure that the algorithm \textbf{must} solve
it before proceeding. Activating a chain coordinate alone must not reveal its
successor. We cancel the successor contribution outside
the workspace gate, so any query that can expose the next
stage must first approach the current hidden center.

The sum $F_{\rm base}+h(x-a)$ still has a defect. Its chain term
$b_{k+1}$ can depend on $u_{k+1}$ as soon as $w_k$ is activated.
For example, on the branch $w_k>1/2$, \eqref{eq:chain-terms} gives
\[
b_{k+1}(w)
=-\Psi(w_k)\Phi(\langle u_{k+1},y(z)\rangle),
\]
even if the workspace query has not approached the new center.
Such values could reveal the next stage before the current
estimation problem is solved. 
We remove precisely this dependence until the query
reaches the appropriate workspace region:

\begin{equation}\label{eq:mean-definition}
\begin{gathered}
F_{\rm new}(z,x)=
\underbrace{P(z)+B_T(w)}_{F_{\rm base}(z)}
+\underbrace{h(x-a(w))}_{\text{moving well}}
-\underbrace{\sum_{i=1}^T[1-\chi(x-a_{<i}(w))]c_i(w)}_{\text{gate correction}},\\
c_i(w)=b_i(w)\tau(\|g_{\ge i}(w)\|),\qquad
g_{\ge i}(w)=(g_i(w),\ldots,g_T(w)).
\end{gathered}
\end{equation}

To localize the correction to nearly inactive suffixes
and the appropriate workspace region, we introduce
two smooth cutoffs. 
Using the previously defined $S_5$, define
\begin{equation}\label{eq:cutoff-profiles}
\tau(t)=1-S_5(2t/\rho_0-1),\qquad
\chi(q)=1-S_5(\|q\|^2/\beta^2-1).
\end{equation}
In \eqref{eq:mean-definition}, the suffix cutoff $\tau$ restricts the correction to
nearly inactive suffixes. The spatial cutoff satisfies
$\chi=1$ on the radius-$\beta$ ball and $\chi=0$ outside
the radius-$\sqrt2\beta$ ball. Thus the factor $1-\chi$
makes the correction vanish in the inner ball; where
$\tau=1$ and $\chi=0$, it cancels the corresponding chain term.

Using the past center $a_{<i}(w)$ ensures that
the $i$-th correction depends only on preceding
workspace directions $v_1,\ldots,v_{i-1}$,
without exposing $v_i,\ldots,v_T$.

Write $a_k=a_{<k}+Ag_kv_k$ for the center through stage $k$.
The stage-$k$ gate is defined by
\begin{equation}\label{eq:workspace-gate}
|w_k|\ge\tfrac12,\qquad
\|x-a_{<k}-Av_k\|<\sqrt2\beta.
\end{equation}

Suppose the future activations vanish, namely
$g_{k+1}=\cdots=g_T=0$. We show that accessing
the successor chain term, which depends on the
hidden direction $u_{k+1}$, requires hitting the
stage-$k$ gate \eqref{eq:workspace-gate}.

If $|w_k|<1/2$, the zero-chain itself gives
$b_{k+1}=0$. Otherwise, $g_k=1$ by
\eqref{eq:smooth-activation}, while
$\tau(\|g_{\ge k+1}\|)=\tau(0)=1$ implies
$c_{k+1}=b_{k+1}$. Since
$a_k=a_{<k}+Av_k$, a query outside the workspace
gate satisfies $\chi(x-a_k)=0$, yielding
\[
\underbrace{b_{k+1}}_{\text{from }B_T}
-\underbrace{[1-\chi(x-a_k)]b_{k+1}}_{\text{gate correction}}
=\chi(x-a_k)b_{k+1}=0.
\]

Thus, the successor contribution vanishes whenever
the gate is missed (see Eq.~\eqref{eq:app-failed-mean}).

Conversely, entering the gate yields
\[
\left\|\frac{x-a_{<k}}{A}-v_k\right\|
<\frac{\sqrt2\beta}{A}.
\]
Hence, locating the new center provides a
constant-accuracy estimate of $v_k$ once the
scales are fixed. The gate correction therefore
makes access to successor information contingent
on resolving the current-stage estimation task,
without requiring the algorithm to move continuously
between query locations.

\Cref{fig:coupled-geometry} collects the coupling between the two
coordinate blocks, with future activations set to zero.
Radial compression bounds the chain projections. A fully activated $w_k$
moves the workspace center by $Av_k$, but the successor still
appears only through $\chi(x-a_k)b_{k+1}$. Thus chain activation
alone does not expose the next hidden direction: the query must
also locate the workspace gate. The arrows show functional
dependence and the change of center, not a required query path.
\begin{figure}[!htbp]
\centering
\input{figures/coupled_geometry}
\caption{Hidden chain geometry and workspace authentication (schematic).}
\label{fig:coupled-geometry}
\end{figure}
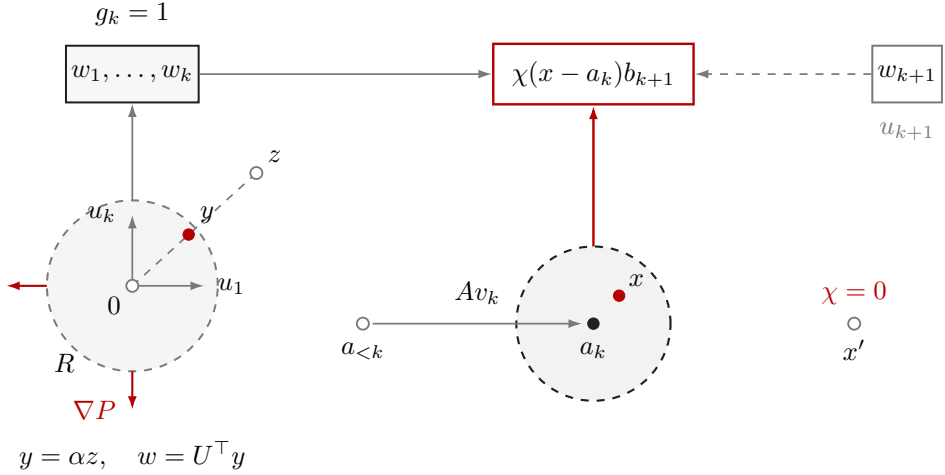

\paragraph{Global regularity and stationarity.}
We now fix $\rho_0=3/32$, $\beta=16$, and $A=128$.
The condition $A>2\sqrt2\beta$ separates successive
workspace gates and ensures constant accuracy
in the direction estimate above.

We next verify that the gate correction does not
introduce spurious small-gradient points.
At any point, at most one correction contributes
nonzero derivatives through order two.\footnote{
An active correction at index $i$ requires
$\|g_{\ge i}\|<\rho_0$, since $\tau$ and its first
two derivatives vanish beyond $\rho_0$.
If another correction at $j>i$ were active,
the zero-chain property would require
$|w_{j-1}|>1/2$, hence $g_{j-1}=1$ by
\eqref{eq:smooth-activation}, contradicting
$\|g_{\ge i}\|<\rho_0<1$ (see Eq.~\eqref{eq:app-single-correction-support}).
}
For such a correction, the suffix cutoff requires
$\|g_{\ge i}\|\le\rho_0$, while the spatial correction
factor $1-\chi(x-a_{<i})$ can contribute only when
$\|x-a_{<i}\|>\beta$. By the orthonormality of the $v_j$,
\[
\|a-a_{<i}\|=A\|g_{\ge i}\|\le A\rho_0=12,
\qquad \|x-a_{<i}\|>\beta.
\]
The reverse triangle inequality therefore gives
\[
\|x-a\|\ge\|x-a_{<i}\|-\|a-a_{<i}\|
>\beta-A\rho_0=4.
\]
Consequently, the workspace well has gradient
magnitude greater than $16/\sqrt{17}$.
The correction contributes at most $3$ to the
workspace gradient by the cutoff derivative bounds
(see Eqs.~\eqref{eq:app-correction-amplitude} and \eqref{eq:app-spatial-cutoff-bounds}). Hence,
\[
\|\nabla_xF_{\rm new}\|
\ge\|\nabla h(x-a)\|
-\left\|\sum_i c_i\nabla\chi(x-a_{<i})\right\|
>\frac{16}{\sqrt{17}}-3>\frac12.
\]
Thus, the correction cannot create small-gradient
points: any sufficiently small-gradient point must
lie where the correction vanishes.

To recover the chain coordinates, one more step is needed.
A gradient norm at most $e_*$ first forces $\|x-a\|<1$,
where all correction derivatives vanish
(see \Cref{lem:global-geometry}). The chain rule then gives
\[
 \|\nabla F_{\rm base}\|
 \le(1+\|D_za\|)\|\nabla F_{\rm new}\|
 \le961e_*\le\frac15.
\]
The first inequality accounts for the moving center; the second
uses $\|D_za\|\le960$ and $e_*=1/4805$.
The base certificate \eqref{eq:base-localization} now gives $w_i>1$.
Hence all $g_i=1$ and $D_za=0$, which improves the base-gradient
bound to $e_*$ itself. These are the two localization passes (see
Eqs.~\eqref{eq:app-first-base-pass} and \eqref{eq:app-second-base-pass}).
Together with the workspace-well bound \eqref{eq:well-localization}, they give the
stationarity certificate (see Eqs.~\eqref{eq:stationary-certificate}
and \eqref{eq:app-terminal-distance}):

\begin{equation}\label{eq:overview-certificate}
\|\nabla F_{\rm new}\|\le e_*
\quad\Longrightarrow\quad
w_i>1\ (\forall i),\qquad
\left\|x-A\sum_i v_i\right\|
\le\frac{e_*}{\sqrt{16-e_*^2}}
\underbrace{<\sqrt{2}\beta}_{\text{terminal gate radius}}.
\end{equation}

The objective is also coercive, with initial gap $O(T)$,
global gradient bound $O(\sqrt T)$, and Hessian bound $O(1)$
(see \Cref{lem:global-geometry}, Eqs.~\eqref{eq:app-gap-identity},
\eqref{eq:app-full-gradient-bound}, and \eqref{eq:app-Hstar}).
The cutoff profiles have the required regularity
(see \Cref{lem:chain-derivatives,lem:global-geometry}).

Finally, \eqref{eq:overview-certificate} turns any
successful stationary-point output into a test of the stage-$T$ gate \eqref{eq:workspace-gate}
(a virtual test requiring no additional oracle call),
even if the algorithm guesses that output after
$N$ adaptive queries without ever querying the
output point.\footnote{The final output need not be an uninformed random guess.
An algorithm may make $N$ adaptive queries without hitting
the terminal gate, acquire information about the hidden
directions $u_i$ and $v_i$, and then output a new point
without querying it. Such an output may be substantially
more accurate than a random guess, so its probability
of hitting the terminal gate cannot be bounded by
the initial guessing probability alone.
The sequential information argument in
\Cref{sec:sequential} proves that this probability
remains small for insufficient query budgets,
even when the final output depends on the entire
adaptive query history.
}
Thus, even unrestricted algorithms must satisfy the
terminal gate condition to succeed on our hard instance,
regardless of their query history.

\subsection{Information bounds under sample reuse}\label{sec:gaussian}

The preceding construction provides $T$ stages
and associates each stage with a fresh hidden-direction
estimation task, whose success is certified by
hitting the stage gate \eqref{eq:workspace-gate}.
We now establish the query cost of passing such
a gate under reusable scalar observations.

\paragraph{Key idea.}
Reusing a sample label could allow an algorithm to cancel
noise and recover exact information about the hidden
direction. We prevent this by representing the hidden mean
and the Gaussian field through the same feature map.
Canceling the noise features then also cancels the hidden
mean. After subtracting the part determined by earlier replies,
the residual response has the Gaussian form \eqref{eq:normalized-response}.
Its information content is bounded, forcing
$\Omega(m\sigma^2)$ queries to identify one direction.

At stage $k$, suppose $g_{k+1}=\cdots=g_T=0$. The workspace-well
definition \eqref{eq:workspace-well} then gives
$a=a_{<k}+Ag_kv_k$. In coordinates centered at the preceding
stages, the one-stage response for label $j$ is
\begin{equation}\label{eq:stage-response}
\begin{gathered}
q=x-a_{<k}(w),\qquad g=g_k(w),\qquad v=v_k,\\
Y=\mu_v(q,g)+Q_j(q,g),\qquad \mu_v(q,g)=h(q-Agv).
\end{gathered}
\end{equation}
Here $\mu_v$ is the translated workspace well, not another potential.
The field $Q_j$ is centered, so $\mu_v$ is the mean of $Y$ at a
fixed query. Under sample reuse this is not a conditional-mean assertion. Reuse requires the conditional
response law derived below (see \Cref{lem:canonical-adaptive-kernel},
Eq.~\eqref{eq:app-conditional-kernel}).

To construct the one-stage field $Q_j$ in \eqref{eq:stage-response}, we first express $\mu_v$
through tensor features. Since $\|v\|=1$, setting
$S=1+\|q\|^2+A^2g^2$ gives
\[
\begin{aligned}
\mu_v(q,g)
&=4\left(\sqrt{S-2Ag\langle q,v\rangle}-1\right)\\
&=4(\sqrt S-1)
+4\sqrt S\sum_{\ell\ge1}
(-1)^\ell\binom{1/2}{\ell}
\left\langle\frac{2Agq}{S},v\right\rangle^\ell.
\end{aligned}
\]
The expansion holds globally because $0\le g\le1$ and
\[
 \frac{2Ag\|q\|}{1+\|q\|^2+A^2g^2}
 \le\frac{Ag}{\sqrt{1+A^2g^2}}
 \le\frac{A}{\sqrt{1+A^2}}<1.
\]
The first inequality uses
$2\|q\|\sqrt{1+A^2g^2}\le\|q\|^2+1+A^2g^2$;
the second uses monotonicity in $g\in[0,1]$.
Using the tensor identity
$\langle q,v\rangle^\ell
=\langle q^{\otimes\ell},v^{\otimes\ell}\rangle$
and the feature definitions (see Eq.~\eqref{eq:feature}),
we obtain
\begin{equation}\label{eq:mean-feature}
\mu_v(q,g)=4(\sqrt S-1)
+\langle\theta(v),\phi(q,g)\rangle_{\cH},
\qquad Q_j(q,g)=
\underbrace{G_j(\phi(q,g))}_{\text{shared-feature Gaussian noise}}.
\end{equation}
The tensor representation has the required derivative bound
for $\theta$, and the sample fields admit one simultaneous
differentiable version (see \Cref{lem:tensor-shift,lem:field-regularity}).

Here $G_j$ is an isonormal Gaussian process on $\cH$.
Both the hidden mean and the one-stage field in
\eqref{eq:mean-feature} use the same feature vector $\phi(q,g)$.

The nonlinear features are necessary here. 
Affine noise in $q$ would disappear from a same-sample
second difference,\footnote{For affine noise $Q_j(q)=\langle\xi_j,q\rangle+c_j$,
three queries using the same sample label $j$ yield
\[
Q_j(q+s)-2Q_j(q)+Q_j(q-s)=0,
\]
while the corresponding second difference of
$h(q-Agv)$ need not vanish and may still depend
on $v$. Thus, only three oracle calls suffice to
obtain a noise-free, potentially informative
second difference. 
} while the
one-stage mean $\mu_v$ in \eqref{eq:stage-response} could still reveal $v$. Instead, the power-series
expansion supplies a noise feature for every hidden mean feature.
The expansion is exact on the whole workspace, so its use does not
require a bound on the algorithm's query radius.

Indeed, suppose a linear combination of queries with one label
cancels its feature vectors. After subtracting the public term
$4(\sqrt S-1)$, that same combination cancels the hidden mean:
\[
 \sum_a c_a\phi(q_a,g_a)=0
 \quad\Longrightarrow\quad
 \sum_a c_a\langle\theta(v),\phi(q_a,g_a)\rangle=0.
\]
Thus canceling the field does not leave a noiseless equation for
$v$. This identity motivates the field design. The regression
argument below extends it from a fixed linear combination to a
fully adaptive sequence of evaluations.

For each sample label $j$, the Gaussian field
$Q_j(q,g)=G_j(\phi(q,g))$ is defined jointly
over all admissible query points $(q,g)$, using
the same realization of $G_j$ across repeated
evaluations. Orthogonal Gaussian regression gives its conditional
response law under adaptive queries (see
\Cref{lem:canonical-adaptive-kernel}, Eq.~\eqref{eq:app-conditional-kernel}). 
Condition on the realized history. Suppose the next
query uses sample label $j$, with feature
$\phi=\phi(q,g)$.

If $j$ has not been queried before, the entire
feature is new, so $\psi=\phi$. Otherwise, let
$\phi_a$ denote its previously queried features and decompose
\[
\phi=\sum_a c_a\phi_a+\psi,
\qquad
\psi\perp\operatorname{span}\{\phi_a\}.
\]
Thus, $\psi$ represents the part of the current
feature not covered by previous queries with
the same label. Here the coefficients $c_a$ represent the projection
onto the previously queried features. Let $Y$ and $Y_a$ be the current and previous one-stage responses \eqref{eq:stage-response}, and
let $S$ and $S_a$ be their respective public quantities.
Subtracting the public terms and the corresponding
previous responses gives the same-label residual response
$\widetilde Y$:
\[
\begin{aligned}
\widetilde Y
=\bigl[Y-4(\sqrt S-1)\bigr]
  -\sum_a c_a
   \bigl[Y_a-4(\sqrt{S_a}-1)\bigr]
=\langle\theta(v),\psi\rangle+G_j(\psi).
\end{aligned}
\]
The same-label residual response $\widetilde Y$ has conditional law
(see \Cref{lem:canonical-adaptive-kernel}, Eq.~\eqref{eq:app-cond-mean}):
\[
\widetilde Y\mid\text{past},v
\sim\mathcal N\bigl(
\langle\theta(v),\psi\rangle,\|\psi\|^2
\bigr).
\]
If $\psi=0$, the response is already determined by the history.
Otherwise, normalization gives
\begin{equation}\label{eq:normalized-response}
\frac{\widetilde Y}{\|\psi\|}
=\left\langle\theta(v),\frac{\psi}{\|\psi\|}\right\rangle+Z,
\qquad
Z\mid\text{past},v\sim\mathcal N(0,1).
\end{equation}

When $\psi\ne0$, the normalized residual response
\eqref{eq:normalized-response} contains a standard Gaussian term
not determined by earlier replies. Together with the derivative bound
below, this lets us uniformly
control the information gained from each additional query,
regardless of how often the same sample label is reused.

To quantify this per-query information bound, we
control the conditional likelihood score of the
normalized residual response \eqref{eq:normalized-response}. Conditioning on the realized
history fixes the chosen feature, even though the
query may depend on $v$ through previous responses.
Thus, the conditional mean depends on $v$ only
through $\theta(v)$. The tensor coefficients give the following
derivative control (see \Cref{lem:tensor-shift},
Eq.~\eqref{eq:app-parameter-deriv}):
\begin{equation}\label{eq:theta-deriv}
\|D\theta(v)\|_{\rm op}^2\le20m/\sigma^2,
\end{equation}

where $D\theta(v)$ is the Fr\'echet derivative,
$\|\cdot\|_{\rm op}$ denotes the operator norm,
and $\sigma$ controls the Gaussian noise amplitude.

For $\psi\ne0$, the conditional score of \eqref{eq:normalized-response} is $Z$ times the tangent
gradient of the normalized mean (see Eq.~\eqref{eq:app-conditional-score}).
Since $\E[Z^2\mid\text{past},v]=1$, this gives
(see Eq.~\eqref{eq:app-conditional-score-budget}):
\[
\E\!\left[
\|\nabla_v\log p_v(\text{reply}\mid\text{past})\|^2
\mid\text{past},v
\right]
=
\frac{\|D\theta(v)^*\psi\|^2}{\|\psi\|^2}
\le\|D\theta(v)\|_{\rm op}^2
\le\frac{20m}{\sigma^2}.
\]
Here $p_v$ denotes the conditional response density,
and the adjoint is restricted to the tangent space
of the sphere. The middle inequality is the operator-norm bound;
the last is \eqref{eq:theta-deriv}. If $\psi=0$, the reply is
history-determined and contributes zero score, so no division is needed.
The same bound applies to all adaptive
queries, regardless of sample reuse or the distance
between query points.

We now pass from a single response to the full
transcript $\mathcal T_n$ of $n$ queries, including the algorithm's
independent random seed. At a fixed seed and history, query and
label choices are fixed. The conditional-likelihood chain rule therefore gives
(see \Cref{thm:tool-adaptive-score})
\[
\nabla_v\log p_v(\mathcal T_n)
=
\sum_{t=1}^n
\nabla_v\log p_v(Y_t\mid\mathcal T_{t-1}),
\]
where $\mathcal T_{t-1}$ is the earlier part of the same transcript,
before query $t$. Each conditional score has zero
conditional mean, so these scores form a
martingale difference sequence (see Eq.~\eqref{eq:app-score-martingale}). Their cross
terms therefore vanish in expectation (see Eq.~\eqref{eq:app-total-score}), yielding
\[
\E\|\nabla_v\log p_v(\mathcal T_n)\|^2
=
\sum_{t=1}^n
\E\|\nabla_v\log p_v(Y_t\mid\mathcal T_{t-1})\|^2
\le\frac{20mn}{\sigma^2}.
\]

Finally, suppose $v$ is uniform on the unit sphere in $s$
dimensions. The transcript changes this uniform prior into a
posterior distribution. Its density is proportional to the
likelihood $p_v(\mathcal T_n)$; the normalizing factor does not
depend on $v$. Thus the posterior's logarithmic derivative is
the same score that we have just bounded. The spherical logarithmic
Sobolev inequality (see \Cref{thm:tool-lsi}) bounds the posterior's
entropy relative to the uniform prior by this squared derivative.
Averaging over transcripts gives
\begin{equation}\label{eq:transcript-information}
I(v;\mathcal T_n)
\le\frac{1}{2(s-2)}
\E\|\nabla_v\log p_v(\mathcal T_n)\|^2
\le\frac{10mn}{(s-2)\sigma^2}.
\end{equation}
Thus, the information acquired about $v$ grows at most linearly
with the number of queries, even under adaptive sample reuse.
The same argument admits a common likelihood and a regular
posterior even with deterministic replies (see
\Cref{lem:posterior-information}, Eq.~\eqref{eq:app-mutual-info}).

To connect \eqref{eq:transcript-information} to optimization, write the stage gate \eqref{eq:workspace-gate} in the coordinates \eqref{eq:stage-response}: a hit has $g=1$ and $\|q-Av\|<\sqrt2\beta$. The algorithm
knows $q$ before receiving the reply. Since $\sqrt2\beta<A$,
a hit implies $q\ne0$, and its unit guess $\widehat v=q/\|q\|$
satisfies
\begin{align*}
 \|\widehat v-v\|
 &\le\left\|\frac q{\|q\|}-\frac qA\right\|
       +\left\|\frac qA-v\right\|
 =\frac{|A-\|q\||+\|q-Av\|}{A}
 \le\frac{2\|q-Av\|}{A},\\
 \langle\widehat v,v\rangle
 &=1-\tfrac12\|\widehat v-v\|^2
 >1-\frac{4\beta^2}{A^2}=\frac{15}{16}.
\end{align*}
The first line uses the triangle and reverse triangle inequalities;
the second uses unit norms and $A=128$, $\beta=16$. For $q=0$
choose any fixed unit vector, since that query cannot hit the gate.
Thus every successful spatial test produces a transcript-measurable
direction guess. Spherical cap concentration and binary data processing
can now be applied to these actual guesses.

Lists are needed because every query can attempt a gate hit,
and the final output supplies another attempt. To see the cost,
compare the actual experiment with one in which $v$ is independent
of the transcript. In the independent experiment, the transcript
fixes the list. A union bound over the $15/16$ spherical caps,
together with the list-size condition below, makes its success
probability exponentially small in $s$.
Binary data processing therefore gives, for a universal
$c'>0$,
\[
 c's\,\PP(\text{some list guess succeeds})
 \le I(v;\mathcal T_n)+\log2
 \le\frac{10mn}{(s-2)\sigma^2}+\log2.
\]
The first inequality says that making an exponentially unlikely
event occur with constant probability requires order-$s$
information; the second is \eqref{eq:transcript-information}.
When $s\asymp m$, constant success probability consequently
requires $n=\Omega(m\sigma^2)$. Keeping the constants gives the
following theorem.

\begin{theorem}[Reusable scalar information bound]\label{thm:reuse}
Let $v$ be uniform on the unit sphere of a known $s$-dimensional
subspace of $\R^m$, $s\ge4$. After $n$ adaptive evaluations of
$\mu_v+Q_j$, with arbitrary label openings and reuse, any list of
$M_{\rm list}$ transcript-measurable unit guesses satisfies
\begin{equation}\label{eq:reuse-bound}
 \PP\left(\max_{a\le M_{\rm list}}\langle v,\widehat v_a\rangle
                    \ge15/16\right)
 \le\frac{64mn}{s^2\sigma^2}+\frac9{4s},
 \qquad \log M_{\rm list}\le s/8.
\end{equation}
Initial auxiliary information must be independent of $v$ and its
fields. In particular, for $s\asymp m$, recovery costs
$\Omega(m\sigma^2)$ evaluations.
\end{theorem}
Spherical-cap concentration and binary data processing give the
stated list bound (see the \hyperref[proof:reuse-list]{proof of
\Cref*{thm:reuse}}, Eq.~\eqref{eq:app-list-recovery}).
The underlying sample construction, conditional regression under
reuse, and transcript information are verified separately
(see \Cref{lem:field-regularity,lem:canonical-adaptive-kernel,lem:posterior-information}).

The two ingredients are now explicit: \eqref{eq:overview-certificate}
forces the terminal gate, while \Cref{thm:reuse} gives the
$\Omega(m\sigma^2)$ cost for the one-stage response
\eqref{eq:stage-response}. Their sequential combination is proved
in \Cref{sec:sequential}.

\subsection{Global regularity and finite-sample realization}\label{sec:regularity}\label{sec:assembly}

Before establishing the multiplicative lower bound, we must
verify that the construction can be realized as a valid
stochastic optimization instance. 
In particular, we need
to construct the sample functions actually queried by
the algorithm, ensure that their population mean agrees
with the hard objective $F_{\rm new}$ up to scaling,
and preserve the information hardness of $T$ independent
stages without accumulating an additional factor of $T$
in the global Lipschitz budget.
Crucially, this realization must also preserve the
Gaussian information hardness under adaptive queries
and sample reuse.
\paragraph{Key idea.}
Independent stagewise noise must not cause the global
sample Lipschitz budget to grow with the number of stages.
Suffix cutoffs ensure that the joint derivative field
has bounded complexity, giving the global estimate
$\E\Lip(N_j)^2=O(\sigma^2)$ (see \Cref{thm:noise-moment},
Eq.~\eqref{eq:app-global-chaining}).
Antithetic sample pairs preserve the population mean
exactly, while hidden label permutations prevent
premature cancellation of their noise.

For sample label $j$ and stage $i$, let $Q_{j,i}$ be an independent
copy of the one-stage field $Q_j$ in \eqref{eq:mean-feature}.
Each $Q_{j,i}$ remains fixed under label reuse. Their combined
noise field is

\begin{equation}\label{eq:multi-noise}
N_j(w,x)=\sum_{i=1}^{T}
\underbrace{\tau(\|g_{>i}\|)}_{\text{suffix cutoff}}
\underbrace{Q_{j,i}(x-a_{<i},g_i)}_{\text{stagewise Gaussian noise}}.
\end{equation}

Here $g_{>i}=(g_{i+1},\ldots,g_T)$ contains the later activations
from \eqref{eq:smooth-activation}, and $a_{<i}$ is the preceding
center in \eqref{eq:workspace-well}. The suffix cutoff $\tau$ is
the same function as in \eqref{eq:cutoff-profiles}; in the noise
field it acts on $g_{>i}$, whereas the correction $c_i$ in
\eqref{eq:mean-definition} uses $g_{\ge i}$.

The two factors suppress noise from different stages.
If later stages become sufficiently active,
$\|g_{>i}\|\ge\rho_0$, the suffix cutoff vanishes
by \eqref{eq:cutoff-profiles}, removing the noise
from stage $i$. Conversely, when stage $i$ is
inactive, $g_i=0$, every tensor feature contains
a positive power of $g_i$, so $\phi(q,0)=0$ and
hence $Q_{j,i}(q,0)=G_{j,i}(0)=0$.
Thus, historical noise is switched off, while
future inactive stages contribute no noise values.

To control the remaining fields, fix $(w,x)$
and let $j_0$ be the smallest index satisfying
$\|g_{>j_0}\|<\rho_0$; such an index exists since
$g_{>T}$ is empty. For every $i<j_0$,
$\|g_{>i}\|\ge\rho_0$, so the corresponding
cutoff and its derivatives through order two
vanish. The remaining activations satisfy
\begin{equation}\label{eq:noise-suffix-control}
\begin{gathered}
\sum_{i\ge j_0}g_i^2
=g_{j_0}^2+\|g_{>j_0}\|^2\le1+\rho_0^2,\\
\left\|a_{<k}-a_{<j_0}\right\|
=\left\|A\sum_{j_0\le i<k}g_iv_i\right\|
\le A\sqrt{1+\rho_0^2}\qquad(k\ge j_0).
\end{gathered}
\end{equation}
The first bound uses $0\le g_{j_0}\le1$; the second uses
orthonormality in the center definition \eqref{eq:workspace-well}.
Hence, although many remaining stages may have
small nonzero activations, their total squared
activation is bounded independently of $T$,
and their moving centers stay within a uniformly
bounded distance of the common center $a_{<j_0}$.

To establish a global Lipschitz bound for $N_j$,
we need to control its gradient uniformly over
all query locations:
\begin{equation}\label{eq:lip-gradient}
\Lip(N_j)=\sup_{w,x}\|\nabla N_j(w,x)\|.
\end{equation}
The fields have a simultaneous differentiable version and a
controlled derivative process (see
\Cref{lem:field-regularity,thm:noise-moment}). Conditional on the hidden frames,
$N_j$ is a sum of Gaussian fields, so its gradients across query
locations form a jointly
Gaussian process. We control the directional derivatives of this
same field $N_j$ in \eqref{eq:multi-noise}. Rather than bounding the
Lipschitz constant of each stage separately
and summing the results, we directly control
the supremum of this combined process.

For the derivative estimate, independence is useful only after the
stage contributions are combined in their common Hilbert space.
In particular, a field with zero value can still have a nonzero
activation derivative:
\[
 Q(q,g)=gG(q)
 \quad\Longrightarrow\quad Q(q,0)=0,
 \qquad \partial_gQ(q,0)=G(q).
\]
Thus, we must account for derivative contributions from all stages,
including those with zero noise value. The squared activation bound
\eqref{eq:noise-suffix-control}, together with $\sum_i h_i^2=1$ for a unit activation
perturbation $h$, controls the sum of the squared derivative
coefficients.\footnote{
For example, let
$N(g)=\frac{\sigma}{\sqrt m}\sum_{i=1}^T g_iZ_i$,
where $Z_i\overset{\rm i.i.d.}{\sim}\mathcal N(0,1)$.
For any unit direction $u\in\mathbb R^T$,
$\sum_i u_i^2=1$, and independence gives
\[
\E|D_gN(g)[u]|^2
=\frac{\sigma^2}{m}
\E\left(\sum_{i=1}^T u_iZ_i\right)^2
=\frac{\sigma^2}{m}\sum_{i=1}^T u_i^2
=\frac{\sigma^2}{m}.
\]
Thus, the variance of each unit directional derivative
does not acquire an additional factor of $T$,
since independent Gaussian contributions combine
through squared coefficients.
}
 The joint estimate includes the suffix cutoffs and moving
centers (see \Cref{thm:noise-moment},
Eq.~\eqref{eq:app-joint-feature-budget}).

The remaining difficulty is the supremum over the unbounded
workspace. Using the center-shift bound \eqref{eq:noise-suffix-control}, write
\begin{equation}\label{eq:workspace-compactification}
 x'=x-a_{<j_0},\qquad
 s_c=\frac1{\sqrt{1+\|x'\|^2}},\qquad
 e_c=\frac{x'}{\sqrt{1+\|x'\|^2}}.
\end{equation}
This workspace compactification differs from the bounded chain map \eqref{eq:soft}. It satisfies $s_c^2+\|e_c\|^2=1$ and covers every query
location in a bounded hemisphere. Locations tending to infinity
approach its boundary $s_c=0$; no query is excluded. The compactified feature formulas have denominators uniformly
bounded away from zero, so their first two derivatives remain
bounded up to that boundary (see Eqs.~\eqref{eq:app-compact-coordinates},
\eqref{eq:app-degree-feature}, \eqref{eq:app-compact-denominator},
and \eqref{eq:app-perstage-Hilbert}).

Each directional derivative
$DN_j(w,x)[u]=\langle\nabla N_j(w,x),u\rangle$
is a scalar Gaussian random variable. By the
duality of the Euclidean norm,
\[
\|\nabla N_j(w,x)\|
=\sup_{\|u\|=1}\langle\nabla N_j(w,x),u\rangle.
\]
Consequently, by \eqref{eq:lip-gradient},
\[
\Lip(N_j)
=\sup_{\substack{w,x\\\|u\|=1}}
DN_j(w,x)[u].
\]

Expressing the gradient norm as a supremum of
directional derivatives turns the global Lipschitz
bound into the supremum of a scalar Gaussian
process, allowing us to apply Gaussian chaining.
The feature estimates therefore control both
the variance of each Gaussian directional derivative
and the variance of its increments between nearby
parameters.
To bound this supremum, approximate the derivative process on
progressively finer grids.
There are $T$ possible first surviving suffixes. For a fixed suffix,
the activations, workspace coordinates \eqref{eq:workspace-compactification}, and unit direction lie in
a bounded set of dimension $O(m+T)$. At mesh size $2^{-n}$ this
set has at most $\exp[C(m+T)(n+1)]$ representative points.
The feature derivative bounds give standard deviation at most
$C\sigma 2^{-n}/\sqrt m$ for the difference between a mesh point
and its nearest representative on the preceding mesh. Thus finer meshes contain more points, but their
Gaussian increments become geometrically smaller.

The $L^2$ norm of a maximum of $K$ centered Gaussian variables is
at most $C\sqrt{\log(2K)}$ times their largest standard deviation
(see \Cref{thm:tool-gaussian}). Applying this at each mesh and adding
the increments gives
\[
 \bigl(\E\Lip(N_j)^2\bigr)^{1/2}
 \le\frac{C\sigma}{\sqrt m}
 \sum_{n\ge0}2^{-n}
       \sqrt{\log(2T)+(m+T)(n+1)}
 \le C\sigma\sqrt{\frac{m+T+\log(2T)}m}.
\]
The logarithm counts the $T$ suffixes and the mesh points.
The last inequality uses $\sqrt{a+b}\le\sqrt a+\sqrt b$ and
$\sum_{n\ge0}2^{-n}\sqrt{n+1}<\infty$.
This is the chaining argument: it uses the correlations between
nearby derivatives instead of adding all stagewise Lipschitz
constants. The same estimate passes from the finite meshes to the actual
infinite-series sample fields (see \Cref{thm:noise-moment},
Eq.~\eqref{eq:app-global-chaining}). Squaring the bound gives
\begin{equation}\label{eq:noise-moment}
 \E\Lip(N_j)^2
 \le C\sigma^2\frac{m+T+\log(2T)}m
 =O(\sigma^2),\qquad m\ge T.
\end{equation}
The estimates are uniform over the hidden frames.
This bounds the largest slope over the whole query space and
allows us to use noise of scale $\sigma$ at every stage
without multiplying the sample Lipschitz budget by $T$.

We now realize the noise using finitely many sample functions.
A finite collection of mean-zero fields need not have zero
empirical mean. To obtain the desired mean for every realization,
we use each field with both signs. Take independent copies
$N_1,\ldots,N_M$ and define
\begin{equation}\label{eq:scaled-samples}
 f_{j,\pm}(Z,X)=\lambda\left[
 F_{\rm new}(Z/b,X/b)
 \pm N_j(U^\top y(Z/b),X/b)\right].
\end{equation}
The scales $b,\lambda>0$ are chosen in \eqref{eq:physical-scales}.
Cancellation within each pair gives
\[
 \frac1{2M}\sum_{j=1}^M(f_{j,+}+f_{j,-})
 =\lambda F_{\rm new}(Z/b,X/b).
\]
Thus the finite population mean is exactly the rescaled hard
objective, independently of the realized Gaussian coefficients.
The sample-moment bound is preserved as well
(see \Cref{lem:finite-population}).

A known pair would let the algorithm cancel its noise using two
queries at the same point. We therefore permute the $2M$ public
labels and keep their pairing hidden. Every label still denotes
one fixed sample function and may be reused arbitrarily. Until both members of an antithetic pair are queried, its
observations couple to an oracle assigning independent fields to
distinct labels and retaining each field under reuse
(see \Cref{lem:static-coupling}).
We apply the Gaussian information bounds to the unconditioned
independent-label experiment, then add the coupling error; we do
not condition that experiment on avoiding a collision.

The hidden permutation also controls adaptive label selection.
Before a collision, each opened label has only one opposite-sign
partner, hidden among the unopened labels by the uniform
permutation. In $N\le M$ queries there are at most $N$ opened
labels and at least $M$ unopened labels. Each opening therefore
has conditional collision probability at most $N/M$; reusing a
label opens no new partner. A union bound gives $N^2/M$.
Deferred matching preserves this conditional symmetry
(see Eq.~\eqref{eq:app-partner}). Increasing $M$ reduces this error without
altering the pairwise cancellation or the sample-moment estimate.
The next subsection chooses $M$ and the remaining parameters.

\subsection{Sequential query complexity}\label{sec:sequential}

The stationarity certificate \eqref{eq:overview-certificate}
forces every chain coordinate past $1$ and the output into the
stage-$T$ gate \eqref{eq:workspace-gate}. It does not require the
algorithm to traverse the stages sequentially.
To obtain a multiplicative lower bound, we must show
that unrestricted adaptive queries cannot reveal future
directions prematurely, and that unusually short stages
cannot collectively complete the chain with high probability.

\paragraph{The prefix simulator.}
Although the static objective has a stagewise structure,
unrestricted queries may reveal future hidden directions,
so the assumptions of the single-stage information bound
cannot be directly guaranteed at the start of each stage.
To overcome this, the prefix simulator maintains a current
stage $k$ and answers using the revealed prefix and the
current-stage information, setting future activations to zero.
This ensures that each new stage retains the conditional
randomness required by \Cref{thm:reuse}.
The simulator is an auxiliary experiment; the actual objective
remains fixed. The coupling below shows that the simulated
replies agree with the static oracle along the algorithm's
queries, except on a controlled future-projection event,
thereby transferring the stagewise information lower bound
back to the original construction.

Figure~\ref{fig:prefix-simulator} separates the two operations.
The static oracle evaluates the future activations at the query;
the simulator sets them to zero. When the actual future activations
are zero, the two replies agree using the same retained sample fields.
The gate determines whether the next query stays at stage $k$ or
starts stage $k+1$.

The two response rules follow from setting
$g_{k+1}=\cdots=g_T=0$.

On a failed test of the stage gate \eqref{eq:workspace-gate},
the simulator uses the mean \eqref{eq:mean-definition} with its
successor term cancelled (see Eq.~\eqref{eq:app-failed-mean}):
\begin{equation}\label{eq:prefix-failed-mean}
\begin{aligned}
P(z)+\sum_{i\le k}b_i(w)
-\sum_{i\le k}[1-\chi(x-a_{<i})]c_i(w)
+h(x-a_{<k}-Ag_kv_k).
\end{aligned}
\end{equation}
The workspace-well term from \eqref{eq:workspace-well} retains information about $v_k$,
whereas the successor chain term is canceled:
\[
b_{k+1}-[1-\chi(x-a_k)]b_{k+1}
=\chi(x-a_k)b_{k+1}=0
\qquad\text{if }|w_k|<\frac12
\text{ or }\|x-a_k\|\ge\sqrt2\beta.
\]
Indeed, the simulator inherits the zero-chain and
gate-cutoff properties of $F_{\rm new}$,
either $|w_k|<1/2$, so $b_{k+1}=0$,
or $\|x-a_k\|\ge\sqrt2\beta$, so $\chi(x-a_k)=0$.
Thus, the response contains no information about
the successor chain direction $u_{k+1}$.

At the first gate hit, $g_k=1$. For $i\le k$,
$\|g_{\ge i}\|\ge g_k=1>\rho_0$, so the suffix cutoff gives
$c_i=0$. The zero future activations give
$c_{k+1}=b_{k+1}$ when $k<T$. The crossing mean is therefore
(see Eq.~\eqref{eq:app-crossing-mean}):
\begin{equation}\label{eq:prefix-crossing-mean}
P(z)+\sum_{i\le k}b_i(w)
+h(x-a_k)+\chi(x-a_k)b_{k+1}(w).
\end{equation}
The last term is absent when $k=T$.
Otherwise, the simulator uses $u_{k+1}$ to
evaluate $b_{k+1}$, while keeping future
activations zero. It then reveals $v_k$ and
$u_{k+1}$ and advances to stage $k+1$.
These free disclosures strengthen the algorithm, so a lower bound
with them also applies to the original scalar oracle
(see \Cref{lem:prefix-simulator}).

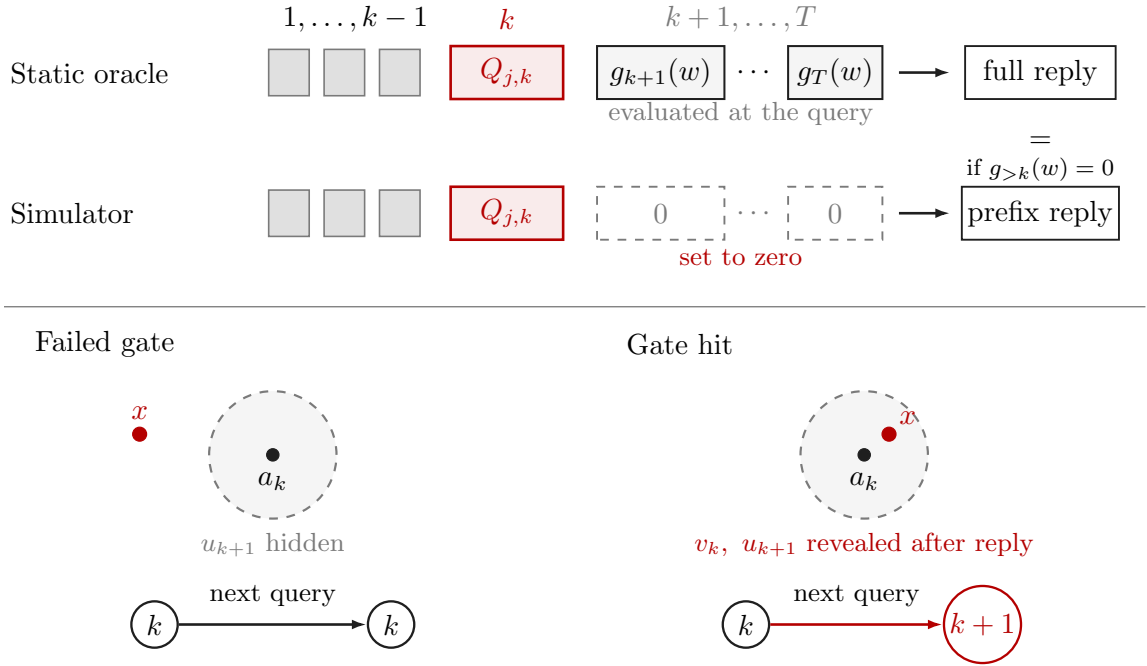
\begin{figure}[!ht]
\centering
\resizebox{\linewidth}{!}{\input{figures/prefix_simulator}}
\caption{Prefix simulation. Top: query-dependent future activations
versus prescribed zeros; the replies agree when the actual future
activations vanish. Both rows retain the same $Q_{j,k}$ under label
reuse. Bottom ($g_k=1$): a failed gate keeps stage $k$; a hit reveals
$v_k,u_{k+1}$ after the reply. The last gate marks completion.}
\label{fig:prefix-simulator}
\end{figure}

For sample label $j$, the simulator obtains
the prefix noise by setting $g_i=0$ for $i>k$
in the combined noise field \eqref{eq:multi-noise}. Since
$Q_{j,i}(q,0)=0$ and $\tau(0)=1$, this gives
\begin{equation}\label{eq:prefix-noise}
\begin{aligned}
N_j^{(k)}(w,x)
={}\sum_{i<k}
\tau\bigl(\|(g_{i+1},\ldots,g_k)\|\bigr)
Q_{j,i}(x-a_{<i},g_i)
+Q_{j,k}(x-a_{<k},g_k).
\end{aligned}
\end{equation}
The first sum contains the Gaussian fields of
earlier stages, whose contributions can be
removed after revealing the completed-stage
information. Subtracting these contributions
and the known prefix terms from the outside-gate
response leaves
\[
h(x-a_{<k}-Ag_kv_k)
+Q_{j,k}(x-a_{<k},g_k).
\]
With the substitutions in \eqref{eq:stage-response}, this is
the same one-stage response used by \Cref{thm:reuse}; its field is
the stage-$k$ copy $Q_{j,k}$ from \eqref{eq:multi-noise}.
The Gaussian fields are independent across
stages, whereas repeated queries with the same
sample label reuse the same field. Thus,
previous-stage observations do not replace
the information needed to estimate the current
hidden direction $v_k$.

To compare the prefix response with the static oracle, suppose
that the undisclosed chain coordinates satisfy
\[
|\langle u_i,y(z)\rangle|<\frac14
\qquad\text{for all }i>k.
\]
The cutoff definition gives $g_i=\Gamma(|w_i|)=0$ for every $i>k$.
Substituting these zeros into the center, noise, and chain terms gives
\[
\begin{aligned}
a(w)&=a_{<k}+Ag_kv_k=a_k,\\
Q_{j,i}(\cdot,g_i)&=0 &&(i>k),\\
b_i(w)&=0 &&(i\ge k+2),\\
b_{k+1}-[1-\chi(x-a_k)]c_{k+1}
&=\chi(x-a_k)b_{k+1}.
\end{aligned}
\]
The first identity uses the center definition; the second uses
$Q(q,0)=0$; the third is the zero-chain cutoff. In the last identity,
$c_{k+1}=b_{k+1}\tau(0)=b_{k+1}$. The suffix weights also reduce to
those in $N_j^{(k)}$. Thus the static and simulated replies agree
term by term, with the same Gaussian fields.

It remains to show that the required projection
condition holds with high probability along the
algorithm's adaptive queries. Since the prefix
simulator does not reveal undisclosed chain
directions, they remain conditionally Haar-distributed
given the preceding simulated transcript, by deferred revelation (see \Cref{lem:static-coupling}).
Together with $\|y(z)\|\le R$, conditional spherical concentration
gives the following bound for every query
(see \Cref{thm:tool-sphere}):

\[
\Pr\left(
|\langle u_i,y(z)\rangle|\ge\frac14
\,\middle|\,\text{past}
\right)
\le
2\exp\left(-\frac{r-T}{32R^2}\right),
\]
where the conditioning fixes the preceding simulation history. A union bound over the
queries and the final output controls the
probability of any premature future activation.

The responses therefore agree term by term, apart from the
exceptional projection event (see \Cref{lem:prefix-simulator,lem:static-coupling},
Eq.~\eqref{eq:app-pbad}).

\paragraph{Information available when a stage starts.}

We now verify that each stage of the prefix simulator
satisfies the assumptions of \Cref{thm:reuse},
so that the established information lower bound
applies to every stage.
At the start of stage $k$, the algorithm knows $v_1,\ldots,v_{k-1}$
but has not received a reply involving $v_k$. It knows that $v_k$
is orthogonal to those directions; within the remaining
$s_k=m-k+1$ dimensions, every unit direction is still equally likely.
The same distinction applies to sample reuse. A label used earlier
retains its old fields, but its current-stage field $Q_{j,k}$ has
not appeared in any previous reply. Independence across stages
makes this field independent of the earlier observations
(see \Cref{lem:prefix-simulator},
Eqs.~\eqref{eq:app-birth-kernel} and \eqref{eq:app-birth-Haar}).

We must also bound the chance of finding the gate after several
failed queries. For this, compare the actual stage with an auxiliary
run that always uses the failed-gate mean \eqref{eq:prefix-failed-mean}
and prefix noise \eqref{eq:prefix-noise}, even at a point that would
hit the gate. Use the same algorithmic randomness and
the same sample fields in both runs. Before the first hit, their
replies and hence their next query points are identical. Therefore,
if the actual stage hits within $\ell_0$ tests, one of the first
$\ell_0$ auxiliary query points must hit. These points form the
candidate list in \Cref{thm:reuse}. That theorem starts from the
uniform direction at stage entry and already allows each query to
depend on earlier replies. It therefore bounds the entire list,
including the dependence created by earlier failed queries
(see \Cref{lem:stage-charging}).

\paragraph{Accumulating the stage costs.}
The auxiliary query list turns stage completion into the
direction-estimation event bounded by \Cref{thm:reuse}.
Since its residual dimension is $s_k\asymp m$, that theorem
bounds the probability of completing within $\ell_0$ tests by
$C\ell_0/(m\sigma^2)+C/m$. Choose $\ell_0$ to be a sufficiently
small constant multiple of $m\sigma^2$. The first term is then
small, and the dimension assumption makes the second small.
With the integer choice (see Eq.~\eqref{eq:app-stage-budget}), this gives
\[
 \PP\bigl(\text{stage $k$ completes within $\ell_0$ tests}
       \mid\text{its starting history}\bigr)\le\frac1{16}.
\]
Call such a completed stage fast. For any fixed set
$\mathcal I$ of stages, repeated conditioning at their starting
histories gives
\[
 \PP(\text{all stages in }\mathcal I\text{ are fast})
 \le16^{-|\mathcal I|}.
\]
No independence between stage durations is required. If all
stages complete in $N+1\le T\ell_0/4$ tests, at least
$\lceil T/2\rceil$ are fast. A union bound over these sets gives
\begin{equation}\label{eq:chain-completion-probability}
 \PP(\text{all stages complete})
 \le\binom{T}{\lceil T/2\rceil}16^{-\lceil T/2\rceil}
 \le2^T16^{-T/2}=2^{-T}.
\end{equation}
The extra test is the final output: it need not be a queried point,
but stationarity still forces it to pass the terminal gate.
This proves how the one-stage costs accumulate without assuming
independent stage durations. The conditional histories, list sizes and integer budgets obey
these same bounds (see \Cref{lem:stage-charging},
Eqs.~\eqref{eq:app-fast-stage} and \eqref{eq:app-fast-count}).

\paragraph{Transfer to unrestricted queries.}
The simulator sets undisclosed chain activations to zero. We must
compare its answers with the actual static oracle, which evaluates
the complete function at every requested point. Since
$\|y(z)\|<R$, Haar concentration bounds the chance that any query
has a large projection onto an undisclosed chain direction:
\begin{equation}\label{eq:bad-projection}
 p_{\rm bad}\le2(N+1)T\exp\left[-\frac{r-T}{32R^2}\right].
\end{equation}
Outside this event, future activations vanish. The failure and
crossing identities \eqref{eq:prefix-failed-mean} and
\eqref{eq:prefix-crossing-mean}, together with $Q(q,0)=0$, identify
the simulated and static replies at every query, including label reuse
(see \Cref{lem:static-coupling}).
This argument permits arbitrary spatial jumps; it uses the bounded
coordinate $y(z)$, not a restriction on the query radius.

For $N+1\le T\ell_0/4$, combine the chain bound
\eqref{eq:chain-completion-probability}, the projection error
\eqref{eq:bad-projection}, and the hidden-pair collision bound
(see Eq.~\eqref{eq:app-partner}). The union bound gives
\begin{equation}\label{eq:master}
 \PP(\text{successful output by }N)
 \le
 \underbrace{2^{-T}}_{\substack{\text{fast completion}\\\text{of the chain}}}
 +\underbrace{p_{\rm bad}}_{\substack{\text{premature future}\\\text{activation}}}
 +\underbrace{\frac{N^2}{M}}_{\substack{\text{opening both labels}\\\text{of a noise pair}}}.
\end{equation}
The static and simulated replies coincide outside the last two
events, so success there requires completing the chain
(see \Cref{lem:static-coupling}, Eq.~\eqref{eq:app-master}). This applies to every admissible
adaptive algorithm, ruling out even unanticipated
strategies for circumventing the lower bound.

\paragraph{Scaling and admissible finite populations.}
Rescale the same samples \eqref{eq:scaled-samples}. Let $H_*,e_*$
be the normalized Hessian bound and stationarity threshold
(see \Cref{lem:global-geometry}). Set
\begin{equation}\label{eq:physical-scales}
 b=\frac{\epsilon H_*}{e_*H},\qquad
 \lambda=\frac{\epsilon^2H_*}{e_*^2H}.
\end{equation}
The population mean then has smoothness $H$, and physical
$\epsilon$-stationarity implies normalized $e_*$-stationarity.
The sufficient gap and sample-moment budgets are
\begin{equation}\label{eq:physical-budgets}
 \frac{\lambda H_*}{b^2}=H,\qquad
 12\lambda T\le\Delta,\qquad
 \frac\lambda b(37\sqrt T+9147+\sqrt{C_N}\sigma)\le L_0/2,
\end{equation}
where $C_N$ is the global noise-moment constant
(see \Cref{thm:noise-moment}). Up to universal constants, the last
two conditions read
\[
 \epsilon^2T\lesssim\Delta H,\qquad
 \epsilon^2(T+\sigma^2+1)\lesssim L_0^2.
\]
The first limits the stages by the available gap. The second
allocates the Lipschitz budget between the mean and the noise.
We allocate the gap and Lipschitz budgets by choosing
\begin{equation}\label{eq:parameter-selection}
 T\asymp B/\epsilon^2,\qquad
 \sigma^2\asymp L_0^2/\epsilon^2,\qquad m\asymp d,
\end{equation}
These choices satisfy the dimension, list and probability
conditions (see \Cref{lem:uniform-parameters}).

The reusable scalar bound \Cref{thm:reuse} uses the unconditioned Gaussian law.
We impose the sample budget afterwards. Put
\[
 S_M=\frac1{2M}\sum_{j,\pm}\Lip(f_{j,\pm})^2.
\]
The $L^2$ triangle inequality and \eqref{eq:physical-budgets} give
\[
 (\E S_M)^{1/2}
 \le\frac\lambda b(37\sqrt T+9147+\sqrt{C_N}\sigma)
 \le L_0/2
 \quad\Longrightarrow\quad
 \PP(S_M\le L_0^2)\ge\frac34,
\]
where the implication is Markov's inequality
(see \Cref{lem:finite-population}).

Choose $N_0\asymp dT\sigma^2$ with a sufficiently small constant
and $M\ge12N_0^2$. The parameter conditions make
\eqref{eq:master} at most $11/48$, so
\[
 \PP(\text{success by }N_0\mid S_M\le L_0^2)
 \le\frac{11/48}{3/4}=\frac{11}{36}<\frac13.
\]
An algorithm succeeding with probability at least $2/3$ on every
admissible population must therefore have cost greater than $N_0$
with probability greater than $1/3$ under this conditioned law.
Using $\E[\mathrm{cost}]\ge N_0\PP(\mathrm{cost}>N_0)$ and then
averaging over populations,
\[
 \sup_{\text{admissible populations}}\E[\mathrm{cost}]
 \ge\frac{N_0}{3}
 \gtrsim dT\sigma^2
 \asymp\frac{dL_0^2B}{\epsilon^4}.
\]
The first step is the cost-tail bound, and the last substitutes
\eqref{eq:parameter-selection} (see \Cref{lem:expected-cost},
Eqs.~\eqref{eq:app-cost-tail} and \eqref{eq:app-expected-cost}).
This proves \Cref{thm:main}.

\section{Gap Amplification}\label{subsec:larger-gaps}

The base construction has $T$ paid stages, but its squared
Lipschitz budget also contains $T$. This causes the cap
$B=\min\{\Delta H,L_0^2\}$. To enlarge the useful gap range,
we reuse a base instance as a block and require entry into its
terminal region before revealing a successor block.

A stationary-point lower bound for an isolated block would not
suffice: an algorithm could enter the region carrying its successor
before producing a stationary point. The required interface is a
lower bound on the first entry itself.

Figure~\ref{fig:depth-comparison} previews the two assemblies.
In (a), $p$ counts recursive chain levels; $\chi$ tests a child's
energy before exposing its successor, and the bottom two rows give
$p=2$. In (b), a child replaces its parent's formula inside a smaller
region; dashed circles stand for the intermediate levels up to $K$.
Thus chain depth $p$ and spatial nesting depth $K$ are different
parameters. The constructions below explain their respective budgets.

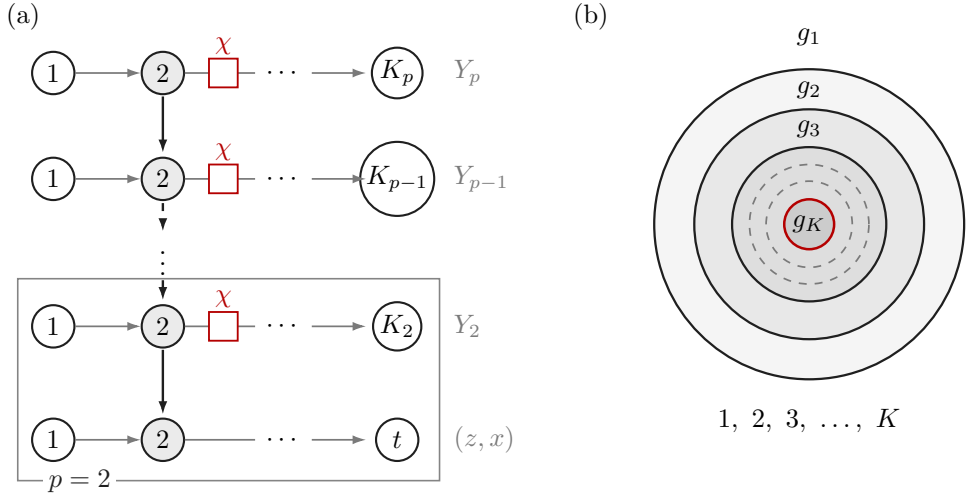
\begin{figure}[!htbp]
\centering
\input{figures/depth_comparison}
\caption{Fixed-depth recursion (a) and general $K$-level spatial nesting (b); schematic.}
\label{fig:depth-comparison}
\end{figure}

\paragraph{The common block.}
We use the one-sided chain and activation in \eqref{eq:wide-source},
which retain the same current-stage scalar channel
(see \Cref{lem:one-sided-source}).
We need a block that is hard to enter and easy to join to its
successor. The first requirement preserves the query cost; the
second permits a new block to replace its terminal minimum.
The block of length $t$ has nonnegative mean $g$, initial energy
$\Theta(t)$, gradient bound $O(\sqrt t)$, and Hessian bound $O(1)$.
Its sample functions $G_j$ have two useful regions. Around the
hidden minimizer $p$, every sample is the same quadratic:
\[
 G_j(X)=\frac\mu2\|X-p\|^2
 \qquad (\|X-p\|\le\sqrt t),
\]
where $\mu>0$ is universal. Near the outer boundary every sample
instead equals one public quadratic, independent of the hidden
directions. The inner region can hold a successor; the outer
region provides a common boundary on which to join it.

These changes preserve first-entry hardness because, before entry
into a larger coarse neighborhood of $p$, a modified response is
computable from at most one original scalar response and known
public terms. Thus an algorithm entering the modified block would
also enter the original target at the same query cost.
The inner and outer regions preserve this response reduction
(see \Cref{lem:wide-core,lem:wide-coarse-entry}).

This entry costs $\Omega(dt\sigma^2)$ scalar queries
(see \Cref{lem:wide-coarse-entry}). The target is deliberately larger than the stationary
region. Its workspace radius is $8\sqrt t$ around
$A\sum_{i=1}^t v_i$. For this estimate, condition additionally
on the current direction, and let $k\le\lfloor t/2\rfloor$ count
the resulting prefix. Conditional on this prefix and the reference
history, the remaining center has a Haar-uniform direction and
norm $A\sqrt{t-k}$. Hence
\[
 \frac{8\sqrt t}{A\sqrt{t-k}}
 \le\frac{\sqrt2}{16}<\frac18.
\]
Thus early entry requires a constant-angle guess for this remaining
direction; otherwise half the inner stages have already been paid.
The reference run always reports a coarse miss, so the estimate
does not condition on previous misses. It includes the crossing
candidate and an arbitrary final output
(see \Cref{lem:wide-coarse-entry}).

\paragraph{Two-level amplification.}
An outer chain of length $K$ selects independent inner blocks,
using shared inner coordinates. Its successor correction prevents
the next block from being revealed before entry into the current
one. The important distinction is between accumulating energy
and summing gradients. Multiplying the outer chain by $t$ and
stretching its coordinates by $t$ multiplies its gap by $t$,
but the factors $t$ and $1/t$ cancel in its first derivative.
Its gap is therefore $O(Kt)$ and its gradient remains $O(\sqrt K)$.

The inner blocks are selected by nonnegative weights $\omega_i$
with $\sum_i\omega_i=1$, rather than all being added. Holding the
outer coordinates fixed, their contribution obeys
\[
 \left\|\nabla_Z\sum_i\omega_i g_i(Z)\right\|
 =\left\|\sum_i\omega_i\nabla g_i(Z)\right\|
 \le\sum_i\omega_i\|\nabla g_i(Z)\|\le C\sqrt t.
\]
The equality uses the weights' independence of $Z$; the
inequalities use the triangle inequality and the block gradient
bound. Derivatives of the weights act on differences between
blocks. Their common public exterior makes these differences
uniformly $O(t)$ even on the unbounded workspace. The suffix
selection and its derivative estimates then avoid a factor $K$;
these are the additional checks (see Eqs.~\eqref{eq:linear-owner-derivatives}
and \eqref{eq:linear-mean-mixture}).
Together with the first-entry cost, this gives
\[
 \text{gap}=O(Kt),\qquad
 \E\Lip(F_\xi)^2=O(t+K+\sigma^2),\qquad
 N=\Omega(dKt\sigma^2).
\]
The scalar assembly preserves global regularity and first-entry
simulation, and rescaling enforces the physical budgets
(see \Cref{prop:two-level-assembly} and the
\hyperref[proof:two-level-bound]{proof of \Cref*{thm:extended-gap}}).
Taking $t,K,\sigma^2$ of order
$L_0^2/\epsilon^2$ gives the larger linear-gap window of
\Cref{thm:extended-gap}. The first-entry costs accumulate even for a query crossing
several gates at once (see \Cref{lem:gap-renewal}).

\paragraph{Fixed-depth amplification.}
For a fixed depth $p\ge3$, the next level needs an energy test
rather than a neighborhood of one minimizer. Normalize a child to
$g\ge0$, put $E=g(0)$, and use a nonincreasing cutoff
$\chi=\psi(g/E)$, equal to one at sufficiently low energy.
For derivatives in the child's coordinates, the successor
coefficient $c\ge0$ is constant. Its correction is $-c\chi$, so
\[
 \nabla(g-c\chi)
 =\left(1-\frac cE\psi'(g/E)\right)\nabla g,
 \qquad 1-\frac cE\psi'(g/E)\ge1.
\]
The multiplier is at least one because $\psi'\le0$ and $c\ge0$,
so this correction cannot cancel the child gradient.
The recursive block is constructed so that a small child gradient
forces low energy, where $\chi=1$ and $\nabla\chi=0$.
Consequently a small modified gradient still forces the energy
test to be passed; the successor correction cannot create an
earlier stopping point.

This guard lets us repeat the two-level selection at every fixed
depth. Each added level multiplies the energy range by its chain
length while adding that length to the derivative budget.
The guard, recursive regularity and first-entry argument survive
this fixed number of iterations; rescaling yields
\Cref{cor:fixed-depth} with $p$-dependent constants
(see \Cref{lem:energy-cutoff,prop:fixed-depth-assembly}).

The depth cannot be sent to infinity at fixed accuracy: the
origin gradient stays fixed while the Lipschitz budget grows
(see \Cref{prop:depth-loss}, Eqs.~\eqref{eq:depth-origin-gradient}
and \eqref{eq:depth-lipschitz-lower}):
\[
 \|\nabla g_p(0)\|=\Phi'(0),\qquad
 \Lip(g_p)^2\ge c\sum_{l=2}^pK_l\ge c(p-1),
\]
where $c>0$ is universal and $K_l\ge1$ is the chain length at
level $l$. After uniform rescaling to meet the sample budget,
\[
 \|\nabla\widetilde g(0)\|
 \le\frac{CL_0}{\sqrt{\sum_{l=2}^pK_l}}
 \le\frac{CL_0}{\sqrt{p-1}}.
\]
At a sufficiently large depth, the public starting point itself
would therefore be $\epsilon$-stationary. This explains why the
accuracy condition in \Cref{cor:fixed-depth} depends on $p$;
the enlarged linear windows do not give a gap-uniform linear bound
at fixed $L_0/\epsilon$.

\paragraph{Logarithmic amplification.}
A different assembly replaces the parent's quadratic core with
an independent child. The equality needed for the join is simple.
For a fixed universal $B_0>1$, the child's public outer collar is
$c_t+\mu\|Y\|^2/(2B_0)$, where $c_t$ is a public constant.
Center the child at the parent's minimum $p$, contract its
locations by $B_0$, and divide its values by $B_0$. On the
overlap with the parent core, its rescaled value is exactly
\[
 B_0^{-1}\left[c_t+\frac{\mu}{2B_0}
                  \|B_0(X-p)\|^2-c_t\right]
 =\frac\mu2\|X-p\|^2.
\]
Substitution uses the child's public collar, the constants cancel,
and the two powers of $B_0$ cancel the two value factors.
Thus the samples themselves agree on an open overlap, and their
available derivatives agree there as well. No interpolation term
is introduced at the join.

The child replaces the whole region in which the parent could
have a small gradient. A parent minimum therefore ceases to be
an alternative stopping point. At any location the formula uses
one block, plus constants inherited from its ancestors; those
constants contribute no gradient. This is the mechanism that
lets energy accumulate without retaining all historical gradients.
The recursive sample formula preserves the exact overlap and
stationary localization (see \Cref{prop:exact-nesting}).

Each contraction increases the local Hessian by $B_0$.
After $K$ levels, a final expansion of both locations and values
by $B_0^{K-1}$ restores a uniform Hessian bound. The first
derivative is preserved by this common expansion, while the
Hessian is divided by its scale. A level-$j$ block consequently
has its original gradient and a Hessian factor $B_0^{j-K}\le1$.
Thus all levels can keep the same local entry cost, while the
energy budget grows as $tB_0^{K-1}$. The affordable number of
levels is therefore $K=\Theta(1+\log(\Delta H/L_0^2))$.
Adding their first-entry costs gives \Cref{thm:logarithmic-gap}.
The physical budgets, dimension condition and expected cost
remain valid after this scaling (see the
\hyperref[proof:logarithmic-bound]{proof of \Cref*{thm:logarithmic-gap}}).

A query may jump directly into a deeper child. Its location and
label are fixed before that child's first response is exposed,
so it still supplies a first-entry candidate for fresh hidden
data. Processing successive crossings before returning the scalar
answer lets the conditional first-entry costs accumulate
(see \Cref{prop:nested-entry,lem:gap-renewal}). It does not require the algorithm to move
continuously between blocks.

The fixed-depth construction enlarges the interval on which the
bound is linear in the gap. Exact nesting instead gives continued
logarithmic growth beyond a fixed interval. Their hypotheses and
rates are compared in \Cref{tab:gap-rates}; when several apply,
the valid lower bound is their maximum.

\section{Goldstein and Exact-Value Consequences}\label{sec:consequences}
\subsection{Goldstein stationarity}\label{subsec:goldstein}

The smooth hard instances also belong to the nonsmooth class.
To use them, we must relate a Goldstein output to ordinary gradient
stationarity. Smoothness gives the comparison
(see \Cref{lem:goldstein-comparison}, Eq.~\eqref{eq:app-Goldstein-ineq})
\begin{equation}\label{eq:goldstein-ineq}
 g_\delta(x)\le\|\nabla f(x)\|
 \le g_\delta(x)+H\delta.
\end{equation}
The first inequality uses $\nabla f(x)\in\partial_\delta^Gf(x)$.
For the second, every gradient in the defining ball is within
$H\delta$ of $\nabla f(x)$, and the same is true of their convex
combinations and limits.

Choose $H=\epsilon/\delta$ and apply the smooth lower bound at
gradient tolerance $2\epsilon$. Then
\begin{equation}\label{eq:goldstein-reduction}
 g_\delta(\widehat x)\le\epsilon
 \quad\Longrightarrow\quad
 \|\nabla f(\widehat x)\|\le2\epsilon.
\end{equation}
Thus a Goldstein algorithm would solve the smooth task on the same
instance, with exactly the same scalar oracle and query cost.
Substitution in \Cref{thm:main} proves \Cref{thm:goldstein}; in the
uncapped range $\Delta\epsilon/\delta\le L_0^2$, its rate is
\[
 \Omega\!\left(\frac{dL_0^2\Delta}{\delta\epsilon^3}\right).
\]
The same substitution transfers the two-level, logarithmic and
fixed-depth bounds (see \Cref{cor:large-gap-goldstein,cor:depth-goldstein}). These reductions preserve the
corresponding gap windows, dimension conditions, and dependence
of constants on depth.

\subsection{Exact-value bounds for randomized algorithms}\label{subsec:noiseless}

Without the Gaussian fields, the same geometry still requires
$T=\Theta(B/\epsilon^2)$ stages. What changes is the information
cost of one stage: a scalar reply now gives one noiseless linear
measurement rather than a noisy Gaussian observation.

At a failed stage-$k$ gate, subtract the known terms to obtain
\[
 \mu_{v_k}(q,g)=4(\sqrt{1+\|q-Agv_k\|^2}-1),
 \qquad q=x-a_{<k},\quad g=g_k.
\] For $g>0$, squaring the displayed
identity gives
\begin{equation}\label{eq:noiseless-invert}
 \langle q,v_k\rangle
 =\frac{1+\|q\|^2+A^2g^2-(1+\mu_{v_k}(q,g)/4)^2}{2Ag}.
\end{equation}
For $g=0$, the response is independent of $v_k$. Each scalar
query can therefore be answered using at most one adaptive linear
measurement of the current hidden direction.

Fix a stage and write its residual dimension as $s$. Let $P_i$
project onto the span of the first $i$ measured directions.
Conditionally on their transcript $\mathcal T_i$, $P_iv$ is known
and $(I-P_i)v$ has uniform orientation in the unmeasured subspace.
Granting an extra orthogonal measurement whenever a query is
redundant only strengthens the oracle. Spherical symmetry gives
\begin{equation}\label{eq:noiseless-measured-energy}
 \E\!\left[\|P_{i+1}v\|^2-\|P_iv\|^2\mid\mathcal T_i\right]
 =\frac{1-\|P_iv\|^2}{s-i},
 \qquad \E\|P_nv\|^2=\frac ns.
\end{equation}
Thus $\PP(\|P_nv\|>1/2)\le4n/s$. Outside this event, a guess
with $\langle v,\widehat v\rangle\ge15/16$ must obtain at least
$7/16$ of that inner product from the still-uniform residual.
Applying spherical-cap concentration to the remaining uniform
direction gives the list estimate (see \Cref{lem:noiseless-list})
\begin{equation}\label{eq:noiseless-list-main}
 \PP\!\left(\max_{a\le M_{\rm list}}
          \langle v,\widehat v_a\rangle\ge15/16\right)
 \le\frac{4n}{s}
      +M_{\rm list}e^{-49(s-n)/512}
 \le\frac1{16},
\end{equation}
for $n,M_{\rm list}\le s/256$ and $s\ge512$.
This controls randomized algorithms because their independent
coins can be included in the conditioning. The conditional energy increment gives the stated expectation
(see Eqs.~\eqref{eq:app-noiseless-projection-increment}
and \eqref{eq:app-noiseless-measured-energy}).

At every new stage, $s_k=m-k+1=\Omega(d)$. Applying this list estimate to the prescribed outside-gate tests
gives a per-stage cost $\ell_{\rm det}=\Theta(d)$
(see \Cref{lem:stage-charging,lem:noiseless-list}). The same counting and projection
coupling therefore yield
\begin{equation}\label{eq:noiseless-master-main}
 \begin{gathered}
 \PP(\text{stationary success by }N)
 \le2^{-T}+2(N+1)T\exp\left[-\frac{r-T}{32R^2}\right],\\
 R=64\sqrt T,\qquad N+1\le T\ell_{\rm det}/4.
 \end{gathered}
\end{equation}
There is no sample-pair collision or random sample-moment event
in this exact-value model. The dimension condition and expected-cost conversion give the
claimed rate (see the \hyperref[proof:exact-bound]{proof of
\Cref*{thm:noiseless}}):
\[
 N=\Omega(dT)
 =\Omega\!\left(\frac d{\epsilon^2}
                    \min\{\Delta H,L_0^2\}\right).
\]
Finally, choosing $H=\epsilon/\delta$ and tolerance $2\epsilon$
gives \Cref{cor:exact-goldstein}
(see \Cref{lem:goldstein-comparison}).

%% file: sections/comparison_table.tex
\begin{table}[!htbp]
\centering
\caption{Comparison of zeroth-order nonconvex optimization bounds.}
\label{tab:prior-rates}
\vspace{0.8em}

\small
\setlength{\tabcolsep}{2pt}
\renewcommand{\arraystretch}{1.25}

\begin{tabular}{@{}cccccc@{}}
\toprule
Work & Oracle & Class & Upper bound & Lower bound & Domain \\
\midrule
\citet{ghadimi2013}
 & Stoch. & S
 & $O(d\epsilon^{-4})$ & --- & Global \\
\citet{lin2022}
 & Stoch. & NS
 & $O(d^{3/2}\delta^{-1}\epsilon^{-4})$ & --- & Global \\
\citet{chen2023}
 & Stoch. & NS
 & $O(d^{3/2}\delta^{-1}\epsilon^{-3})$ & --- & Global \\
\citet{kornowski2024}
 & Stoch. & NS
 & $O(d\delta^{-1}\epsilon^{-3})$ & --- & Global \\
\citet{cartis2012}
 & Exact & S
 & $O(d\epsilon^{-2})$ & --- & Global \\
\citet{dingyuan2026}$^{\dagger}$
 & Exact & NS
 & $(1+(\delta\epsilon)^{-1})e^{O(d\log(1/\epsilon))}$
 & $e^{\Omega(d)}$ & Global \\
\citet{zhang2026}
 & Stoch. & NS
 & --- & $\Omega(d\delta^{-1}\epsilon^{-3})$ & Ball \\
\midrule
\textbf{Ours (\resultlink{thm:main})}
 & \textbf{Stoch.} & \textbf{S}
 & --- & $\boldsymbol{\Omega(d\epsilon^{-4})}$
 & \textbf{Global} \\
\textbf{Ours (\resultlink{thm:goldstein})}
 & \textbf{Stoch.} & \textbf{NS}
 & --- & $\boldsymbol{\Omega(d\delta^{-1}\epsilon^{-3})}$
 & \textbf{Global} \\
\textbf{Ours (\resultlink{thm:noiseless})}
 & \textbf{Exact} & \textbf{S}
 & --- & $\boldsymbol{\Omega(d\epsilon^{-2})}$
 & \textbf{Global} \\
\textbf{Ours (\resultlink{cor:exact-goldstein})}
 & \textbf{Exact} & \textbf{NS}
 & --- & $\boldsymbol{\Omega(d\delta^{-1}\epsilon^{-1})}$
 & \textbf{Global} \\
\bottomrule
\end{tabular}

\par\vspace{0.4em}
\begin{minipage}{\linewidth}
\footnotesize\raggedright
\emph{Notes.}
S/NS: smooth/Goldstein stationarity;
Stoch./Exact: reusable stochastic/fixed-function values;
Global/Ball: unrestricted/ball-restricted queries and outputs
(global/local gap).
Other parameters are suppressed.
$^\dagger$ Deterministic algorithms only; exponential lower
bound for fixed small accuracy.
Our lower bounds allow randomized algorithms.
Goldstein rates assume $\Delta\epsilon/\delta\le L_0^2$.
\end{minipage}
\end{table}

%% file: sections/gap_table.tex
\begin{table}[!ht]
\centering
\small
\setlength{\tabcolsep}{5pt}
\renewcommand{\arraystretch}{1.15}

\caption{Smooth stochastic zeroth-order lower bounds
under the global reusable oracle of Section~2.}
\label{tab:gap-rates}
\medskip

\begin{tabularx}{\linewidth}{@{}
  >{\centering\arraybackslash}m{.16\linewidth}
  >{\centering\arraybackslash}m{.34\linewidth}
  >{\centering\arraybackslash}X@{}}
\toprule
Construction & Query lower bound & Sufficient conditions \\
\midrule

\begin{tabular}{@{}c@{}}
Basic\\[1pt]
{\footnotesize(\Cref{thm:main})}
\end{tabular}
& $\displaystyle
  \begin{gathered}
  \Omega\!\left(\dfrac{dL_0^2}{\epsilon^4}
  \min\{\Delta H,L_0^2\}\right)
  \end{gathered}$
& $\displaystyle
  \begin{gathered}
  \epsilon^2\le c_0\min\{\Delta H,L_0^2\}\\[4pt]
  d\ge C\frac{\min\{\Delta H,L_0^2\}}{\epsilon^2}\\[-2pt]
  {}\times\log\!\left(
  2+\frac{L_0^2\min\{\Delta H,L_0^2\}}{\epsilon^4}
  \right)
  \end{gathered}$ \\[5pt]

\midrule

\begin{tabular}{@{}c@{}}
Two-level\\[1pt]
{\footnotesize(\Cref{thm:extended-gap})}
\end{tabular}
& $\displaystyle
  \begin{gathered}
  \Omega\!\left(\dfrac{dL_0^2}{\epsilon^4}
  \min\!\left\{\Delta H,\dfrac{L_0^4}{\epsilon^2}\right\}\right)
  \end{gathered}$
& $\displaystyle
  \begin{gathered}
  \epsilon\le c_eL_0,\quad
  \Delta H\ge C_*L_0^2\\[4pt]
  d\ge C\frac{L_0^2}{\epsilon^2}
  \log\!\left(2+\frac{L_0^2}{\epsilon^2}\right)
  \end{gathered}$ \\[5pt]

\midrule

\begin{tabular}{@{}c@{}}
Fixed depth\\[1pt]
{\footnotesize(\Cref{cor:fixed-depth})}
\end{tabular}
& $\displaystyle
  \Omega\!\left(\dfrac{dL_0^2}{\epsilon^4}
  \min\!\left\{\Delta H,
  \dfrac{L_0^{2p}}{\epsilon^{2p-2}}\right\}\right)$
& $\displaystyle
  \begin{gathered}
  p\ge2\ \text{(fixed)}\\[2pt]
  \epsilon\le c_{e,p}L_0,\quad
  \Delta H\ge C_{*,p}L_0^2\\[4pt]
  d\ge C_p\frac{L_0^2}{\epsilon^2}
  \log\!\left(2+\frac{L_0^2}{\epsilon^2}\right)
  \end{gathered}$ \\[5pt]

\midrule

\begin{tabular}{@{}c@{}}
Exact nesting\\[1pt]
{\footnotesize(\Cref{thm:logarithmic-gap})}
\end{tabular}
& $\displaystyle
  \Omega\!\left(\dfrac{dL_0^4}{\epsilon^4}
  \left[1+\log\dfrac{\Delta H}{C_0L_0^2}\right]\right)$
& $\displaystyle
  \begin{gathered}
  \epsilon\le c_eL_0,\quad
  \Delta H\ge C_0L_0^2\\[4pt]
  \begin{aligned}
  d&\ge C\Bigl[\frac{L_0^2}{\epsilon^2}
  \log\!\left(2+\frac{L_0^2}{\epsilon^2}\right)\\[-2pt]
  &\hspace{37pt}+\log(2K)\Bigr]
  \end{aligned}
  \end{gathered}$ \\

\bottomrule
\end{tabularx}

\par\vspace{7pt}
\begin{minipage}{\linewidth}
\footnotesize
\emph{Notes.}
The constants $c_0,c_e,C,C_*,C_0,C_1$ are
positive universal constants, with $C_1>1$.
The generic constant $C$ may differ between results.
Constants indexed by $p$ may depend on the fixed
depth $p\ge2$.
For exact nesting,
$K=1+\lfloor
\log(\Delta H/(C_0L_0^2))/\log C_1
\rfloor$.
\end{minipage}
\end{table}

%% file: figures/coupled_geometry.tex
\begin{tikzpicture}[x=1cm,y=1cm]
\path[use as bounding box] (0,.05) rectangle (13.8,6.95);
\coordinate (O) at (2.1,3.2);
\coordinate (Y) at (2.846,3.878);
\coordinate (Z) at ($(O)!2.2!(Y)$);
\fill[black!4] (O) circle (1.13);
\draw[figgray,dashed,line width=.75pt] (O) circle (1.13);
\draw[figgray,dashed,line width=.65pt] (O) -- (Z);
\draw[edge,->] (O) -- (3.05,3.2);
\draw[edge,->] (O) -- (2.1,4.14);
\node[right=3pt] at (3.05,3.2) {$u_1$};
\node[left=4pt] at (2.1,4.14) {$u_k$};
\node[open] at (O) {};
\node[below left=3pt] at (O) {$0$};
\node[circle,fill=figred,draw=figred,inner sep=1.5pt] at (Y) {};
\node[above right=3pt] at (Y) {$y$};
\node[open] at (Z) {};
\node[above right=3pt] at (Z) {$z$};
\node at (1.2,2.18) {$R$};
\draw[figred,line width=.8pt,->] (.97,3.2) -- (.43,3.2);
\draw[figred,line width=.8pt,->] (2.1,2.07) -- (2.1,1.56);
\node[left=4pt,text=figred] at (2.1,1.56) {$\nabla P$};
\node at (2.1,.96) {$y=\alpha z,\quad w=U^\top y$};
\node[draw=figbase,line width=.75pt,fill=black!3,minimum height=7.5mm] (prefix) at (2.1,6.0) {$w_1,\ldots,w_k$};
\node[above=5pt] at (prefix.north) {$g_k=1$};
\draw[edge,->] (2.1,4.33) -- (prefix.south);
\node[draw=figred,line width=.9pt,minimum width=2.65cm,minimum height=8mm] (term) at (8.2,6.0) {$\chi(x-a_k)b_{k+1}$};
\draw[edge,->] (prefix.east) -- (term.west);
\node[draw=figgray,line width=.75pt,minimum height=7.5mm] (future) at (12.35,6.0) {$w_{k+1}$};
\node[below=5pt,text=figgray] at (future.south) {$u_{k+1}$};
\draw[edge,dashed,->] (future.west) -- (term.east);
\coordinate (center) at (8.2,2.7);
\fill[black!5] (center) circle (1.02);
\draw[figbase,dashed,line width=.8pt] (center) circle (1.02);
\draw[figred,line width=.9pt,->] (8.2,3.72) -- (term.south);
\node[open] at (5.15,2.7) {};
\node[below=5pt] at (5.15,2.7) {$a_{<k}$};
\draw[edge,->] (5.28,2.7) -- node[above=5pt] {$Av_k$} (8.05,2.7);
\node[dot] at (center) {};
\node[below=5pt] at (center) {$a_k$};
\node[circle,fill=figred,draw=figred,inner sep=1.5pt] at (8.54,3.07) {};
\node[above right=2pt] at (8.54,3.07) {$x$};
\node[open] at (11.65,2.7) {};
\node[below=4pt] at (11.65,2.7) {$x'$};
\node[above=5pt,text=figred] at (11.65,2.7) {$\chi=0$};
\end{tikzpicture}

%% file: figures/prefix_simulator.tex
\begin{tikzpicture}[x=1cm,y=1cm]
\path[use as bounding box] (0,0) rectangle (13.8,8.2);
\node[anchor=west] at (0,7.35) {Static oracle};
\node[anchor=west] at (0,5.65) {Simulator};
\node at (4.25,8.00) {$1,\ldots,k-1$};
\node[text=figred] at (6.08,8.00) {$k$};
\node[text=figgray] at (8.91,8.00) {$k+1,\ldots,T$};
\foreach \yy in {7.35,5.65}{
  \foreach \xx in {3.20,3.87,4.54}{
    \fill[black!12] (\xx,\yy-.28) rectangle +(0.49,.56);
    \draw[figgray,line width=.55pt] (\xx,\yy-.28) rectangle +(0.49,.56);
  }
  \fill[figred!8] (5.40,\yy-.32) rectangle +(1.36,.64);
  \draw[figred,line width=.9pt] (5.40,\yy-.32) rectangle +(1.36,.64);
  \node[text=figred] at (6.08,\yy) {$Q_{j,k}$};
}
\draw[figbase,line width=.7pt,fill=black!4] (7.17,7.03) rectangle (8.71,7.67);
\node at (7.94,7.35) {$g_{k+1}(w)$};
\node at (9.09,7.35) {$\cdots$};
\draw[figbase,line width=.7pt,fill=black!4] (9.48,7.03) rectangle (10.62,7.67);
\node at (10.05,7.35) {$g_T(w)$};
\node[text=figgray,font=\footnotesize] at (8.91,6.86) {evaluated at the query};
\foreach \xx/\ww in {7.17/1.54,9.48/1.14}{
  \draw[figgray,dashed,line width=.65pt] (\xx,5.33) rectangle +(\ww,.64);
  \node[text=figgray] at (\xx+.5*\ww,5.65) {$0$};
}
\node[text=figgray] at (9.09,5.65) {$\cdots$};
\node[text=figred,font=\footnotesize] at (8.91,5.12) {set to zero};
\draw[signal] (10.82,7.35) -- (11.44,7.35);
\draw[signal] (10.82,5.65) -- (11.44,5.65);
\node[draw=figbase,line width=.65pt,minimum width=1.82cm,minimum height=6.4mm] at (12.53,7.35) {full reply};
\node[draw=figbase,line width=.65pt,minimum width=1.82cm,minimum height=6.4mm] at (12.53,5.65) {prefix reply};
\node at (12.53,6.52) {$=$};
\node[font=\scriptsize] at (12.53,6.18) {if $g_{>k}(w)=0$};
\draw[figgray,line width=.45pt] (0,4.52) -- (13.8,4.52);

\node[anchor=west] at (.3,4.07) {Failed gate};
\node[anchor=west] at (7.45,4.07) {Gate hit};
\foreach \xx in {3.25,10.4}{
  \fill[black!4] (\xx,2.73) circle (.77);
  \draw[figgray,dashed,line width=.75pt] (\xx,2.73) circle (.77);
  \node[dot] at (\xx,2.73) {};
  \node[below=4pt] at (\xx,2.73) {$a_k$};
}
\node[circle,fill=figred,draw=figred,inner sep=1.7pt] at (1.64,2.98) {};
\node[above=3pt,text=figred] at (1.64,2.98) {$x$};
\node[circle,fill=figred,draw=figred,inner sep=1.7pt] at (10.70,2.98) {};
\node[above right=2pt,text=figred] at (10.70,2.98) {$x$};
\node[text=figgray,font=\footnotesize] at (3.25,1.64) {$u_{k+1}$ hidden};
\node[text=figred,font=\footnotesize] at (10.4,1.64) {$v_k,\ u_{k+1}$ revealed after reply};

\node[stage] (ka) at (1.82,.68) {$k$};
\node[stage] (kb) at (4.68,.68) {$k$};
\draw[signal] (ka.east) -- node[above=4pt,font=\footnotesize] {next query} (kb.west);
\node[stage] (kc) at (8.97,.68) {$k$};
\node[stage,draw=figred,text=figred] (kd) at (11.83,.68) {$k+1$};
\draw[signal,draw=figred] (kc.east) -- node[above=4pt,font=\footnotesize] {next query} (kd.west);
\end{tikzpicture}

%% file: figures/depth_comparison.tex
\begin{tikzpicture}[x=1cm,y=1cm]
\path[use as bounding box] (0,-.1) rectangle (13.8,6.75);
\node[anchor=west] at (.05,6.4) {(a)};
\node[anchor=west] at (7.55,6.4) {(b)};
\foreach \yy/\last/\coord in {5.65/{K_p}/{Y_p},4.25/{K_{p-1}}/{Y_{p-1}},2.3/{K_2}/{Y_2},.8/{t}/{(z,x)}} {
  \node[stage] at (.75,\yy) {$1$};
  \node[stage,fill=figbase!9] at (2.2,\yy) {$2$};
  \node at (3.8,\yy) {$\cdots$};
  \node[stage] at (5.3,\yy) {$\last$};
  \draw[edge,->] (1.04,\yy) -- (1.91,\yy);
  \draw[edge] (2.49,\yy) -- (3.42,\yy);
  \draw[edge,->] (4.17,\yy) -- (4.89,\yy);
  \node[right=4pt,text=figgray] at (5.83,\yy) {$\coord$};
}
\foreach \yy in {5.65,4.25,2.3} {
  \node[rectangle,draw=figred,line width=.7pt,fill=white,inner sep=1.2pt,minimum size=3.8mm] at (3.0,\yy) {};
  \node[above=5pt,text=figred] at (3.0,\yy) {$\chi$};
}
\draw[signal] (2.2,5.34) -- (2.2,4.59);
\draw[signal,dashed] (2.2,3.93) -- (2.2,3.55);
\node at (2.2,3.24) {$\vdots$};
\draw[signal,dashed] (2.2,2.91) -- (2.2,2.64);
\draw[signal] (2.2,1.99) -- (2.2,1.13);
\draw[figbase!55,line width=.6pt] (.28,.26) rectangle (5.86,2.93);
\node[fill=white,text=figbase] at (1.1,.26) {$p=2$};
\fill[black!4] (10.75,3.65) circle (2.05);
\draw[figbase,line width=.85pt] (10.75,3.65) circle (2.05);
\fill[black!9] (10.75,3.65) circle (1.52);
\draw[figbase,line width=.85pt] (10.75,3.65) circle (1.52);
\fill[black!13] (10.75,3.65) circle (1.02);
\draw[figbase,line width=.85pt] (10.75,3.65) circle (1.02);
\draw[figgray,dashed,line width=.65pt] (10.75,3.65) circle (.79);
\draw[figgray,dashed,line width=.65pt] (10.75,3.65) circle (.57);
\fill[black!20] (10.75,3.65) circle (.33);
\draw[figred,line width=.95pt] (10.75,3.65) circle (.33);
\node at (10.75,6.13) {$g_1$};
\node at (10.75,5.43) {$g_2$};
\node at (10.75,4.91) {$g_3$};
\node at (10.75,3.65) {$g_K$};
\node at (10.75,1.05) {$1,\ 2,\ 3,\ \ldots,\ K$};
\end{tikzpicture}

%% file: sections/external_tools.tex
\section{External Probabilistic Tools}\label[appendix]{app:tools}
The following standard results are used with the hypotheses stated here.
The intrinsic sphere dimension is $s$; the optimization dimension is $d$.

\subsection{Gaussian maxima and fixed-feature regression}
\begin{theorem}[Gaussian maximum and dyadic chaining; \citealp{dudley1967,vershynin2018}]\label{thm:tool-gaussian}
If $Z_1,\ldots,Z_K$ are centered Gaussian variables with $\max_i\E Z_i^2\le v^2$, not necessarily independent, then
\[
\|\max_{i\le K}|Z_i|\|_{L^2}\le C v\sqrt{\log(2K)}.
\]
More generally, suppose $(Z_t)_{t\in\mathcal T}$ is a separable centered Gaussian process on a totally bounded metric space, with canonical metric $d_Z(t,t')=\|Z_t-Z_{t'}\|_{L^2}$. If $\mathcal N(\eta)$ is the covering number for $d_Z$, then for any base point $t_0$,
\[
\left\|\sup_{t\in\mathcal T}|Z_t|\right\|_{L^2}
\le C\left[\|Z_{t_0}\|_{L^2}+
\int_0^{\operatorname{diam}(\mathcal T,d_Z)}\sqrt{\log(2\mathcal N(\eta))}\,d\eta\right].
\]
\end{theorem}
\begin{proof}[Explanation of the needed $L^2$ version]
The union bound and the Gaussian tail inequality imply
\[
\PP\!\left(\max_i|Z_i|>t\right)
\le\sum_{i=1}^K\PP(|Z_i|>t)
\le 2K\exp[-t^2/(2v^2)].
\]
Integrating $\E W^2=\int_0^\infty 2t\PP(W>t)\,dt$, with $t_0=v\sqrt{2\log(2K)}$, yields
\begin{align*}
\E W^2&\le t_0^2+\int_{t_0}^\infty 4Kt e^{-t^2/(2v^2)}\,dt\\
&=2v^2\log(2K)+4Kv^2e^{-t_0^2/(2v^2)}
=2v^2\bigl[\log(2K)+1\bigr].
\end{align*}
For the process bound, take nested nets at distances $2^{-j}$ in the canonical metric and write $Z_t-Z_{t_0}$ as the telescoping sum of net increments. Each increment has standard deviation $O(2^{-j})$ and there are at most $\mathcal N(2^{-j})\mathcal N(2^{-j+1})$ possible pairs. Apply the preceding finite-maximum estimate to each level, and then Minkowski's inequality to the sum. The resulting Riemann sum is the displayed entropy integral. Separable continuity identifies the net limit with the process supremum; this is the standard Dudley argument \citep{dudley1967}.
\end{proof}

\begin{theorem}[Conditional Gaussian projection; see \citealp{gpml}]\label{thm:tool-gaussian-regression}
Let $G$ be an isonormal Gaussian process over a real separable Hilbert space $\mathcal H$: for finite $u_i$, the vector $(G(u_i))_i$ is jointly centered Gaussian with $\E[G(u)G(v)]=\langle u,v\rangle$. For deterministic $u_1,\ldots,u_k,u$, after observing $G(u_1),\ldots,G(u_k)$, let $P_k$ project onto $\operatorname{span}\{u_1,\ldots,u_k\}$. Conditional on those observations,
\[
G(u)=G(P_ku)+G((I-P_k)u),\qquad
G((I-P_k)u)\sim N(0,\|(I-P_k)u\|^2),
\]
and the last term is independent of the observed projections.
This statement concerns fixed features. For adaptive queries to the
retained tensor-series samples, the simultaneous version and conditional
kernel are proved separately in \Cref{lem:canonical-adaptive-kernel}.
\end{theorem}

\subsection{Conditional laws and likelihood scores}
\begin{theorem}[Regular conditional distributions and disintegration; \citealp{kallenberg2021}]\label{thm:tool-disintegration}
If $X,Y$ take values in standard Borel spaces, there is a measurable
probability kernel $K$ such that, for every nonnegative measurable $F$,
\[
 \E F(X,Y)=\int\!\int F(x,y)K(y,dx)\,P_Y(dy).
\]
The kernel is a regular conditional law of $X$ given $Y$, unique
$P_Y$-almost everywhere. If an event $A$ has probability one, then
$K(y,A)=1$ for $P_Y$-almost every $y$. Countably many such assertions
can be required on one full-measure set.
\end{theorem}

\begin{theorem}[Adaptive Gaussian likelihood and score chain rule]\label{thm:tool-adaptive-score}
Let $v$ be a parameter on a smooth manifold and $\mathcal T_i$ an adaptive transcript. Conditional on $\mathcal T_{i-1}$ and $v$, suppose response $Y_i$ is either a parameter-independent point mass or has a Gaussian density with history-measurable variance $s_i^2>0$ and a differentiable conditional mean $m_i(v)$. The branch, variance and deterministic response are functions of the history, independent of the candidate parameter. Assume
$\sum_{i\le n}\E[\ind_{\{s_i>0\}}\|\nabla m_i(v)\|^2/s_i^2]<\infty$,
with the summand defined as zero at deterministic steps. Then, holding each realized history fixed while differentiating, the log-likelihood score is the sum of conditional scores
\[
\nabla_v\log p(\mathcal T_n\mid v)=\sum_{i=1}^n S_i,
\qquad S_i:=\frac{Y_i-m_i(v)}{s_i^2}\nabla_v m_i(v)
\]
for Gaussian responses, with $S_i=0$ for parameter-independent deterministic responses. Moreover
\[
\E[S_i\mid v,\mathcal T_{i-1}]=0,
\qquad
\E[\|S_i\|^2\mid v,\mathcal T_{i-1}]
=\frac{\|\nabla_v m_i(v)\|^2}{s_i^2}.
\]
The displayed second-moment expression is interpreted as zero at deterministic steps. In particular, conditional scores from different rounds are orthogonal in $L^2$, without independence of adaptive query locations.
\end{theorem}
\begin{proof}
Fix a realized history $h$. The positive-variance conditional kernel is
\[
p_i(y\mid v,h)=\frac1{\sqrt{2\pi}s_i(h)}
\exp\left[-\frac{(y-m_i(v,h))^2}{2s_i(h)^2}\right].
\]
The derivative at fixed $h$ is
$\nabla_v\log p_i=(y-m_i)\nabla_v m_i/s_i^2$.
Its conditional mean is zero and its conditional second moment is
$\|\nabla_v m_i\|^2/s_i^2$, by the first two moments of the
standardized Gaussian innovation.

For a common reference kernel, use $N(0,1)$ on a Gaussian branch
and the stated point mass on a deterministic branch. The choice is
history-measurable and parameter-independent. Gaussian densities
are positive, so their ratios to this reference kernel define a common
transcript likelihood. Predictable query rules contribute no extra
factor: at a fixed history their queries and labels are fixed.
Differentiating its finite product gives
$\nabla_v\log p(\mathcal T_n\mid v)=\sum_{i\le n}S_i$.

The finite energy assumption gives $S_i\in L^2$; Cauchy--Schwarz
makes every cross inner product integrable. For $i<j$, $S_i$ is
measurable with respect to $\mathcal T_{j-1}$, hence
\begin{align*}
 \E\langle S_i,S_j\rangle
 &=\E\langle S_i,\E[S_j\mid v,\mathcal T_{j-1}]\rangle=0,\\
 \E\Big\|\sum_{i\le n}S_i\Big\|^2
 &=\sum_{i\le n}\E\|S_i\|^2.
\end{align*}
The first equality is the conditional-expectation pull-out identity;
the second line expands the finite squared norm and cancels the
cross terms. This proves the claims with no positive lower bound
on the nonzero variances.
\end{proof}

\subsection{Sphere entropy and concentration}
\begin{theorem}[Spherical logarithmic Sobolev inequality; \citealp{beckner1993}]\label{thm:tool-lsi}
Let $\pi$ be uniform probability measure on the unit sphere $\mathbb S^{s-1}\subset\mathbb R^s$ with $s\ge4$. For every nonnegative smooth $p$ with $\int p\,d\pi=1$,
\[
\operatorname{Ent}_\pi(p):=\int p\log p\,d\pi
\le \frac1{2(s-2)}\int\frac{\|\nabla_{\mathbb S}p\|^2}{p}\,d\pi.
\]
The conclusion extends to densities with finite weak Fisher information by smooth approximation and lower semicontinuity.
\end{theorem}

\begin{theorem}[Spherical concentration and Haar symmetry; \citealp{vershynin2018}]\label{thm:tool-sphere}
For $V\sim\operatorname{Unif}(\mathbb S^{s-1})$ and deterministic unit $u$,
\[
\E e^{\lambda\langle u,V\rangle}\le e^{\lambda^2/(2s)},\qquad
\PP(\langle u,V\rangle\ge t)\le e^{-st^2/2},\quad t\ge0.
\]
If an orthonormal frame is sampled according to Haar measure, then conditioned on its first $k$ columns, the next column is uniform on the sphere in their orthogonal complement. The same conditional law holds after observing any random variables measurable with respect to the revealed prefix and independent auxiliary randomness, provided the unexposed column did not enter them.
\end{theorem}

\subsection{Information contraction}
\begin{theorem}[Binary information contraction; \citealp{duchi2013}]\label{thm:tool-binary}
For joint law $P_{XY}$ and product $Q=P_X\otimes P_Y$, and any measurable event $E\subseteq\mathcal X\times\mathcal Y$, set $p=P_{XY}(E)$ and $q=Q(E)\in(0,1)$. Then
\[
I(X;Y)=D_{\rm KL}(P_{XY}\|Q)
\ge D_{\rm KL}(\operatorname{Ber}(p)\|\operatorname{Ber}(q))
\ge p\log(1/q)-\log 2.
\]
\end{theorem}
\begin{proof}[Derivation of the last inequality]
Expand the binary divergence without omitting either term:
\begin{align*}
D_{\rm KL}(\operatorname{Ber}(p)\|\operatorname{Ber}(q))
&=p\log\frac pq+(1-p)\log\frac{1-p}{1-q}\\
&=p\log\frac1q+\bigl[p\log p+(1-p)\log(1-p)\bigr]
 +(1-p)\log\frac1{1-q}\\
&\ge p\log\frac1q-h_2(p)
\ge p\log\frac1q-\log2.
\end{align*}
The first inequality discards a nonnegative term, and the last uses the binary entropy bound $h_2(p)\le\log2$. The data-processing inequality gives the first displayed inequality by mapping the pair $(X,Y)$ to $\mathbf 1_E$.
\end{proof}

\paragraph{Elementary inequalities.}
For nonnegative random variables $Z$, $\PP(Z>t)\le\E Z/t$. If $X$ and $Y$ lie in $L^2$, then $\|X+Y\|_{L^2}\le\|X\|_{L^2}+\|Y\|_{L^2}$. For a nonnegative integer-valued cost $C$ and threshold $N$, $\E C\ge N\PP(C>N)$.

%% file: sections/geometry_proofs.tex
\section{Global Geometry and Stationary Certificates}
\label[appendix]{app:geometry}
\subsection*{Scalar profiles and the complete mean}
Let $U\in\R^{r\times T}$ and $V\in\R^{m\times T}$ be independent
Haar orthonormal frames, $r,m\ge T$. Use $y,w,P$ from
\eqref{eq:soft}--\eqref{eq:confinement}, and the clamped quintic
$S_5$ and activations from \eqref{eq:smooth-activation}.
The zero-chain of \citet{carmon2019} is
\begin{gather*}
 \Psi(t)=\begin{cases}0,&t\le1/2,\\e^{1-(2t-1)^{-2}},&t>1/2,\end{cases}
 \qquad \Phi(t)=\sqrt e\int_{-\infty}^t e^{-u^2/2}\,du,\\
 b_1(w)=-\Phi(w_1),\qquad
 b_i(w)=\Psi(-w_{i-1})\Phi(-w_i)-\Psi(w_{i-1})\Phi(w_i)\quad(i>1),\\
 B_T=\sum_{i=1}^T b_i.
\end{gather*}
The profiles are $C^2$, including at $w_i=0$, where $g_i$ is locally
constant. Define
\begin{gather*}
 A=128,\quad\beta=16,\quad\rho_0=3/32,\qquad
 h(q)=4(\sqrt{1+\|q\|^2}-1),\\
 a=A\sum_i g_iv_i,\qquad a_{<i}=A\sum_{j<i}g_jv_j,\qquad
 c_i(w)=b_i(w)\tau(\|g_{\ge i}\|).
\end{gather*}
Together with the cutoffs \eqref{eq:cutoff-profiles}, these specify
every term of \eqref{eq:mean-definition}.

\subsection{Elementary derivatives of the chain and activations}
\begin{lemma}[Chain and activation derivatives]\label[lemma]{lem:chain-derivatives}
For the profiles in this appendix,
\begin{equation}\label{eq:app-chainbounds}
\begin{aligned}
|b_i|&\le12,\quad \|Db_i\|\le20,\quad \|D^2b_i\|\le160,\\
\|DB_T\|&\le24\sqrt T,\quad \|D^2B_T\|\le160,\\
\|Dg\|&\le\frac{15}{2},\qquad\|D^2g\|\le93.
\end{aligned}
\end{equation}
Every coordinate derivative of $B_T$ is nonpositive. If
$i=\min\{j:|w_j|\le1\}$ exists, then
\begin{equation}\label{eq:app-chain-barrier}
 \partial_iB_T(w)\le-1.
\end{equation}
\end{lemma}
\begin{proof}
Here derivatives of the vector $g=(g_1,\ldots,g_T)$ are in the Euclidean operator norms. To obtain \eqref{eq:app-chainbounds}, bound the elementary factors by
\[
0\le\Psi<11/4,\quad0\le\Psi'<9/2,\quad|\Psi''|<34,
\qquad 0<\Phi<25/6,\quad0<\Phi'<5/3,\quad|\Phi''|\le1.
\]
For $u=2t-1>0$ and $a=u^{-2}$, direct differentiation gives
\[
\Psi'(t)=4e a^{3/2}e^{-a},\qquad
\Psi''(t)=8e(2a^3-3a^2)e^{-a}.
\]
The maximum of the square of the first derivative is $54/e<(9/2)^2$. For $0\le a\le3/2$, $a^2(3-2a)\le1$ bounds the second derivative by $8e<22$. On $a\ge3/2$, its maximum occurs at $a_+=(9+\sqrt{33})/4\in(11/3,37/10)$; monotonicity of $a^2(2a-3)$ gives
\[
 \Psi''(t)\le8e(37/10)^2(22/5)e^{-11/3}
 =\frac{60236}{125e^{8/3}}<\frac{60236}{1775}<34.
\]
Here $e^{8/3}>71/5$ follows from its positive Taylor sum through
order eight. Moreover
\[
\Phi'(t)=e^{(1-t^2)/2},\quad |t|\Phi'(t)\le1,
\quad \sqrt{2\pi e}<25/6.
\]
The supports of $\Psi(t)$ and $\Psi(-t)$, including their first two
derivative supports, are disjoint. Direct differentiation gives
\begin{equation}\label{eq:app-chain-sign}
\begin{aligned}
\partial_1b_1&=-\Phi'(w_1)\le0,\\
\partial_i b_i
&=-[\Psi(-w_{i-1})+\Psi(w_{i-1})]\Phi'(w_i)\le0
&& (2\le i\le T),\\
\partial_i b_{i+1}
&=-\Psi'(-w_i)\Phi(-w_{i+1})
  -\Psi'(w_i)\Phi(w_{i+1})\le0
&& (1\le i<T),\\
\partial_i B_T
&=\begin{cases}
\partial_i b_i+\partial_i b_{i+1},&i<T,\\
\partial_T b_T,&i=T
\end{cases}
\ \le0.
\end{aligned}
\end{equation}
The first line differentiates $b_1$; the second uses the evenness
of $\Phi'$. The third uses
$\Psi'\ge0$ and $\Phi>0$. Only the current and successor terms
depend on $w_i$, giving the last line, including the endpoint $i=T$.
If $i>1$ is the first index with $|w_i|\le1$, then
\begin{equation}\label{eq:app-chain-predecessor}
\begin{aligned}
 \Psi(w_{i-1})+\Psi(-w_{i-1})
 &=\Psi(|w_{i-1}|)\\
 &=\exp\!\left(1-(2|w_{i-1}|-1)^{-2}\right)\ge1.
\end{aligned}
\end{equation}
Here $|w_{i-1}|>1$ makes the opposite-sign term zero and
$2|w_{i-1}|-1>1$ makes the exponent nonnegative. For $i=1$,
$\partial_1b_1=-\Phi'(w_1)$ directly. Thus in both cases
\begin{equation}\label{eq:app-chain-barrier-proof}
\begin{aligned}
 \partial_iB_T(w)
 &\le\partial_i b_i
 &&\text{by \eqref{eq:app-chain-sign}},\\
 &\le-\Phi'(w_i)
 &&\text{by \eqref{eq:app-chain-predecessor}, or $i=1$},\\
 &=-\exp((1-w_i^2)/2)
 &&\text{by the definition of $\Phi$},\\
 &\le-1
 &&\text{since $|w_i|\le1$}.
\end{aligned}
\end{equation}
Each $b_i$ depends only on $(w_{i-1},w_i)$; the Hessian of $B_T$ is therefore tridiagonal. The support disjointness keeps the derivative estimates uniform rather than introducing a factor of $T$. In particular, the first-coordinate identity is $\partial_1b_1=-\Phi'(w_1)$ and all successor contributions remain nonpositive. Hence
\begin{align*}
|b_i|&\le(11/4)(25/6)=275/24<12,\\
\|Db_i\|&\le\sqrt{(75/4)^2+(55/12)^2}<20,\\
|(DB_T)_i|&\le(9/2)(25/6)+(11/4)(5/3)=70/3<24,\\
\max\{\|D^2b_i\|,\|D^2B_T\|\}&\le34(25/6)+11/4+2(9/2)(5/3)=1913/12<160.
\end{align*}
The same last expression bounds every absolute Hessian row sum:
one own-coordinate term, one successor term, and two neighboring
mixed terms. The symmetric-matrix row-sum bound gives the stated
operator norm; the coordinate gradient bound gives $24\sqrt T$.

On $[0,1]$,
\[
 S_5'(u)=30u^2(1-u)^2,\qquad
 S_5''(u)=60u(1-u)(1-2u).
\]
Both derivatives vanish at $0$ and $1$, so clamping preserves
$C^2$ regularity. The extrema give
$\max|S_5'|=15/8$ at $u=1/2$ and
$\max|S_5''|=10/\sqrt3$ at $u=(3\pm\sqrt3)/6$.
Since $g_i$ is constant near $w_i=0$, the absolute value in
\eqref{eq:smooth-activation} introduces no singularity. The chain
rule and the diagonal coordinate dependence give
\begin{equation}\label{eq:app-activation-derivatives}
\begin{aligned}
 \|Dg\|&\le4\max|S_5'|=15/2,\\
 \|D^2g\|&\le16\max|S_5''|=160/\sqrt3<93.
\end{aligned}
\end{equation}
For the bilinear norm in the second line, use
$\sum_i u_i^2v_i^2\le\|u\|^2\|v\|^2$. The derivatives are zero
on the constant pieces, including their endpoints.

\end{proof}

\subsection{All-space stationary localization for the base zero-chain}
\begin{lemma}[Base-chain regularity and stationary localization]
\label[lemma]{lem:base-chain}
The base objective satisfies
\begin{equation}\label{eq:base-cert}
\begin{gathered}
 F_{\rm base}(0)-\inf F_{\rm base}\le12T,\quad
 \|\nabla F_{\rm base}\|\le37\sqrt T,\quad
 \|\nabla^2F_{\rm base}\|\le161,\\
 \|\nabla F_{\rm base}\|\le1/5\ \Longrightarrow\ w_i>1\ (\forall i).
\end{gathered}
\end{equation}
\end{lemma}
\begin{proof}
Put $q=\nabla B_T(w)$ and abbreviate $\alpha=\alpha(z)$. Direct differentiation gives
\begin{equation}\label{eq:app-soft-deriv}
Dy=\alpha\left(I-\frac{yy^\top}{R^2}\right),
\qquad\|Dy\|\le1,
\qquad\|D^2y\|\le\frac2R,
\end{equation}
\begin{equation}\label{eq:app-base-grad}
\nabla F_{\rm base}=\alpha Uq+cy,
\qquad c=\frac15-\frac{\alpha\langle y,Uq\rangle}{R^2}.
\end{equation}
Indeed for unit $u,v$,
\[
D^2y[u,v]=-\frac{\alpha^3}{R^2}
\big(\langle z,v\rangle u+\langle z,u\rangle v+\langle u,v\rangle z\big)
+\frac{3\alpha^5}{R^4}\langle z,u\rangle\langle z,v\rangle z.
\]
Writing $t=\|z\|/R$, the bound is
\[
R\|D^2y\|\le\frac{3t}{(1+t^2)^{3/2}}+\frac{3t^3}{(1+t^2)^{5/2}}
\le\frac2{\sqrt3}+\frac{18\sqrt3}{25\sqrt5}<\frac76+\frac35<2.
\]
Suppose first that $\|z\|\ge R$ and write $\widehat z=z/\|z\|$. The radial eigenvalue of $Dy$ is $\alpha^3$, giving
\begin{equation}\label{eq:app-base-radial-barrier}
\begin{aligned}
\langle\widehat z,\nabla F_{\rm base}\rangle
&\ge\frac{\alpha\|z\|}{5}-24\alpha^3\sqrt T\\
&\ge\frac{R}{5\sqrt2}-\frac{24\sqrt T}{2\sqrt2}
=\frac4{5\sqrt2}\sqrt T>\frac12.
\end{aligned}
\end{equation}
Therefore a point with $\|\nabla F_{\rm base}\|\le e_0\le1/5$ must satisfy $\|z\|<R$. In this region $\alpha\ge1/\sqrt2$, $\alpha\|y\|\le R/2$, and
\begin{equation}\label{eq:app-base-radial-coefficient}
\left|c-\frac15\right|\le\frac{24\sqrt T}{2R}=\frac3{16},
\qquad c\in[1/80,31/80].
\end{equation}
By \eqref{eq:app-chain-sign}, $q_i\le0$ for every $i$.
If $i$ is the first coordinate with $|w_i|\le1$, then
$q_i\le-1$ by \eqref{eq:app-chain-barrier-proof}.
The $i$-th gradient component obeys the explicit chain
\begin{equation}\label{eq:app-base-component-barrier}
\begin{aligned}
\langle u_i,\nabla F_{\rm base}\rangle
&=\alpha q_i+cw_i
&&\text{(differentiate $B_T(U^\top y)$ and $P$)},\\
&\le-\alpha+|c||w_i|
&&\text{($q_i\le-1$, $|w_i|\le1$, $c>0$)},\\
&\le-\frac1{\sqrt2}+\frac{31}{80}
&&\text{($\|z\|<R$, hence $\alpha\ge1/\sqrt2$)},\\
&<-\frac14.
\end{aligned}
\end{equation}
This contradicts $\|\nabla F_{\rm base}\|\le e_0$. Thus $|w_i|>1$ for all $i$, and $\|y\|\ge\|w\|>\sqrt T\ge1$. Taking inner product with $y$ in \eqref{eq:app-base-grad} yields
\begin{equation}\label{eq:app-base-positive-sign}
\alpha\langle q,w\rangle\left(1-\frac{\|y\|^2}{R^2}\right)
+\frac15\|y\|^2
=\langle\nabla F_{\rm base},y\rangle
\le e_0\|y\|<\frac15\|y\|^2.
\end{equation}
Since $\|y\|<R$, the preceding inequality gives
$\langle q,w\rangle<0$, and \eqref{eq:app-base-grad} then gives $c>1/5$.
If any $w_i<-1$, \eqref{eq:app-chain-sign} yields
\begin{equation}\label{eq:app-base-negative-coordinate}
\begin{aligned}
 \langle u_i,\nabla F_{\rm base}\rangle
 &=\alpha q_i+cw_i\le cw_i<-c<-1/5.
\end{aligned}
\end{equation}
The equality is \eqref{eq:app-base-grad}; the three inequalities use,
respectively, $q_i\le0$, $w_i<-1$ with $c>0$, and $c>1/5$.
This contradicts $\|\nabla F_{\rm base}\|\le e_0\le1/5$.
Thus $w_i>1$ for every $i$. The orthogonal component gives
\begin{equation}\label{eq:app-base-orthogonal-component}
 \|P_{U^\perp}y\|=c^{-1}\|P_{U^\perp}\nabla F_{\rm base}\|
 \le c^{-1}e_0<5e_0,
\end{equation}
using $c>1/5$ and contraction of orthogonal projection.
Finally, $P$ is coercive and nonnegative. Since $b_1(0)<0$ and
$b_i(0)=0$ for $i>1$, \eqref{eq:app-chainbounds} gives
\begin{equation}\label{eq:app-base-gap}
\begin{aligned}
 F_{\rm base}(0)-\inf_zF_{\rm base}(z)
 &=B_T(0)-\inf_z\{P(z)+B_T(w(z))\}\\
 &\le0+12T=12T.
\end{aligned}
\end{equation}
The equality uses $P(0)=0$; the inequality uses $B_T(0)<0$,
$P\ge0$, and $B_T\ge-12T$.
Equations \eqref{eq:app-chainbounds} and \eqref{eq:app-soft-deriv} give
\begin{equation}\label{eq:app-base-regularity-proof}
\begin{aligned}
\|\nabla F_{\rm base}\|&\le24\sqrt T+R/5
=(184/5)\sqrt T<37\sqrt T,\\
\|\nabla^2 F_{\rm base}\|&\le160+48\sqrt T/R+1/5<161.
\end{aligned}
\end{equation}
\end{proof}

\subsection{Past-centered gate, gap identity, and global stationarity}
\begin{lemma}[Global geometry and stationary localization]
\label[lemma]{lem:global-geometry}
The function \eqref{eq:mean-definition} is globally
$C^2$ and coercive; it satisfies
\begin{equation}\label{eq:geometry-summary}
\inf F_{\rm new}=\inf F_{\rm base},
\qquad
F_{\rm new}(0,0)-\inf F_{\rm new}\le12T,
\end{equation}
\[
\norm{\nabla F_{\rm new}}\le37\sqrt T+9147,
\qquad
\norm{\nabla^2F_{\rm new}}
\le H_*:=\frac{66243245}{8}.
\]
The threshold is smaller than the base certificate's $1/5$
because the moving center can amplify the $z$-gradient by the
factor $1+\norm{D_za}\le961$. We therefore choose
$J=961$ and $e_*=1/(5J)=1/4805$. Whenever $0<e_0\le e_*$,
\begin{equation}\label{eq:stationary-certificate}
\norm{\nabla F_{\rm new}(z,x)}\le e_0
\Longrightarrow
\begin{cases}
w_i>1 & (1\le i\le T),\\
\displaystyle
\norm{x-A\sum_i v_i}
\le\frac{e_0}{\sqrt{16-e_0^2}},\\
\norm{P_{U^\perp}y(z)}\le5e_0.
\end{cases}
\end{equation}
\end{lemma}
\begin{proof}
Write $\tau_e(u)=\tau(\|u\|)$, with $\tau$ from
\eqref{eq:cutoff-profiles}.
The gate is $C^2$ and flat at its two support boundaries. For $i>1$, a nonzero two-jet of $b_i$ requires $|w_{i-1}|>1/2$, hence $g_{i-1}=1$, killing every earlier tail cutoff and its derivatives. More explicitly, whenever $c_i$ has a potentially nonzero two-jet, $\tau(\|g_{\ge i}\|)$ or one of its first two derivatives must be nonzero; since $\tau$ is flat on $[\rho_0,\infty)$, this forces $\|g_{\ge i}\|<\rho_0$. If simultaneously $c_j$ with $j>i$ had a nonzero two-jet, the necessary support of $b_j$ would force $g_{j-1}=1$, where $j-1\ge i$, giving
\begin{equation}\label{eq:app-single-correction-support}
1=g_{j-1}\le\|g_{\ge i}\|<\rho_0<1,
\end{equation}
a contradiction. The same argument covers derivatives at cutoff boundaries because $\tau$ is $C^2$ and flat there. Thus at most one $c_i=b_i\tau(\|g_{\ge i}\|)$ has a nonzero two-jet at a point. On its support,
\begin{equation}\label{eq:app-suffix-displacement}
\|a-a_{<i}\|\le A\|g_{\ge i}\|\le A\rho_0=12.
\end{equation}
On this support $g_i<\rho_0<S_5(1/4)=53/512$; monotonicity
of $S_5$ on $[0,1]$ gives $|w_i|<5/16$. Consequently,
\begin{equation}\label{eq:app-correction-amplitude}
\begin{aligned}
 \Phi(\pm w_i)&\le\Phi(0)+\tfrac5{16}\sup_t\Phi'(t)
 <\tfrac{25}{12}+\tfrac5{16}\tfrac53
 =\tfrac{125}{48}<\tfrac83,\\
 |c_i|&\le|b_i|<\tfrac{11}{4}\tfrac83<8=:C_0.
\end{aligned}
\end{equation}
The first line integrates $\Phi'$ from $0$ and uses
$\Phi(0)=\sqrt{\pi e/2}<25/12$. The second uses $0\le\tau\le1$,
$\Psi<11/4$, and disjoint supports of $\Psi(\pm w_{i-1})$.
This bound is used only on the correction support.

For $s=\|q\|^2/\beta^2-1\in(0,1)$, the spatial cutoff has
\[
 D\chi(q)=-\frac{2S_5'(s)}{\beta^2}q^\top,\qquad
 D^2\chi(q)=-\frac{4S_5''(s)}{\beta^4}qq^\top
             -\frac{2S_5'(s)}{\beta^2}I.
\]
Thus the scalar derivative bounds already proved above yield
\begin{equation}\label{eq:app-spatial-cutoff-bounds}
\begin{aligned}
 \|D\chi\|&\le\frac{2\sqrt2}{\beta}\frac{15}{8}
 <\frac6\beta,\\
 \|D^2\chi\|&\le\frac1{\beta^2}
       \left(\frac{80}{\sqrt3}+\frac{15}{4}\right)
 <\frac{52}{\beta^2}.
\end{aligned}
\end{equation}
The first line uses $\|q\|\le\sqrt2\beta$; the second uses
$\|qq^\top\|=\|q\|^2\le2\beta^2$ and the triangle inequality.
Outside this annulus the derivatives vanish, and flat endpoints
extend these bounds to its boundary.

Set $q=x-a(w)$ and $e_i=x-a_{<i}(w)$. At $x=a(w)$, every correction gate is inactive by \eqref{eq:app-suffix-displacement} and $\|e_i\|\le12<\beta=16$, so $F_{\rm new}=F_{\rm base}$. If a correction is nonzero then $\|e_i\|>\beta$, hence
\[
\|q\|\ge\|e_i\|-\|a-a_{<i}\|>\beta-12=4.
\]
The correction's spatial cutoff satisfies $0\le1-\chi\le1$; at the only potentially nonzero index we also have $|c_i|\le|b_i|<8$. Because $\|q\|>4$ and $h$ is increasing in $\|q\|$,
\begin{align*}
F_{\rm new}(z,x)-F_{\rm base}(z)
&=h(q)-\sum_i[1-\chi(e_i)]c_i(w)\\
&\ge4\bigl(\sqrt{1+\|q\|^2}-1\bigr)-8\\
&>4(\sqrt{17}-1)-8>0.
\end{align*}
When no correction is active, the same difference equals $h(q)\ge0$.
On the graph $x=a(w)$, for every live correction $\|e_i\|=\|a-a_{<i}\|\le12<\beta$, so $\chi(e_i)=1$; also $h(0)=0$. Hence $F_{\rm new}(z,a(w))=F_{\rm base}(z)$ for every $z$. Taking first the infimum in $(z,x)$ and then evaluating on the graph yields both inequalities $\inf F_{\rm new}\ge\inf F_{\rm base}$ and $\inf F_{\rm new}\le\inf F_{\rm base}$. Therefore
\begin{equation}\label{eq:app-gap-identity}
\inf F_{\rm new}=\inf F_{\rm base},
\qquad
F_{\rm new}(0,0)=F_{\rm base}(0),
\qquad
F_{\rm new}(0,0)-\inf F_{\rm new}\le12T.
\end{equation}
If the workspace gradient of a correction is nonzero, its gate is annular and $\|q\|>4$. The reverse triangle inequality and $\|D\chi\|\le6/\beta$ imply
\begin{equation}\label{eq:app-workspace-separation}
\|\nabla_xF_{\rm new}\|
\ge\|\nabla h(q)\|-C_0\|D\chi\|
>\frac{16}{\sqrt{17}}-\frac{6C_0}{\beta}
=\frac{16}{\sqrt{17}}-3>\frac12.
\end{equation}
If the correction's workspace gradient vanishes but $\|q\|\ge1$, then
\[
\|\nabla_xF_{\rm new}\|
=\frac{4\|q\|}{\sqrt{1+\|q\|^2}}\ge\frac4{\sqrt2}>\frac12.
\]
Hence $\|\nabla F_{\rm new}\|\le e_0\le e_*$ forces $\|q\|<1$. Every potentially live gate then has $\|e_i\|<13<\beta$, so all correction two-jets vanish. The remaining derivative identities are
\[
\nabla_xF_{\rm new}=\nabla h(q),
\qquad
\nabla_zF_{\rm base}=\nabla_zF_{\rm new}+(D_za)^\top\nabla_xF_{\rm new}.
\]
Since $\|D_za\|\le A\|Dg\|\|U^\top Dy\|\le15A/2$, defining $J=1+15A/2=961$ gives
\begin{align*}
\|\nabla_zF_{\rm base}\|
&=\|\nabla_zF_{\rm new}+(D_za)^\top\nabla_xF_{\rm new}\|\\
&\le\|\nabla_zF_{\rm new}\|+\|D_za\|\,\|\nabla_xF_{\rm new}\|\\
&\le(1+15A/2)\|\nabla F_{\rm new}\|
\le Je_0\le\frac15.
\end{align*}
The first identity uses vanishing correction jets for $\|q\|<1$, the second is the triangle inequality, and the third uses the operator-norm chain rule. This is the point at which the global full-variable stationarity request reduces to a zero-chain request.

We next apply the base localization argument \emph{twice}, because the
first application alone gives the weaker tolerance $Je_0$.  Indeed,
\eqref{eq:app-base-grad} and the base lemma, applied with threshold
$\eta=Je_0\le1/5$, imply
\begin{equation}\label{eq:app-first-base-pass}
\|\nabla F_{\rm base}(z)\|\le Je_0\quad\Longrightarrow\quad
w_i>1\quad(1\le i\le T).
\end{equation}
The definition $g_i(w)=\Gamma(|w_i|)$, and the fact that $\Gamma$ is
identically one on $[1/2,\infty)$, then give, \emph{at this particular
point},
\begin{align*}
g_i(w)&=1, & D_wg_i(w)&=0 &&(1\le i\le T),\\
a(w)&=A\sum_{i=1}^T v_i,
& D_za(w(z))&=AVD_wg(w)U^\top Dy(z)=0.
\end{align*}
The correction two-jets had already vanished on $\|q\|<1$, so the
ordinary chain rule for $F_{\rm new}=F_{\rm base}+h(x-a(w))$ gives
\begin{align*}
\nabla_zF_{\rm new}
&=\nabla F_{\rm base}-(D_za)^\top\nabla h(q)
=\nabla F_{\rm base},\\
\nabla_xF_{\rm new}&=\nabla h(q).
\end{align*}
In particular, the \emph{improved} bound is
\begin{equation}\label{eq:app-second-base-pass}
\|\nabla F_{\rm base}(z)\|
=\|\nabla_zF_{\rm new}(z,x)\|
\le\|\nabla F_{\rm new}(z,x)\|\le e_0,
\end{equation}
not merely $Je_0$. Apply the \emph{same base lemma a second time}, now
with tolerance $e_0\le1/5$.  The base-gradient identity
\eqref{eq:app-base-grad} then has $c>1/5$, and projecting it onto the
orthogonal complement of $\operatorname{span}(U)$ yields
\begin{align*}
P_{U^\perp}\nabla F_{\rm base}&=cP_{U^\perp}y,\\
\|P_{U^\perp}y\|
&=c^{-1}\|P_{U^\perp}\nabla F_{\rm base}\|
\le c^{-1}\|\nabla F_{\rm base}\|
<5e_0.
\end{align*}
Finally, since
$\nabla_xF_{\rm new}=\nabla h(q)=4q/\sqrt{1+\|q\|^2}$, we have
\begin{equation}\label{eq:app-terminal-distance}
\begin{aligned}
\frac{4\|q\|}{\sqrt{1+\|q\|^2}}
&=\|\nabla_xF_{\rm new}\|\le e_0,\\
16\|q\|^2&\le e_0^2(1+\|q\|^2),\\
(16-e_0^2)\|q\|^2&\le e_0^2,\\
\|x-A\sum_i v_i\|=\|q\|
&\le\frac{e_0}{\sqrt{16-e_0^2}},
\end{aligned}
\end{equation}
where $16-e_0^2>0$ follows from $e_0\le e_*\le1/5$.
Thus both the $5e_0$ projection estimate and the workspace estimate
have the claimed tolerance, completing the global certificate.

For the derivative budget, differentiate the radial cutoff only on its transition annulus $\rho_0/2<\|u\|<\rho_0$, where the norm is $C^2$. With $t=\|u\|$, the chain rule gives
\begin{align*}
D\tau_e(u)[h]&=\tau'(t)\frac{\langle u,h\rangle}{t},\\
D^2\tau_e(u)[h,k]
&=\tau''(t)\frac{\langle u,h\rangle\langle u,k\rangle}{t^2}
+\tau'(t)\left(\frac{\langle h,k\rangle}{t}
-\frac{\langle u,h\rangle\langle u,k\rangle}{t^3}\right).
\end{align*}
Since the cutoff is flat outside the annulus and $\tau(t)=1-S_5(2t/\rho_0-1)$ inside, the bounds on $S_5'$ and $S_5''$ imply
\[
\|D\tau_e\|\le\frac{15}{4\rho_0}=40,
\qquad \|D^2\tau_e\|<2700.
\]
For the composition $\tau_e\circ g_{\ge i}$, use
\[
D^2(\tau_e\circ g)[h,k]
=D^2\tau_e(g)[Dg[h],Dg[k]]+D\tau_e(g)[D^2g[h,k]].
\]
The operator-norm chain is therefore
\begin{align*}
\|D(\tau_e\circ g)\|
&\le40\cdot\frac{15}{2}=300,\\
\|D^2(\tau_e\circ g)\|
&\le2700\left(\frac{15}{2}\right)^2+40\cdot93
=151875+3720=155595.
\end{align*}
A second application of the product rule to $c_i=b_i(\tau_e\circ g_{\ge i})$ then gives
\begin{align*}
\|Dc_i\|&\le\|Db_i\|+|b_i|\,\|D(\tau_e\circ g)\|
\le20+8\cdot300=2420,\\
\|D^2c_i\|&\le\|D^2b_i\|+2\|Db_i\|\,\|D(\tau_e\circ g)\|
+|b_i|\,\|D^2(\tau_e\circ g)\|\\
&\le160+2\cdot20\cdot300+8\cdot155595=1256920.
\end{align*}
These bounds apply on the relevant support, where $|b_i|<8$; outside it all correction jets vanish. Composing with $w=U^\top y(z)$ and $\|D^2y\|\le2/R$ yields
\[
\|D_zc_i\|\le2420,
\quad \|D_z^2c_i\|<1257000.
\]
For $E_i(z,x)=x-a_{<i}(w(z))$, the first derivative obeys
\[
\|DE_i\|\le1+A\|Dg\|\|U^\top Dy\|
\le1+A(15/2)=J.
\]
Its second derivative is obtained by the bilinear chain rule:
\begin{align*}
D^2a_{<i}(w(z))[h,k]
&=A\sum_{j<i}v_j\Bigl(D^2g_j(w)[U^\top Dy[h],U^\top Dy[k]]
+Dg_j(w)[U^\top D^2y[h,k]]\Bigr),\\
\|D^2E_i\|
&\le A\left(\|D^2g\|\|Dy\|^2+\|Dg\|\|D^2y\|\right)\\
&\le A\left(93+\frac{15}{R}\right)
<A\cdot94=:K_0=12032.
\end{align*}
The first inequality uses orthogonality of the $v_j$ frame to control the vector sum; the second uses $\|Dy\|\le1$ and $\|D^2y\|\le2/R$.
Using \eqref{eq:app-spatial-cutoff-bounds}, the product rule gives
\begin{align*}
\|D[(1-\chi(E_i))c_i]\|&\le2420+48J/\beta,\\
\|D^2[(1-\chi(E_i))c_i]\|&\le1257000+29040J/\beta
+416J^2/\beta^2+48K_0/\beta.
\end{align*}
Only one correction two-jet is live, so the same bounds hold for their sum. The workspace well contributes at most $4J$ to the gradient and $4J^2+4K_0$ to the Hessian. Thus the numerical value
\begin{equation}\label{eq:app-Hstar}
H_*=161+4J^2+4K_0+1257000
+\frac{29040J}{\beta}+\frac{416J^2}{\beta^2}+\frac{48K_0}{\beta}
=\frac{66243245}{8}
\end{equation}
controls $\|\nabla^2F_{\rm new}\|$ globally. Similarly
\begin{equation}\label{eq:app-full-gradient-bound}
\|\nabla F_{\rm new}\|
\le37\sqrt T+2420+4J+48J/\beta
=37\sqrt T+9147\le9184\sqrt T.
\end{equation}
Coercivity follows from $F_{\rm new}\ge F_{\rm base}$ for $\|z\|\to\infty$, and, when $z$ is bounded, from
$F_{\rm new}\ge F_{\rm base}+h(x-a)-8\to\infty$ as $\|x\|\to\infty$.
\end{proof}

\paragraph{Exact prefix dependence.}
If all future $|w_j|<1/4$ beyond a disclosed stage $k$, all later activations vanish. The corrections through index $k$ involve only the already known past centers $a_{<i}$. Away from the complete stage-$k$ gate, the only new $v_k$-dependent term is
\[
h(r_k-Ag_kv_k),\qquad r_k=x-a_{<k}.
\]
At a stage crossing with $|w_k|\ge1/2$, the next correction reduces to $c_{k+1}=b_{k+1}$ and its gate is centered at $a_{<k+1}=a_k$. This gives the constant-angle recovery interface in the assembly theorem.

%% file: sections/gaussian_proofs.tex
\section{The Reusable Gaussian Experiment}
\label[appendix]{app:gaussian}
\subsection{The fixed sample field and its features}
For $q\in\R^m$, $g\in[0,1]$, and $v\in\sph^{m-1}$, put
\[
 \mu_v(q,g)=4(\sqrt{1+\|q-Agv\|^2}-1),\qquad
 S=1+\|q\|^2+A^2g^2,
\]
\[
 t=2Agq/S,\qquad
 \rho=A/\sqrt{1+A^2}<1.
\]
Let $c_\ell=\binom{1/2}{\ell}(-1)^\ell$ and
$\omega_\ell=\ell^2|c_\ell|$ for $\ell\ge1$. Each label retains
independent standard Gaussian tensors $G_{j,\ell}$, and
\begin{equation}\label{eq:noise-definition}
 Q_j(q,g)=\frac\sigma{\sqrt m}\sum_{\ell\ge1}
 \omega_\ell\sqrt S\,\langle G_{j,\ell},t^{\otimes\ell}\rangle.
\end{equation}
In $\cH=\bigoplus_{\ell\ge1}(\R^m)^{\otimes\ell}$ set
\begin{equation}\label{eq:feature}
 \phi_\ell(q,g)=\frac\sigma{\sqrt m}\omega_\ell\sqrt S\,t^{\otimes\ell},
 \qquad \theta_\ell(v)=-\frac{4\sqrt m}{\sigma\ell^2}v^{\otimes\ell}.
\end{equation}
The Gaussian coefficients define an isonormal process on $\cH$;
they are not assumed to form an $\cH$-valued Gaussian vector.
We prove first that the scalar series is an actual globally $C^1$
sample function, and then use its exact feature representation.

\subsection{Uniform convergence and sample regularity}
\begin{lemma}[A fixed regular version of each field]\label[lemma]{lem:field-regularity}
The series \eqref{eq:noise-definition} defines a jointly measurable
version whose samples are $C^1$ on the query strip
$\R^m\times[0,1]$. Its value and first derivatives are simultaneous
limits of the finite tensor sums. Each label retains this same
version at all adaptive and repeated queries. Its global
Lipschitz constant on this strip has finite second moment.
\end{lemma}
\begin{proof}
For each degree $\ell$, put
\[
P_\ell(e)=\langle G_\ell,e^{\otimes\ell}\rangle,
\qquad
Z_{\ell,a}=\sup_{\|e\|=\|h_1\|=\cdots=\|h_a\|=1}
\big|D^aP_\ell(e)[h_1,\ldots,h_a]\big|,
\quad a=0,1.
\]
The proof reduces the global value and derivative series to these
two polynomial suprema, bounds the suprema by Gaussian nets, and
then sums over degrees.

\paragraph{Uniform bounds for one degree.}
Recall $S=1+\|q\|^2+A^2g^2$, $t=2Agq/S$, and $\rho=A/\sqrt{1+A^2}<1$. Since
\[
\|t\|=\frac{2Ag\|q\|}{1+A^2g^2+\|q\|^2}
\le\frac{Ag}{\sqrt{1+A^2g^2}}\le\rho,
\qquad \sqrt S\,\|t\|\le2Ag\le2A,
\]
homogeneity gives
\begin{equation}\label{eq:app-value-sup}
\sup_{q,g}\big|\sqrt S P_\ell(t)\big|
\le2A\rho^{\ell-1}Z_{\ell,0}.
\end{equation}
The relevant derivatives are
\begin{align*}
D_qt&=\frac{2Ag}{S}I-\frac{4Ag}{S^2}qq^\top,
& D_gt&=\frac{2Aq}{S}-\frac{4A^3g^2q}{S^2},\\
\|D\sqrt S\|&\le1+A,
& \sqrt S\,\|Dt\|&\le3A.
\end{align*}
Indeed the parenthesized operator in $D_qt=(2Ag/S)(I-2qq^\top/S)$ has norm at most one; the analogous scalar factor in $D_gt$ does as well. The two derivative blocks, after multiplication by $\sqrt S$, have norm at most $2A$ each. Consequently
\begin{equation}\label{eq:app-derivative-sup}
\sup_{q,g}\big\|D\big(\sqrt S P_\ell(t)\big)\big\|
\le (1+A)\rho^\ell Z_{\ell,0}
+3A\rho^{\ell-1}Z_{\ell,1}.
\end{equation}

Thus a bound on $Z_{\ell,0}$ and $Z_{\ell,1}$ controls every query
location, including arbitrarily large $\|q\|$.

\paragraph{Gaussian bounds for the polynomial suprema.}
We derive the degree-uniform estimate from the finite Gaussian
maximum inequality in \Cref{thm:tool-gaussian}, including the entropy
count.  For $a=0,1$ write the indexed Gaussian variable as
$Z_x=\langle G_\ell,A_{\ell,a}(x)\rangle$, with
$x=(e,h_1,\ldots,h_a)$ a tuple of unit vectors and
\[
A_{\ell,0}(e)=e^{\otimes\ell},\qquad
A_{\ell,1}(e,h)=\sum_{j=1}^\ell
 e^{\otimes(j-1)}\otimes h\otimes e^{\otimes(\ell-j)}.
\]
The tensor-product norm is multiplicative, and the derivative tensor
contains at most $\ell^a$ unit simple tensors.  Thus
\[
\|A_{\ell,a}(x)\|\le\ell^a,
\qquad \|Z_x\|_{L^2}=\|A_{\ell,a}(x)\|\le\ell^a.
\]
For two unit tuples $x,x'$, telescope each tensor product one slot at
a time.  The difference of one rank-one tensor product is bounded by
the sum of the differences of its factors; summing over the at most
$\ell^a$ derivative placements therefore gives
\begin{equation}\label{eq:app-polynomial-coefficient-metric}
\|A_{\ell,a}(x)-A_{\ell,a}(x')\|
\le C_a\ell^{a+1}\max_{0\le j\le a}\|x_j-x_j'\|,
\quad (a=0,1).
\end{equation}
For clarity, at $a=1$ this also follows directly by splitting
$A_{\ell,1}(e,h)-A_{\ell,1}(e',h')$ into a term with $h-h'$,
whose norm is at most $\ell\|h-h'\|$, and a term with $e-e'$,
whose $\ell$ summands each telescope over $\ell-1$ slots, yielding
at most $\ell(\ell-1)\|e-e'\|$.  The $a=0$ case is the same
$\ell$-slot telescoping identity.

For each $n\ge0$ construct a Euclidean $2^{-n}/\ell$-net
$\mathcal N_n$ in the product of $a+1$ unit spheres, using the
maximum of the coordinate distances.  The volumetric sphere covering
bound gives
\begin{align*}
|\mathcal N_n|
&\le(1+2^{n+1}\ell)^{m(a+1)},\\
\log(2|\mathcal N_n|)
&\le C_a m\,[\log(2\ell)+n+1].
\end{align*}
Choose a nearest-net representative $x_n$ of each index $x$, so
$\max_j\|x_j-(x_n)_j\|\le2^{-n}/\ell$.
Equation \eqref{eq:app-polynomial-coefficient-metric} implies that
$Z_x-Z_{x_0}$ has Gaussian standard deviation at most
$C_a\ell^a$, and that $Z_{x_{n+1}}-Z_{x_n}$ has standard deviation
at most $C_a\ell^a2^{-n}$.  There are at most
$|\mathcal N_{n+1}|\,|\mathcal N_n|$ possible links at level $n$.
Applying the $L^2$ finite-maximum inequality from
\Cref{thm:tool-gaussian} to the base net and each collection of links,
\begin{align*}
\Bigl\|\max_{x_0\in\mathcal N_0}|Z_{x_0}|\Bigr\|_{L^2}
&\le C_a\ell^a\sqrt{m\log(2\ell)},\\
\Bigl\|\max_{\text{level-$n$ links}}
  |Z_{x_{n+1}}-Z_{x_n}|\Bigr\|_{L^2}
&\le C_a\ell^a2^{-n}
   \sqrt{m[\log(2\ell)+n+1]}.
\end{align*}
For each fixed finite tensor $G_\ell$, the polynomial map in the
unit-vector arguments is continuous; hence $Z_{x_n}\to Z_x$.
Finite telescoping followed by Minkowski's inequality and monotone
passage over the net levels gives
\begin{align*}
\|Z_{\ell,a}\|_{L^2}
&\le C_a\ell^a\sqrt m\left[
 \sqrt{\log(2\ell)}+
 \sum_{n\ge0}2^{-n}\sqrt{\log(2\ell)+n+1}\right]\\
&\le C_a'\ell^a\sqrt{m\log(2\ell)}.
\end{align*}
The final step uses $\sqrt{u+v}\le\sqrt u+\sqrt v$,
$\sum_{n\ge0}2^{-n}=2$ and
$\sum_{n\ge0}2^{-n}\sqrt{n+1}<\infty$.
Absorb the numerical $C_a'$ into $C_a$ to record
\begin{equation}\label{eq:app-polynomial-sup}
\|Z_{\ell,a}\|_{L^2}\le C_a\ell^a\sqrt{m\log(2\ell)}.
\end{equation}

\paragraph{Summation and the retained sample function.}
Combining \eqref{eq:app-value-sup}, \eqref{eq:app-derivative-sup},
and \eqref{eq:app-polynomial-sup} with $\omega_\ell\le\ell^2$ yields
\begin{align}
\sum_{\ell\ge1}\frac{\sigma\omega_\ell}{\sqrt m}
\sum_{a=0}^1\left\|\sup_{q,g}
\big\|D^a\big(\sqrt S P_\ell(t)\big)\big\|\right\|_{L^2}
&\le C_A\sigma\sum_{\ell\ge1}\ell^3\sqrt{\log(2\ell)}\,\rho^{\ell-1}<\infty.
\label{eq:app-gaussian-series}
\end{align}
Tonelli's theorem and $\|X\|_{L^1}\le\|X\|_{L^2}$ imply almost-sure uniform summability of the values and first derivatives. The same estimates on $|g|\le2$ hold with $A$ replaced by $2A$
and absolute values on $g$, giving a geometric ratio
$2A/\sqrt{1+4A^2}<1$. Thus the derivative convergence also holds
on an open neighborhood of the query strip, including its endpoints.
Therefore the infinite field defines a globally $C^1$ process on $\R^m\times[0,1]$. Minkowski's inequality gives
\[
\E\sup_{q,g}|Q_j(q,g)|^2+\E\Lip(Q_j)^2\le C_A\sigma^2.
\]
Convergence in $L^1$ permits taking expectations, so $\E Q_j=0$. Since $t=0$ when $g=0$ or $q=0$, $Q_j(q,0)=Q_j(0,g)=0$ identically. 
\end{proof}

\begin{lemma}[Simultaneous tensor-series version]\label[lemma]{lem:simultaneous-tensor-version}
The retained tensors in \eqref{eq:noise-definition} determine one
jointly measurable version of every $Q_j$. There is a single event
$\Omega_*$ of probability one on which all labels have $C^1$ sample
paths and their value and derivative series converge absolutely and
uniformly on $\R^m\times[0,1]$. On this event, for every finite set
of points $x_a=(r_a,g_a)$ and real coefficients $d_a$,
\begin{equation}\label{eq:canonical-finite-combination}
 \sum_a d_aQ_j(x_a)
 =\sum_{\ell\ge1}\left\langle G_{j,\ell},
                       \sum_a d_a\phi_\ell(x_a)\right\rangle.
\end{equation}
In particular, $\sum_a d_a\phi(x_a)=0$ implies
$\sum_a d_aQ_j(x_a)=0$. Both assertions hold simultaneously for all
finite configurations, including configurations selected from earlier
responses.
\end{lemma}
\begin{proof}
For each label, \eqref{eq:app-gaussian-series} bounds the expectation
of the sum of the global value and derivative suprema. These suprema
are measurable: each degree is continuous, so its supremum equals the
supremum over a countable dense set. A nonnegative random variable
with finite expectation is finite almost surely. Intersect the resulting
events over the finite or countable label set to obtain $\Omega_*$.
Set the entire noise field to zero on its complement. The limits are
jointly measurable, and uniform derivative convergence gives $C^1$
paths as above.

Fix a sample in $\Omega_*$. Every finite weighted sum of value
series is absolutely convergent, so its finite sum can be interchanged
with the degree series. Each degree is a linear functional of its
feature, which proves \eqref{eq:canonical-finite-combination}.
If the total Hilbert feature is zero, each degree component is zero,
so the right-hand side vanishes term by term. Since the argument is
pointwise on $\Omega_*$, it requires no point-dependent exceptional
sets. It also makes no assertion that the Gaussian coordinate array
is an $\mathcal H$-valued vector or a bounded functional on all of
$\mathcal H$.
\end{proof}

\subsection{The exact Gaussian shift representation}
\begin{lemma}[Mean representation and the actual tangent derivative]
\label[lemma]{lem:tensor-shift}
For the features in \eqref{eq:feature}, $\phi$ is continuous on the
query strip and $\theta:\sph^{m-1}\to\cH$ is $C^1$.
For every $v\in\sph^{m-1}$,
\[
 \mu_v(q,g)=4(\sqrt S-1)+\langle\theta(v),\phi(q,g)\rangle_{\cH}.
\]
For $h\perp v$, the actual tangent derivative satisfies
\begin{equation}\label{eq:app-parameter-deriv}
 \|D\theta(v)[h]\|_{\cH}^2
 =\frac{16m}{\sigma^2}\sum_{\ell\ge1}\ell^{-3}\|h\|^2
 \le\frac{20m}{\sigma^2}\|h\|^2.
\end{equation}
The derivative is continuous in sphere charts; no second derivative
of $\theta$ is required.
\end{lemma}
\begin{proof}
The estimates $\sqrt S\|t\|^\ell\le2A\rho^{\ell-1}$ and
$\omega_\ell\le\ell^2$ give
\[
 \sum_{\ell\ge1}\sup_{q,g}\|\phi_\ell(q,g)\|
 \le\frac{2A\sigma}{\sqrt m}
       \sum_{\ell\ge1}\ell^2\rho^{\ell-1}<\infty.
\]
The first step inserts \eqref{eq:feature}; the second series
converges because $\rho<1$. Each finite block is continuous, so
uniform convergence proves continuity of $\phi$.

Put $h\perp v$. The product rule and the tensor inner-product
identity give
\begin{align*}
 D(v^{\otimes\ell})[h]
 &=\sum_{a=1}^{\ell}
       v^{\otimes(a-1)}\otimes h\otimes v^{\otimes(\ell-a)},\\
 \big\langle v^{\otimes(a-1)}\otimes h\otimes v^{\otimes(\ell-a)},
       v^{\otimes(b-1)}\otimes h\otimes v^{\otimes(\ell-b)}\big\rangle
 &=\begin{cases}\|h\|^2,&a=b,\\0,&a\ne b,\end{cases}\\
 \|D\theta_\ell(v)[h]\|^2
 &=\frac{16m}{\sigma^2\ell^4}\ell\|h\|^2
 =\frac{16m}{\sigma^2\ell^3}\|h\|^2.
\end{align*}
For $a\ne b$, one tensor factor contains $\langle v,h\rangle=0$;
for $a=b$, the other factors have norm one. The last line uses
the coefficient of $\theta_\ell$.
Consequently
\begin{align*}
 \sum_{\ell\ge1}\left(
     \sup_v\|\theta_\ell(v)\|
     +\sup_{\substack{\|h\|=1\\h\perp v}}
                  \|D\theta_\ell(v)[h]\|\right)
 &\le\frac{4\sqrt m}{\sigma}
       \sum_{\ell\ge1}(\ell^{-2}+\ell^{-3/2})<\infty.
\end{align*}
Both scalar series are convergent $p$-series. A sphere chart $u$
has tangent derivatives, since
\[
 0=D(\|u\|^2)[h]=2\langle u,Du[h]\rangle.
\]
The first equality differentiates the constant unit norm; the
second differentiates its inner product. In a sphere chart
with bounded derivative, these bounds give uniform convergence
of the finite sums and of their derivatives. To verify that the
limit is the derivative, apply the fundamental theorem of calculus
on a convex chart neighborhood to each partial sum, and pass to
the limit under its one-dimensional integral. Continuity of the
limiting derivative then gives the Fr\'echet derivative of the
limit. This proves the asserted $C^1$ regularity.

Orthogonality between degree blocks now gives
\begin{align*}
 \|D\theta(v)[h]\|^2
 &=\sum_{\ell\ge1}\|D\theta_\ell(v)[h]\|^2\\
 &\le\frac{16m}{\sigma^2}
       \left(1+\frac18+\int_2^\infty u^{-3}\,du\right)\|h\|^2
 =\frac{20m}{\sigma^2}\|h\|^2.
\end{align*}
The first equality is the Hilbert direct-sum norm. The inequality
bounds the decreasing tail of $\sum\ell^{-3}$ by its integral;
the final equality uses $\int_2^\infty u^{-3}du=1/8$.

Finally $\|t\|\le\rho<1$ permits the binomial expansion, and
$c_\ell<0$, $\omega_\ell=\ell^2|c_\ell|$ give
\begin{align*}
 1+\|q-Agv\|^2
 &=S-2Ag\langle q,v\rangle
 =S(1-\langle t,v\rangle),\\
 \mu_v(q,g)
 &=4(\sqrt S-1)
   +4\sqrt S\sum_{\ell\ge1}c_\ell\langle t,v\rangle^\ell\\
 &=4(\sqrt S-1)
   +\sum_{\ell\ge1}\langle\theta_\ell(v),\phi_\ell(q,g)\rangle.
\end{align*}
The first line expands the squared norm and uses $\|v\|=1$.
The second is the absolutely convergent series for
$\sqrt{1-z}$. The last substitutes the two features and
$-|c_\ell|=c_\ell$; Hilbert Cauchy--Schwarz justifies their sum.
\end{proof}

\subsection{Canonical regression under adaptive reuse}
\begin{lemma}[Canonical adaptive response kernel]\label[lemma]{lem:canonical-adaptive-kernel}
Fix $v$ and initial randomization and auxiliary data independent of
all current-stage tensors. Let $\mathcal F_i$ contain these initial
data and the first $i$ queries, labels and responses. Use the actual
sample version of \Cref{lem:simultaneous-tensor-version}. At round $i$,
the label $j_i$ and point $x_i=(r_i,g_i)$ are
$\mathcal F_{i-1}$-measurable and
\begin{equation}\label{eq:app-observation-hilbert}
 Y_i=b_i+\langle\theta(v),\phi_i\rangle+Q_{j_i}(x_i),
 \quad b_i=4(\sqrt{S(x_i)}-1),\quad\phi_i=\phi(x_i).
\end{equation}
For the chosen label let $P_{j_i,i-1}$ project onto the span of its
previous features, and choose history-measurable coefficients satisfying
\[
 P_{j_i,i-1}\phi_i=\sum_{t<i:j_t=j_i}c_{it}\phi_t,
 \qquad\psi_i=(I-P_{j_i,i-1})\phi_i.
\]
Then the exact conditional response kernel is, almost surely,
\begin{equation}\label{eq:app-conditional-kernel}
 Y_i\mid(v,\mathcal F_{i-1}=h)\sim
 \begin{cases}
 N(m_i(v,h),s_i^2(h)),&s_i^2(h)>0,\\
 \delta_{m_i(h)},&s_i^2(h)=0,
 \end{cases}
\end{equation}
where
\begin{equation}\label{eq:app-cond-mean}
 \begin{aligned}
 m_i(v,h)&=b_i+\sum_{t<i:j_t=j_i}c_{it}(Y_t-b_t)
                         +\langle\theta(v),\psi_i\rangle,\\
 s_i^2(h)&=\|\psi_i\|^2.
 \end{aligned}
\end{equation}
The coefficients and variance are functions of the realized history,
independent of the candidate $v$. If $s_i=0$, the mean $m_i(h)$ is
also independent of $v$. This holds with interleaved labels, repeated
features and singular Gram matrices.
\end{lemma}
\begin{proof}
\par\medskip\noindent\textbf{Measurable projections and exact residuals.}\enspace
For one label, write $\Phi c=\sum_t c_t\phi_t$ for its finite old-feature
synthesis map and $K=\Phi^*\Phi$ for its Gram matrix. The Moore--Penrose
inverse is Borel measurable on positive semidefinite matrices, since
\begin{equation}\label{eq:gram-borel-regularization}
 K^\dagger=\lim_{q\to\infty}(K^2+q^{-1}I)^{-1}K.
\end{equation}
Indeed, an eigenvalue $a\ge0$ is transformed to
$a/(a^2+q^{-1})$, whose limit is $a^{-1}$ for $a>0$ and zero for
$a=0$. Each approximating matrix is a continuous function of $K$,
so its entrywise limit is Borel measurable. Finite-dimensional spectral
diagonalization justifies the displayed limit without requiring a
measurable choice of eigenvectors.

The projection is $P=\Phi K^\dagger\Phi^*$ and the required
coefficient vector is $c=K^\dagger\Phi^*\phi_i$. To check the
projection formula, $P\phi_i$ is in the old span and
$\Phi^*(\phi_i-P\phi_i)=0$: here
$\operatorname{ran}\Phi^*=\operatorname{ran}K$ because
$\ker K=\ker\Phi$, as follows from
$\langle c,Kc\rangle=\|\Phi c\|^2$.
Thus the residual is orthogonal to that span. The feature map is
continuous, and the query rules are measurable; all these projections
and coefficients are therefore history-measurable even when the rank
changes.

By \Cref{lem:simultaneous-tensor-version}, on the one event $\Omega_*$,
\begin{equation}\label{eq:canonical-pathwise-residual}
 Y_i-b_i-\sum_{t<i:j_t=j_i}c_{it}(Y_t-b_t)
 =\langle\theta(v),\psi_i\rangle
       +\sum_{\ell\ge1}\langle G_{j_i,\ell},\psi_{i,\ell}\rangle.
\end{equation}
The last series is a finite combination of actual feature series.
This equality consequently holds at the selected points, rather than
only almost surely for each fixed point separately. If $\psi_i=0$,
it already proves the stated parameter-independent point mass.

\par\medskip\noindent\textbf{Conditional coordinate law.}\enspace
Condition first on $v$ and the initial auxiliary data. The remaining
randomness is a countable product of real Gaussian coordinates, and
a finite history is a measurable function of the real response vector.
Regular conditional laws therefore exist by
\Cref{thm:tool-disintegration}. Denote the retained coordinate array
of label $j$ by $\Gamma_j$.

We induct on the number of responses. Before round $i$, conditional
on the realized history, the arrays of different labels have independent
Gaussian cylinder laws. For label $j$ their mean and covariance are
\begin{equation}\label{eq:canonical-posterior-coordinates}
 a_j=\Phi_jK_j^\dagger z_j\in\operatorname{ran}\Phi_j,
 \qquad
 \operatorname{Cov}_h(\Gamma_j[u],\Gamma_j[u'])
       =\langle(I-P_j)u,(I-P_j)u'\rangle,
\end{equation}
where
$z_j=(Y_t-b_t-\langle\theta(v),\phi_t\rangle)_{t<i:j_t=j}$.
The notation $\Gamma_j[u]$ initially denotes only finite-coordinate
linear combinations. Their mean is $\langle a_j,u\rangle$.
At the initial history the mean is zero and every $P_j=0$, which is
the original independent product law.

Choose versions of the conditional coordinate characteristic functions
first for finite label sets and rational finite-coordinate vectors.
There are countably many such tests; hence at each round one common
full-measure history set satisfies all of them. Continuity of each
finite-dimensional characteristic function extends the identities to
all real finite-coordinate vectors. For a history in this set, the
conditional joint characteristic function is therefore
\begin{equation}\label{eq:canonical-cylinder-characteristic}
 \E_h\exp\!\left(\mathrm i\sum_j\Gamma_j[u_j]\right)
 =\exp\!\left(\sum_j\left[
    \mathrm i\langle a_j,u_j\rangle
       -\tfrac12\|(I-P_j)u_j\|^2\right]\right)
\end{equation}
for every finite set of labels and finite-coordinate vectors $u_j$.
This both specifies the cylinder laws and states their conditional
independence.

\par\medskip\noindent\textbf{The selected feature and the coordinate limit.}\enspace
After fixing such a history, the selected label, features and Gram
coefficients are deterministic. Truncate a feature $u$ after complete
tensor degrees, writing $u^{\le L}$ for its finite-coordinate
truncation. The conditional Gaussian second moment of a finite tail
satisfies
\begin{equation}\label{eq:canonical-conditional-tail}
 \E_h|\Gamma_j[u^{\le L'}-u^{\le L}]|^2
 \le(1+\|a_j\|^2)\|u^{\le L'}-u^{\le L}\|^2.
\end{equation}
The covariance in \eqref{eq:canonical-posterior-coordinates} is a
contraction, and Cauchy--Schwarz bounds the squared mean by
$\|a_j\|^2$ times the tail norm squared. Since $u\in\mathcal H$,
these finite sums are Cauchy in conditional $L^2$. Their limit has
mean $\langle a_j,u\rangle$ and covariance given by
\eqref{eq:canonical-posterior-coordinates}. For any finite list of
features and coordinate tests, the same bound applies to each linear
combination. Their joint characteristic functions converge, since
$|e^{\mathrm ix}-e^{\mathrm iy}|\le|x-y|$ and $L^2$ convergence
implies $L^1$ convergence. The limits are therefore jointly Gaussian
with the stated means and covariances, rather than merely random
variables sharing those first two moments.

Disintegration also gives conditional probability one to $\Omega_*$
for almost every history. On this event, the same complete-degree
partial sums converge to the actual selected sample value, or to its
finite residual combination in \eqref{eq:canonical-pathwise-residual}.
The almost-sure limit and the conditional $L^2$ limit agree by uniqueness
of limits in probability. Thus the coordinate-law calculation applies
to the actual adaptive response. Only the countable coordinate tests,
not an uncountable family of fixed-feature null sets, were intersected.

In particular $a_j\perp\psi_i$, so the residual noise in
\eqref{eq:canonical-pathwise-residual} is conditionally
$N(0,\|\psi_i\|^2)$. Its conditional mean before subtraction is
$\langle a_j,\phi_i\rangle=\sum_t c_{it}z_{j,t}$. Consequently
\[
 b_i+\langle\theta(v),\phi_i\rangle+\langle a_j,\phi_i\rangle
 =b_i+\sum_t c_{it}(Y_t-b_t)+\langle\theta(v),\psi_i\rangle,
\]
which is exactly \eqref{eq:app-cond-mean}.

\par\medskip\noindent\textbf{Updating the state and completing the induction.}\enspace
For any finite-coordinate vector $u$, the preceding joint law gives
\[
 \operatorname{Cov}_h(\Gamma_j[u],Y_i-m_i)
 =\langle(I-P_j)u,\psi_i\rangle=\langle u,\psi_i\rangle,
 \qquad \operatorname{Var}_h(Y_i-m_i)=s_i^2.
\]
The second equality in the covariance uses $\psi_i\perp\operatorname{ran}P_j$.
If $s_i>0$, deterministic Gaussian conditioning thus gives, for the queried
label,
\begin{equation}\label{eq:canonical-posterior-update}
 P'_j=P_j+\frac{\psi_i\otimes\psi_i}{s_i^2},
 \qquad
 a'_j=a_j+\frac{Y_i-m_i(v,h)}{s_i^2}\psi_i.
\end{equation}
Here $(u\otimes u)w=\langle u,w\rangle u$. These are the orthogonal projection onto the enlarged span and its minimum-norm interpolant. The formula follows by subtracting the old
projection and conditioning the one remaining Gaussian coordinate;
\Cref{thm:tool-gaussian-regression} is used here only for deterministic
finite-dimensional coordinate laws. If $s_i=0$, the response is
already determined and neither quantity changes. Other labels retain
their conditional laws, since the selected label is fixed by the old
history. Applying this argument to the countable coordinate tests
establishes \eqref{eq:canonical-cylinder-characteristic} at the next
round, completing the induction.

\par\medskip\noindent\textbf{A parameter-independent reference kernel.}\enspace
Let $z_j^0=(Y_t-b_t)_{t<i:j_t=j}$. The old noise data are consistent
precisely when
\[
 z_j=z_j^0-\Phi_j^*\theta(v)\in\operatorname{ran}K_j.
\]
Because $\Phi_j^*\theta(v)\in\operatorname{ran}K_j$, this condition
is equivalent to $z_j^0\in\operatorname{ran}K_j$ and is independent
of $v$. On inconsistent histories assign any common measurable
extension. On consistent histories, use $N(0,1)$ as reference kernel
when $s_i>0$, and $\delta_{b_i+\sum_tc_{it}(Y_t-b_t)}$ when $s_i=0$.
These are probability kernels depending on the history but not on $v$.
The iterated reference law dominates every parameter's transcript
law. In the positive-variance case its density differs from the
Lebesgue Gaussian density by a parameter-independent factor. At a
zero-variance step the likelihood factor is one. Hence subsequent
likelihood differentiation at a fixed history has no hidden
parameter-dependent atom or support term.
\end{proof}

\subsection{Finite-record posteriors and their information}
\begin{lemma}[Posterior regularity and information under arbitrary reuse]
\label[lemma]{lem:posterior-information}
Let $v$ have the uniform prior $\pi$ on the unit sphere of a known
$s$-dimensional subspace of $\R^m$, $s\ge4$. Initial auxiliary data
and algorithmic coins are independent of $v$ and the fields.
For any $n$ adaptive scalar queries, including label reuse and
interleaving, the complete record has a common reference measure.
Its posterior density is positive and $C^1$ on the sphere for
almost every record, and
\[
 I(v;\mathcal F_n)\le\frac{10mn}{(s-2)\sigma^2}.
\]
The seed and record spaces need not be standard Borel.
\end{lemma}
\begin{proof}
Use the field convergence in \Cref{lem:field-regularity}, its
simultaneous version in \Cref{lem:simultaneous-tensor-version},
and the conditional kernels in \Cref{lem:canonical-adaptive-kernel}.
The initial coins and auxiliary data are independent of the
parameter and its fields. Work under the unconditioned tensor
law throughout.

\noindent\textbf{Likelihood differentiation at fixed history.}\ 
For $s_i(h)>0$, conditional on the \emph{realized} $h$ the density with
respect to Lebesgue measure is
\[
p_i(y\mid v,h)
=(2\pi s_i^2(h))^{-1/2}
\exp\!\left[-\frac{(y-m_i(v,h))^2}{2s_i^2(h)}\right].
\]
Only the last term in $m_i(v,h)$ depends on the candidate $v$.
In particular, the derivatives do not act on the algorithm's chosen
query, the old observations, the coefficients $c_{i,t}(h)$, the
projection $P_{j_i,i-1}(h)$, or the variance $s_i^2(h)$: all of them
are frozen by the conditioning.  The parameter-independent conditional
atom when $s_i=0$ contributes zero likelihood score.  For tangent
$h_v\in T_v\mathbb S^{s-1}$, direct differentiation gives
\begin{align*}
D_{h_v}m_i(v,h)
 &=\langle D\theta(v)[h_v],\psi_i(h)\rangle,\\
D_{h_v}\log p_i(Y_i\mid v,h)
 &=\frac{Y_i-m_i(v,h)}{s_i^2(h)}
    \langle D\theta(v)[h_v],\psi_i(h)\rangle.
\end{align*}
Set $z_i=\psi_i/\|\psi_i\|$ and
$\xi_i=(Y_i-m_i(v,\mathcal F_{i-1}))/s_i$ when $s_i>0$.
Then $\xi_i\mid(v,\mathcal F_{i-1})\sim N(0,1)$ and
\begin{equation}\label{eq:app-conditional-score}
D_{h_v}\log p_i(Y_i\mid v,\mathcal F_{i-1})
=\xi_i\langle D\theta(v)[h_v],z_i\rangle.
\end{equation}
The common Hilbert feature map is exactly what makes this derivative
valid after subtracting the previously revealed field projection:
no separate derivative of a conditional-noise mean remains.

Let $S_i\in T_v\mathbb S^{s-1}$ denote this directional score vector,
with $S_i=0$ when $\psi_i=0$.  For an orthonormal basis
$(h_a)_{a=1}^{s-1}$ of the tangent space and $\psi_i\ne0$,
\begin{equation}\label{eq:app-conditional-score-budget}
\begin{aligned}
\E[\|S_i\|^2\mid v,\mathcal F_{i-1}]
&=\E[\xi_i^2\mid v,\mathcal F_{i-1}]
   \sum_{a=1}^{s-1}|\langle D\theta(v)[h_a],z_i\rangle|^2\\
&=\|\operatorname{Proj}_{T_v\mathbb S^{s-1}}
            D\theta(v)^*z_i\|^2\\
&\le\|D\theta(v)\|_{\rm op}^2\|z_i\|^2
\le\frac{20m}{\sigma^2},
\end{aligned}
\end{equation}
where the last inequality is \eqref{eq:app-parameter-deriv}.
In particular, \emph{one} Gaussian innovation contributes an operator
norm, not $s-1$ copies of that bound.  Both cases of the kernel satisfy
\begin{equation}\label{eq:app-score-martingale}
\E[S_i\mid v,\mathcal F_{i-1}]=0,
\qquad
\E[\|S_i\|^2\mid v,\mathcal F_{i-1}]
\le\frac{20m}{\sigma^2}.
\end{equation}

\noindent\textbf{The full-record score.}\
Apply \Cref{thm:tool-adaptive-score}. Its history-independent
variance and atom assumptions follow from
\Cref{lem:canonical-adaptive-kernel}; its square-integrability
assumption follows from \eqref{eq:app-score-martingale} for finite
$n$. The theorem gives
\begin{align}
 \nabla_{\sph}\log p(\mathcal F_n\mid v,\mathcal F_0)
 &=\sum_{i=1}^nS_i,\notag\\
 \E\left\|\sum_{i=1}^nS_i\right\|^2
 &=\sum_{i=1}^n\E\|S_i\|^2
 \le\frac{20mn}{\sigma^2}.
 \label{eq:app-total-score}
\end{align}
The first equality is its likelihood chain rule. The next is its
martingale orthogonality identity, and the inequality sums
\eqref{eq:app-score-martingale}. Independent initial data add no
prior score. 

\noindent\textbf{A common reference and the actual posterior.}
Write $m_i^{\rm pub}$ for the parameter-independent part of
$m_i$ in \eqref{eq:app-cond-mean}. At fixed history, use
$N(m_i^{\rm pub},s_i^2)$ as reference kernel when $s_i>0$,
and its point mass when $s_i=0$. Include the independent initial
data in the reference record law $\nu$. The one-step likelihood
ratio is
\[
 \ell_i(v;t)=
 \begin{cases}
 \displaystyle\exp\left[
 \frac{(y_i-m_i^{\rm pub})\langle\theta(v),\psi_i\rangle}{s_i^2}
 -\frac{\langle\theta(v),\psi_i\rangle^2}{2s_i^2}\right],&s_i>0,\\
 1,&s_i=0.
 \end{cases}
\]
The positive-variance formula is the ratio of two Gaussian
densities with the same variance. If $s_i=0$, then $\psi_i=0$,
so the actual and reference atoms coincide. The finite product
$L_t(v)=\prod_{i=1}^n\ell_i(v;t)$ therefore satisfies
\begin{equation}\label{eq:app-common-likelihood}
 P_v(dt)=L_t(v)\nu(dt),\qquad
 L_t(v)>0,\qquad \int L_t(v)\nu(dt)=1.
\end{equation}
The measure identity follows by iteration of the conditional
kernels; the last equality is their probability normalization.
This construction uses measurable kernels on the original seed
space and introduces no topological assumption on that space.

Define the evidence and normalized density by
\[
 z_t=\int L_t(v)\pi(dv),\qquad p_t(v)=L_t(v)/z_t.
\]
Tonelli's theorem and \eqref{eq:app-common-likelihood} give
\begin{align*}
 \int z_t\nu(dt)
 &=\int\!\int L_t(v)\pi(dv)\nu(dt)
 =\int\!\int L_t(v)\nu(dt)\pi(dv)=1,\\
 P_{\mathcal F_n}(dt)&=z_t\nu(dt),\qquad
 \int p_t(v)\pi(dv)=1.
\end{align*}
The first line interchanges nonnegative integrals and uses the
normalization of $P_v$. The second is marginalization followed
by division by the positive finite evidence. Positivity follows
from $L_t>0$; finiteness holds almost everywhere by the first
line. For every nonnegative measurable $a(v,t)$, Bayes' formula is
verified directly by
\begin{align*}
 \int\!\int a(v,t)p_t(v)\pi(dv)P_{\mathcal F_n}(dt)
 &=\int\!\int a(v,t)\frac{L_t(v)}{z_t}\pi(dv)z_t\nu(dt)\\
 &=\int\!\int a(v,t)L_t(v)\nu(dt)\pi(dv)
 =\E[a(v,\mathcal F_n)].
\end{align*}
Cancellation of $z_t$ gives the second line, Tonelli swaps the
integrals, and \eqref{eq:app-common-likelihood} gives the final
expectation.

\noindent\textbf{Regularity and the same score.}
At a fixed record all queries, projections, residuals, and
variances in $\ell_i$ are fixed. By \Cref{lem:tensor-shift},
$\theta$ is genuinely $C^1$ on the sphere. Each $\ell_i$, hence
$L_t$ and $p_t$, is positive and $C^1$. Compactness makes $p_t$
bounded away from zero and its first derivative bounded, with
bounds allowed to depend on $t$. Thus
\begin{align}
 \nabla_{\sph}\log p_t(v)
 &=\nabla_{\sph}\log L_t(v)=\sum_{i=1}^n S_i,
 \notag\\
 \frac{\|\nabla_{\sph}p_t(v)\|^2}{p_t(v)}
 &=p_t(v)\left\|\sum_{i=1}^nS_i\right\|^2.
 \label{eq:app-posterior-score}
\end{align}
The first line differentiates the fixed denominator and then the
finite product; the second is the chain rule
$\nabla p_t=p_t\nabla\log p_t$. The derivatives here are on the
actual $s$-dimensional prior sphere. Restricting the ambient score
to its tangent space is an orthogonal projection, whose norm is
at most one, so no factor $s$ or number of sample labels appears.
For each fixed record, the density, its entropy integrand, and
its Fisher integrand are integrable by the compactness and
positivity just established.

Apply \Cref{thm:tool-lsi} to this positive $C^1$ density. To justify
averaging without assuming log-likelihood integrability, use the
nonnegative integrand $u\log u-u+1$. The common likelihood gives
\begin{align}
 I(v;\mathcal F_n)
 &=\int\!\int
       [p_t(v)\log p_t(v)-p_t(v)+1]\pi(dv)P_{\mathcal F_n}(dt)
 \notag\\
 &=\E_t\operatorname{Ent}_\pi(p_t)
 \notag\\
 &\le\frac1{2(s-2)}\E_t\int
       \frac{\|\nabla_{\sph}p_t(v)\|^2}{p_t(v)}\pi(dv)
 \notag\\
 &=\frac1{2(s-2)}\E_{v,t}
          \left\|\sum_{i=1}^nS_i\right\|^2
 \le\frac{10mn}{(s-2)\sigma^2}.
 \label{eq:app-mutual-info}
\end{align}
The first equality uses the density $p_t$ of the joint law
relative to $\pi\otimes P_{\mathcal F_n}$; the linear terms cancel
because $\int p_t\,d\pi=1$. The second equality is the definition
of entropy for a normalized density. The inequality is spherical
log-Sobolev. Equation \eqref{eq:app-posterior-score} and the Bayes
identity for nonnegative integrands give the next equality, and
\eqref{eq:app-total-score} gives the final bound. All averaging is
valid first for nonnegative integrands by Tonelli; the final finite
bound proves the required averaged integrability. Omitting any
independent coins or auxiliary data can only decrease information
by data processing.
\end{proof}

\subsection{List recovery from the information bound}
\begin{proof}[Proof of \Cref{thm:reuse}]
\phantomsection\label{proof:reuse-list}
Use the complete-record information bound in
\Cref{lem:posterior-information}.
By \Cref{thm:tool-sphere}, for every fixed unit $u$,
\[
 \PP(\langle V,u\rangle\ge15/16)\le e^{-225s/512},
 \qquad V\sim\operatorname{Unif}(\mathbb S^{s-1}).
\]
Let $p_{\rm hit}$ be the list-success probability under the joint law of
$(v,\mathcal F_n)$, and let $q_{\rm hit}$ be the probability of the same
event under the product of the prior and transcript marginal. The cap
bound and the union bound give
\[
0\le q_{\rm hit}\le M_{\rm list}e^{-225s/512}
\le e^{-161s/512}<1,
\]
where the second inequality uses $\log M_{\rm list}\le s/8$.
If $q_{\rm hit}=0$, finiteness of the mutual information in
\eqref{eq:app-mutual-info} implies absolute continuity of the joint law
with respect to the product law, so $p_{\rm hit}=0$ and the claim follows.
Otherwise, applying \Cref{thm:tool-binary} to this event yields
\begin{align*}
I(v;\mathcal F_n)
&\ge D_{\rm KL}(\operatorname{Ber}(p_{\rm hit})\|\operatorname{Ber}(q_{\rm hit}))\\
&\ge p_{\rm hit}\log(1/q_{\rm hit})-\log2\\
&\ge p_{\rm hit}\left(\frac{225s}{512}-\log M_{\rm list}\right)-\log2.
\end{align*}
The first step is binary data processing with the actual product-law
probability, the second expands the binary divergence and bounds binary
entropy by $\log2$, and only the last step substitutes the upper bound
on $q_{\rm hit}$.
Using $\log M_{\rm list}\le s/8$ and \eqref{eq:app-mutual-info},
\begin{equation}\label{eq:app-list-recovery}
p_{\rm hit}\le
\frac{5120mn}{161s(s-2)\sigma^2}+\frac{512\log2}{161s}
\le\frac{64mn}{s^2\sigma^2}+\frac9{4s},
\end{equation}
where $s-2\ge s/2$ for $s\ge4$ and $\log2<7/10$. This proves the claim, including arbitrary interleaving of sample openings and reuse.
\end{proof}

%% file: sections/noise_proofs.tex
\section{Global Regularity of the Multistage Sample Field}
\label[appendix]{app:noise}
The centered sample noise is the static field
\[
N_j(w,x)=\sum_{i=1}^T\tau(\|g_{>i}\|)
Q_{j,i}(x-a_{<i},g_i),
\]
where all stage/degree/sample tensor families are independent and each label keeps its entire family permanently. A current stage appears with coefficient one when all future activations vanish; later fields then vanish because $Q_{j,i}(r,0)=0$. The derivative estimate includes $D_gQ_{j,i}(r,0)$ at the first
potentially active suffix.

\subsection{Uniform global sample-noise moment}
\begin{theorem}[Uniform global sample-noise moment]\label{thm:noise-moment}
For the independent stage tensors defining \eqref{eq:multi-noise},
\begin{equation*}
\E\Lip_{(w,x)}(N_j)^2\le C_A\sigma^2\frac{m+T+\log(2T)}m.
\end{equation*}
In particular, when $m\ge T$, the right-hand side is bounded by a universal constant multiple of $\sigma^2$, independently of $T$. The same order bound holds after composition with the soft coordinate map.
For $m\ge\max\{T,8\}$, the composed field admits $C_N=2^{128}$ in the bound $\E\Lip(N_j)^2\le C_N\sigma^2$.
\end{theorem}

\begin{proof}
\textbf{Bounded activation suffix and common shift.}\enspace
Treat $g\in[0,1]^T$ temporarily as independent coordinates, and let $j_0$ be the first index satisfying $\|g_{>j_0}\|<\rho_0$. Every earlier tail cutoff and its first two derivatives vanish, including at equality, since $\tau$ is $C^2$ and flat outside its support. Therefore only indices $i\ge j_0$ can contribute to the value or its first two derivatives. Moreover
\begin{equation}\label{eq:app-suffix-energy}
\|g_{\ge j_0}\|^2=g_{j_0}^2+\|g_{>j_0}\|^2
\le1+\rho_0^2=:B_{\rm suf}^2.
\end{equation}
Translate the workspace by the preceding center,
\[
x'=x-A\sum_{i<j_0}g_iv_i,
\qquad
B_i=A\sum_{j_0\le a<i}g_av_a,
\]
so $\|B_i\|\le AB_{\rm suf}$ and $\|DB_i\|\le A$. Earlier activations enter the active suffix only through the same spatial shift $x\mapsto x'$, whose differential has norm at most $1+A$ by orthogonality of $V$. It remains to control one suffix on
\begin{equation}\label{eq:app-suffix-domain}
\left\{g\in[0,1]^{T-j_0+1}:\|g\|\le B_{\rm suf}\right\}\times\R^m,
\end{equation}
then take the maximum over the at most $T$ possible suffix choices.

\paragraph{Compactifying the workspace.}
Write
\[
s_c=(1+\|x'\|^2)^{-1/2},\qquad e_c=s_cx',
\qquad s_c^2+\|e_c\|^2=1,\quad s_c\ge0.
\]
Differentiation in the workspace variable gives
\[
D_{x'}s_c[h]=-s_c^3\langle x',h\rangle,
\qquad D_{x'}e_c[h]=s_ch-s_c^3\langle x',h\rangle x'.
\]
The same radial/tangential decomposition as for $Dy$ in \Cref{app:geometry} gives
\[\|D_{x'}(s_c,e_c)[h]\|^2\le\|h\|^2.\]
This is the compactification's nonexpansion property. Define, for stage $i$,
\begin{align}
d_i&=e_c-s_cB_i,
&K_i&=s_c^2+\|d_i\|^2+A^2g_i^2s_c^2,\notag\\
p_i&=\frac{2Ag_id_i}{\sqrt{K_i}},
&t_i&=\frac{2Ag_is_cd_i}{K_i}.
\label{eq:app-compact-coordinates}
\end{align}
For the inverse-shear estimate, use
\begin{align*}
\bigl(s_c^2+\|e_c\|^2\bigr)^{1/2}
&=\bigl(s_c^2+\|d_i+s_cB_i\|^2\bigr)^{1/2}\\
&\le\bigl(s_c^2+\|d_i\|^2\bigr)^{1/2}+s_c\|B_i\|\\
&\le(1+\|B_i\|)\bigl(s_c^2+\|d_i\|^2\bigr)^{1/2}.
\end{align*}
The first equality substitutes $e_c=d_i+s_cB_i$; the second is the Euclidean triangle inequality; the third uses $s_c\le(s_c^2+\|d_i\|^2)^{1/2}$. Allowing the slightly looser denominator $2+\|B_i\|$, we have
\begin{equation}\label{eq:app-compact-denominator}
\begin{aligned}
K_i&=s_c^2+\|e_c-s_cB_i\|^2+A^2g_i^2s_c^2\\
&\ge s_c^2+\|e_c-s_cB_i\|^2\\
&\ge\frac{s_c^2+\|e_c\|^2}{(2+\|B_i\|)^2}
\ge(2+AB_{\rm suf})^{-2}=:c_A>0.
\end{aligned}
\end{equation}
The first inequality drops a nonnegative term and the last uses $s_c^2+\|e_c\|^2=1$ and $\|B_i\|\le AB_{\rm suf}$.
Thus the coordinates extend smoothly to $s_c=0$. For $s_c>0$ one has $p_i=\sqrt{S_i}\,t_i$, where $S_i$ and $t_i$ are the one-stage expressions evaluated at $(r,g)=(x'-B_i,g_i)$. Globally,
\[
\|p_i\|\le2A|g_i|,\qquad\|t_i\|\le\rho<1.
\]
For degree $\ell$, the unscaled feature can therefore be written without division by $s_c$:
\begin{equation}\label{eq:app-degree-feature}
f_{i,\ell}=\omega_\ell p_i\otimes t_i^{\otimes(\ell-1)}.
\end{equation}
Here are the differential identities behind the Hilbert bounds. Treat $q_i=(B_i,s_c,e_c)$ and $g_i$ as independent variables; write $d=d_i$, $K=K_i$, and let $\eta$ be any increment. The quotient and product rules give
\begin{align*}
Dd[\eta]&=\eta_e-\eta_sB_i-s_c\eta_B,\\
DK[\eta]&=2s_c\eta_s+2\langle d,Dd[\eta]\rangle
+2A^2g_is_c^2\eta_g+2A^2g_i^2s_c\eta_s,\\
Dp[\eta]&=\frac{2A}{\sqrt K}(\eta_gd+g_iDd[\eta])
-\frac{Ag_id}{K^{3/2}}DK[\eta],\\
Dt[\eta]&=\frac{2A}{K}(\eta_gs_cd+g_i\eta_sd+g_is_cDd[\eta])
-\frac{2Ag_is_cd}{K^2}DK[\eta].
\end{align*}
The denominators are uniformly separated from zero by the preceding calculation. For the compact domain $\|g\|\le B_{\rm suf}$ and $s_c^2+\|e_c\|^2=1$, the norms of $d,Dd,DK,D^2K$ are bounded by numbers depending only on $A$, not on $m$ or $T$. In the $q_i$ derivatives the numerators have an explicit factor $g_i$, while $g_i$ differentiation can remove that factor. Thus
\[
\|D_{q_i}^ap\|+\|D_{q_i}^at\|\le C_A|g_i|\ (0\le a\le2),
\qquad \|D_{g_i}p\|+\|D_{g_i}t\|+\|D_{g_i}^2(p,t)\|\le C_A.
\]
For the tensor product, show the two first product-rule levels explicitly:
\begin{align*}
D f_{i,\ell}[\eta]
&=\omega_\ell\Bigl(Dp[\eta]\otimes t^{\otimes(\ell-1)}
+\sum_{a=1}^{\ell-1}p\otimes t^{\otimes(a-1)}\otimes Dt[\eta]
\otimes t^{\otimes(\ell-a-1)}\Bigr),\\
\|Df_{i,\ell}[\eta]\|
&\le\omega_\ell\Bigl(\|Dp[\eta]\|\rho^{\ell-1}
+(\ell-1)\|p\|\|Dt[\eta]\|\rho^{\ell-2}\Bigr).
\end{align*}
A second derivative differentiates at most two tensor factors; there are at most $1+2(\ell-1)+(\ell-1)(\ell-2)$ product-rule terms and at least $\ell-3$ untouched factors. Substitution of $\|p\|\le2A|g_i|$ and $\|t\|\le\rho$ bounds derivatives through order two by
\[
C_A\omega_\ell\ell^2\rho^{\max\{\ell-3,0\}}.
\]
The squared degree bounds are summable. Let $f_i=(f_{i,\ell})_{\ell\ge1}$ and $q_i=(B_i,s_c,e_c)$. Summation in the Hilbert direct sum gives
\begin{align}
\|D_{q_i}^a f_i\|&\le C_A|g_i|\quad(a=0,1,2),\notag\\
\|D_{g_i}f_i\|+\|D^2_{g_i}f_i\|+\|D_{q_i}D_{g_i}f_i\|&\le C_A.
\label{eq:app-perstage-Hilbert}
\end{align}
These are multilinear operator bounds, not entrywise bounds on random tensor coefficients.

\paragraph{Joint rather than stagewise derivative bounds.}
Let $F$ be the Hilbert direct sum over degrees and active stages of
\[
F(g,s_c,e_c)=\big(\tau(\|g_{>i}\|)f_i\big)_{i\ge j_0}.
\]
For an increment $\eta=(u,\xi)$ put $R_\eta=A\|u\|+\|\xi\|$. Equation \eqref{eq:app-perstage-Hilbert} and $\|DB_i\|\le A$ yield
\begin{equation}\label{eq:app-fi-one-deriv}
\|Df_i[\eta]\|\le C_A\big(|u_i|+|g_i|R_\eta\big).
\end{equation}
The second derivative in $\eta=(u,\xi)$ and $\zeta=(v,\upsilon)$ is bounded by $C_A$ times
\[
|u_iv_i|+|u_i|R_\zeta+|v_i|R_\eta+|g_i|R_\eta R_\zeta.
\]
Using \eqref{eq:app-suffix-energy}, Cauchy--Schwarz, and $R_\eta\le(1+A)\|\eta\|$, we obtain
\begin{align}
\sum_i\|f_i\|^2&\le C_A^2\sum_i g_i^2\le C_A^2B_{\rm suf}^2,\notag\\
\sum_i\|Df_i[\eta]\|^2&\le2C_A^2(\|u\|^2+B_{\rm suf}^2R_\eta^2)
\le C_A\|\eta\|^2,\notag\\
\sum_i\|D^2f_i[\eta,\zeta]\|^2
&\le4C_A^2\big(\|u\|^2\|v\|^2+\|u\|^2R_\zeta^2
+\|v\|^2R_\eta^2+B_{\rm suf}^2R_\eta^2R_\zeta^2\big)
\le C_A\|\eta\|^2\|\zeta\|^2.
\label{eq:app-joint-deriv}
\end{align}
For the cutoff, set $R_i=\|g_{>i}\|$ and $\tau_i=\tau(R_i)$. On the cutoff transition interval $R_i\in[\rho_0/2,\rho_0]$ the norm is smooth and bounded away from zero. In directions $u,v$,
\begin{align*}
DR_i[u]&=\frac{\langle g_{>i},u_{>i}\rangle}{R_i},\\
D^2R_i[u,v]&=\frac{\langle u_{>i},v_{>i}\rangle}{R_i}
-\frac{\langle g_{>i},u_{>i}\rangle\langle g_{>i},v_{>i}\rangle}{R_i^3},\\
D\tau_i[u]&=\tau'(R_i)DR_i[u],\\
D^2\tau_i[u,v]&=\tau''(R_i)DR_i[u]DR_i[v]+\tau'(R_i)D^2R_i[u,v].
\end{align*}
Outside the transition interval the derivatives vanish because $\tau$ is flat. Hence $|\tau_i|\le1$, $\|D\tau_i\|\le C_{\rho_0}$, $\|D^2\tau_i\|\le C_{\rho_0}$ uniformly in the number of coordinates. Applying the product rule once and twice gives
\begin{align*}
D(\tau_if_i)[\eta]
&=D\tau_i[\eta]f_i+\tau_iDf_i[\eta],\\
D^2(\tau_if_i)[\eta,\zeta]
&=D^2\tau_i[\eta,\zeta]f_i+D\tau_i[\eta]Df_i[\zeta]
+D\tau_i[\zeta]Df_i[\eta]+\tau_iD^2f_i[\eta,\zeta].
\end{align*}
Cauchy--Schwarz in the \emph{orthogonal Hilbert sum over stages}, together with $\sum_i g_i^2\le B_{\rm suf}^2$, now yields, with constants independent of the suffix length,
\begin{equation}\label{eq:app-joint-feature-budget}
\|F\|+\|DF\|+\|D^2F\|\le C_A.
\end{equation}

\paragraph{The directional derivative process.}
The Lipschitz constant is the supremum of directional derivatives.
To express each derivative through the same Gaussian coefficients,
write
$C(x')=(s_c(x'),e_c(x'))$ and, for a unit physical increment
$\eta=(u,h)$ with $\|u\|^2+\|h\|^2=1$, define the coefficient
\begin{equation}\label{eq:app-Xi-definition}
\Xi(g,C(x'),\eta)
:=DF(g,C(x'))\,[u,J_c(x')h],\qquad J_c=DC.
\end{equation}
In the compact coordinates $s=s_c$ and $e=e_c$ this Jacobian has the
explicit extension to the closed hemisphere
\begin{equation}\label{eq:app-compact-Jacobian}
J_c(s,e)h
=\left(-s^2\langle e,h\rangle,
       s(I-ee^\top)h\right).
\end{equation}
Since $s^2+\|e\|^2=1$, it is the differential of the normalized
vector $(1,x')/\sqrt{1+\|x'\|^2}$: for every $h$,
\[
\|J_ch\|^2=s^2\bigl(\|h\|^2-\langle e,h\rangle^2\bigr)
\le\|h\|^2.
\]
Thus $\|J_c\|\le1$.  Differentiating the two blocks of
\eqref{eq:app-compact-Jacobian}, for an increment
$(\dot s,\dot e)$ of unit norm, gives
\begin{align*}
D J_c[(\dot s,\dot e)]h
=\bigl(&-2s\dot s\langle e,h\rangle
       -s^2\langle\dot e,h\rangle,\\
&\dot s(I-ee^\top)h
 -s(\dot e\,e^\top+e\,\dot e^\top)h\bigr).
\end{align*}
The absolute value of the first block is at most
$3\|h\|\| (\dot s,\dot e)\|$ and the norm of the second at most
$3\|h\|\|(\dot s,\dot e)\|$; hence $\|DJ_c\|<8$ uniformly,
including at the compactification boundary $s=0$.

The chain rule and \eqref{eq:app-joint-feature-budget} therefore give
\begin{align*}
\sup\|\Xi\|&\le\|DF\|\max\{1,\|J_c\|\}\le C_A,\\
\|D_{(g,s,e)}\Xi\|
&\le\|D^2F\|\max\{1,\|J_c\|\}+\|DF\|\|DJ_c\|
\le C_A,\\
\|D_\eta\Xi\|&\le\|DF\|\max\{1,\|J_c\|\}\le C_A.
\end{align*}
The first and third bounds use $\|J_c\|\le1$; the middle bound
also uses $\|DJ_c\|<8$. All three constants are independent of the
number of active stages.

\paragraph{A uniform Gaussian supremum.}
Within each suffix, the index $(g,s,e,\eta)$ lies in a bounded set
in at most $p=2(m+T)+1$ Euclidean dimensions. Choose
$\delta_n=C_A^{-1}2^{-n}$, increasing $C_A\ge1$ if necessary,
and take a $\delta_n$-net in each suffix. Their disjoint union
$\mathcal N_n$ satisfies
\[
 |\mathcal N_n|\le T(C_A2^n)^p,
 \qquad
 \log(2|\mathcal N_n|)
 \le\log(2T)+C_A(m+T)(n+1).
\]
Here the volumetric covering bound supplies the power $p$, and the
union over suffixes supplies the factor $T$. For each fine-net point
$t\in\mathcal N_{n+1}$ choose a parent $\pi_n(t)\in\mathcal N_n$
in the same suffix at distance at most $\delta_n$. The isonormal
covariance identity yields
\begin{align*}
\|G(\Xi(t))\|_{L^2}&=\|\Xi(t)\|\le C_A,\\
\|G(\Xi(t))-G(\Xi(\pi_n(t)))\|_{L^2}
&=\|\Xi(t)-\Xi(\pi_n(t))\|
\le C_A\delta_n=2^{-n}.
\end{align*}
Thus the base-net variables have bounded standard deviation, and
the level-$n$ increments have geometrically decreasing standard
deviation. There is at most one such increment per fine-net point.
Telescoping from a fine-net point to its successive parents and
applying \Cref{thm:tool-gaussian} to each finite maximum gives
\begin{align}
 \|\sup|G(\Xi)|\|_{L^2}
 &\le C_A\sum_{n\ge0}2^{-n}
          \sqrt{\log(2T)+(m+T)(n+1)}\notag\\
 &\le C_A\sqrt{m+T+\log(2T)}.
 \label{eq:app-global-chaining}
\end{align}
The first line uses Minkowski's inequality to add the base maximum
and all increment maxima. The second uses
$\sqrt{a+b}\le\sqrt a+\sqrt b$ and
$\sum_{n\ge0}2^{-n}\sqrt{n+1}<\infty$.

For finite degree truncations, continuity passes the net bound to
the whole compact index set. For a degree tail, the product-rule
estimates above replace each coefficient bound by a constant times
\[
\left[\sum_{\ell>L}\ell^8\rho^{2\max\{\ell-3,0\}}\right]^{1/2}
\longrightarrow0.
\]
The same net bound then proves uniform $L^2$ convergence of the
derivative processes. Their limits agree on a countable dense set,
and hence everywhere by continuity, with the retained sample
functions of \Cref{lem:field-regularity}. Consequently the supremum
in \eqref{eq:app-global-chaining} bounds the actual sample derivative,
rather than only a formal Gaussian process. Restoring the noise
factor $\sigma/\sqrt m$, the common-shift factor $1+A$, and the
activation derivative factor $15/2$ gives
\[
\|\Lip(N_j(U^\top y(z),x))\|_{L^2}
\le C_A\sigma\sqrt{\frac{m+T+\log(2T)}m}.
\]
The same argument without the soft coordinate composition bounds
$\Lip_{(w,x)}(N_j)$. Squaring proves the first two assertions.

\paragraph{Verification of $C_N=2^{128}$.}
We now give explicit bounds for the two coefficient norms used in
the preceding argument. This calculation fixes the sample budget;
it does not change the field or the qualitative proof.
The binomial coefficients satisfy
\[
|c_\ell|=\frac{a_\ell}{2\ell-1},\qquad
 a_\ell=\binom{2\ell}{\ell}4^{-\ell},\qquad
\frac{a_{\ell+1}}{a_\ell}=\frac{2\ell+1}{2\ell+2}.
\]
Starting from $a_1=1/2$, induction and the elementary comparison
\[
\frac{\ell+1}{\ell+2}-\left(\frac{2\ell+1}{2\ell+2}\right)^2
=\frac{3\ell+2}{4(\ell+1)^2(\ell+2)}>0
\]
yield $a_\ell\le(\ell+1)^{-1/2}$, and therefore
\[
|c_\ell|\le\ell^{-3/2},\quad\omega_\ell^2\le\ell.
\]
For $q=\rho^2=A^2/(1+A^2)$, define
\[
S_0^2=\sum_{\ell\ge1}\ell q^{\ell-1}=(1-q)^{-2},\qquad
S_1^2=\sum_{\ell\ge1}\ell^3q^{\ell-1}
=\frac{1+4q+q^2}{(1-q)^4}.
\]
Thus $S_0=1+A^2=16385$ and $S_1<\sqrt6(1+A^2)^2<2^{30}$. For a single-stage feature $f$ with the normalization $\sigma/\sqrt m$ removed, direct calculation gives
\[
\|f\|\le2AgS_0,
\quad\|D_rf\|\le3AgS_1,
\quad\|D_gf\|\le3AS_1.
\]
The active suffix, with $B_{\rm suf}<33/32$, satisfies a combined directional coefficient bound
\begin{equation}\label{eq:app-explicit-standard-deviation}
80AB_{\rm suf}S_0+3AS_1\big(1+B_{\rm suf}(1+A)\big)<2^{46}.
\end{equation}
This includes stages with $g_i=0$ but nonzero derivatives. The corresponding compactified degree features admit explicit first- and second-derivative bounds; we use normalized coordinates $(s_c,e_c-s_cB_i,Ag_is_c)$ whose norm is above $2^{-8}$, with first and second derivatives bounded by $2^9$. The normalization map then yields $\|Dn\|\le2^{17}$ and $\|D^2n\|\le2^{37}$. The product-rule bounds for the degree features sum by
\begin{align*}
 \sum_{\ell\ge1}\omega_\ell^2\ell^4q^{\max\{\ell-3,0\}}
 &\le q^{-3}\sum_{\ell\ge1}\ell^5q^\ell\\
 &=\frac{1+26q+66q^2+26q^3+q^4}{q^2(1-q)^6}\\
 &\le\frac{120}{q^2(1-q)^6}
 \le\frac{960q}{(1-q)^6}<2^{100}.
\end{align*}
The first step uses $\omega_\ell^2\le\ell$ and
$q^{\max\{\ell-3,0\}}\le q^{\ell-3}$. The equality follows by
applying $(q\,d/dq)^5$ to the geometric series. The next bound
sums the numerator coefficients, and $q\ge1/2$ gives the following
one. Finally $(1-q)^{-1}=1+A^2=16385<2^{15}$ and $960q<2^{10}$
give the displayed power of two.
The normalization derivative estimates are obtained from $n(x)=x/\|x\|$ at $\|x\|\ge2^{-8}$:
\begin{align*}
Dn(x)[h]&=\frac{h}{\|x\|}-\frac{x\langle x,h\rangle}{\|x\|^3},
&\|Dn(x)\|&\le\|x\|^{-1}\le2^8,\\
\|D^2n(x)[h,k]\|
&\le3\frac{\|h\|\|k\|}{\|x\|^2}
&\|D^2n(x)\|&\le3\cdot2^{16}.
\end{align*}
Composing with the source coordinates, whose first two derivatives are bounded by $2^9$, gives
\begin{align*}
\|D(n\circ x)\|&\le2^8\cdot2^9=2^{17},\\
\|D^2(n\circ x)\|
&\le(3\cdot2^{16})(2^9)^2+2^8\cdot2^9
<2^{37}.
\end{align*}
The source-coordinate bounds $2^9$ used here can be checked without concealing any dependence on $T$. For
$X_i=(s_c,d_i,Ag_is_c)$, the mixed derivatives of $B_i=A\sum_{a<i}g_av_a$ satisfy $\|DB_i\|\le A$ and $D^2B_i=0$. Setting $C_x=1+AB_{\rm suf}+A$, we obtain
\begin{align*}
\|DX_i\|^2&\le1+C_x^2+4A^2<2^{18},\\
\|D^2X_i\|&\le4A=2^9.
\end{align*}
The bounds use $\|B_i\|\le AB_{\rm suf}$, the product rule in $d_i=e_c-s_cB_i$, and the Euclidean norm of the three blocks.

We establish the joint tensor estimate. Set $n=X_i/\|X_i\|$, denote its spatial-vector part by $n_d$ and its last scalar coordinate by $n_a$. The compact coordinates satisfy the \emph{exact} identities
\[
p_i=2Ag_i n_d,
\qquad
 t_i=2n_a n_d.
\]
From the just-proved normalization bounds and $0\le g_i\le1$, the product rule yields
\begin{align*}
\|p_i\|&\le2A=2^8,&
\|Dp_i\|&\le2A(1+2^{17})<2^{26},&
\|D^2p_i\|&\le2A(2\cdot2^{17}+2^{37})<2^{46},\\
\|t_i\|&\le\rho,&
\|Dt_i\|&\le4\cdot2^{17}=2^{19},&
\|D^2t_i\|&\le4\cdot2^{37}+4\cdot2^{34}<2^{40}.
\end{align*}
For a degree-$\ell$ tensor, differentiating once and twice as in the displayed product rules gives, for unit directions,
\begin{align*}
\|Df_{i,\ell}\|
&\le\omega_\ell\bigl(2^{26}\rho^{\ell-1}
+(\ell-1)2^8\,2^{19}\rho^{\ell-2}\bigr)
\le2^{28}\omega_\ell\ell\rho^{\max\{\ell-2,0\}},\\
\|D^2f_{i,\ell}\|
&\le\omega_\ell\bigl(2^{46}\rho^{\ell-1}
+2(\ell-1)2^{26}\,2^{19}\rho^{\ell-2}
+(\ell-1)2^8\,2^{40}\rho^{\ell-2}\\
&\hspace{50pt}
+(\ell-1)(\ell-2)2^8(2^{19})^2\rho^{\ell-3}\bigr)
\le2^{51}\omega_\ell\ell^2\rho^{\max\{\ell-3,0\}}.
\end{align*}
Terms with negative tensor exponents are absent when $\ell=1,2$; in the upper bounds the exponent is truncated at zero. The first inequalities are the product rule; the second replace every untouched tensor factor by $\rho$ and bound the combinatorial number of derivative placements by $\ell$ or $\ell^2$.

With $q=\rho^2$, the exact Hilbert direct-sum identity and the explicit degree series give
\begin{align*}
\|D^2 f_i\|_{\cH}
&\le 2^{51}\left(\sum_{\ell\ge1}\omega_\ell^2\ell^4q^{\max\{\ell-3,0\}}\right)^{1/2}\\
&<2^{51}\cdot2^{50}=2^{101},
\end{align*}
For the first derivative, the Hilbert direct-sum calculation is separate:
\begin{align*}
\|Df_i\|_{\cH}^2
&\le 2^{56}\sum_{\ell\ge1}\omega_\ell^2\ell^2
q^{\max\{\ell-2,0\}}\\
&\le 2^{56}q^{-1}\sum_{\ell\ge1}\ell^3q^{\ell-1}
=2^{56}q^{-1}S_1^2<2^{118},
\end{align*}
where $\omega_\ell^2\le\ell$, $q^{-1}<2$, and $S_1<2^{30}$. Thus $\|Df_i\|_{\cH}<2^{59}$, while the previous line gives $\|D^2f_i\|_{\cH}<2^{101}$. The cutoff derivatives satisfy
$\|D\tau(\|g_{>i}\|)\|\le40$ and
$\|D^2\tau(\|g_{>i}\|)\|<2^{12}$ on the transition region, by the explicit radial derivative formulas above.
Finally, perform the stagewise summation \emph{before} bounding the operator norm: $\sum g_i^2\le B_{\rm suf}^2$, $\sum u_i^2\le\|u\|^2$, and
$\sum u_i^2v_i^2\le\|u\|^2\|v\|^2$. In particular, preserving the factors $g_i$ and $u_i$ in the first derivative product rule yields the weighted bounds
\begin{align*}
\|Df_i[\eta]\|&\le2^{78}\bigl(|u_i|+g_iR_\eta\bigr),\\
\|D^2f_i[\eta,\zeta]\|&\le2^{101}\bigl(|u_iv_i|+|u_i|R_\zeta+|v_i|R_\eta+g_iR_\eta R_\zeta\bigr),
\end{align*}
where $\eta=(u,\xi)$, $\zeta=(v,\upsilon)$ and $R_\eta=A\|u\|+\|\xi\|\le2^8\|\eta\|$. These follow from the preceding $p,t$ identities by distinguishing a derivative on $g_i$ (which removes one factor $g_i$) from a derivative on shared coordinates (which preserves it).

Now insert the cutoff product rule and sum the squared stage norms. Since $B_{\rm suf}<2$ and the cutoff bounds are $\|D\tau_i\|\le2^6$, $\|D^2\tau_i\|\le2^{12}$,
\begin{align*}
\|DF[\eta]\|
&\le2^{78}(\|u\|+B_{\rm suf}R_\eta)
 +2^6\cdot2^{58}B_{\rm suf}\|\eta\|\\
&\le\bigl[2^{78}(1+2^9)+2^{65}\bigr]\|\eta\|
<2^{90}\|\eta\|,\\
\|D^2F[\eta,\zeta]\|
&\le2^{101}\bigl(\|u\|\|v\|+\|u\|R_\zeta
+\|v\|R_\eta+B_{\rm suf}R_\eta R_\zeta\bigr)\\
&\quad+2\cdot2^6\cdot2^{78}(1+2^9)\|\eta\|\|\zeta\|
+2^{12}\cdot2^{58}B_{\rm suf}\|\eta\|\|\zeta\|\\
&\le\bigl[2^{101}(1+2^9+2^{17})+2^{85}(1+2^9)+2^{71}\bigr]\|\eta\|\|\zeta\|\\
&<2^{122}\|\eta\|\|\zeta\|.
\end{align*}
The first inequality in each calculation is the Hilbert direct-sum triangle inequality; the second uses Cauchy--Schwarz over stages, $R_\eta\le2^8\|\eta\|$, and $B_{\rm suf}<2$. In particular, we obtain the deliberately looser advertised constants
$\|DF\|<2^{123}$ and $\|D^2F\|<2^{130}$. No factor $T$ appears in either expression. For the sharper \emph{size} bound on $\Xi$, it is useful to differentiate
the original stage features before compactifying.  The one-stage
bounds \eqref{eq:app-explicit-standard-deviation} give
$\|f_i\|\le2Ag_iS_0$,
$\|D_{x'} f_i\|\le3Ag_iS_1$, and
$\|D_{g_i}f_i\|\le3AS_1$.  Changing an earlier activation changes
its center $B_i$ by at most $A\|u\|$, so the product and chain rules,
then Cauchy--Schwarz across orthogonal stage blocks, imply
\begin{align*}
\|\Xi\|
&\le
\underbrace{(40)(2A S_0)\Bigl(\sum_i g_i^2\Bigr)^{1/2}
     \|u\|}_{\text{differentiating the cutoffs}}\\
&\quad+3AS_1\left[\|u\|+
\Bigl(\sum_i g_i^2\Bigr)^{1/2}
 (\|h\|+A\|u\|)\right]\\
&\le80AB_{\rm suf}S_0
 +3AS_1\bigl(1+B_{\rm suf}(1+A)\bigr)
<2^{46},
\end{align*}
where $\sum_i g_i^2\le B_{\rm suf}^2$ and
$\|u\|,\|h\|\le1$.  This is why the coefficient size uses the
substantially sharper bound $2^{46}$, not the coarse second-derivative
bound used for metric entropy.

For the \emph{increment metric}, regard $\Xi$ as a function of
$(g,s,e,\eta)$ on a fixed active-suffix index set, using
\eqref{eq:app-Xi-definition} and the smooth extension
\eqref{eq:app-compact-Jacobian}.  In a unit change of $(g,s,e)$,
the chain and product rules give
\begin{align*}
\|D_{(g,s,e)}\Xi\|
&\le\|D^2F\|\| (u,J_ch)\|
      +\|DF\|\|DJ_c\|\|h\|\\
&<2^{130}+8\cdot2^{90}<2^{132}.
\end{align*}
The derivative in the direction parameter $\eta$ itself is bounded by
$\|DF\|\max\{1,\|J_c\|\}<2^{90}$.  Combining both blocks gives
\begin{equation}\label{eq:app-Xi-metric}
\sup\|\Xi\|<2^{46},\qquad
\operatorname{Lip}(\Xi)<2^{134},
\end{equation}
within each fixed suffix (including the boundary of its compactified
index set).  No factor $T$ was inserted: the sums precede the
operator-norm estimates, and the dependence on the chosen active
suffix is handled by taking a union of its compact index sets below.

For the numerical specialization choose suffix nets of radius
$\delta_n$ and write $\mathcal N_n$ for their disjoint union. Use
\[
\delta_n=2^{-88-n},\qquad
|\mathcal N_n|\le T2^{p(92+n)},\qquad p=2(m+T)+1\le5m.
\]
For each $t\in\mathcal N_{n+1}$ choose a parent
$\pi_n(t)\in\mathcal N_n$ in the same suffix with
$\|t-\pi_n(t)\|\le\delta_n$. There is one increment per fine-net point,
so the number of level-$n$ increments is at most $|\mathcal N_{n+1}|$.
The isonormal covariance identity and the coefficient bounds give
\begin{align*}
\|G(\Xi(t))\|_{L^2}&=\|\Xi(t)\|\le2^{46},\\
\|G(\Xi(t))-G(\Xi(\pi_n(t)))\|_{L^2}
&=\|\Xi(t)-\Xi(\pi_n(t))\|\\
&\le2^{134}\delta_n=2^{46-n}.
\end{align*}
When $m\ge\max\{T,8\}$, the finite-maximum
proof in \Cref{thm:tool-gaussian} gives
$\|\max_{j\le K}|Z_j|\|_{L^2}\le
2v\sqrt{\log(2K)+1}$. Since $T\le m$ and $m\ge8$,
$\log(2T)\le m$ and $1\le m$. Using $\log2<1$ and $p\le5m$,
the base-net and increment counts satisfy, respectively,
\begin{align*}
\log(2|\mathcal N_0|)+1
&\le\log(2T)+92p\log2+1\le462m,\\
\log(2|\mathcal N_{n+1}|)+1
&\le\log(2T)+p(93+n)\log2+1\\
&\le m+5m(93+n)+m=m(467+5n),\\
\sqrt{467+5n}&\le22+3(n+1).
\end{align*}
The last inequality follows from $22^2=484\ge467$ and
$\sqrt{5n}\le3(n+1)$ for all $n\ge0$. Evaluating both the base-net
maximum and the dyadic increments by the finite-maximum inequality,
then applying Minkowski, gives
\begin{align*}
\left\|\sup|G(\Xi)|\right\|_{L^2}
&\le 2^{47}\sqrt m
\left(22+\sum_{n\ge0}2^{-n}[22+3(n+1)]\right)\\
&=2^{47}\sqrt m\left(22+22\cdot2+3\cdot4\right)
=156\cdot2^{46}\sqrt m
<2^{54}\sqrt m.
\end{align*}
We used the exact geometric series
$\sum_{n\ge0}2^{-n}=2$ and
$\sum_{n\ge0}(n+1)2^{-n}=4$; the inequality therefore has an
explicit numerical margin, rather than an unspecified chaining constant.
Multiplying by the composition factor $(1+A)(15/2)<2^{10}$ and the noise scaling $\sigma/\sqrt m$ yields
\[
\big\|\Lip(N_j(U^\top y(z),x))\big\|_{L^2}
\le2^{64}\sigma,
\qquad
\E\Lip(N_j(U^\top y(z),x))^2\le2^{128}\sigma^2.
\]

\end{proof}

\subsection{A lower bound on the uniform noise constant}
\begin{proposition}[Normalization loss of the chosen field]
\label[proposition]{prop:noise-normalization-loss}
For $A=128$, any dimension-uniform constant satisfying
$\E\Lip(N_j)^2\le C_N\sigma^2$ for this field must obey
\begin{equation}\label{eq:app-intrinsic-CN}
 C_N\ge\frac{A^2(1+A^2)(1+2A^2)}8=1099612293120.
\end{equation}
This lower estimate does not replace the upper-budget constant
$2^{128}$ in \Cref{thm:noise-moment}.
\end{proposition}
\begin{proof}
Take $g_1=1$ and $g_i=0$ for $i>1$. The workspace noise is
$Q_{j,1}(x,1)$. Put $q=A^2/(1+A^2)$ and evaluate at
$x=\sqrt{1+A^2}\,e$, $\|e\|=1$. For every unit $h\perp e$,
differentiating \eqref{eq:noise-definition} and using tensor
orthogonality gives
\[
 \E|D_hQ_{j,1}(x,1)|^2
 =\frac{2\sigma^2}{m}\sum_{\ell\ge1}\ell^5c_\ell^2q^\ell.
\]
For $a_\ell=\binom{2\ell}{\ell}4^{-\ell}$, the recurrence
$a_{\ell+1}/a_\ell=(2\ell+1)/(2\ell+2)$ and
\[
 \left(\frac{2\ell+1}{2\ell+2}\right)^2
      -\frac\ell{\ell+1}=\frac1{4(\ell+1)^2}>0
\]
imply $a_\ell\ge1/(2\sqrt\ell)$ by induction from $a_1=1/2$.
Thus $c_\ell^2=a_\ell^2/(2\ell-1)^2\ge1/(16\ell^3)$.
For an orthonormal tangent basis $(h_a)_{a=1}^{m-1}$,
\begin{align*}
 \E\Lip(N_j)^2
 &\ge\E\sum_{a=1}^{m-1}|D_{h_a}N_j|^2\\
 &=\frac{2(m-1)\sigma^2}{m}
      \sum_{\ell\ge1}\ell^5c_\ell^2q^\ell\\
 &\ge\frac{m-1}{m}\frac{\sigma^2}{8}
      \sum_{\ell\ge1}\ell^2q^\ell\\
 &=\frac{m-1}{m}\frac{\sigma^2}{8}
      \frac{q(1+q)}{(1-q)^3}.
\end{align*}
The first step uses the definition of the global Lipschitz
constant and drops the other gradient components. The equality
uses the displayed directional variance; the next step substitutes
the coefficient lower bound. The last evaluates the second
moment of the geometric series. Divide by $\sigma^2$, let
$m\to\infty$, and substitute $q=A^2/(1+A^2)$ to obtain
\eqref{eq:app-intrinsic-CN}.
\end{proof}

%% file: sections/assembly_proofs.tex
\section{Sequential Charging and the Static Scalar Oracle}
\label[appendix]{app:assembly}
\subsection{Physical scales and the exact finite population}
\begin{lemma}[Finite population and sample-moment promise]\label[lemma]{lem:finite-population}
Under \eqref{eq:app-moment-budgets}, the paired samples have population
mean $\lambda F_{\rm new}(\cdot/b)$ for every realization and
$\E S_M\le L_0^2/4$. In particular,
$\PP(E_{\rm adm})\ge3/4$.
\end{lemma}
\begin{proof}
Fix $L_0,\Delta,H>0$. Let $H_*,c_*=37,C_*=9147$ and $e_*=1/4805$ be as above. Use the physical scales
\begin{equation}\label{eq:app-physical-scales}
b=\frac{\epsilon H_*}{e_*H},\qquad
\lambda=\frac{\epsilon^2H_*}{e_*^2H}.
\end{equation}
For $m\ge T$, let $C_N$ satisfy $\E\Lip(N_j)^2\le C_N\sigma^2$.
Keep $T$ and $\sigma$ subject to the sufficient budgets
\begin{equation}\label{eq:app-moment-budgets}
12\lambda T\le\Delta,\qquad
\frac{\lambda}{b}\left(c_*\sqrt T+C_*+\sqrt{C_N}\sigma\right)\le\frac{L_0}{2}.
\end{equation}
\Cref{lem:uniform-parameters} verifies these budgets for the
parameter choice used in \Cref{thm:main}.
Then $\lambda F_{\rm new}(Z/b,X/b)$ is globally $H$-smooth, because $\lambda H_*/b^2=H$. Its population gradient is normalized by $\lambda/b=\epsilon/e_*$, so physical $\epsilon$-stationarity implies the normalized certificate with threshold $e_*$.

Draw $M$ independent base random functions $N_j$, define
\begin{equation}\label{eq:app-antithetic}
f_{j,\pm}(Z,X)
=\lambda\big[F_{\rm new}(Z/b,X/b)\pm N_j(U^\top y(Z/b),X/b)\big],
\end{equation}
and conceal their positive/negative pairing under a uniformly random permutation of $2M$ publicly queryable labels. Every fixed realization has exact uniform population mean
\[
\frac1{2M}\sum_{j=1}^M(f_{j,+}+f_{j,-})=\lambda F_{\rm new}(Z/b,X/b).
\]
Put $M_T=(\lambda/b)(c_*\sqrt T+C_*)$ and let $Z_j$ denote the global Lipschitz constant of the scaled noise sample. For the empirical second moment
$S_M=(2M)^{-1}\sum_{j,\nu}\Lip(f_{j,\nu})^2$, the bound
$\Lip(f_{j,\nu})\le M_T+Z_j$ for both signs and Minkowski's
inequality give
\begin{align*}
\E S_M
&=\frac1{2M}\sum_{j=1}^M\sum_{\nu\in\{+,-\}}\E\Lip(f_{j,\nu})^2\\
&\le\frac1M\sum_{j=1}^M\E(M_T+Z_j)^2\\
&\le\frac1M\sum_{j=1}^M\bigl(M_T+\sqrt{\E Z_j^2}\bigr)^2\\
&\le\Bigl[\frac{\lambda}{b}(c_*\sqrt T+C_*+\sqrt{C_N}\sigma)\Bigr]^2
\le L_0^2/4.
\end{align*}
The first bound uses the deterministic-plus-noise Lipschitz inequality, Minkowski in $L^2$ gives the next step, and the sample budget gives the last step. Consequently the \emph{global admissibility} event
\begin{equation}\label{eq:app-admissible}
E_{\rm adm}=\{S_M\le L_0^2\}
\end{equation}
has probability at least $3/4$ by Markov's inequality. The rest of the analysis uses the \emph{unconditioned} Gaussian and Haar laws; condition on $E_{\rm adm}$ only at the end.

\end{proof}

\subsection{The unconditioned sequential stage simulator}
\begin{lemma}[Prefix simulator and conditional birth law]\label[lemma]{lem:prefix-simulator}
The disclosed-prefix simulator has the failure and crossing
responses \eqref{eq:app-failed-mean} and
\eqref{eq:app-crossing-mean} below. At stage $k$'s birth,
$v_k$ is conditionally Haar in dimension $m-k+1$, and its
label-stage arrays are fresh and independent of the birth history.
These assertions hold under the unconditioned field law.
\end{lemma}
\begin{proof}
\textbf{State and gate test.}\enspace
Work in normalized coordinates $(z,x)$, which are known rescalings of the physical coordinates. Initially disclose $u_1$. At stage $k$, disclose the prefix
$u_1,\ldots,u_k$ and $v_1,\ldots,v_{k-1}$. For every current query compute its disclosed chain coordinates and set future activations to zero. Write
\begin{equation}\label{eq:app-stage-gate}
a_k=A\sum_{i\le k}g_iv_i,
\qquad
G_k=\big\{|w_k|\ge1/2,\ \|x-a_k\|<\sqrt2\,\beta\big\}.
\end{equation}
At most one gate is processed per paid query.

\paragraph{Failed query.}
If the gate fails, the mean answer is the fixed outside-gate extension
\begin{equation}\label{eq:app-failed-mean}
P+\sum_{i\le k}b_i(w)
-\sum_{i\le k}\big[1-\chi(x-a_{<i})\big]c_i(w)
+h(x-a_{<k}-Ag_kv_k),
\end{equation}
with future activations zero. This same outside-gate formula is used later for the scripted list experiment. For an actual failed gate, the static mean agrees with it whenever future projections are below threshold. There are two cases.
If $|w_k|<1/2$, the successor term $b_{k+1}$ and its active chain transition vanish since $\Psi(w_k)=\Psi(-w_k)=0$. In this case the full mean has exactly the $i\le k$ terms of \eqref{eq:app-failed-mean}. If $|w_k|\ge1/2$, then $g_k=1$ and, because the gate failed, $\|x-a_k\|\ge\sqrt2\beta$ so $\chi(x-a_k)=0$. The potential successor contribution is
\[
b_{k+1}-[1-\chi(x-a_k)]b_{k+1}
=b_{k+1}-(1-0)b_{k+1}=0.
\]
All remaining future terms vanish because their predecessor
coordinates have magnitude less than $1/4$ and their activations
are zero. In both cases, the full mean reduces exactly to
\eqref{eq:app-failed-mean}. The projection condition is used to
identify the static and simulated responses; the cancellation
itself is exact.

\paragraph{Crossing query.}
At a successful stage crossing the exact mean answer, using the disclosed new $u_{k+1}$ if $k<T$, is
\begin{equation}\label{eq:app-crossing-mean}
P+\sum_{i\le k}b_i(w)+h(x-a_k)+\chi(x-a_k)b_{k+1}(w),
\end{equation}
with the last term omitted at $k=T$. This formula defines the
simulator's crossing response even on exceptional queries. To verify
its agreement with the static mean \eqref{eq:mean-definition}, suppose
that $|w_i|<1/4$ for all $i>k$. Then the actual future activations
vanish, while the gate condition gives $g_k=1$. Hence
\begin{align*}
\tau(\|g_{\ge i}\|)&=0\quad(i\le k)
&&\text{($\|g_{\ge i}\|\ge g_k=1>\rho_0$)},\\
c_i&=0\quad(i\le k)
&&\text{(definition of $c_i$)},\\
c_{k+1}&=b_{k+1}\tau(0)=b_{k+1}
&&\text{($\tau(0)=1$)},\\
b_i&=0\quad(i>k+1)
&&\text{($|w_{i-1}|<1/4$ and the zero-chain cutoff)},\\
a(w)&=a_k,\quad a_{<k+1}=a_k
&&\text{(zero future activations)}.
\end{align*}
Substitution into the full objective gives
\begin{align*}
F_{\rm new}
&=P+\sum_{i\le k}b_i+b_{k+1}+h(x-a_k)
-[1-\chi(x-a_k)]b_{k+1}\\
&=P+\sum_{i\le k}b_i+h(x-a_k)+\chi(x-a_k)b_{k+1}.
\end{align*}
At $k=T$ the next chain term is absent. Thus equality with the
static mean is algebraically exact under the stated projection
condition; \Cref{lem:static-coupling} bounds violations of that
condition. In the simulator itself, the crossing answer uses only
noise terms through stage $k$ and keeps future activations zero,
even if the newly used $u_{k+1}$ has a large coordinate at this query.
After answering, reveal $v_k$ and, if $k<T$, $u_{k+1}$, and enter the
next stage. At $k=T$ mark completion.

\paragraph{Fields retained under label reuse.}
The simulator may additionally reveal the \emph{past-stage} tensors of already opened labels and the past-stage tensors of a newly opened label at its paid query, but never current or future stage tensors or unopened labels' tensors. Gate notifications may also be granted. These disclosures only strengthen the simulated algorithm. Previously opened labels retain their Gaussian families forever.

\paragraph{The conditional law at stage birth.}
We formalize the preceding deferred-disclosure argument in the
\emph{ideal independent-label experiment}; the separate coupling below
passes from that experiment to the fixed finite antithetic population.
Let $\mathcal R$ contain the algorithm's random coins and any other
initial randomness independent of $(U,V)$ and the stage tensors.  For
$1\le k\le T$, define the enlarged pre-birth sigma-algebra
\begin{equation}\label{eq:app-prebirth-sigma}
\mathcal G_{<k}
=\sigma\Bigl(\mathcal R,\ U,\ v_1,\ldots,v_{k-1},\
\{G_{j,i,\ell}:i<k,\ j\ \text{any public label},\ \ell\ge1\}\Bigr).
\end{equation}
This sigma-algebra deliberately grants \emph{more} than the simulator
actually reveals: it contains the entire $U$ frame and all past-stage
tensors, including tensors of as-yet-unopened labels.  Such enlargement
is safe for the one-stage information argument because the $U$ and $V$
frames are independently Haar, the Gaussian tensor blocks are
independent across stages and labels in the ideal experiment, and none
of the variables in $\mathcal G_{<k}$ involves $v_k$ beyond the known
orthonormality constraint with $v_1,\ldots,v_{k-1}$.

The simulated transcript before the birth of stage $k$ is
$\mathcal G_{<k}$-measurable.  This follows by induction on paid
queries and gate tests: a stage-$i<k$ failed reply is given by
\eqref{eq:app-failed-mean}, using only $U$, $v_1,\ldots,v_i$ and the
corresponding stage-$j\le i$ tensors; a successful crossing is given
by \eqref{eq:app-crossing-mean}, with $u_{i+1}$ but \emph{without}
$v_{i+1}$ or any stage-$(i+1)$ tensor.  The label chosen next and the
time of crossing are functions of the transcript and $\mathcal R$.
Thus they too are $\mathcal G_{<k}$-measurable until stage $k$ is
born.  This proves, in particular,
\begin{equation}\label{eq:app-birth-inclusion}
\mathcal H_{\rm birth}(k)\subseteq\mathcal G_{<k},
\end{equation}
where the birth history includes the identities of all labels opened
so far and the already revealed frame prefix.

For a bounded measurable test function $f$ on the sphere, independent
Haar completion of $V$ now gives the conditional identity
\begin{equation}\label{eq:app-birth-kernel}
\E[f(v_k)\mid\mathcal G_{<k}]
=\int_{\mathbb S^{s_k-1}\cap\operatorname{span}(v_{<k})^\perp}
  f(u)\,\pi_{v_{<k}}(du),
\qquad s_k=m-k+1,
\end{equation}
where $\pi_{v_{<k}}$ is the normalized Haar measure on that residual
sphere.  The same identity holds after conditioning on
$\mathcal H_{\rm birth}(k)$ \emph{together with} the disclosed
$v_{<k}$: apply the tower property to \eqref{eq:app-birth-kernel}
and use \eqref{eq:app-birth-inclusion}.  Equivalently,
\begin{equation}\label{eq:app-birth-Haar}
v_k\mid\mathcal H_{\rm birth}(k)
\sim\operatorname{Unif}\left(
\mathbb S^{s_k-1}\cap\operatorname{span}(v_1,\ldots,v_{k-1})^\perp
\right),\qquad s_k=m-k+1.
\end{equation}
The right-hand side is a \emph{conditional kernel indexed by the
revealed prefix}, not an unconditional law on a fixed sphere.

Similarly, conditionally on $\mathcal G_{<k}$ and $v_k$, the entire
stage-$k$ family $(G_{j,k,\ell})_{j,\ell}$ has its original
independent Gaussian law.  A label used in earlier stages reveals only
blocks with index $i<k$; its stage-$k$ block is still fresh.
This is precisely the independence required to invoke the adaptive
Gaussian result \Cref{thm:reuse} at birth.  It is important that no
conditioning on the future-projection good event, the absence of an
antithetic partner collision, or $E_{\rm adm}$ has been performed here:
each of those events depends on currently hidden variables and is
charged separately in the final coupling and conditioning argument.

\end{proof}

\subsection{Charging adaptive reuse in each stage}
\begin{lemma}[Conditional stage charging]\label[lemma]{lem:stage-charging}
With the per-stage budget $\ell_0$ chosen below, each stage has
conditional probability at most $1/16$ of completing within
$\ell_0$ tests. If $N+1\le T\ell_0/4$, the probability that all
stages complete within $N+1$ tests is at most $2^{-T}$.
No independence of stage durations is assumed.
\end{lemma}
\begin{proof}
\textbf{Reduction to one current-stage experiment.}\enspace
In stage $k$, all earlier noise fields can be subtracted after the permitted past-tensor disclosures, even if their suffix weights depend on the known $g_k$. The current noise coefficient is one. Therefore subtraction of all known terms from \eqref{eq:app-failed-mean} leaves exactly
\begin{equation}\label{eq:app-stage-innovation}
\mu_{v_k}(r_k,g_k)+Q_{j,k}(r_k,g_k),
\qquad r_k=x-a_{<k}.
\end{equation}
To apply \Cref{thm:reuse}, condition first on the stronger
$\mathcal G_{<k}$ of \eqref{eq:app-prebirth-sigma} and on the revealed
frame prefix, \emph{before} observing any current-stage reply.
By \eqref{eq:app-birth-kernel}, $v_k$ is uniform on the residual
sphere; by independence of tensor blocks, every stage-$k$ sample
field has its original Gaussian law, jointly independent of $v_k$.
The extra past-stage information is therefore auxiliary information of
the exact sort permitted in \Cref{thm:reuse}.  After conditioning on
it, the algorithm's current-stage label and query choices remain
predictable functions of past \emph{current-stage} responses and its
coins, even if it requests disclosures of further past-stage fields.
Those disclosures are measurable functions of $\mathcal G_{<k}$ and
the selected label, and hence contribute no new conditional likelihood
factor in $v_k$.  In particular, the selected label identities must
not be conditioned upon as unexplained free advice about $v_k$:
they are part of the adaptive transcript and are handled by the
conditional-score factorization
\eqref{eq:app-conditional-kernel}--\eqref{eq:app-total-score}.
The finite antithetic-pool simulator uses a stricter \emph{opened-only}
disclosure policy, because revealing unopened labels' tensors would
destroy the hidden-pair exchangeability used in
\eqref{eq:app-partner}.

Put
\begin{equation}\label{eq:app-stage-budget}
s=m-T+1,
\qquad
\ell_0=\left\lfloor\frac{s^2\sigma^2}{2048m}\right\rfloor,
\qquad s\ge72,
\quad \ell_0\ge1,
\quad\log\ell_0\le\frac{s}{8}.
\end{equation}
From each stage's birth, script $\ell_0$ consecutive failures using \eqref{eq:app-failed-mean}, withholding all crossing disclosures, and record the workspace guesses generated on this script. Up to the first actual gate hit, the scripted transcript agrees with the actual stage run. To turn any potential gate hit into a direction guess, project $r_k$ orthogonally away from the disclosed workspace prefix and call the result $p$. At a hit, $g_k=1$ and $x-a_k=r_k-Av_k$. Let $P_{<k}^\perp$ project onto the orthogonal complement of the already revealed $v_1,\ldots,v_{k-1}$. Since $P_{<k}^\perp v_k=v_k$, we have
\begin{align*}
p-Av_k
&=P_{<k}^\perp r_k-P_{<k}^\perp(Av_k)
=P_{<k}^\perp(r_k-Av_k)\\
&=P_{<k}^\perp(x-a_k),\\
\|p-Av_k\|&\le\|x-a_k\|<\sqrt2\,\beta<A.
\end{align*}
The first line is the definition of $p$, the second uses $v_k\perp v_{<k}$, and the inequality is nonexpansivity of orthogonal projection followed by the gate radius. Since $\sqrt2\,\beta<A$, we also have $p\ne0$. By reverse triangle inequality,
\[
\left\|\frac{p}{\|p\|}-v_k\right\|
\le2\left\|\frac pA-v_k\right\|<\frac{2\sqrt2\beta}{A},
\]
and because both vectors are unit,
\begin{equation}\label{eq:app-stage-angle}
\left\langle\frac p{\|p\|},v_k\right\rangle
=1-\frac12\left\|\frac{p}{\|p\|}-v_k\right\|^2
>1-\frac{4\beta^2}{A^2}=\frac{15}{16}.
\end{equation}
For $p=0$ use a fixed residual-space unit vector; such a candidate cannot hit. The list-recovery theorem with birth dimension $s_k\ge s$ yields
\begin{align}
\PP(\text{stage $k$ passes within $\ell_0$ tests}\mid\mathcal H_{\rm birth})
&\le\frac{64m\ell_0}{s_k^2\sigma^2}+\frac9{4s_k}\\
&\le\frac{64m\ell_0}{s^2\sigma^2}+\frac9{4s}
\le\frac1{32}+\frac1{32}=\frac1{16}.
\label{eq:app-fast-stage}
\end{align}
Scripted failure gate notifications give no additional hidden-parameter information before the first actual hit because their value is fixed there; that hit is contained in the scripted list-success event.

\paragraph{Accumulating the conditional costs.}
Call a stage \emph{fast} if it finishes within $\ell_0$ of its own tests; an unborn stage has fast indicator zero. Let $k_1<\cdots<k_j$ be any fixed collection of indices and $F_k$ the event that stage $k$ is born and finishes fast. Applying the tower property at the birth of the last selected stage, then successively at the earlier births, gives the complete recursion
\begin{align*}
\PP\Bigl(\bigcap_{a=1}^jF_{k_a}\Bigr)
&=\E\left[\mathbf1_{\cap_{a<j}F_{k_a}}
\PP(F_{k_j}\mid\mathcal H_{\rm birth}(k_j))\right]\\
&\le\frac1{16}\PP\Bigl(\bigcap_{a<j}F_{k_a}\Bigr)
\le\cdots\le\left(\frac1{16}\right)^j.
\end{align*}
The indicators for earlier selected stages are measurable at this birth; a stage not born has indicator zero. No independence between durations is invoked. After an algorithm stops, pad with $(z,x)=(0,0)$, which never passes a gate. If fewer than $\lceil T/2\rceil$ stages were fast, more than $T/2$ stages would each consume more than $\ell_0$ tests, giving strictly more than $T\ell_0/2>T\ell_0/4\ge N+1$ tests. Thus finishing in the allotted budget forces at least $\lceil T/2\rceil$ fast stages. Union-bounding the joint probability over all subsets of this size (and then crudely over all $2^T$ subsets) yields
\begin{align*}
\PP(\text{all $T$ stages finish})
&\le\binom{T}{\lceil T/2\rceil}(1/16)^{\lceil T/2\rceil}\\
&\le 2^T(1/16)^{T/2}=2^{-T}.
\end{align*}
We record this as the fast-stage counting bound:
\begin{equation}\label{eq:app-fast-count}
\PP(\text{all $T$ stages finish})\le2^{-T}.
\end{equation}
No independent-stage-duration assumption is made. The output point contributes one additional gate test, but receives no value observation and opens no sample label. For missing scripted observations one can pad at $g=0$, where both the hidden mean component and Gaussian field are parameter independent.

\end{proof}

\subsection{Exact coupling to the unrestricted static oracle}
\begin{lemma}[Static-oracle coupling and finite-pool error]\label[lemma]{lem:static-coupling}
For $N\le M$ and $N+1\le T\ell_0/4$, the static finite oracle
and the prefix simulator admit a coupling under which stationary
success by $N$ queries has probability at most
$2^{-T}+2(N+1)T e^{-(r-T)/(32R^2)}+N^2/M$. Labels may be chosen
adaptively and reused, and the final output may be arbitrary.
\end{lemma}
\begin{proof}
\textbf{Future activations.}\enspace
Undisclosed $U$ columns never enter simulated responses before their
gate crossings. Conditional on the simulated history $\mathcal H$,
the next query is fixed and each undisclosed column is Haar in the
unrevealed orthogonal complement. Since $\|y(z)\|\le R$, an inner
product of magnitude at least $1/4$ requires a normalized inner
product of magnitude at least $1/(4R)$. Applying
\Cref{thm:tool-sphere} with the weaker remaining dimension $r-T$
and taking both signs gives
\begin{align*}
\PP\left(|\langle u_j,y(z)\rangle|\ge\tfrac14\mid\mathcal H\right)
&\le 2\exp\left[-\frac{r-T}{2}\left(\frac1{4R}\right)^2\right]\\
&=2\exp\left[-\frac{r-T}{32R^2}\right].
\end{align*}
There are at most $T$ hidden columns and $N$ paid queries plus one final unpaid output test. A union bound, conditional at each adaptive query and then averaged, gives
\begin{equation}\label{eq:app-pbad}
p_{\rm bad}\le2(N+1)T\exp\left(-\frac{r-T}{32R^2}\right).
\end{equation}
On the complementary good-projection event, every undisclosed future activation is identically zero. Therefore the complete static noise field and the simulator's prefix field coincide at every queried label and location, including reused labels and stage crossings. The mean identities follow from the past-centered correction and the crossing and failure formulas. This is an exact pathwise coupling, not replacement of reused noise with fresh independent noise.

The global stationary certificate forces $w_i>1$ for every $i$ at any small-gradient output. Before final stage completion, the necessary future activations are absent on the good-projection event, so such an output would contradict the certificate. At stage $T$, the same certificate forces the final gate. Thus stationary success requires completing all stages or encountering a bad future projection.

\paragraph{Hidden antithetic partners.}
To realize independent labels inside the finite antithetic pool, couple the $q$-th \emph{distinct} newly opened public label to a fresh Gaussian family until the opposite-signed partner of an opened label is itself opened. For the collision probability, the hidden pairing is a uniformly random perfect matching of $2M$ public positions. Conditional on the history of previously opened \emph{nonpartner} positions, when a new label is chosen, its partner is uniform among the remaining unexposed positions. At most $q-1$ of those positions are already represented by a previously opened counterpart. On the event that no partner has yet been opened, distinct base tensor families are independent. The sign of a new family is symmetric because $G$ and $-G$ have identical centered Gaussian laws. Reused labels keep their assigned fields. Conditioned on no previous partner collision, the $q$-th new label has at most $q-1$ partners among $2M-q+1$ unexposed positions. For $N\le M$,
\begin{align}
\PP(\text{partner collision})
&\le\sum_{q=1}^N\frac{q-1}{2M-q+1}
\le\frac1M\sum_{q=1}^N(q-1)
=\frac{N(N-1)}{2M}\le\frac{N^2}{M}.
\label{eq:app-partner}
\end{align}
The random permutation and Gaussian sign symmetry ensure that before collisions newly opened labels have the law of independent Gaussian families. Only opened-label past fields are disclosed in this finite-pool coupling; revealing all unopened fields would invalidate the hidden-pair exchangeability argument.

\paragraph{Combining the three events.}
The three events can now be separated explicitly. Let $E_1$ be completion of all gates in the scripted simulator, $E_2$ be a large projection onto a future undisclosed $U$ column, and $E_3$ be opening both signs of some hidden antithetic pair. On $E_2^c\cap E_3^c$, the static evaluation transcript equals the simulator transcript at every paid query and its stationary-output gate test. Thus
\begin{align*}
\PP(\text{static stationary success})
&\le\PP(E_1\cup E_2\cup E_3)\\
&\le\PP(E_1)+\PP(E_2)+\PP(E_3)\\
&\le2^{-T}+2(N+1)T e^{-(r-T)/(32R^2)}+N^2/M.
\end{align*}
The first inequality uses the pathwise coupling and the all-space stationary certificate; the second is the union bound; and the last inserts the three estimates already proved. This gives the unconditioned master bound
\begin{equation}\label{eq:app-master}
\PP(\text{stationary success within $N$ evaluations})
\le2^{-T}+p_{\rm bad}+\frac{N^2}{M},
\qquad N+1\le\frac{T\ell_0}{4}.
\end{equation}

\end{proof}

\subsection{Conditioning on the global moment promise}
\begin{lemma}[Admissibility conditioning and expected cost]\label[lemma]{lem:expected-cost}
Suppose an unconditioned finite-family law satisfies
$\PP(E_{\rm adm})\ge3/4$ and
$\PP(\mathrm{success},\mathrm{cost}\le N_0)\le11/48$.
Every algorithm succeeding with probability at least $2/3$ on
all admissible fixed families has expected query cost at least
$N_0/3$ on some such fixed family.
\end{lemma}
\begin{proof}
Condition on \eqref{eq:app-admissible}. The two assumed probability
bounds give
\begin{equation}\label{eq:app-conditioned-success}
\PP(\text{success},\ \mathrm{cost}\le N_0\mid E_{\rm adm})
\le\frac{11/48}{3/4}=\frac{11}{36}<\frac13.
\end{equation}
The conditioned distribution is supported on admissible fixed
finite sample families. The preceding bound concerns success by
$N_0$, not eventual success; the following truncation step converts
it into an expected-cost lower bound.
For a variable-length algorithm truncate the run after $N_0$ evaluations. Let $S$ denote eventual successful output and $C$ its evaluation cost. The events $S\cap\{C\le N_0\}$ and $\{C>N_0\}$ cover $S$. Averaging over admissible families and using the fixed-budget result gives the full implication
\begin{equation}\label{eq:app-cost-tail}
\begin{aligned}
\frac23
&\le\PP(S\mid E_{\rm adm})\\
&\le\PP(S,C\le N_0\mid E_{\rm adm})
 +\PP(C>N_0\mid E_{\rm adm})\\
&<\frac13+\PP(C>N_0\mid E_{\rm adm}).
\end{aligned}
\end{equation}
The three inequalities use, respectively, the uniform success guarantee,
the union bound, and the truncated-run bound. Thus
$\PP(C>N_0\mid E_{\rm adm})>1/3$. The inequality $\E C\ge N_0\PP(C>N_0)$ yields
\begin{equation}\label{eq:app-expected-cost}
\E[C\mid E_{\rm adm}]
\ge N_0\PP(C>N_0\mid E_{\rm adm})\ge N_0/3.
\end{equation}
Another averaging argument yields the worst-case expected-cost claim
on a fixed admissible family.

\end{proof}

%% file: sections/scaling_and_noiseless_proofs.tex
\section{Parameter-Explicit Scales and the Dimension Threshold}
\label[appendix]{app:parameters}
\subsection{The Lipschitz constraint on chain length}
\begin{proof}[Verification of the stated budget constraints]
The physical gap and sample-moment budgets imply
\begin{equation}\label{eq:app-cap}
T\lesssim\frac{\Delta H}{\epsilon^2},
\qquad
T\lesssim\frac{L_0^2}{\epsilon^2}.
\end{equation}
The second restriction also has a direct necessary explanation \emph{for the present wrapper}. Along the graph $x=a(w)$, the correction jets and workspace-well gradients vanish. Along rays $z=tu$ with $\|u\|=1$, one has $Dy(z)\to0$ and $\nabla P(z)\to(R/5)u$ as $t\to\infty$. Hence $\Lip(F_{\rm new})\ge R/5=64\sqrt T/5$. Under physical scaling and the sample Lipschitz promise,
\[
\frac{64\epsilon\sqrt T}{5e_*}
\le\Lip(f)\le\E_\xi\Lip(f_\xi)
\le\sqrt{\E_\xi\Lip(f_\xi)^2}\le L_0.
\]
Therefore $T\le(5e_*L_0/(64\epsilon))^2$, which, together
with the gap budget, gives $B=\min\{\Delta H,L_0^2\}$.

\end{proof}

\subsection{Uniform parameter theorem}
\begin{lemma}[Uniform physical parameters and dimension conditions]
\label[lemma]{lem:uniform-parameters}
Under \eqref{eq:main-regime}, the constants and scales
\eqref{eq:app-universal-constants}--\eqref{eq:app-universal-scales}
satisfy the gap, smoothness and sample budgets, the one-stage list
conditions, and the future-projection bound required by
\Cref{lem:stage-charging,lem:static-coupling}. They yield the
expected-cost lower bound of \Cref{thm:main}, with a positive
universal prefactor given in \eqref{eq:app-uniform-prefactor}.
\end{lemma}
\begin{proof}

Define the source universal constants
\begin{equation}\label{eq:app-universal-constants}
a_T=\frac{e_*^2}{12H_*},\qquad
a_\sigma=\frac{9e_*^2}{64C_N},\qquad
c_0=\min\left\{\frac{a_T}{8},\frac{e_*^2}{81C_*^2},a_\sigma\right\}.
\end{equation}
For $\epsilon^2\le c_0B$, take
\begin{equation}\label{eq:app-universal-scales}
T=\left\lfloor\frac{a_TB}{\epsilon^2}\right\rfloor,
\qquad
\sigma^2=\frac{a_\sigma L_0^2}{\epsilon^2}.
\end{equation}
The hypotheses and $B\le L_0^2$ give
\[
 a_TB/\epsilon^2\ge8,\qquad
 8\le T,\quad T\ge a_TB/(2\epsilon^2),\qquad
 \sigma^2\ge a_\sigma B/\epsilon^2\ge1.
\]
Here $\lfloor x\rfloor\ge x/2$ for $x\ge2$. The gap estimate is
\begin{align*}
 f(0)-\inf f
 &=\lambda[F_{\rm new}(0,0)-\inf F_{\rm new}]\\
 &\le12\lambda T
 \le12\frac{\epsilon^2H_*}{e_*^2H}\frac{a_TB}{\epsilon^2}
 =\frac BH\le\Delta.
\end{align*}
The equality is the global rescaling; the inequalities use
\Cref{lem:global-geometry}, the upper floor bound, and $B\le\Delta H$.
The last equality uses $a_T=e_*^2/(12H_*)$.
For the moment budget, the source separately bounds
\begin{align*}
\frac{\epsilon}{e_*}c_*\sqrt T
&\le\frac{c_*}{\sqrt{12H_*}}\sqrt B<\frac{L_0}{72},\\
\frac{\epsilon}{e_*}C_*&\le\frac{\sqrt B}{9}\le\frac{L_0}{9},\\
\frac{\epsilon}{e_*}\sqrt{C_N}\sigma&=\frac{3L_0}{8}.
\end{align*}
Therefore
\begin{equation}\label{eq:app-budget-half}
\frac{\lambda}{b}(c_*\sqrt T+C_*+\sqrt{C_N}\sigma)
\le(1/72+1/9+3/8)L_0=L_0/2.
\end{equation}
The first two terms alone bound the deterministic scaled Lipschitz constant by $L_0/8$.

\paragraph{Dimension and list conditions with an explicit sufficient
constant.}
Take $r=\lfloor d/2\rfloor$, $m=d-r$,
$s=m-T+1$, and
$N_0=\lfloor a_NdT\sigma^2\rfloor$, where $a_N=2^{-19}$.
We may choose the universal constants
\begin{equation}\label{eq:app-concrete-dim-constants}
C_1=2^{28},\qquad C_2=2^{36}.
\end{equation}
These constants give the following dimension and list estimates.
Suppose first that
\begin{equation}\label{eq:app-dimension-intermediate}
d\ge C_1T\log(2+dT^2\sigma^2).
\end{equation}
Because $T\ge8$, $\sigma^2\ge1$ and the logarithm is at least
$\log2$, the bound gives $d\ge4T+4$,
$d\ge 8T$, and $d\sigma^2\ge2^{17}$.  Since
$m=\lceil d/2\rceil\ge d/2$ and $m\le d$,
\begin{align*}
m&\ge2T,\\
s=m-T+1&\ge\frac d2-T+1\ge\frac d4,\\
s&\ge72.
\end{align*}
In particular the real number defining the stage budget,
$x_\ell=s^2\sigma^2/(2048m)$, satisfies
\[
x_\ell\ge\frac{(d/4)^2\sigma^2}{2048d}
=\frac{d\sigma^2}{32768}\ge4.
\]
Thus $\lfloor x_\ell\rfloor\ge x_\ell/2$ and
\begin{equation}\label{eq:app-list-budget-explicit}
\ell_0=\lfloor x_\ell\rfloor
\ge\frac{s^2\sigma^2}{4096m}
\ge\frac{d\sigma^2}{65536}\ge2.
\end{equation}
For the total gate-test budget, use the upper floor bound for $N_0$
and the preceding lower bound:
\begin{align*}
\frac{T\ell_0}{4}-(N_0+1)
&\ge\frac{dT\sigma^2}{262144}
     -\frac{dT\sigma^2}{524288}-1\\
&=\frac{dT\sigma^2}{524288}-1\ge0,
\end{align*}
where $d\sigma^2\ge2^{17}$ and $T\ge8$.
Finally, because $s\le m\le d$,
\[
\ell_0\le\frac{s^2\sigma^2}{2048m}
\le d\sigma^2,
\qquad
\log\ell_0\le\log(2+dT^2\sigma^2).
\]
Meanwhile $s\ge d/4$ and \eqref{eq:app-dimension-intermediate}
give
\[
\frac{s}{8}\ge\frac d{32}
\ge\frac{2^{28}T}{32}\log(2+dT^2\sigma^2)
\ge\log\ell_0.
\]
This verifies $s\ge72$, $\ell_0\ge1$ and
$\log\ell_0\le s/8$ with all floor effects included.

\paragraph{Explicit bad-projection margin.}
With $R^2=4096T$ and $r=\lfloor d/2\rfloor$, the coupling estimate
\eqref{eq:app-pbad} reads
\begin{align*}
p_{\rm bad}
&\le2(N_0+1)T
   \exp\left[-\frac{r-T}{131072T}\right]\\
&\le2(a_NdT\sigma^2+1)T
   \exp\left[-\frac{d}{1048576T}\right],
\end{align*}
because $r-T\ge d/4-1\ge d/8$ for $d\ge8$.
Put $X=dT^2\sigma^2$. Since $X\ge T$ and $C_1=2^{28}$,
\[
\exp\left[-\frac{d}{1048576T}\right]
\le\exp[-256\log(2+X)]=(2+X)^{-256}.
\]
Consequently
\begin{align*}
p_{\rm bad}
&\le2(a_NX+T)(2+X)^{-256}\\
&\le4X(2+X)^{-256}
\le4(2+X)^{-255}
<\frac1{12}.
\end{align*}
Thus the intermediate dimension requirement controls both list
recovery and the unrestricted-query coupling with a strict margin.

\paragraph{Removing $d$ from inside the logarithm.}
It suffices instead to require
\begin{equation}\label{eq:app-dimension-simplified}
d\ge C_2T\log(2+T\sigma^2),\qquad C_2=2^{36}.
\end{equation}
To prove this directly, set $X_0=T\sigma^2\ge8$,
$L=\log(2+X_0)\ge\log3>1$, and $d_*=C_2TL$.
Since $T\le X_0$ and $L\le X_0$,
\begin{align*}
\log(2+d_*T^2\sigma^2)
&=\log(2+C_2T^2X_0L)\\
&\le\log(2+C_2X_0^4)\\
&\le\log(2+C_2)+4\log X_0\\
&\le\bigl[\log(2+C_2)+4\bigr]L.
\end{align*}
Here $\log(2+2^{36})+4<30$ and $C_2/C_1=256$,
so $d_*\ge C_1T\log(2+d_*T^2\sigma^2)$.
For $\alpha=T^2\sigma^2>0$, the function
$d\mapsto\log(2+\alpha d)/d$ is decreasing: with
$t=2+\alpha d\ge2$, the numerator of its derivative is
$1-2/t-\log t<0$.
Thus the same intermediate inequality holds for every $d\ge d_*$.
Finally, $T\le a_TB/\epsilon^2$ and
$T\sigma^2\le a_Ta_\sigma BL_0^2/\epsilon^4$.
Since the fixed source constants $a_T,a_\sigma$ are at most one,
a sufficiently large universal $C$ in the main theorem makes
\[
d\ge C\frac B{\epsilon^2}
\log\left(2+\frac{BL_0^2}{\epsilon^4}\right)
\quad\Longrightarrow\quad
\eqref{eq:app-dimension-simplified}.
\]
No extra parameter restriction is used in replacing the intermediate
dimension inequality by the advertised one.
Under these conditions, the finite-pool argument gives an expected-cost lower bound at least $N_0/3$. The complete numerical chain is
\begin{align*}
\frac{N_0}{3}
&=\frac13\lfloor a_NdT\sigma^2\rfloor\\
&\ge\frac16a_NdT\sigma^2
&&\text{($a_NdT\sigma^2\ge2$, so $\lfloor x\rfloor\ge x/2$)},\\
&\ge\frac16a_Nd\left(\frac{a_TB}{2\epsilon^2}\right)
\left(\frac{a_\sigma L_0^2}{\epsilon^2}\right)
&&\text{(lower floor bound for $T$ and chosen $\sigma^2$)},\\
&=\frac{a_Na_Ta_\sigma}{12}\frac{dL_0^2B}{\epsilon^4}.
\end{align*}
The multiplier $d$ arises from the hidden-space dimension $s_k\asymp d$, not from summing the global Lipschitz costs of $T$ samples.
The source records the following single sufficient universal prefactor:
\begin{equation}\label{eq:app-uniform-prefactor}
c=\frac{2^{-19}a_Ta_\sigma}{12}
=\frac{e_*^4}{2^{29}H_*C_N}
\approx\frac{4.2199\times10^{-31}}{C_N}.
\end{equation}
With $C_N=2^{128}$, these constants are
\[
c\approx1.2401\times10^{-69},
\qquad c_0\approx1.7899\times10^{-47},
\qquad\epsilon/\sqrt B\le\sqrt{c_0}\approx4.2308\times10^{-24}.
\]

\end{proof}

\section{Goldstein and Noiseless Lower-Bound Proofs}
\label[appendix]{app:noiseless}
\noindent\textbf{Proof obligations for \Cref{thm:goldstein,thm:noiseless}.}
The first part supplies the smooth-to-Goldstein reduction; the second
uses adaptive exact linear measurements instead of a Gaussian channel
to prove the lower bound for randomized exact-value algorithms.
\subsection{Smooth-to-Goldstein comparison}
\begin{lemma}[$C^1$ Clarke subgradients and Goldstein comparison]
\label[lemma]{lem:goldstein-comparison}
If $f$ is $C^1$, then its Clarke subdifferential at $x$ is
$\{\nabla f(x)\}$. If $\nabla f$ is $H$-Lipschitz, then
\begin{equation}\label{eq:app-Goldstein-ineq}
 g_\delta(x)\le\|\nabla f(x)\|\le g_\delta(x)+H\delta.
\end{equation}
\end{lemma}
\begin{proof}
For every direction $h$, the fundamental theorem of calculus and
continuity of the gradient give
\begin{align*}
 f^\circ(x;h)
 &=\limsup_{z\to x,\ t\downarrow0}
       \frac{f(z+th)-f(z)}t\\
 &=\limsup_{z\to x,\ t\downarrow0}
       \int_0^1\langle\nabla f(z+uth),h\rangle\,du
 =\langle\nabla f(x),h\rangle.
\end{align*}
The first equality is the Clarke directional-derivative definition;
the second integrates the derivative along the segment. For the
last, $z+uth\to x$ uniformly in $u\in[0,1]$, so the continuous
integrand converges uniformly. Hence
\begin{align*}
 v\in\partial^Cf(x)
 &\Longleftrightarrow
   \langle v-\nabla f(x),h\rangle\le0\quad(\forall h)\\
 &\Longleftrightarrow v=\nabla f(x).
\end{align*}
The first equivalence is the definition of the Clarke
subdifferential; choosing $h=v-\nabla f(x)$ proves the forward
implication in the second, and substitution proves its converse.
Thus the Goldstein subdifferential is the closed convex hull of
the gradients in its defining ball.
The first inequality holds because
$\nabla f(x)\in\partial_\delta^Gf(x)$.
For the second, let
\[
 v=\sum_a\lambda_a\nabla f(y_a),\qquad
 \lambda_a\ge0,\quad\sum_a\lambda_a=1,\quad\|y_a-x\|\le\delta.
\]
Then
\begin{align*}
\|\nabla f(x)\|
&\le\|v\|+\|\nabla f(x)-v\|
&&\text{(triangle inequality)},\\
&=\|v\|+\left\|\sum_a\lambda_a
[\nabla f(x)-\nabla f(y_a)]\right\|\\
&\le\|v\|+\sum_a\lambda_a\|\nabla f(x)-\nabla f(y_a)\|
&&\text{(convexity of the norm)},\\
&\le\|v\|+\sum_a\lambda_aH\|x-y_a\|
&&\text{(global $H$-smoothness)},\\
&\le\|v\|+H\delta
&&\text{($\|y_a-x\|\le\delta$ and $\sum_a\lambda_a=1$).}
\end{align*}
Continuity passes this bound to the closed convex hull; taking the infimum over $v$ gives the second inequality. The $H\delta$ term is necessary: for $f(t)=Ht^2/2$ at $x=\delta$, one has $g_\delta(x)=0$ and $|f'(x)|=H\delta$.

\end{proof}

\begin{proof}[Proof of \Cref{thm:goldstein}]
\phantomsection\label{proof:goldstein-bound}
Set $H=\epsilon/\delta$ and apply the smooth theorem with gradient tolerance $\epsilon_s=2\epsilon$. If a returned $x$ were Goldstein $(\delta,\epsilon)$-stationary, the comparison just proved would yield
\[
\|\nabla f(x)\|
\le g_\delta(x)+H\delta
\le\epsilon+\frac{\epsilon}{\delta}\delta=2\epsilon=\epsilon_s.
\]
Thus every algorithm solving the Goldstein task would also solve gradient stationarity for the same smooth witness and the tolerance $\epsilon_s$, without any oracle conversion. The smooth lower bound equals
\[
\Omega\left(\frac{dL_0^2\min\{\Delta H,L_0^2\}}{\epsilon_s^4}\right)
=\Omega\left(\frac{dL_0^2\min\{\Delta\epsilon/\delta,L_0^2\}}{16\epsilon^4}\right),
\]
where $16$ is absorbed into the universal constant. This is the capped lower bound
\[
N\ge c\frac{dL_0^2B_G}{\epsilon^4},
\qquad B_G=\min\left\{\frac{\Delta\epsilon}{\delta},L_0^2\right\},
\]
subject to $\epsilon^2\le c_0B_G$ and $d\ge C(B_G/\epsilon^2)\log(2+B_GL_0^2/\epsilon^4)$ with appropriate universal constants. In the unsaturated interval $\Delta\epsilon/L_0^2\le\delta\le c_0\Delta/\epsilon$, this yields $\Omega(dL_0^2\Delta/(\delta\epsilon^3))$.

\end{proof}

\begin{corollary}[Larger-gap Goldstein bounds]\label[corollary]{cor:large-gap-goldstein}
There are universal positive constants such that, if
$\epsilon\le c_eL_0$, $\Delta\epsilon/\delta\ge C_*L_0^2$, and
\[
 d\ge C\frac{L_0^2}{\epsilon^2}
              \log\left(2+\frac{L_0^2}{\epsilon^2}\right),
\]
the stochastic model of \Cref{thm:goldstein} has worst-case expected
query cost
\begin{equation}\label{eq:large-gap-goldstein-linear}
 N\ge c\frac{dL_0^2}{\epsilon^4}
       \min\left\{\frac{\Delta\epsilon}{\delta},
                         \frac{L_0^4}{\epsilon^2}\right\}.
\end{equation}
Also, if $\Delta\epsilon/\delta\ge C_0L_0^2$, set
\[
 K=1+\left\lfloor
    \frac{\log(\Delta\epsilon/(C_0\delta L_0^2))}{\log C_1}
    \right\rfloor.
\]
Under $\epsilon\le c_eL_0$ and
$d\ge C[(L_0^2/\epsilon^2)\log(2+L_0^2/\epsilon^2)+\log(2K)]$,
\begin{equation}\label{eq:large-gap-goldstein-log}
 N\ge c\frac{dL_0^4}{\epsilon^4}
       \left[1+\log\frac{\Delta\epsilon}{C_0\delta L_0^2}\right].
\end{equation}
Both conclusions allow randomized algorithms, unrestricted queries
and outputs, and adaptive sample reuse.
\end{corollary}
\begin{proof}
Take $H=\epsilon/\delta$ in
\Cref{thm:extended-gap,thm:logarithmic-gap}, with gradient tolerance
$2\epsilon$. Equation \eqref{eq:goldstein-ineq} gives
$g_\delta(\widehat x)\le\epsilon\Rightarrow
\|\nabla f(\widehat x)\|\le2\epsilon$.
The two factors $(2\epsilon)^{-4}$ and
$L_0^4/(2\epsilon)^2$ change only universal constants; in particular
$\min\{a,b/4\}\ge\min\{a,b\}/4$ for $a,b\ge0$.
The dimension and accuracy conditions follow by adjusting the same
universal constants. The smooth hard families are contained in the
stated nonsmooth class, completing the reduction.
\end{proof}

\subsection{Adaptive noiseless spherical measurements}
The inverse formula \eqref{eq:noiseless-invert} gives the scalar
measurement $\langle r,v\rangle$. Conditional spherical symmetry
then bounds the probability of recovering $v$ from adaptive measurements.

\begin{lemma}[Noiseless spherical measurement list bound]\label[lemma]{lem:noiseless-list}
Let $v$ be uniform on the unit sphere in a known $s$-dimensional Euclidean subspace, $s\ge4$. After $n<s$ adaptive exact linear measurements of $v$, let the algorithm output a transcript-measurable list of $M_{\rm list}$ unit-vector guesses; its initial auxiliary information and random coins are independent of $v$. Then
\begin{equation}\label{eq:app-noiseless-list}
\PP\left(\max_{a\le M_{\rm list}}\langle v,\widehat v_a\rangle\ge15/16\right)
\le\frac{4n}{s}+M_{\rm list}\exp\left(-\frac{49(s-n)}{512}\right).
\end{equation}
For $s\ge512$ and $n,M_{\rm list}\le s/256$, the probability is at most $1/16$.
\end{lemma}

\begin{proof}
By Gram--Schmidt, each nonredundant exact linear measurement reveals the projection onto a unit direction in the unmeasured subspace. If a measurement is redundant, disclose an additional orthogonal component for free, strengthening the experiment. Conditional on the first $i$ measured components, the residual is uniform on the sphere in the remaining $(s-i)$-dimensional subspace, with squared radius $1-\|P_iv\|^2$. This follows inductively from rotational invariance because the next query direction is fixed by the recorded history and conditioning on its exact component preserves spherical symmetry in the orthogonal complement.

Let $r_i=v-P_iv$ and $\lambda_i=\|r_i\|$. Under the conditional spherical law, a new unit measurement direction $h_i$ orthogonal to the previous span satisfies
\[
\E[\langle r_i,h_i\rangle\mid\mathcal T_i]=0,
\qquad
\E[\langle r_i,h_i\rangle^2\mid\mathcal T_i]
=\frac{\lambda_i^2}{s-i},
\]
because the covariance of a uniform sphere of radius $\lambda_i$ in $s-i$ dimensions is $\lambda_i^2 I/(s-i)$. Since the next projection is $P_{i+1}v=P_iv+\langle r_i,h_i\rangle h_i$ with orthogonal summands, Pythagoras gives
\begin{equation}\label{eq:app-noiseless-projection-increment}
\begin{aligned}
\|P_{i+1}v\|^2-\|P_iv\|^2
&=\langle r_i,h_i\rangle^2,\\
\E[\|P_{i+1}v\|^2-\|P_iv\|^2\mid\mathcal T_i]
&=\frac{\lambda_i^2}{s-i}
=\frac{1-\|P_iv\|^2}{s-i}.
\end{aligned}
\end{equation}
This establishes the conditional increment identity from rotational symmetry rather than assuming the adaptive directions are independent of $v$.
Starting with $P_0=0$, write $a_i=\E\|P_iv\|^2$. Taking expectation in the displayed conditional formula gives
\begin{equation}\label{eq:app-noiseless-measured-energy}
\begin{aligned}
a_{i+1}&=a_i+\frac{1-a_i}{s-i}
&&\text{(tower property of conditional expectation)},\\
&=\frac{i}{s}+\frac{1-i/s}{s-i}
&&\text{(induction hypothesis $a_i=i/s$)},\\
&=\frac{i}{s}+\frac{1}{s}=\frac{i+1}{s}
&&\text{(since $(1-i/s)/(s-i)=1/s$).}
\end{aligned}
\end{equation}
Thus $\E\|P_nv\|^2=n/s$. Markov's inequality then yields
\[
\PP(\|P_nv\|>1/2)\le4n/s.
\]
On the event $\|P_nv\|\le1/2$, a guess satisfying $\langle v,\widehat v\rangle\ge15/16$ necessarily obeys
\begin{align*}
\langle v-P_nv,\widehat v\rangle
&=\langle v,\widehat v\rangle-\langle P_nv,\widehat v\rangle\\
&\ge\frac{15}{16}-\|P_nv\|\|\widehat v\|
\ge\frac{15}{16}-\frac12=\frac7{16}.
\end{align*}
The first line is linearity, the next uses Cauchy--Schwarz and $\|\widehat v\|=1$, and the last uses the low-projection event. Conditioned on the transcript, $\widehat v$ is fixed and the residual
has uniform orientation in the remaining $(s-n)$-dimensional subspace,
with radius at most one. Thus the spherical cap theorem \Cref{thm:tool-sphere} gives
\[
\PP\left(\langle v-P_nv,\widehat v\rangle\ge\frac7{16}\mid\mathcal T_n\right)
\le\exp\left[-\frac{s-n}{2}\left(\frac7{16}\right)^2\right]
=\exp\left[-\frac{49(s-n)}{512}\right].
\]
A union bound over $M_{\rm list}$ guesses and then Markov over the high-projection event yield
\begin{align*}
\PP\left(\max_{a\le M_{\rm list}}\langle v,\widehat v_a\rangle\ge\frac{15}{16}\right)
&\le\PP(\|P_nv\|>1/2)
+M_{\rm list}e^{-49(s-n)/512}\\
&\le\frac{4n}{s}+M_{\rm list}e^{-49(s-n)/512},
\end{align*}
which is \eqref{eq:app-noiseless-list}. For the stated specialization,
$4n/s\le1/64$. Also, when $s\ge512$ and $n\le s/256$,
\[\frac{s}{256}\exp\left[-\frac{49(s-n)}{512}\right]\le\frac3{64}.\]
Thus the total success probability is at most $1/16$.
\end{proof}

\subsection{Smooth exact-value lower bound}\label{proof:exact-bound}
\begin{proof}[Proof of \Cref{thm:noiseless,cor:exact-goldstein}]
In the deterministic oracle experiment, the algorithm queries a single fixed function with exact precision; the algorithm itself may be randomized. Use the same Haar frames, geometric mean $F_{\rm new}$, and physical scales \eqref{eq:app-physical-scales}, but no Gaussian field and no sample pool. Take
\[
 T=\lfloor a_TB/\epsilon^2\rfloor,\qquad
 c_0=\min\{a_T/8,e_*^2/(81C_*^2)\}.
\]
These choices require no noise constant. The deterministic scaled Lipschitz bound is at most $L_0/8$ under the small-tolerance condition, while smoothness, gap, and stationary localization follow from the same geometric estimates.

At a failed stage gate the new hidden contribution is $\mu_{v_k}(r_k,g_k)$. For $g_k>0$, put $u=\langle r_k,v_k\rangle$ and recall $\mu_{v_k}(r_k,g_k)=4\bigl(\sqrt{1+\|r_k-Ag_kv_k\|^2}-1\bigr)$. A complete algebraic inversion is
\begin{align*}
1+\mu_{v_k}(r_k,g_k)/4
&=\sqrt{1+\|r_k-Ag_kv_k\|^2}\\
&=\sqrt{1+\|r_k\|^2+A^2g_k^2-2Ag_k\langle r_k,v_k\rangle},\\
\bigl(1+\mu_{v_k}(r_k,g_k)/4\bigr)^2
&=1+\|r_k\|^2+A^2g_k^2-2Ag_ku,\\
2Ag_ku
&=1+\|r_k\|^2+A^2g_k^2-\bigl(1+\mu_{v_k}(r_k,g_k)/4\bigr)^2.
\end{align*}
The first line is the definition, the second expands the squared distance using $\|v_k\|=1$, and the third solves for the unknown inner product. Thus one exact scalar value is equivalent to \emph{at most} one exact linear measurement:
\begin{equation}\label{eq:app-exact-inversion}
\langle r_k,v_k\rangle
=\frac{1+\|r_k\|^2+A^2g_k^2-(1+\mu_{v_k}(r_k,g_k)/4)^2}{2Ag_k}.
\end{equation}
For $g_k=0$, the response is independent of $v_k$. Thus the scripted failure experiment is simulable with at most one adaptive exact linear measurement per scalar evaluation. Let $s=m-T+1\ge512$ and $\ell_{\rm det}=\lfloor s/256\rfloor$. The noiseless list bound and the same gate-to-angle implication give fast-stage probability at most $1/16$. Consequently, if $N+1\le T\ell_{\rm det}/4$,
\[
\PP(\text{all stages finish})\le2^{-T}.
\]
The static mean/Haar coupling needs no sample collision term or final admissibility conditioning; hence
\[
\PP(\text{stationary success within $N$ evaluations})
\le2^{-T}+2(N+1)T\exp\left(-\frac{r-T}{32R^2}\right).
\]
For $N_0=\lfloor a_N d T\rfloor$ and a sufficiently small universal $a_N>0$, the stage budget is valid: since $s\ge d/4$ and $m\ge d/2$, the deterministic stage allowance satisfies
\[
\ell_{\rm det}=\lfloor s/256\rfloor
\ge\frac{s}{512}\ge\frac{d}{2048}
\quad\text{for $s\ge512$}.
\]
Taking, for example, $a_N\le1/16384$ (and assuming $T,d$ above the fixed numerical thresholds),
\[
N_0+1\le a_N d T+1\le\frac{dT}{8192}
\le\frac{T\ell_{\rm det}}4.
\]
The $+1$ pays for the unqueried candidate output gate test. The remaining condition $d\ge CT\log(2+T)$ makes the future-$U$ projection probability negligible by the same exponential calculation as in the stochastic proof. Under $d\ge CT\log(2+T)$ the future-projection event contributes at most $1/12$. Therefore success is strictly below $1/3$ at budget $N_0$, and the same averaging/truncation argument gives worst-case and worst-case expected scalar cost
\[
N=\Omega(dT)=\Omega\left(\frac{d\min\{\Delta H,L_0^2\}}{\epsilon^2}\right).
\]
Averaging over the hidden orientation and the algorithm's random
coins yields a fixed hard function with the stated expected-query cost.

Applying \eqref{eq:app-Goldstein-ineq} with $H=\epsilon/\delta$ and gradient tolerance $2\epsilon$ gives the noiseless Goldstein consequence
\[
N=\Omega\left(\frac{dB_G}{\epsilon^2}\right)
=\Omega\left(d\min\left\{\frac\Delta{\delta\epsilon},\frac{L_0^2}{\epsilon^2}\right\}\right),
\]
with its corresponding small-tolerance and dimension conditions. In particular, the unsaturated term is $\Omega(d\Delta/(\delta\epsilon))$, without the additional stochastic factor $L_0^2\epsilon^{-2}$.

\end{proof}

%% file: sections/gap_amplification_proofs.tex
\section{Common Block and Two-Level Gap Amplification}
\label[appendix]{app:linear-gap}

We use the scalar Gaussian fields of \eqref{eq:noise-definition},
the information theorem \Cref{thm:reuse}, and the unconditioned stage
argument in \Cref{lem:prefix-simulator,lem:stage-charging}. Their Gaussian proofs are not repeated. This appendix proves \Cref{thm:extended-gap}.
We verify the changed geometric and scalar interfaces below. Constants
denoted by $C$ are numerical and may increase between displays.
Throughout, $S_5$ is the clamped quintic used in
\eqref{eq:smooth-activation}.

\subsection{The one-sided source and a coarse first-entry bound}
\label{subsec:wide-source}

Let $t\ge4096$, $L=\sqrt t$, $R=2^{16}L$, and retain
$A=128$, $\beta=16$, $\rho_0=3/32$, $h$, $\chi$, and $\tau$
from \Cref{sec:geometry}. Take independent Haar $t$-frames $U$ and $V$
in $\R^r$ and $\R^m$, respectively. For $X=(z,x)$ set
\[
 \alpha=(1+\|z\|^2/R^2)^{-1/2},\qquad y=\alpha z,
 \quad w=U^Ty,\qquad P(z)=\frac{R^2}{5}(\alpha^{-1}-1).
\]
Replace only the chain and activation by the following one-sided versions:
\begin{equation}\label{eq:wide-source}
\begin{aligned}
 b_1&=-\Phi(w_1),\qquad b_i=-\Psi(w_{i-1})\Phi(w_i)\quad(i>1),\\
 B&=\sum_{i=1}^t b_i,\qquad g_i=\Gamma(w_i)=S_5(4w_i-1),\\
 a(w)&=A\sum_{i=1}^t g_iv_i,\qquad a_{<i}=A\sum_{j<i}g_jv_j,\\
 c_i&=b_i\tau(\|g_{\ge i}\|),\qquad
 C(w,x)=\sum_i[1-\chi(x-a_{<i})]c_i,\\
 N_j(w,x)&=\sum_i\tau(\|g_{>i}\|)Q_{j,i}(x-a_{<i},g_i),\\
 F^{\rm src}_j(X)&=P+B+h(x-a)-C+N_j.
\end{aligned}
\end{equation}
The functions $\Psi,\Phi$ are those of \Cref{sec:geometry}; in particular
$\Phi'(s)=\sqrt e\exp(-s^2/2)$. Here $e$ in an exponential denotes
Euler's constant. We use $\epsilon_0$ for a normalized stationarity
tolerance, avoiding a second use of $e$.

\begin{lemma}[One-sided source estimates and scalar interface]
\label[lemma]{lem:one-sided-source}
The source \eqref{eq:wide-source} satisfies the derivative and
correction bounds \eqref{eq:wide-source-estimates} below. On a
failed gate its only undisclosed current-stage contribution is
\eqref{eq:wide-current-channel}. For $K$ arrays its noise satisfies
\eqref{eq:wide-noise-maximum}, uniformly in the hidden frames.
\end{lemma}
\begin{proof}
The elementary estimates of \Cref{lem:chain-derivatives} give
\begin{equation}\label{eq:wide-source-estimates}
 q:=\nabla B\le0\ \hbox{coordinatewise},\quad
 \|q\|\le24L,\quad \|D^2B\|\le160,\quad
 \|D_wa\|\le960,\qquad J=961.
\end{equation}
At the first index $i$ with $w_i\le1$, every predecessor has
$w_{i-1}>1$ when $i>1$. The one-sided own derivative is therefore
$-\Phi'(w_i)$ for $i=1$, and $-\Psi(w_{i-1})\Phi'(w_i)\le-\Phi'(w_i)$
for $i>1$. Since the successor derivative is nonpositive,
\begin{equation}\label{eq:wide-chain-barrier}
 q_i\le-\Phi'(w_i)=-\exp((1-w_i^2)/2)
 \le\begin{cases}-1,&-1\le w_i\le1,\\0,&w_i<-1.\end{cases}
\end{equation}
The first inequality uses the adjacent-term derivative calculation
\eqref{eq:app-chain-sign} with the negative-predecessor term removed;
the last uses $w_i^2\le1$ in the first case and positivity of the
exponential in the second. At most one $c_i$ has a nonzero two-jet: a live
$b_i$, for $i>1$, requires $w_{i-1}>1/2$, hence $g_{i-1}=1$,
whereas a live suffix cutoff requires $\|g_{\ge i}\|<\rho_0$.
On its support,
\[
 \|a-a_{<i}\|\le12,\qquad |c_i|\le8,\qquad
 \|D_xC\|\le8(6/\beta)=3.
\]
The support-specific bound on $c_i$ follows from $g_i<\rho_0$,
which gives $w_i<5/16$, and $\Phi(w_i)<8/3$, $\Psi<11/4$.
If $\|x-a\|<1$, every correction two-jet is zero since
$\|x-a_{<i}\|<13<\beta$. All statements include cutoff endpoints.

\paragraph{Compatibility with the scalar experiment.}
At stage $k$ use the gate
\[
 w_k\ge\tfrac12,\qquad \|x-a_{<k}-Av_k\|<\sqrt2\,\beta.
\]
On $|w_i|<1/4$ for $i>k$, all future activations vanish. If
$w_k<1/2$, then $b_{k+1}=0$. Otherwise $g_k=1$, all $c_i$ with
$i\le k$ vanish, and $c_{k+1}=b_{k+1}$; a failed gate has
$\chi(x-a_{<k}-Av_k)=0$ and cancels this successor exactly.
After subtracting the revealed terms the failed response is therefore
\begin{equation}\label{eq:wide-current-channel}
 h(x-a_{<k}-Ag_kv_k)+Q_{j,k}(x-a_{<k},g_k).
\end{equation}
A crossing adds $\chi(x-a_k)b_{k+1}$, using the newly disclosed
$u_{k+1}$ but no $v_{k+1}$ or future Gaussian arrays. Thus the
birth filtration, reset bound, and stopped coupling in
\Cref{lem:prefix-simulator,lem:stage-charging,lem:static-coupling} apply with the new $R$.

The proof of \Cref{thm:noise-moment} first treats $g\in[0,1]^t$
as free coordinates. The new $\Gamma$ has the same bounded first
two derivatives. For any $K$ arrays, the same Gaussian nets, with
the array index included, give
\begin{equation}\label{eq:wide-noise-maximum}
 \E\max_{j\le K}\Lambda_j^2
 \le C\sigma^2\frac{m+t+\log(2Kt)}m,
 \qquad \Lambda_j=\Lip_{(w,x)}(N_j),\qquad
 |N_j(w,x)|\le\Lambda_j\|w\|.
\end{equation}
For the last inequality use $N_j(0,x)=0$. For the first, add
$\log K$ to the net log-cardinality in
\eqref{eq:app-global-chaining}; the summable increment bounds
and passage to the infinite field remain unchanged. The estimate
is uniform in the frames and also applies to independent blocks
for one retained label. 

\end{proof}

\begin{lemma}[Coarse first entry]\label[lemma]{lem:wide-coarse-entry}
Put $p_x=A\sum_{i=1}^t v_i$ and
$\mathcal P=\{(z,x):\|x-p_x\|<8\sqrt t\}$. Define
\begin{equation}\label{eq:wide-local-budget}
 h_t=\lfloor t/2\rfloor,\quad s=m-t+1,\quad
 \ell=\left\lfloor\frac{s^2\sigma^2}{2048m}\right\rfloor,
 \qquad n_* =\lfloor h_t\ell/8\rfloor.
\end{equation}
Suppose $m\ge2t$, $s\ge72$, $\ell\ge16$, $\log\ell\le s/8$, and
\begin{equation}\label{eq:wide-local-dimension}
 r-t\ge32R^2\log(128t(n_*+t)).
\end{equation}
In the ideal independent-label experiment, at most $n_*$ adaptive
candidate tests hit $\mathcal P$ with probability less than $1/16$.
Tests may include an unqueried final output and reuse any label.
Independent auxiliary information is allowed. The same conclusion
holds for a family whose evaluations outside $\mathcal P$ are public
affine postprocessing of at most one source evaluation at the same point.
\end{lemma}

\begin{proof}
Apply the reset argument of \Cref{app:assembly} to the first $h_t$
records of \eqref{eq:wide-current-channel}. Completing them in
$n_*$ tests has probability at most $2^{-h_t}$. Before their
completion, suppose $k<h_t$ records have finished. In addition
to the revealed past, condition on the current $v_{k+1}$; do not
condition on coarse misses or on a projection-good event. With
$j=k+1\le h_t$, the reference history is independent of the future
frame conditional on these data. In their orthogonal complement $E$,
\[
 A\sum_{i>j}v_i=A\sqrt{t-j}\,U_E,\qquad
 U_E\sim\operatorname{Unif}(\sph(E)),\qquad \dim E=m-j.
\]
Projecting a successful candidate onto $E$ requires guessing $U_E$
within radius
$8\sqrt t/(A\sqrt{t-j})\le\sqrt2/16<1/4$.
For a uniform point on $\sph^{d'-1}$, any Euclidean ball of radius
$\rho\le1/4$ has probability at most $\rho^{d'-1}$: its angular
radius $\theta$ obeys $\sin\theta\le\rho$, and polar integration gives
$c_{d'}\int_0^\theta\sin^{d'-2}u\,du
\le c_{d'}\sin^{d'-1}\theta/((d'-1)\cos\theta)
\le\rho^{d'-1}$, since $c_{d'}/(d'-1)\le1/2$.
Thus the conditional chance at each candidate is at most
$4^{-(m-h_t-1)}$.

Run the additional coarse test with scripted answer \emph{miss},
and couple only until its first actual hit. The unconditioned
reference has at most $n_*+t$ candidate/birth positions. The same
future-$U$ estimate as \eqref{eq:app-pbad} yields
\begin{equation}\label{eq:wide-entry-error}
 \PP(\text{first entry})\le
 2^{-h_t}+2t(n_*+t)e^{-(r-t)/(32R^2)}
 +(n_*+t)4^{-(m-h_t-1)}<\frac1{16}.
\end{equation}
Indeed the first two terms are at most $1/256$ and $1/64$.
Using $\ell\ge16$, $\log\ell\le(m-t+1)/8$, and $t\ge4096$ gives
$\log(n_*+t)\le\log(2t)+\log\ell\le m/4$;
also $m-h_t-1\ge m/2$. The last term is below $1/64$.
The floor estimates give $n_*\ge ctm\sigma^2$.
For postprocessing, answer every scripted miss with the public affine
transform of the source response and stop before returning the hitting
response. No assertion about the samples after entry is needed.
\end{proof}

\subsection{A common block with a wide quadratic core}

\begin{lemma}[A wide core with an exact public collar]\label[lemma]{lem:wide-core}
For the preceding source there are numerical constants
$B_0>1000$, $C_H,C_N,C_G$, and $\epsilon_0>0$, independent of
$t,m,r,\sigma$, and static blocks $G_j=g+Z_j$ on $\R^{r+m}$,
with $C^1$ samples and a $C^2$ coercive mean, satisfying the following.
The unique critical point and minimizer is $p=(z_*,p_x)$, with
$\|z_*\|<121\sqrt t$, $p_x=A\sum_i v_i$, and
\begin{equation}\label{eq:wide-core-interface}
\begin{gathered}
 \min g=0,\qquad t\le g(0)\le22t,\qquad
 \Lip(g)\le C_G\sqrt t,\quad \|D^2g\|\le C_H,\\
 \|\nabla g(X)\|\le\epsilon_0
 \ \Longrightarrow\quad
 \|X-p\|\le\|\nabla g(X)\|/\mu,\qquad \mu=1/100,\\
 G_j(X)=\frac\mu2\|X-p\|^2\quad(\|X-p\|\le L).
\end{gathered}
\end{equation}
Set $B_{\rm out}=B_0L$ and $a_{\rm c}=L/B_{\rm out}=1/B_0$.
Every sample equals the same public quadratic
\begin{equation}\label{eq:wide-public-collar}
 c_t+\frac{\mu a_{\rm c}}2\|X\|^2
 \quad\hbox{for }B_{\rm out}/2\le\|X\|\le B_{\rm out},
 \qquad 0<c_t\le C_Gt.
\end{equation}
Beyond $B_{\rm out}$ the samples have one common public coercive
continuation with gradient norm greater than $\epsilon_0$. Moreover,
\begin{equation}\label{eq:wide-compiled-noise}
 Z_j=0\quad(\|X\|\ge B_{\rm out}),\qquad
 \E\max_{j\le K}\Lip(Z_j)^2
 \le C_N\sigma^2\frac{m+t+\log(2Kt)}m.
\end{equation}
The constants and public functions depend on $t$, but not on the hidden
frames or labels. Outside $\mathcal P$, each $G_j(X)$ is a public
affine transform of at most one $F^{\rm src}_j(X)$ evaluation.
Consequently \Cref{lem:wide-coarse-entry} applies to its first entry.
\end{lemma}

\begin{proof}
\par\smallskip\noindent\textbf{Phase localization.}
Let $F_b=P+B(w)$. For any $q_{\rm eff}\le q$ coordinatewise with
$\|q_{\rm eff}\|\le24L+960(4+1)$,
\begin{equation}\label{eq:wide-phase-vector}
 V=\nabla P+(D_zw)^Tq_{\rm eff}
 =\alpha Uq_{\rm eff}
 +\left(\frac15-\frac\alpha{R^2}\langle w,q_{\rm eff}\rangle\right)y.
\end{equation}
For $\|z\|\ge R$ its outward component is at least
$R/(5\sqrt2)-\|q_{\rm eff}\|>1$. Inside that ball
$\alpha\ge1/\sqrt2$, and the coefficient of $y$ belongs to
$[.19,.21]$. At the first $w_i\le1$ the $u_i$ component is at
most $-.19$ when $w_i<-1$, or $-1/\sqrt2+.21$ otherwise.
Thus $\|V\|\le1/100$ implies $\|z\|<R$ and all $w_i>1$.

In $y$ coordinates, $P$ has Hessian at least $I/5$.
The one-sided chain is convex on $w_i\ge1$. To check this last
claim, for $s,u\ge1$ and $a=(2s-1)^{-2}$,
\[
 \Psi''(s)<0,\quad \Phi''(u)=-u\Phi'(u)<0,\qquad
 \frac{\Psi'(s)^2}{-\Psi''(s)\Psi(s)}
 =\frac{2a}{3-2a}\le2<\frac{u\Phi(u)}{\Phi'(u)}.
\]
Both diagonal Hessian entries of $-\Psi(s)\Phi(u)$ are positive
and this inequality makes its determinant positive. The last ratio
increases on $u\ge1$ and exceeds $2$ at $1$.
Therefore the coercive $F_b$ has a unique critical point $z_*=Uz_0$,
where $z_0$ is deterministic. Writing $w_*=U^Ty(z_*)$ and
$q_* =\nabla B(w_*)$, its equation is
\[
 \frac15(1-\|y_*\|^2/R^2)^{-3/2}y_*+Uq_*=0.
\]
It gives $\|y_*\|\le120L$, a radial coefficient below $.201$,
$\|z_*\|<121L$, and $w_{*i}>9/8$, $-q_{*i}>9/40$.
For if $1<w_{*i}\le9/8$, then
$-q_{*i}\ge\Phi'(9/8)\ge111/128>
.201(9/8)$, a contradiction.

Define
\[
 B_{\rm lin}(w)=B(w_*)+\langle q_*,w-w_*\rangle,
 \qquad F_{\rm lin}=P+B_{\rm lin}(w).
\]
Fix $\epsilon_0=2^{-28}$. For every $s_0\in[0,1]$,
\begin{equation}\label{eq:wide-phase-localization}
 \|\nabla[P+(1-s_0)B+s_0B_{\rm lin}]\|\le\epsilon_0
 \ \Longrightarrow\quad w_i>1\ (\forall i),\quad
 \|z-z_*\|<40\epsilon_0.
\end{equation}
Indeed $q_{s_0}=(1-s_0)q+s_0q_*$ is nonpositive and has norm
at most $24L$. The exterior estimate first gives $\|z\|<R$.
Inverting $Dy$ there bounds the $y$ gradient by $3\epsilon_0$,
so $\|y\|<121L$ and its radial coefficient is below $.201$.
At a first $w_i\le1$, the component is at most $-1/5$ if
$w_i<-1$, and at most $.201-9/40<-.024$ otherwise.
This proves $w_i>1$. Strong convexity on that convex region gives
$\|y-y_*\|\le15\epsilon_0$; the inverse soft map has differential
norm below $2$ along the joining segment. Also
\begin{equation}\label{eq:wide-tangent-bounds}
 B_{\rm lin}-B\le0\quad(w_i\ge1\ \forall i),\qquad
 \tfrac1{10}I\preceq D^2F_{\rm lin}\preceq I
 \quad(\|z-z_*\|\le2L).
\end{equation}
For the second assertion use $\|z\|<123L$, the least eigenvalue
$\alpha^3/5$ of $D^2P$, and the bound
\[
 \|D^2\langle q_*,U^Ty\rangle\|\le2\|q_*\|/R\le48/65536.
\]

\par\smallskip\noindent\textbf{One scalar formula for the core.}
Put $r_x=\|x-p_x\|$, $h_a=h(x-a(w))$, $h_p=h(x-p_x)$, and
\begin{align*}
 \psi&=1-S_5((r_x-7L)/L),&\theta&=1-S_5((r_x-5L)/L),\\
 \eta&=1-S_5((r_x-3L)/L),&\zeta&=1-S_5((r_x-L)/L),\\
 \chi_z&=1-S_5((\|z-z_*\|-L)/L),&
 E_z&=\chi_z\left[F_{\rm lin}(z_*)+
          \tfrac\mu2\|z-z_*\|^2-F_{\rm lin}(z)\right].
\end{align*}
Here $E_z\le0$ by \eqref{eq:wide-tangent-bounds}. Set
$d_t=(4A+1)L$, $e_t=2\mu t$, and define
\begin{equation}\label{eq:wide-core-formula}
\begin{aligned}
 \widehat f={}&P+B+h_a-(1-\psi)C
       +\theta(h_p-h_a-d_t)+\eta(B_{\rm lin}-B)\\
       &+\zeta[E_z+\mu r_x^2/2-h_p-e_t],\\
 \widehat F_j={}&\widehat f+(1-\psi)N_j.
\end{aligned}
\end{equation}
Outside $r_x<8L$ this equals the source sample exactly.

\par\smallskip\noindent\textbf{Excluding small gradients in every core transition.}
For $r_x\ge6L$ only the removal of $C$ can be active.
Its extra derivative $C\nabla\psi$ has norm at most $16/L\le1/4$.
If $D_xC\ne0$, then $\|x-a\|>4$, giving
$\|\nabla_x\widehat f\|\ge16/\sqrt{17}-3-1/4>1/2$.
If $D_xC=0$ and $\|x-a\|\ge1$, the bound is
$4/\sqrt2-1/4>1$. A full gradient at most $\epsilon_0$ would
therefore force $\|x-a\|<1$, where all correction jets vanish.
Adding $(D_za)^T\nabla_x\widehat f$ to $\nabla_z\widehat f$
gives $\|\nabla F_b\|\le J\epsilon_0<1/100$.
Equation \eqref{eq:wide-phase-vector} forces $a=p_x$, contradicting
$r_x\ge6L$.

For $5L\le r_x\le6L$, write $u=x-p_x$, $\delta_a=p_x-a$,
$d=u+\delta_a$, and
\[
 c=\frac{4(1-\theta)}{\sqrt{1+\|d\|^2}},\qquad
 b_x=\frac{4\theta}{\sqrt{1+r_x^2}}
       +\frac{\theta'}{r_x}(h_p-h_a-d_t).
\]
Since $h_p-h_a-d_t\le-L$, both coefficients are nonnegative
and $c+b_x>0$. With $v_x=\nabla_x\widehat f=cd+b_xu$ and
$\tau_*=c/(c+b_x)$,
\begin{align*}
 cd&=\tau_*v_x+\frac{cb_x}{c+b_x}\delta_a,\\
 (D_wa)^T\delta_a&=A^2(\Gamma'(w_i)(1-g_i))_i\ge0,\\
 q_{\rm eff}&=q-\frac{cb_x}{c+b_x}(D_wa)^T\delta_a\le q,\\
 \nabla_z\widehat f&=\nabla P+(D_zw)^Tq_{\rm eff}
                 -\tau_*(D_za)^Tv_x.
\end{align*}
If the full gradient is at most $\epsilon_0$, then
$\|\frac{cb_x}{c+b_x}\delta_a\|\le4+\epsilon_0$.
Equation \eqref{eq:wide-phase-vector} applies and forces $a=p_x$.
The workspace radial derivative is then
$h'(r_x)-d_t\theta'(r_x)\ge h'(r_x)>\epsilon_0$, a contradiction.

For $2L\le r_x\le5L$, the phase gradient is that of
$P+(1-\eta)B+\eta B_{\rm lin}$. A small gradient implies $w_i>1$
by \eqref{eq:wide-phase-localization}; hence
$\partial_{r_x}\widehat f=h'(r_x)+\eta'(B_{\rm lin}-B)
\ge h'(r_x)>\epsilon_0$. This covers the flat gaps between bands too.

For $r_x\le2L$, the phase is $F_{\rm lin}+\zeta E_z$.
Let $d_z=\|z-z_*\|$ and
$Q_z=F_{\rm lin}(z_*)+\mu d_z^2/2$. For $d_z\le2L$ and
$s_0\in[0,1]$,
\[
 \partial_{d_z}(F_{\rm lin}+s_0E_z)
 =(1-s_0\chi_z)\partial_{d_z}F_{\rm lin}
   +s_0\chi_z\mu d_z+s_0\chi_z'(Q_z-F_{\rm lin})
 \ge\mu d_z.
\]
For $d_z>2L$, $E_z$ vanishes and
\eqref{eq:wide-phase-localization} excludes a small phase gradient.
Moreover $E_z\le0$ and $\mu r_x^2/2\le e_t$ imply
\[
 \partial_{r_x}\widehat f
 =(1-\zeta)h'(r_x)+\zeta\mu r_x
   +\zeta'[E_z+\mu r_x^2/2-h_p-e_t]
 \ge\min\{h'(r_x),\mu r_x\}.
\]
Thus a full gradient at most $\epsilon_0$ forces
$d_z,r_x\le\epsilon_0/\mu$, placing $X$ in the exact quadratic
ball, where $\|\nabla\widehat f\|=\mu\|X-p\|$.
This proves uniqueness of the critical point and the stronger
gradient-dependent certificate in \eqref{eq:wide-core-interface}.

\par\smallskip\noindent\textbf{Gap and budgets before the public completion.}
All changes to the source mean are uniformly $O(t)$, so coercivity
is preserved. At its minimum,
$\widehat f(p)=F_b(z_*)-d_t-e_t$ is public by orthogonal invariance.
Set $f_0=\widehat f-\widehat f(p)$. Since the origin has
$r_x=AL>8L$,
$f_0(0)=F_b(0)-F_b(z_*)+d_t+e_t$.
Testing $w_i=1$ gives $t\le F_b(0)-F_b(z_*)\le12t$;
the upper bound also follows from $B\ge-12t$, $P\ge0$.
As $d_t/t=513/L\le513/64$ and $e_t/t=1/50$,
$t\le f_0(0)\le22t$.

The cutoff contrasts in \eqref{eq:wide-core-formula} have
(value, gradient, Hessian) orders
$C:(1,1,1)$, $h_p-h_a-d_t:(L,1,1)$,
$B_{\rm lin}-B:(t,L,1)$, $E_z:(t,L,1)$, and
$\mu r_x^2/2-h_p-e_t:(t,L,1)$ on their cutoff supports.
For the linear contrast use $\|w\|\le R$, $\|q_*\|\le24L$;
for $E_z$ use its matching value and gradient at $z_*$ and
\eqref{eq:wide-tangent-bounds}. Workspace cutoff derivatives are
$O(1/L)$ and $O(1/t)$. The product rule yields
$\Lip(f_0)\le CL$, $\|D^2f_0\|\le C$.
For each array, \eqref{eq:wide-noise-maximum} gives
\begin{equation}\label{eq:wide-noise-cutoff}
 \Lip_{(z,x)}((1-\psi)N_j)
 \le(1+2R/L)\Lambda_j.
\end{equation}
Here $|N_j(w,x)|\le R\Lambda_j$ and the soft map is a contraction.
The factor is numerical, independent of $t$.

\par\smallskip\noindent\textbf{An explicit public completion.}
Set
\[
 R_z=16R,\quad R_x=512L,\quad S=4R_z,\quad
 B_{\rm out}=8S=B_0L,\quad a_{\rm c}=L/(8S).
\]
With fixed sufficiently large $C_M$, define the public function and cutoff
\[
\begin{gathered}
 M(X)=f_0(0)+P(z)+h(x)+C_Mt,\\
 \chi_{\rm rec}(X)=
 [1-S_5((\|z\|-R_z)/R_z)]
 [1-S_5((\|x\|-R_x)/R_x)].
\end{gathered}
\]
The preceding contrast bounds and $|h(x-a)-h(x)|\le4AL$ give
$0<M-f_0\le Ct$. Define the rectangular sample
\[
 G_{1,j}=\chi_{\rm rec}[f_0+(1-\psi)N_j]+(1-\chi_{\rm rec})M.
\]
On $\|z\|\ge R_z$, $E_z=0$ and the remaining nonpublic terms
factor through $w$, with $w$ gradient bounded by $24L+9147$.
Their radial derivative is attenuated by $\alpha^3$, giving
\[
 \partial_{\|z\|}f_0
 \ge\alpha\|z\|/5-(24L+9147)\alpha^3>4L.
\]
On $\|x\|\ge R_x$, $r_x\ge384L$, so every new hidden modification
and every workspace derivative of the original correction vanish.
Since $\|a\|\le128L$, the outward derivative of $h(x-a)$ exceeds
$2$. The same weaker positive bounds hold for $M$.
Both coordinate radial derivatives of $\chi_{\rm rec}$ are nonpositive;
the product term $(f_0-M)\partial\chi_{\rm rec}$ is nonnegative.
Thus the rectangular transition cannot contain an
$\epsilon_0$-stationary point. It preserves the origin and the core.
Its mean has gradient $O(L)$ and Hessian $O(1)$; its noise has
Lipschitz constant at most $C\Lambda_j$ by
\eqref{eq:wide-noise-cutoff} and $\|\nabla\chi_{\rm rec}\|=O(1/L)$.

Write $r_X=\|X\|$. For $r_X\ge S$ the rectangular cutoff is zero,
so every sample is already $M$. Its total radial derivative is
at least $2$: either $\|z\|\ge r_X/\sqrt2$ and the $P$ contribution
suffices, or $\|x\|\ge r_X/\sqrt2$ and the $h$ contribution does.
Choose $L_{\rm out}=C_{\rm out}L>\Lip(M)$ and put
\[
 U_{\rm out}(r_X)=M(0)+L_{\rm out}r_X,\qquad
 \chi_S(r_X)=1-S_5((r_X-S)/S).
\]
On $[S,2S]$ use the public scalar function
$\chi_SM+(1-\chi_S)U_{\rm out}$. Since $U_{\rm out}\ge M$,
\[
 \partial_{r_X}[\chi_SM+(1-\chi_S)U_{\rm out}]
 =\chi_S\partial_{r_X}M+(1-\chi_S)L_{\rm out}
       +\chi_S'(M-U_{\rm out})\ge2.
\]
For $2S\le u\le4S$ define the positive derivative
\[
 p_{\rm out}(u)=[1-S_5((u-2S)/(2S))]L_{\rm out}
              +S_5((u-2S)/(2S))\mu a_{\rm c}u,
\]
and set
\[
 c_t=U_{\rm out}(2S)+\int_{2S}^{4S}p_{\rm out}(u)\,du
                         -\frac{\mu a_{\rm c}}2(4S)^2.
\]
For $u\ge B_{\rm out}$ use
\[
 p_{\rm ext}(u)=\mu a_{\rm c}
 \left[B_{\rm out}+\int_{B_{\rm out}}^u
       (1-S_5((v-B_{\rm out})/B_{\rm out}))\,dv\right].
\]
The complete sample formula is
\begin{equation}\label{eq:wide-public-formula}
G_j(X)=
\begin{cases}
 G_{1,j}(X),&r_X\le S,\\
 \chi_S(r_X)M(X)+(1-\chi_S(r_X))U_{\rm out}(r_X),
       &S\le r_X\le2S,\\
 U_{\rm out}(2S)+\displaystyle\int_{2S}^{r_X}p_{\rm out}(u)\,du,
       &2S\le r_X\le4S,\\
 c_t+\mu a_{\rm c}r_X^2/2,&4S\le r_X\le B_{\rm out},\\
 c_t+\mu a_{\rm c}B_{\rm out}^2/2
       +\displaystyle\int_{B_{\rm out}}^{r_X}p_{\rm ext}(u)\,du,
       &r_X\ge B_{\rm out}.
\end{cases}
\end{equation}
Flat quintic endpoints match the first two jets. The derivative on
$[2S,\infty)$ is at least $\mu L/4>\epsilon_0$.
Beyond $2B_{\rm out}$ it is constant and positive, and everywhere
it is $O(L)$ with derivative $O(1)$. The transition on $[S,2S]$
also has these bounds because its contrast is $O(Lr_X)$.
Since $U_{\rm out}(2S)\ge2L_{\rm out}S$ and the subtracted
quadratic is $\mu LS$, taking $C_{\rm out}>\mu$ ensures
$0<c_t\le C_Gt$. This proves coercivity, the public collar, and
all deterministic budgets without appealing to an unspecified extension.

The only random part of \eqref{eq:wide-public-formula} is
$\chi_{\rm rec}(1-\psi)N_j$. Equation
\eqref{eq:wide-noise-maximum} gives \eqref{eq:wide-compiled-noise}.
Outside $\mathcal P$, for $r_X\le S$, the formula is exactly
\[
 G_j(X)=\chi_{\rm rec}(X)
          [F^{\rm src}_j(X)-\widehat f(p)]
          +(1-\chi_{\rm rec}(X))M(X);
\]
for $r_X\ge S$ it is public. Every coefficient and additive constant
here is public. This proves the scalar simulation assertion.
The unchanged core contains $p$, every transition has positive
gradient bounded away from zero, and $g(p)=0$; all conclusions follow.
\end{proof}

\paragraph{Exact-collar use.}
The collar in \eqref{eq:wide-public-collar} has a fixed relative
width, and $a_{\rm c}=1/B_0$ is independent of $t$. When two
rescaled pieces agree there, equality holds for each retained label
on an open overlap, including their first two jets wherever defined.
Gluing along that overlap adds no cutoff derivative term.

\begin{lemma}[Charging successive first entries]\label[lemma]{lem:gap-renewal}
Suppose $K$ blocks are processed in order. At every actual birth,
conditional on the preceding history, completion within $n_*$ local
candidate tests has probability at most $1/16$. A crossing candidate
is tested before its response and can start the next block at the same
point. Assume that each paid query or final output is noncrossing for
at most one current block. If $N+1\le Kn_*/4$, the probability of
completing all blocks in $N$ queries and one output is at most $1/16$.
\end{lemma}
\begin{proof}
Let $n_i$ count noncrossing tests of block $i$. A completed block uses
$n_i+1$ tests, including its crossing; hence
\[
 \sum_i n_i\le N+1,\qquad
 \#\{i\text{ completed}:n_i+1>n_*\}\le\frac{N+1}{n_*}\le K/4.
\]
The first inequality charges a noncrossing test to its unique paid
query or output. For the second, each counted block has $n_i\ge n_*$.
Let $I_i$ indicate birth and completion within $n_*$ tests. Set $I_i=0$
if the block is never born or never completed. Earlier indicators are
measurable at a later birth. Repeated conditional expectation, starting
with the largest index in a fixed subset $A$, therefore gives
\[
 \E\prod_{i\in A}I_i\le16^{-|A|}.
\]
On completion of all $K$ blocks, at least $\lceil3K/4\rceil$ are fast.
For $K\ge2$, the union bound gives
\[
 \PP(\text{all complete})
 \le\binom K{\lceil3K/4\rceil}16^{-\lceil3K/4\rceil}
 \le2^K16^{-3K/4}=2^{-2K}\le1/16.
\]
Here $\binom Kj\le2^K$ and $\lceil3K/4\rceil\ge3K/4$ justify
the second inequality. For $K=1$, completion must be fast and the
birth hypothesis gives the result directly.
\end{proof}

\subsection{The two-level static family}
\label{subsec:linear-family}
\begin{proposition}[Normalized two-level family]\label[proposition]{prop:two-level-assembly}
For $t\ge4096$, $\sigma^2\ge1$ and the dimension allocation and
condition \eqref{eq:linear-dimension-normalized}, the samples
\eqref{eq:linear-samples} have a coercive $C^2$ mean with gap
$O(Kt)$, Hessian bound $C_H$, and the moment bound
\eqref{eq:linear-global-budgets}. Their stationary outputs require
completion of the outer chain. In the independent-label experiment,
the allowed budget $N\ge c dKt\sigma^2$ has the early-success
bound \eqref{eq:linear-early-success}, for a universal $c>0$.
\end{proposition}
\begin{proof}

Use coordinates $(Y,Z)\in\R^{r_o}\times\R^{r_i+m}$, $Z=(z,x)$,
and take $K$ independent length-$t$ blocks
$G_{i,j}=g_i+N_{i,j}$ from \Cref{lem:wide-core}, all on the same
$Z$ space. Here $N_{i,j}$ is the corresponding copy of the
cutoff-modified noise $Z_j$ in \eqref{eq:wide-compiled-noise}.
Let $p_i=(z_{*,i},p_{x,i})$ be their minima.
Their public exteriors are identical. Let $g_{K+1}$ be one fixed
public block with that same exterior and set $N_{K+1,j}=0$.
The common boundary and the Lipschitz bound imply
\begin{equation}\label{eq:linear-inner-interfaces}
 \|g_i-g_{K+1}\|_\infty\le Ct,\qquad
 \E\max_{i\le K}\Lip(N_{i,j})^2
 \le C\sigma^2\frac{m+t+\log(2Kt)}m.
\end{equation}
For the first bound, inside the common boundary compare to a boundary
point at distance $O(L)$; outside it the difference is zero.
For the second apply \eqref{eq:wide-compiled-noise} to the
independent blocks for label $j$.

Independently draw a Haar $K$-frame $U^o$ in $\R^{r_o}$. Define
\begin{align*}
 v&=Y/t,&R_o&=C_R\sqrt K,&S_o&=C_S\sqrt K,\\
 w&=(U^o)^T\frac{v}{\sqrt{1+\|v\|^2/R_o^2}},&
 P_o(v)&=\frac{S_o^2}{5}\left(\sqrt{1+\|v\|^2/S_o^2}-1\right),\\
 b_1&=-\Phi(w_1),&b_i&=-\Psi(w_{i-1})\Phi(w_i)\quad(i>1),&
 B&=\sum_{i=1}^K b_i,\\
 a_i&=S_5(4w_i-1),&s_i&=\sum_{\ell=i}^K a_\ell^2.
\end{align*}
The constants $C_R,C_S$ will be chosen below. Define two staggered
cutoffs, with their numerical thresholds displayed directly:
\begin{align*}
 T_1(s)&=1-S_5\left(\frac{s-64^{-2}}{32^{-2}-64^{-2}}\right),&
 T_2(s)&=1-S_5\left(\frac{s-16^{-2}}{8^{-2}-16^{-2}}\right),\\
 d_i(w)&=b_i(w)T_1(s_i),&
 \alpha_i&=T_2(s_i),\quad\alpha_0=0,\quad\alpha_{K+1}=1,\\
 \omega_i&=\alpha_i-\alpha_{i-1},&
 \chi_i(Z)&=1-S_5(2\|Z-p_i\|/L-1).
\end{align*}
Thus $\omega_i\ge0$, $\sum_i\omega_i=1$. Correction ends at
suffix norm $1/32$, strictly before ownership changes at $1/16$.
The support of $\chi_i$ lies in
$\mathcal P_i=\{Z:\|x-p_{x,i}\|<8L\}$.
For a fixed sufficiently large $C_0$ the samples are
\begin{equation}\label{eq:linear-samples}
 F_j(Y,Z)=C_0t\left[P_o(Y/t)+B(w)
                  -\sum_{i=1}^K(1-\chi_i(Z))d_i(w)\right]
           +\sum_{i=1}^{K+1}\omega_i(w)G_{i,j}(Z).
\end{equation}
Every label specifies one static function. The mean $f$ is obtained
by replacing $G_{i,j}$ with $g_i$.

\paragraph{Ownership derivatives without a factor $K$.}
For scalar coefficients $A_i$ with $|A_i|\le M$ write
\[
 W(a)=\sum_i\omega_iA_i
      =A_{K+1}+\sum_{j=1}^KT_2(s_j)(A_j-A_{j+1}).
\]
Put $A'_n=\sum_{j\le n}(A_j-A_{j+1})T_2'(s_j)$ and
$B'_n=\sum_{j\le n}(A_j-A_{j+1})T_2''(s_j)$. Then
\[
 \partial_{a_l}W=2a_lA'_l,\qquad
 \partial^2_{a_ka_l}W=2\mathbf1_{k=l}A'_l
                         +4a_ka_lB'_{\min(k,l)}.
\]
Since $s_j$ decreases with $j$, summation by parts bounds
$|A'_n|+|B'_n|\le CM$, using the finite total variations of
$T_2'$ and $T_2''$. If $j_0$ is the first index where either
derivative is nonzero, all derivative terms are supported on the
suffix with $\sum_{l\ge j_0}a_l^2=s_{j_0}\le1/64$.
Cauchy--Schwarz therefore gives
\begin{equation}\label{eq:linear-owner-derivatives}
 \|\nabla_a W\|+\|D_a^2W\|\le CM.
\end{equation}
If no such index exists both derivatives vanish. Composition with
the coordinatewise quintic preserves this estimate. The first-derivative
argument also holds for vector coefficients with norm at most $M$.

Let $M(w,Z)=\sum_i\omega_i(w)g_i(Z)$. Apply
\eqref{eq:linear-owner-derivatives} to $g_i-g_{K+1}$ and their
$Z$ gradients; subtracting the common public function is essential
because the individual $g_i$ are unbounded. It gives
\begin{equation}\label{eq:linear-mean-mixture}
 \|\nabla_wM\|+\|D_w^2M\|\le Ct,\quad
 \|D_wD_ZM\|\le CL,\quad
 \|\nabla_ZM\|\le CL,\quad\|D_Z^2M\|\le C.
\end{equation}
The last two estimates also follow directly from the convex weights.
At most one $d_i$ has a nonzero two-jet: a live $b_i$, $i>1$,
forces $a_{i-1}=1$, while a live $T_1(s_i)$ requires
$s_i<1/1024$. This also bounds its value and first two derivatives
by numerical constants. Flat endpoints are included.

The product rule, $L^2=t$, and $Y=tv$ give
\begin{equation}\label{eq:linear-global-budgets}
\begin{gathered}
 \Lip(f)\le C(\sqrt K+\sqrt t),\qquad \|D^2f\|\le C_H,\\
 \E\Lip(F_j)^2\le C_N\left[K+t+
                 \sigma^2\frac{m+t+\log(2Kt)}m\right].
\end{gathered}
\end{equation}
In particular the correction has $ZZ,YZ,YY$ Hessian orders
$C_0t/L^2$, $C_0/L$, $C_0/t$, respectively. The outer chain
has Hessian $O(C_0/t)$ since its $w$ gradient is $O(\sqrt K)$
and the compactifier's second derivative is $O(1/R_o)$.
For noise, put $\Lambda_j=\max_i\Lip(N_{i,j})$.
Its $Z$ gradient has norm at most $\Lambda_j$, while the common
zero exterior implies $\|N_{i,j}\|_\infty\le2B_0L\Lambda_j$.
Equation \eqref{eq:linear-owner-derivatives} bounds its $Y$ gradient
by $C\Lambda_j/L$. Apply \eqref{eq:linear-inner-interfaces}.
This is one global sample Lipschitz estimate; there is no sum of
$K$ layerwise noise budgets.

\paragraph{Every small-gradient point completes the outer chain.}
For fixed $Z$ let
$q_w=\nabla_w[B-\sum_i(1-\chi_i)d_i]$. The signs of the
one-sided chain, the decreasing cutoff, and the increasing activations give
\begin{equation}\label{eq:linear-outer-sign}
\begin{aligned}
 q_w&=\sum_i[1-(1-\chi_i)T_1(s_i)]\nabla b_i
       -\sum_i(1-\chi_i)b_iT_1'(s_i)\nabla s_i\le0,\\
 \|q_w\|&\le C\sqrt K,\qquad
 |\langle w,q_w\rangle|\le CK.
\end{aligned}
\end{equation}
The pairing bound uses bounded $s\Phi'(s)$ and $s\Psi'(s)$.
The single live suffix derivative has nonzero coordinates only at
$w_l\in[1/4,1/2]$, and its pairing is bounded by
$C\sum_{l\ge i}|w_la_la_l'|\le C\sqrt K\sqrt{s_i}$.
These are global bounds, including unbounded query points.

Choose $C_0$ large enough that
$E=\nabla_wM/(C_0t)$ has norm at most $1/100$. Then
\begin{equation}\label{eq:linear-outer-gradient}
 C_0^{-1}\nabla_Y f=\nabla P_o(v)+(D_vw)^T(q_w+E).
\end{equation}
At $r_v=\|v\|>0$ its radial component is at least
$P_o'(r_v)-CK/r_v-1/100$. Choose a numerical $C_z$, and then
$C_S$ sufficiently large, so this exceeds $1$ whenever
$r_v\ge C_z\sqrt K$. Thus small gradients lie inside that ball.
Inside it, with $\alpha=(1+\|v\|^2/R_o^2)^{-1/2}$,
\[
 \nabla P_o+(D_vw)^Tq_w=\alpha U^oq_w+c_vv,\qquad
 c_v=\frac1{5\sqrt{1+\|v\|^2/S_o^2}}
                   -\frac{\alpha^2\langle w,q_w\rangle}{R_o^2}.
\]
Increasing $C_S$ relative to $C_z$ and then choosing $C_R$ gives
$\alpha\ge.99$ and $.19\le c_v\le.21$, uniformly in $K,t$.

Suppose $i$ is the first coordinate with $w_i\le1$.
If $s_i\ge1/256$, all correction jets with index $j\le i$ vanish.
A later correction can remove only a fraction of the successor's
contribution to coordinate $i$, never its own contribution
$\partial_i b_i$. Hence \eqref{eq:wide-chain-barrier} still gives $q_{w,i}\le-1$
for $-1\le w_i\le1$ and $q_{w,i}\le0$ for $w_i<-1$. The $u_i^o$ component of
$\alpha U^oq_w+c_vv$ is at most
$-.99+.21/.99<-.77$ in the former case and at most $-.19$
in the latter. The perturbation has norm at most $.01$.

If instead $s_i<1/256$, all earlier coordinates have activation one,
so $\omega_i=1$ and every other owner weight and all ownership
derivatives are zero. Only $d_i$ can have a nonzero two-jet.
The $Z$ gradient is exactly
$\nabla g_i+C_0td_i\nabla\chi_i$. Outside $\|Z-p_i\|\le L$,
the correction is zero and \Cref{lem:wide-core} excludes a small
gradient. Inside that ball it is
\[
 \nabla_Zf=
 \left[\mu+C_0td_i\frac{\chi_i'(\rho)}\rho\right](Z-p_i),
 \qquad \rho=\|Z-p_i\|,
\]
with its value at zero given by continuity. Since $d_i\le0$ and
$\chi_i'\le0$, the bracket is at least $\mu$.
Our choice $\epsilon_0=2^{-28}$ is below $\mu/4$ and $1/100$.
A full gradient at most $\epsilon_0$ forces
$\rho\le\epsilon_0/\mu<L/2$, where $\chi_i=1$ and every
correction jet vanishes. Now $E=0$, and the preceding coordinate
argument again contradicts smallness. The flat endpoints cover
$s_i=1/256$. We have proved
\begin{equation}\label{eq:linear-stationary}
 \|\nabla f(Y,Z)\|\le\epsilon_0
       \quad\Longrightarrow\quad w_1,\ldots,w_K>1.
\end{equation}
There the only owner is the public terminal block $K+1$.
No alignment of gradients of different inner blocks was assumed.

\paragraph{Coercivity and initial gap.}
All $g_i\ge0$, $B\ge-CK$, and the subtracted corrections are
nonnegative. Thus $f\ge-CKt$. As $\|Y\|\to\infty$, $P_o(Y/t)$
is coercive; as $\|Z\|\to\infty$, every inner block equals the same
coercive public exterior. These bounds also handle simultaneous escape.
At the origin $\chi_1=0$, $w=0$, and $b_1$ cancels exactly, giving
$f(0,0)=g_1(0)\le22t$. Conversely take $w_i=1$ for all $i$,
$v=U^o\mathbf1/\sqrt{1-K/R_o^2}$, and $Z=p_{K+1}$.
All corrections vanish, $B=-K\Phi(1)$, and
$P_o(v)\le\|v\|^2/10$. Therefore
\begin{equation}\label{eq:linear-gap-budget}
 cKt\le f(0,0)-\inf f\le C_GKt.
\end{equation}

\subsection{Scalar simulation and accumulated information cost}

Process each query and the final output from the first uncompleted
block $i$. Test $Z\in\mathcal P_i$ before returning any scalar. A hit
completes $i$ and discloses its frames and the outer direction $u_i^o$.
Test the next block at the same point. These extra disclosures strengthen
the algorithm; they impose no restriction on its query locations.

First use independent label arrays, retained on every reuse. In a
reference process set undisclosed outer coordinates to zero in all
hidden terms, leaving the public $P_o(Y/t)$ unchanged. Before a coarse
hit of block $i$, $\chi_i=0$. If the actual undisclosed coordinates
satisfy $|w_l|\le1/4$ for $l\ge i$, then $a_l=0$ for those indices.
Consequently all owner weights beyond $i$ are zero; $s_i=0$, so
$d_i=b_i$ cancels the entire $b_i$ term; and $b_l=d_l=0$ for $l>i$.
The earlier weights and corrections depend only on the disclosed
prefix. The response is therefore exactly
\begin{equation}\label{eq:linear-scalar-reference}
 A_{i,j}(Y,Z)+\omega_i(w)G_{i,j}(Z),
\end{equation}
where $A_{i,j}$ uses only completed data and public parameters. In
particular it is independent of $u_i^o,\ldots,u_K^o$ and of all later
inner frames and arrays. If $\omega_i=0$, no source query is needed;
otherwise \Cref{lem:wide-core} simulates the response before entry
using one source evaluation. Counting both cases as tests only
weakens the lower bound.

Conditional on the disclosed outer prefix, the residual outer Haar
frame is independent of this reference history. Its compactified query
has norm at most $R_o$. Apply the spherical tail bound used in
\eqref{eq:app-pbad} to these unconditioned reference queries, including
the output and birth positions. A union bound gives
\begin{equation}\label{eq:linear-outer-error}
 p_{\rm out}\le2K(N+K+1)
       \exp\left(-\frac{r_o-K-1}{32R_o^2}\right).
\end{equation}
Couple the actual and reference transcripts until the first projection
violation. Equation \eqref{eq:linear-scalar-reference} proves equality
up to that time.

Before block $i$ is born, both reference responses and membership
notifications are pointwise independent of its frames and all its
label-indexed arrays. The candidate and selected label at birth have
already been chosen, before the membership test. Hence that block
has its original independent prior at birth, including when its label
has been used in other blocks or several blocks are crossed at one
point. In a single-block reduction the completed arrays simulate
$A_{i,j}$ as independent auxiliary information. They need not be
disclosed in the actual finite-label experiment. It follows from
\Cref{lem:wide-coarse-entry} that the conditional probability of a
completion within $n_*$ local tests is at most $1/16$.

Allocate $m,r_i,r_o$ comparable shares of the total dimension $d$.
There is a numerical $C$ such that
\begin{equation}\label{eq:linear-dimension-normalized}
 d\ge C(t+K)\log(2+Kt\sigma^2)
\end{equation}
ensures the hypotheses of \Cref{lem:wide-coarse-entry},
$m\ge t+\log(2Kt)$, $n_*\ge16$, and $p_{\rm out}\le1/64$ for
$N=\lfloor Kn_*/8\rfloor$. To verify this for all larger dimensions,
write $u=d/(t+K)$. The logarithms of $t(n_*+t)$ and
$K(N+K+1)$ are at most $C\log(2+Kt\sigma^2)+\log u$, since
$n_*\le Ctm\sigma^2$. The dimension exponents in
\eqref{eq:wide-local-dimension} and \eqref{eq:linear-outer-error}
are at least $cu$. Choosing the constant in
\eqref{eq:linear-dimension-normalized} sufficiently large makes
$cu-\log u$ dominate both logarithms. The requirements on $\ell$
in \eqref{eq:wide-local-budget} follow from $s\asymp m$ and
$\sigma^2\ge1$. In particular $n_*\ge ctd\sigma^2$.

For this choice of $N$, $N+1\le Kn_*/4$.
\Cref{lem:gap-renewal} bounds reference completion by $1/16$.
Until a projection violation, an incomplete output has an undisclosed
outer coordinate with absolute value at most $1/4$, contradicting
\eqref{eq:linear-stationary}. Therefore
\begin{equation}\label{eq:linear-early-success}
 \PP(\|\nabla f(\widehat X)\|\le\epsilon_0)
 \le1/16+p_{\rm out}\le1/16+1/64.
\end{equation}
This includes arbitrary transcript-measurable outputs.

\end{proof}

\subsection{Physical scaling and a fixed finite family}
\begin{proof}[Proof of \Cref{thm:extended-gap}]
\phantomsection\label{proof:two-level-bound}
We prove \Cref{thm:extended-gap}. Choose a small numerical $\kappa>0$
and a large numerical $C_*>0$, and put
\begin{equation}\label{eq:linear-physical-choice}
\begin{gathered}
 t=\left\lfloor\kappa\frac{L_0^2}{\epsilon^2}\right\rfloor,
 \quad \sigma^2=\kappa\frac{L_0^2}{\epsilon^2},\quad
 K=\left\lfloor\min\left\{\kappa\frac{L_0^2}{\epsilon^2},
                    \frac{\Delta H}{C_*L_0^2}\right\}\right\rfloor,\\
 \lambda=\frac{C_H\epsilon^2}{\epsilon_0^2H},
 \qquad b=\frac{C_H\epsilon}{\epsilon_0H},
 \qquad \widetilde F_j(X)=\lambda F_j(X/b).
\end{gathered}
\end{equation}
Choose $c_e$ small enough that $\epsilon\le c_eL_0$ implies
$t\ge4096$. The theorem's gap condition ensures $K\ge1$.
For $\widetilde f=\E\widetilde F_j$, the chain rule and
\eqref{eq:linear-global-budgets} give
\begin{equation}\label{eq:linear-scaled-budgets}
\begin{aligned}
 \Lip(\nabla\widetilde f)&\le\lambda C_H/b^2=H,\\
 \|\nabla\widetilde f(X)\|\le\epsilon
 &\ \Longrightarrow\ \|\nabla f(X/b)\|\le b\epsilon/\lambda
        =\epsilon_0,\\
 \E\Lip(\widetilde F_j)^2
 &\le C\epsilon^2\epsilon_0^{-2}(K+t+\sigma^2)
        \le L_0^2/4,\\
 \widetilde f(0)-\inf\widetilde f
 &\le\lambda C_GKt
        \le (C_HC_G\kappa/\epsilon_0^2)KL_0^2/H\le\Delta.
\end{aligned}
\end{equation}
The third line uses $m\ge t+\log(2Kt)$ and then small $\kappa$;
the fourth uses \eqref{eq:linear-gap-budget} and then large $C_*$.
Fix these constants in that order. Since $K\le t+1\le2t$, the
dimension assumption of \Cref{thm:extended-gap} implies
\eqref{eq:linear-dimension-normalized}.

For $N=\lfloor Kn_*/8\rfloor$, form a finite pool of $M\ge64N^2$
independent complete $K$-block arrays and their negatives, with an
independent hidden uniform permutation of the $2M$ labels. Negation
is applied to the whole noise function. The finite mean is exactly
$\widetilde f$. The coupling \eqref{eq:app-partner} adds at most
$N^2/M\le1/64$ to \eqref{eq:linear-early-success}, even for
adaptive label selection and reuse.

Let $E_{\rm adm}$ mean that the finite average of the squared global
sample Lipschitz constants is at most $L_0^2$. The third line of
\eqref{eq:linear-scaled-budgets} and Markov's inequality imply
$\PP(E_{\rm adm})\ge3/4$. Condition the instance law on this
event only after the unconditional coupling and first-entry proof:
\begin{equation}\label{eq:linear-conditioned-success}
 \PP(\text{success by }N\mid E_{\rm adm})
 \le\frac{1/16+1/64+1/64}{3/4}=1/8.
\end{equation}
All families in this law satisfy \eqref{eq:model-assumptions}.
The unconditioned numerator is at most
$1/16+1/64+1/64<11/48$, and admissibility has probability at
least $3/4$. Thus \Cref{lem:expected-cost} supplies one fixed
admissible family with expected cost at least $N/3$.
Finally, floors change the positive quantities below by at
most numerical factors, so
\[
 N\ge cKdt\sigma^2
 \ge c\frac{dL_0^4}{\epsilon^4}
          \min\left\{\frac{\Delta H}{L_0^2},
                         \frac{L_0^2}{\epsilon^2}\right\}
 =c\frac{dL_0^2}{\epsilon^4}
          \min\left\{\Delta H,\frac{L_0^4}{\epsilon^2}\right\}.
\]
This is the asserted enlarged linear-gap range.

\end{proof}

%% file: sections/fixed_depth_proofs.tex
\section{Fixed-Depth Gap Amplification}\label[appendix]{app:fixed-depth}
We prove \Cref{cor:fixed-depth}. The two-level construction of
\Cref{app:linear-gap} is the base case. Iterating it requires two
additional properties: small gradients must imply low energy, and
first entry into that low-energy set must itself be expensive.
Both properties are established below before the recursion is used.
Throughout the normalized construction, $t\ge4096$ and $\sigma^2\ge1$.

\subsection{An energy cutoff and the normalized two-level block}
Write $C_b=e\Phi(+\infty)$, so $-C_b\le b_i\le0$, and fix
\[
 \theta=\frac1{16(C_b+1)},\qquad
 \psi(s)=1-S_5((s-\theta)/\theta).
\]
Here $S_5$ is the clipped quintic from \Cref{app:geometry}. Thus
$\psi=1$ on $(-\infty,\theta]$, $\psi=0$ on
$[2\theta,\infty)$, and $\psi'\le0$.

\begin{lemma}[Energy cutoff]\label[lemma]{lem:energy-cutoff}
Let $g\ge0$ be $C^2$, let $E=g(0)>0$, and suppose
$\|\nabla g\|\le L$, $\|D^2g\|\le H$. Set
$\chi=\psi(g/E)$. Then
\begin{equation}\label{eq:energy-cutoff-bounds}
 \|\nabla\chi\|\le CL/E,\qquad
 \|D^2\chi\|\le C(H/E+L^2/E^2).
\end{equation}
If $\|\nabla g\|\le e$ implies $g\le\theta E$, then for
every $a\le0$ and $\lambda\ge0$,
\begin{equation}\label{eq:energy-cutoff-guard}
 \|\nabla(g+\lambda a\chi)\|\le e
 \quad\Longrightarrow\quad \chi=1,\quad\nabla\chi=0.
\end{equation}
\end{lemma}
\begin{proof}
The chain rule gives
\[
 \nabla\chi=\frac{\psi'(g/E)}E\nabla g,\qquad
 D^2\chi=\frac{\psi'(g/E)}E D^2g+
       \frac{\psi''(g/E)}{E^2}\nabla g\nabla g^T.
\]
Bounded derivatives of the fixed quintic give
\eqref{eq:energy-cutoff-bounds}. Moreover,
\[
 \nabla(g+\lambda a\chi)
 =\left(1+\frac{\lambda a}E\psi'(g/E)\right)\nabla g,
 \qquad 1+\lambda a\psi'(g/E)/E\ge1.
\]
Thus a small modified gradient implies a small original gradient,
then $g/E\le\theta$, where the cutoff is identically one.
\end{proof}

\begin{proposition}[Fixed-depth assembly]\label[proposition]{prop:fixed-depth-assembly}
For every fixed $p\ge2$, the recursion below preserves the
low-energy first-entry interface \eqref{eq:depth-interface}.
Its mean and sample budgets are \eqref{eq:depth-mean-budgets}
and \eqref{eq:depth-sample-budget}; its dimension and query
bounds are \eqref{eq:depth-dimension} and
\eqref{eq:depth-total-cost}. The scaling below proves
\Cref{cor:fixed-depth}, with constants depending only on $p$.
\end{proposition}
\begin{proof}
Choose $C_0$ in \eqref{eq:linear-samples} sufficiently large once
and for all. Denote its two-level mean by $f_2$, with bottom length
$t$ and outer length $K_2$. For each level, subtract its minimum
from every sample. This minimum is public by the following calculation.
For a bare outer phase write
\[
 P(v)=\frac{S^2}{5}\left(\sqrt{1+\|v\|^2/S^2}-1\right),
 \quad w=U^Tv/\sqrt{1+\|v\|^2/R^2},
 \quad F_b(v)=P(v)+B(w),
\]
where $R=C_R\sqrt K$, $S=C_S\sqrt K$, and $K$ is its length.
The minimum $m_K=\min_vF_b(v)$ is independent of the orthonormal
frame $U$, by orthogonal invariance. The proof of
\eqref{eq:linear-stationary}, with no inner mixture or correction,
puts every minimizer in $w_i>1$. Nonnegative child means and
nonnegative corrections give
\begin{equation}\label{eq:depth-minimum}
 f_2\ge C_0t F_b\ge C_0t m_{K_2},\qquad
 \min f_2=C_0t m_{K_2}.
\end{equation}
For equality, choose a bare minimizer and a minimizer of the public
terminal child: all activations are one, all corrections vanish,
and the terminal child is the unique owner.
Since $P\ge0$ and $B\ge-C_bK$, while the test point $w_i=1$
has $P\le K/5$ and $B=-K\Phi(1)\le-2K$, we have
\begin{equation}\label{eq:depth-bare-min}
 -C_bK\le m_K\le-K.
\end{equation}
Put $g_2=f_2-\min f_2$ and $E_2=g_2(0)$. At zero the first
outer term cancels and the first child has public energy in
$[t,22t]$. Taking $C_0\ge22$ gives
\begin{equation}\label{eq:depth-base-energy}
 C_0K_2t\le E_2\le(C_b+1)C_0K_2t.
\end{equation}

To strengthen the location certificate \eqref{eq:linear-stationary}
to an energy bound, consider the outer phase.
On the span of $U$, use the compactified vector
$y=v/\sqrt{1+\|v\|^2/R^2}$ and set $u=\|y\|^2$,
$D=1-u/R^2$. The inverse is $v=y/\sqrt D$, and
\begin{equation}\label{eq:depth-phase-hessian}
 \begin{aligned}
 \nabla_y(P(y/\sqrt D))&=a(u)y,
 & a(u)&=\frac1{5D^2\sqrt{1+u/(S^2D)}},\\
 \frac{a'(u)}{a(u)}&=\frac2{R^2D}
             -\frac1{2S^2D^2(1+u/(S^2D))}.
 \end{aligned}
\end{equation}
The Hessian is $aI+2a'yy^T$. On $\|y\|\le C_z\sqrt K$,
choose $C_S,C_R$ large relative to $C_z$. Its tangential eigenvalue
is $a\ge1/6$, and its radial eigenvalue is at least
\[
 a\left(1-\frac{u}{S^2D^2(1+u/(S^2D))}\right)\ge1/10.
\]
The same formula applies in all ambient directions. The chain $B$
is convex on $w_i\ge1$: for $s,u\ge1$ and $a=(2s-1)^{-2}$,
\[
 \Psi''(s)<0,\quad\Phi''(u)=-u\Phi'(u)<0,\qquad
 \frac{\Psi'(s)^2}{-\Psi''(s)\Psi(s)}
   =\frac{2a}{3-2a}\le2<\frac{u\Phi(u)}{\Phi'(u)}.
\]
Thus each two-variable Hessian of $-\Psi(s)\Phi(u)$ is positive
definite, as checked in the proof of \Cref{lem:wide-core}.
The terminal phase in $y$ coordinates is therefore $1/10$-strongly
convex on the convex set
$\{\|y\|\le C_z\sqrt K,\ U^Ty\ge\mathbf1\}$.
Both the bare minimizer and every terminal small-gradient point
belong to this set by radial localization. The inverse compactifier
has differential norm $D^{-3/2}\le2$ there. Strong convexity gives
\begin{equation}\label{eq:depth-phase-energy}
 F_b(v)-m_K
 \le5\|\nabla_y(F_b(v(y)))\|^2
 \le20\|\nabla_vF_b(v)\|^2.
\end{equation}
For the first inequality, apply the strong-convexity inequality
between $y$ and its minimizer and maximize
$\langle\nabla_yF_b,y-y_*\rangle-\|y-y_*\|^2/20$ over
$y-y_*$. The second follows from the inverse differential bound.

At a two-level point with $\|\nabla g_2\|\le e\le\epsilon_0$,
\eqref{eq:linear-stationary} selects the public terminal child.
Its gradient is at most $e$, and \Cref{lem:wide-core} places it
in its quadratic core. The outer gradient gives
$\|\nabla_vF_b\|\le e/C_0$. Hence
\begin{equation}\label{eq:depth-base-small-energy}
 g_2\le\frac{20t e^2}{C_0}+\frac{e^2}{2\mu}
       \le\theta C_0K_2t\le\theta E_2.
\end{equation}
Choose a fixed $e\in(0,\epsilon_0]$ small enough for the middle
inequality for all $t\ge4096$ and $K_2\ge1$.
This establishes the geometric base case.

\subsection{Recursive scalar formula}
Suppose a child family $G_{i,j}=g_i+N_{i,j}$ has a nonnegative
coercive $C^2$ mean, the same public initial energy $E$, and
the following properties, uniformly over its independent copies:
\begin{equation}\label{eq:depth-interface}
 \begin{gathered}
 \|\nabla g_i\|\le e\ \Longrightarrow\ g_i\le\theta E,
 \qquad \|g_i-\bar g\|_\infty\le C_dE,\\
 \text{first entry into }\{g_i<2\theta E\}
 \text{ within }n\text{ local tests has probability at most }1/16.
 \end{gathered}
\end{equation}
Here $\bar g$ is one public fixed copy, and the probability statement
allows independent auxiliary information, adaptive sample reuse,
and an arbitrary final candidate. We prove the two-level instance
of this interface and its induction below.

Use the definitions $w,b_i,a_i,s_i,d_i,\omega_i$ from
\Cref{subsec:linear-family}, with outer length $K$ and
$v=Y/E$. Let all children share coordinates $X$, take an independent
outer frame, and use the public noiseless child $g_{K+1}=\bar g$.
Define
\begin{equation}\label{eq:depth-scalar}
 \begin{aligned}
 \chi_i(X)&=\psi(g_i(X)/E),\\
 F_j^{\rm raw}(Y,X)&=C_0E\left[P(Y/E)+B(w)
                   -\sum_{i=1}^K(1-\chi_i(X))d_i(w)\right]
                   +\sum_{i=1}^{K+1}\omega_i(w)G_{i,j}(X),\\
 G_j(Y,X)&=F_j^{\rm raw}(Y,X)-C_0E m_K.
 \end{aligned}
\end{equation}
Its mean is denoted $g$. Every term is a globally defined scalar
function. Nonnegativity and the terminal test used in
\eqref{eq:depth-minimum} show that $\min g=0$. At zero,
$\chi_1=0$, the first outer term cancels, and the first child
is the only owner. Thus the new public energy is
\begin{equation}\label{eq:depth-energy-recursion}
 E_{\rm new}=E-C_0E m_K,\qquad
 C_0KE\le E_{\rm new}\le(C_b+1)C_0KE.
\end{equation}
The public radial part is coercive in $Y$. The convex child mixture
is at least $\bar g-C_dE$ and is coercive in $X$; all remaining
terms are bounded. This proves coercivity also along simultaneous
escape in both coordinate blocks.

The owner derivative estimate \eqref{eq:linear-owner-derivatives}
applied to $g_i-\bar g$ bounds the outer mixture error by
$CC_d/C_0$. Choose $C_0$ so this is at most $1/100$.
The outer radial localization and coordinate argument in
\eqref{eq:linear-outer-sign}--\eqref{eq:linear-stationary}
then apply without change until the single-owner case. In that case
the exact inner gradient is
\[
 \nabla_X g=\nabla g_i+C_0E d_i\nabla\chi_i.
\]
Since $d_i\le0$, \Cref{lem:energy-cutoff} makes $\chi_i=1$ at
a small-gradient point. The outer correction vanishes, restoring
the same coordinate contradiction. Hence
\begin{equation}\label{eq:depth-terminal}
 \|\nabla g\|\le e\quad\Longrightarrow\quad
 w_1,\ldots,w_K>1,\qquad
 g=C_0E(F_b-m_K)+\bar g.
\end{equation}
This proves the stationary certificate from the child energy
property; it does not assume alignment of different child gradients.
At such a point, \eqref{eq:depth-phase-energy} and the child
energy property now give
\begin{equation}\label{eq:depth-small-energy-induction}
 g\le\frac{20Ee^2}{C_0}+\theta E
   \le\theta C_0KE\le\theta E_{\rm new}.
\end{equation}
The middle inequality follows by taking $C_0\ge2$ and decreasing
the same fixed $e$ if necessary. Equations
\eqref{eq:depth-base-small-energy} and
\eqref{eq:depth-small-energy-induction} prove the small-gradient
energy property at every fixed depth.

\paragraph{Differences of copies and regularity.}
In comparing two copies, the public radial term and the public
minimum cancel. Each chain has absolute value at most $C_bK$,
and at most one correction has a nonzero two-jet and magnitude
at most $C_b$. The difference of two convex child mixtures is at
most $2C_dE$. Consequently
\begin{equation}\label{eq:depth-difference}
 \frac{\|g-\bar g_{\rm new}\|_\infty}{E_{\rm new}}
 \le 2C_b(1+1/K)+\frac{2C_d}{C_0K}\le C_d.
\end{equation}
Choose $C_d\ge8C_b$ large enough for the base family, then
$C_0\ge4$ sufficiently large for all preceding requirements.
The base comparison follows by the same estimate with child
differences $O(t)$, supplied by their common exterior.
At higher levels no common exterior is required.

If child gradients and Hessians are bounded by $L,H$, the owner
estimate and \eqref{eq:energy-cutoff-bounds} give
\begin{equation}\label{eq:depth-regularity-recursion}
 \Lip(g)\le C(L+\sqrt K),\qquad
 \|D^2g\|\le C(H+L^2/E+1/E).
\end{equation}
Indeed the correction has inner, mixed, and outer Hessian bounds
$C(H+L^2/E)$, $CL/E$, and $C/E$; the mixed term is absorbed
by $L/E\le(L^2/E+1/E)/2$. The owner mixture has the same
orders after subtracting $\bar g$. The bare outer term contributes
$O(\sqrt K)$ to the gradient and $O(1/E)$ to its Hessian,
using $R\asymp\sqrt K$. Thus for depth $p$,
\begin{equation}\label{eq:depth-mean-budgets}
 \Lip(g_p)^2\le C_p\left(t+\sum_{l=2}^pK_l\right),\qquad
 \|D^2g_p\|\le C_p,
\end{equation}
provided $t\ge4096$, the internal lengths are comparable to $t$,
and $1\le K_p\le Ct$. The energy recursion makes
$L^2/E$ bounded at each level; the finite iteration accounts for
the subscript $p$.

\paragraph{The sample budget.}
Let $\Lambda_j$ be the maximum Lipschitz constant among all bottom
noise fields in array $j$. Their common zero exterior gives a
uniform amplitude bound $A_j\le C\sqrt t\Lambda_j$.
Higher-level noise is a convex mixture, so its amplitude remains
at most $A_j$. Its inner derivative is bounded by the largest
child derivative; its outer derivative is at most $CA_j/E$.
These derivatives occupy orthogonal coordinate blocks, whence
\[
 \Lip(N_{{\rm new},j})^2\le
       \max_i\Lip(N_{i,j})^2+CA_j^2/E^2.
\]
Since energies increase by at least $C_0$ at each recursion,
\[
 \Lip(N_{p,j})^2
 \le\Lambda_j^2+CA_j^2\sum_{l\ge0}(C_0^lt)^{-2}
 \le C\Lambda_j^2.
\]
There are $\prod_{l=2}^pK_l$ bottom blocks. Applying
\eqref{eq:wide-noise-maximum} to their independent arrays yields
\begin{equation}\label{eq:depth-sample-budget}
 \E\Lip(G_{p,j})^2\le C_p\left[t+\sum_{l=2}^pK_l+
 \sigma^2\frac{m+t+\log(2t\prod_{l=2}^pK_l)}m\right].
\end{equation}
All cutoffs use deterministic means, so noise enters linearly.
Negating the complete bottom array therefore negates the complete
noise at the root, as required by the antithetic realization.

\subsection{Charging low-energy entry and completing the induction}
We give the information argument separately from the geometric
energy property. At a level of length $K$, process all current-child
entries at a query before returning its value. Let $b$ be the number
of completed children. If the current child has not been entered
and every undisclosed outer coordinate has magnitude at most $1/4$,
the next outer term cancels, all later terms vanish, and
\begin{equation}\label{eq:depth-low-energy-cost}
 f^{\rm raw}\ge-C_bC_0bE,\qquad
 g\ge C_0E(K-C_bb).
\end{equation}
Here $P\ge0$, the child mixture is nonnegative, the first $b$
chain terms are at least $-C_b$ each, and
$-m_K\ge K$. For $b\le K/(4C_b)$,
\begin{equation}\label{eq:depth-paid-target}
 g\ge\tfrac34C_0KE
   >2\theta(C_b+1)C_0KE\ge2\theta E_{\rm new}.
\end{equation}
Thus entry into $\{g<2\theta E_{\rm new}\}$ requires at least
$h=\lfloor K/(4C_b)\rfloor$ child entries or a projection violation.

At level two, the child entries are the coarse sets
$\mathcal P_i$ of \Cref{lem:wide-coarse-entry}. At subsequent
levels they are the energy sets in \eqref{eq:depth-interface}.
Off the corresponding set the cutoff is zero, and the scalar
response is exactly
\begin{equation}\label{eq:depth-reference}
 A_{i,j}(Y,X)+\omega_i(w)G_{i,j}(X)-C_0E m_K,
\end{equation}
with $A_{i,j}$ depending only on public and completed data.
The derivation is the cancellation calculation
\eqref{eq:linear-scalar-reference}; it uses only the cutoff being
zero and the future activations vanishing. In particular it does
not require a quadratic core at higher levels.

For clarity, a target test in this first-entry experiment returns
the scripted answer ``miss'' until its first hit; the policy is
continued with that scripted answer in the reference process.
We bound the first true hit along these unconditioned reference
queries. We do not condition the reference law on earlier true
misses or on good projections. Before a child is born,
\eqref{eq:depth-reference} and all prior notifications are pointwise
independent of its complete frames and arrays. At birth its prior
is therefore fresh, even if its sample label has already been used
in other children. Completed arrays can be supplied as independent
auxiliary data in a single-child reduction. A query crossing several
targets is processed in order before its scalar answer, exactly as
in \Cref{lem:gap-renewal}.

For the induction, the same counting argument gives a stronger tail.
Suppose each fresh child has probability at most $1/16$ of entry
within $n$ of its tests. In at most $hn/8$ noncrossing tests,
completion of $h$ children, with $n\ge16$, forces at least
$3h/4$ fast children: a slow completion uses at least $n-1$
noncrossing tests and
$(hn/8)/(n-1)\le2h/15<h/4$. Successive conditioning at birth
and a union bound over the possible fast sets give
\begin{equation}\label{eq:depth-renewal-tail}
 \PP(h\text{ completions})\le 2^h16^{-3h/4}=2^{-2h}.
\end{equation}
The final output is one further test, so the budgets below can be
decreased by one without changing their constant order.

Assume internal lengths $K_l\ge32C_b$. Start with
$n_1\ge ctm\sigma^2$ from \Cref{lem:wide-coarse-entry}, and set
\begin{equation}\label{eq:depth-local-budgets}
 h_l=\lfloor K_l/(4C_b)\rfloor,\qquad
 n_l=\lfloor h_l n_{l-1}/8\rfloor-1
       \ge c_lm\sigma^2t\prod_{j=2}^lK_j.
\end{equation}
Choose the dimension so the outer projection error at each local
experiment is at most $1/256$. Since $h_l\ge8$,
\eqref{eq:depth-paid-target} and \eqref{eq:depth-renewal-tail}
give entry probability at most $2^{-16}+1/256<1/16$.
This proves the probabilistic part of \eqref{eq:depth-interface}
at each internal level. At the top, $K_p$ may be one: stationarity
requires all $K_p$ children, and \Cref{lem:gap-renewal} gives
constant failure probability up to
\begin{equation}\label{eq:depth-total-cost}
 N\ge c_pm\sigma^2t\prod_{l=2}^pK_l.
\end{equation}

Allocate one coordinate block to the bottom workspace, one to its
phase, and one to each outer level. Within a level all children
share the same blocks. For fixed $p$ the block dimensions are
comparable to $d$. The sufficient dimension condition is
\begin{equation}\label{eq:depth-dimension}
 d\ge C_p\left(t+\sum_{l=2}^pK_l\right)
                 \log\left(2+\sigma^2t\prod_{l=2}^pK_l\right).
\end{equation}
Indeed a local outer error is bounded by
$2K_l(n_l+K_l+1)\exp[-c r_l/K_l]$, as in
\eqref{eq:linear-outer-error}; the bottom error is
\eqref{eq:wide-entry-error}. Writing
$u=d/(t+\sum_lK_l)$, every prefactor has logarithm at most
$C_p\log(2+\sigma^2t\prod_lK_l)+\log u$, and every exponent
is at least $c_pu$. Increasing $C_p$ makes
$c_pu-\log u$ dominate these prefactors. This also supplies
$m\ge t+\log(2t\prod_lK_l)$ and $n_1\ge16$.
The induction uses a conditional local error at each birth,
not a union bound over every node of the construction tree.

\subsection{Scaling and finite samples}
The geometric implication is
\eqref{eq:depth-small-energy-induction}; the probabilistic interface
is \eqref{eq:depth-local-budgets}.
Let $H_p$ be a finite Hessian bound from
\eqref{eq:depth-mean-budgets}, and let $e>0$ be the common
normalized tolerance. Set
\begin{equation}\label{eq:depth-scaling}
 \lambda=\frac{H_p\epsilon^2}{e^2H},\qquad
 b=\frac{H_p\epsilon}{eH},\qquad
 \widetilde G_j(X)=\lambda G_{p,j}(X/b).
\end{equation}
Then the mean has Hessian at most $H$, and an $\epsilon$-stationary
output gives an $e$-stationary normalized output.

Write $q=L_0^2/\epsilon^2$. Choose a sufficiently small
$\kappa_p>0$, and take
\[
 t=K_2=\cdots=K_{p-1}=\lfloor\kappa_pq\rfloor,
 \qquad \sigma^2=\kappa_pq.
\]
The accuracy constant makes all internal lengths at least
$\max\{4096,32C_b\}$. When
$\Delta H/L_0^2\ge C_{g,p}q^{p-2}$, take
\begin{equation}\label{eq:depth-top-length}
 K_p=\left\lfloor\min\left\{\kappa_pq,
          \frac{\Delta H}{C_{g,p}L_0^2q^{p-2}}\right\}\right\rfloor.
\end{equation}
For $p=2$ these choices are exactly the two-level scaling.
Equations \eqref{eq:depth-energy-recursion},
\eqref{eq:depth-sample-budget}, and \eqref{eq:depth-scaling} give
\begin{align*}
 \widetilde g(0)
 &\le C_p\frac{\epsilon^2}{H}t\prod_{l=2}^pK_l
 \le C_p\frac{L_0^2}{H}q^{p-2}K_p\le\Delta,\\
 \E\Lip(\widetilde G_j)^2
 &\le C_p\epsilon^2\left(t+\sum_{l=2}^pK_l+\sigma^2\right)
 \le C_p\kappa_pL_0^2\le L_0^2/4.
\end{align*}
Choose $\kappa_p$ first for the second line, then $C_{g,p}$
for the first. Since all lengths are $O_p(q)$,
\eqref{eq:depth-dimension} follows from
$d\ge C_pq\log(2+q)$. Floors cost only $p$-dependent factors,
and \eqref{eq:depth-total-cost} becomes
\begin{equation}\label{eq:depth-rate-calculation}
 N\ge c_pdq^pK_p
 \ge c_pdq^2\min\{\Delta H/L_0^2,q^{p-1}\}
 =c_p\frac{dL_0^2}{\epsilon^4}
             \min\{\Delta H,L_0^2q^{p-1}\}.
\end{equation}

For the remaining range
$\Delta H/L_0^2<C_{g,p}q^{p-2}$, invoke the depth-$(p-1)$
result. For $s<C_{g,p}q^{p-2}$,
\[
 \min\{s,q^{p-2}\}\ge
 \frac1{\max\{1,C_{g,p}\}}\min\{s,q^{p-1}\}.
\]
Thus it implies \eqref{eq:depth-rate-calculation} after decreasing
$c_p$. Start this induction at \Cref{thm:extended-gap}; increase
$C_{*,p}$ and the dimension constant, and decrease $c_{e,p}$,
to include the depth-$(p-1)$ hypotheses.

Finally take $M\ge64N^2$ independent complete arrays and their
negatives, and hide their pairing by a uniform permutation.
The mean is unchanged exactly. The finite-label coupling adds
at most $1/64$; the outer projection error contributes at most
$1/64$; and the ideal top completion probability is at most
$1/16$. The sample budget makes admissibility have probability
at least $3/4$. As in \eqref{eq:linear-conditioned-success},
\[
 \PP(\text{success by }N\mid\text{admissible})\le1/8,
 \qquad \sup_{\text{admissible families}}\E[\mathrm{cost}]\ge N/3.
\]
This proves the claimed expected-cost lower bound, including
adaptive label selection, reuse, and unrestricted output.

\end{proof}

\subsection{Dependence on recursion depth}
\begin{proposition}[Loss of the initial gradient at growing depth]
\label[proposition]{prop:depth-loss}
The normalized depth-$p$ mean satisfies
$\|\nabla g_p(0)\|=\Phi'(0)$ and
$\Lip(g_p)^2\ge(C_0^2C_S^2/25)\sum_{l=2}^pK_l$.
Consequently every uniform rescaling satisfying the sample
Lipschitz budget obeys
$\|\nabla\widetilde g(0)\|\le CL_0/\sqrt{p-1}$,
where $C$ is independent of $p$.
\end{proposition}
\begin{proof}

The restriction to fixed depth also follows from the gradient at the
initial point. The bottom block agrees with the source, up to an
additive constant, in a neighborhood of zero. There the activations
and the well gradient vanish, the confinement gradient is zero, and
only $-\Phi(w_1)$ contributes. Consequently its initial gradient has
norm $\Phi'(0)$. At every outer level the first child is the sole
owner near zero, its authentication cutoff is zero, and the first
chain term cancels identically. The new radial gradient also vanishes
at zero. Induction therefore gives
\begin{equation}\label{eq:depth-origin-gradient}
 \|\nabla g_p(0)\|=\Phi'(0).
\end{equation}

By contrast, the outer radial gradients can be simultaneously large.
At each level $l\ge2$, the phase-space allocation in
\eqref{eq:depth-dimension} has dimension larger than $K_l$.
Follow the first-child path and choose each normalized coordinate
$v_l$ orthogonal to that level's outer frame, so $w=0$. All ownership weights
are locally constant, with the first child the sole owner.
First choose the bottom coordinates sufficiently far away. Coercivity
then makes each first-child energy exceed its cutoff threshold;
at level two it also places the bottom point outside the spatial
cutoff. Taking the outer radii large preserves these strict inequalities,
because the radial terms and the minimum-normalizing constants are
nonnegative when the cutoffs are zero. Thus all authentication cutoffs
vanish in a neighborhood of these points. The first chain term cancels,
and the other chain terms vanish. In the coordinate block $Y_l$
of the normalized recursion, differentiation of the radial term gives
\[
 \nabla_{Y_l}g_p
 =\frac{C_0v_l}{5\sqrt{1+\|v_l\|^2/(C_S^2K_l)}}.
\]
The child energy used to scale $Y_l$ cancels its prefactor; at level
two the same cancellation uses $t$. These derivatives occupy
orthogonal coordinate blocks. Taking all $\|v_l\|$ to infinity and
using $\Lip(g_p)=\sup_X\|\nabla g_p(X)\|$ yields
\begin{equation}\label{eq:depth-lipschitz-lower}
 \Lip(g_p)^2
 \ge\lim\sum_{l=2}^p\|\nabla_{Y_l}g_p\|^2
 =\frac{C_0^2C_S^2}{25}\sum_{l=2}^pK_l.
\end{equation}
The argument follows one first-child path; it does not require the
frames of different copies at a given level to share an orthogonal
direction.

For the uniform rescaling $\widetilde g(X)=\lambda g_p(X/b)$,
the sample moment constraint implies
\[
 \frac\lambda b\Lip(g_p)=\Lip(\widetilde g)
 \le\E\Lip(\widetilde G_j)
 \le\bigl(\E\Lip(\widetilde G_j)^2\bigr)^{1/2}\le L_0.
\]
The first equality is the scaling rule, the first inequality follows
by averaging sample differences, and the second is Cauchy--Schwarz.
Combining this with \eqref{eq:depth-origin-gradient} and
\eqref{eq:depth-lipschitz-lower} gives
\[
 \|\nabla\widetilde g(0)\|
 =\frac\lambda b\Phi'(0)
 \le\frac{5\Phi'(0)L_0}{C_0C_S\sqrt{\sum_{l=2}^pK_l}}
 \le\frac{5\Phi'(0)L_0}{C_0C_S\sqrt{p-1}}.
\]
Thus, for this recursion and fixed $L_0/\epsilon$, taking $p$
sufficiently large either violates the sample budget or makes the
public origin an $\epsilon$-stationary point. This is consistent
with the fixed-depth theorem, whose accuracy condition depends on $p$.

\end{proof}

\subsection{Goldstein consequence}
\begin{corollary}[Fixed-depth Goldstein bound]\label[corollary]{cor:depth-goldstein}
For fixed $p\ge2$ and suitable positive constants depending
only on $p$, the conditions
$\epsilon\le c_{e,p}L_0$,
$\Delta\epsilon/\delta\ge C_{*,p}L_0^2$, and
$d\ge C_p(L_0^2/\epsilon^2)\log(2+L_0^2/\epsilon^2)$ imply
\begin{equation}\label{eq:depth-goldstein}
 N=\Omega_p\left(\frac{dL_0^2}{\epsilon^4}
 \min\left\{\frac{\Delta\epsilon}{\delta},
             L_0^2\left(\frac{L_0^2}{\epsilon^2}\right)^{p-1}\right\}\right).
\end{equation}
\end{corollary}
\begin{proof}
Apply \Cref{cor:fixed-depth} with $H=\epsilon/\delta$ and
smooth tolerance $2\epsilon$. Equation \eqref{eq:goldstein-ineq}
validates the reduction. The outer factor loses $16$, and
$(L_0^2/(2\epsilon)^2)^{p-1}=4^{-(p-1)}(L_0^2/\epsilon^2)^{p-1}$; these are
absorbed into $c_p$. Decrease the smooth accuracy constant
$c_{e,p}$ by a factor of two, and enlarge the dimension and gap
constants if needed, to obtain the displayed hypotheses.

\end{proof}

%% file: sections/logarithmic_gap_proofs.tex
\section{Proof of logarithmic gap amplification}
\label[appendix]{app:logarithmic-gap}

We use the common block and its coarse first-entry estimate from
\Cref{lem:wide-core,lem:wide-coarse-entry}.  The inner length is $t$, the number of levels is
$K$, and $m=\lfloor d/2\rfloor$, $r=d-m$.  All constants in this
appendix are universal.  The construction uses the fixed normalized
stationarity threshold $2^{-28}$ and $\mu=1/100$.

\subsection{Exact nesting and global regularity}
\begin{proposition}[Exact nested ownership and regularity]
\label[proposition]{prop:exact-nesting}
Take the common blocks of \Cref{lem:wide-core}. The recursion
\eqref{eq:log-recursion} defines fixed sample functions with the
ownership formula \eqref{eq:log-owner}. Their mean is coercive
and $C^2$, the samples are $C^1$, and the gap, Hessian,
stationary-location and sample-moment bounds are
\eqref{eq:log-geometry}--\eqref{eq:log-moment}.
Outside a child's insertion region no ancestor response depends
on that child's data, including additive constants.
\end{proposition}
\begin{proof}

For $1\le j\le K$, let $G_{j,\xi}=g_j+Z_{j,\xi}$ be independent
blocks with centers $p_j$ and the common public collar constant $c_t$.
One label $\xi$ fixes its arrays at every level.  Define
$\mathcal F_{K,\xi}=G_{K,\xi}$ and, backwards for $j<K$, set
\begin{equation}\label{eq:log-recursion}
\mathcal F_{j,\xi}(X)=
\begin{cases}
 B_0^{-1}\mathcal F_{j+1,\xi}\bigl(B_0(X-p_j)\bigr)-B_0^{-1}c_t,
       &\norm{X-p_j}<\sqrt t,\\
 G_{j,\xi}(X),&\norm{X-p_j}\ge\sqrt t.
\end{cases}
\end{equation}
Every modification lies in $B(p_j,\sqrt t)\subset
B(0,B_0\sqrt t/2)$.  Consequently the recursively modified block
retains the original public outer collar.  On
$\sqrt t/2\le\norm{X-p_j}\le\sqrt t$, its two defining expressions
coincide sample by sample, since
\begin{equation}\label{eq:log-collar}
 B_0^{-1}\left[c_t+
       \frac{\mu}{2B_0}\norm{B_0(X-p_j)}^2\right]
       -B_0^{-1}c_t
 =\frac\mu2\norm{X-p_j}^2.
\end{equation}
There is an overlap of positive width; equality there and continuity
of the available derivatives match the jets at the joining boundary.
Thus the samples are $C^1$ and their mean is $C^2$.

Expand both values and locations by the same factor:
\begin{equation}\label{eq:log-expanded}
 \widehat F_\xi(X)=B_0^{K-1}
       \mathcal F_{1,\xi}(X/B_0^{K-1}),
 \qquad \widehat f=\E_\xi\widehat F_\xi.
\end{equation}
To describe ownership, define origins and nested open sets by
\begin{equation}\label{eq:log-origins}
 o_1=0,\qquad o_{j+1}=o_j+B_0^{K-j}p_j,\qquad
 \mathcal C_1=\R^d,\qquad
 \mathcal C_{j+1}=B(o_{j+1},B_0^{K-j}\sqrt t)
       \quad(1\le j\le K).
\end{equation}
The inequality $\norm{p_j}+\sqrt t<B_0\sqrt t/2$ makes these sets
strictly nested. Indeed, for $X\in\mathcal C_{j+1}$,
\[
 \|X-o_j\|
 \le\|X-o_{j+1}\|+\|o_{j+1}-o_j\|
 <B_0^{K-j}(\sqrt t+\|p_j\|)
 <\tfrac12B_0^{K-j+1}\sqrt t.
\]
The first inequality is the triangle inequality, the next uses
the child radius and center increment, and the last is the
common-block center bound. For $j\ge2$ this lies inside
$\mathcal C_j$; $\mathcal C_1=\R^d$ is automatic. The last set $\mathcal C_{K+1}$ is a terminal
target, with no further inserted block.  For $j<K$ and
$X\in\mathcal C_j\setminus\mathcal C_{j+1}$, and for $j=K$ and
$X\in\mathcal C_K$, the exact sample value is
\begin{equation}\label{eq:log-owner}
 \widehat F_\xi(X)=
 -c_t\sum_{i=2}^{j}B_0^{K-i}
 +B_0^{K-j}G_{j,\xi}\left(\frac{X-o_j}{B_0^{K-j}}\right).
\end{equation}
To verify this formula, set the current local coordinate to
$Z=(X-o_j)/B_0^{K-j}$. On insertion, the recursive term transforms as
\begin{align*}
 B_0(Z-p_j)
 &=\frac{X-o_j-B_0^{K-j}p_j}{B_0^{K-j-1}}
 =\frac{X-o_{j+1}}{B_0^{K-j-1}},\\
 B_0^{K-j}\mathcal F_{j,\xi}(Z)
 &=B_0^{K-j-1}\mathcal F_{j+1,\xi}
       \left(\frac{X-o_{j+1}}{B_0^{K-j-1}}\right)
       -c_tB_0^{K-j-1}.
\end{align*}
The first line substitutes $Z$ and \eqref{eq:log-origins};
the second substitutes \eqref{eq:log-recursion} and distributes
the scale. Induction adds the constants in \eqref{eq:log-owner}.
The empty sum is zero.
In particular, outside $\mathcal C_j$ the
sample has no dependence on any data of levels $j,\ldots,K$,
including additive constants.  Boundary conventions do not change
this identity, by \eqref{eq:log-collar}.

The gradient on a region owned by level $j$ equals the local block
gradient, and its mean Hessian equals $B_0^{j-K}D^2g_j$.
The matching collars therefore give
\begin{equation}\label{eq:log-geometry}
 \norm{D^2\widehat f}\le C_H,\qquad
 \widehat f(0)-\inf\widehat f\le C t B_0^{K-1},\qquad
 \norm{\nabla\widehat f(X)}\le2^{-28}
       \Longrightarrow X\in\mathcal C_{K+1}.
\end{equation}
For the gap bound, observe that the origin is
unmodified because $\norm{p_1}\ge128\sqrt t$.  At each insertion the
child minimum is below the nonnegative parent exterior; induction gives
\[
 \inf\E_\xi\mathcal F_{1,\xi}
       =-c_t\sum_{i=1}^{K-1}B_0^{-i},\qquad
 \E_\xi\mathcal F_{1,\xi}(0)=g_1(0)\le22t.
\]
The minimizing point is inside the child's body, so it is retained by
the insertion.  The sum is at most $(B_0-1)^{-1}$ and $c_t\le C_Gt$.
Coercivity follows from the unchanged public exterior.  For $j<K$,
a small gradient at a point owned by $j$ would require its local
distance from $p_j$ to be at most $2^{-28}/\mu<\sqrt t$, whereas
that region has been replaced by the child.  At the last level the
same certificate places the point in $\mathcal C_{K+1}$.

At each point only one block contributes a gradient.  Hence the global
Lipschitz constant is bounded by the maximum of the block Lipschitz
constants, and the maximum-noise estimate in \Cref{lem:wide-core} gives
\begin{equation}\label{eq:log-moment}
 \E\Lip(\widehat F_\xi)^2
 \le C\left[t+\sigma^2\frac{m+t+\log(2Kt)}m\right].
\end{equation}
This uses the finite-union version of \Cref{thm:noise-moment}, already
included in the shared block estimate.  There is no sum of $K$
sample-gradient bounds.

\end{proof}

\subsection{First-entry charging across levels}
\begin{proposition}[Accumulated cost of nested entries]
\label[proposition]{prop:nested-entry}
Under \eqref{eq:log-dimension-normalized}, let $n_*$ be
\eqref{eq:log-local-tests}. In the independent-label experiment,
for $N+1\le Kn_*/4$ the nested family has stationary-success
probability at most $1/16$ by $N$ scalar queries. Queries and
outputs may jump between any regions and may reuse labels.
\end{proposition}
\begin{proof}

For the rest of the argument assume
\begin{equation}\label{eq:log-dimension-normalized}
 d\ge C\left[t\log(2+t\sigma^2)+\log(2K)\right],
 \qquad t\ge4096,\quad\sigma^2\ge1,
\end{equation}
and use the single local test allowance
\begin{equation}\label{eq:log-local-tests}
 n_*=\left\lfloor\frac{\lfloor t/2\rfloor}{8}
       \left\lfloor\frac{(m-t+1)^2\sigma^2}{2048m}\right\rfloor
       \right\rfloor.
\end{equation}
The shared coarse first-entry estimate gives $n_*\ge16$,
$n_*\ge ctm\sigma^2$, and a conditional probability at most
$1/16$ of reaching $B(p_j,\sqrt t)$ within $n_*$ local tests from
an independent birth history.  Its application to the compiled block
is a first-hit argument: the target ball is contained in
$\{\norm{x-128\sum_i v_{j,i}}<8\sqrt t\}$, and before first entering
this cylinder every response is simulated by the primitive with the
same retained label.  Subsequent responses inside the cylinder play
no role in this probability bound.  The Gaussian information step
is exactly \Cref{thm:reuse} through the shared prefix interface.

Work first with independent label arrays.  Process each paid query,
and finally the output without a response, starting at the first level
whose target has not yet been entered.  A candidate outside the current
$\mathcal C_j$ is an ancestor query.  Otherwise test whether it lies
in $\mathcal C_{j+1}$.  A positive test completes level $j$ and starts
level $j+1$ at the same point and with the already selected label.
Continue until a test fails or all levels are completed.  The one
scalar response is the value of the actual owner in
\eqref{eq:log-owner}.  Completed centers and membership notifications
may be disclosed; they only enlarge the algorithm's information.

We verify the birth law in this actual geometric process.
Before the first entry into $\mathcal C_j$, every response is
pointwise independent of all data of levels $j,\ldots,K$, by
\eqref{eq:log-owner}.  Membership in $\mathcal C_j$ depends only on
the preceding levels.  Thus, immediately before testing the local
candidate $(X-o_j)/B_0^{K-j}$ against level $j$, this level's entire
frame and all label-indexed Gaussian arrays retain their original
joint independent law.  The candidate and its selected label are
already part of that history; the scalar response from level $j$
is not.  The assertion also holds if the label was used in ancestors.

During the lifetime of level $j$, ancestor responses are functions of
ancestor arrays and the evolving transcript.  For a local reduction
one may condition on all ancestor arrays in the ideal experiment:
they are independent auxiliary randomness for the new block, and
they simulate these interleaved responses.  Queries in
$\mathcal C_j\setminus\mathcal C_{j+1}$ return exactly the transformed
$G_{j,\xi}$ value.  A query entering $\mathcal C_{j+1}$ is the parent's
terminal candidate, tested before any descendant scalar is returned.
These facts meet the shared first-entry lemma's hypotheses.
In particular, an exceptional parent may expose all its own data
without affecting this argument, and the same reasoning applies
successively when one query crosses several levels.

Each candidate is noncrossing at at most one current level; ancestor
queries do not count as local tests. A completed level uses its
noncrossing tests plus the single crossing candidate. The verified
birth law and \Cref{lem:gap-renewal} therefore give, whenever
$N+1\le Kn_*/4$,
\begin{equation}\label{eq:log-ideal-success}
 \PP(\norm{\nabla\widehat f(\widehat X)}\le2^{-28})\le1/16.
\end{equation}
Indeed \eqref{eq:log-geometry} requires completion of all levels.
The first-entry bound includes the phase-projection error for each
block, and \Cref{lem:gap-renewal} applies to these conditional bounds.

\end{proof}

\subsection{Physical scaling and the finite population}
\begin{proof}[Proof of \Cref{thm:logarithmic-gap}]
\phantomsection\label{proof:logarithmic-bound}

Choose a sufficiently small universal $\kappa>0$ and set
\begin{equation}\label{eq:log-physical-choice}
 t=\left\lfloor\kappa\frac{L_0^2}{\epsilon^2}\right\rfloor,
 \qquad\sigma^2=\kappa\frac{L_0^2}{\epsilon^2},\qquad
 \lambda=\frac{2^{56}C_H\epsilon^2}{H},\qquad
 b=\frac{2^{28}C_H\epsilon}{H}.
\end{equation}
The condition $\epsilon\le c_eL_0$, with $c_e$ small enough after
choosing $\kappa$, ensures $t\ge4096$.  The scaled samples are
$F_\xi(X)=\lambda\widehat F_\xi(X/b)$.
Their mean is $H$-smooth, since $\lambda C_H/b^2=H$, and
$\lambda/b=2^{28}\epsilon$ converts physical
$\epsilon$-stationarity into the normalized certificate above.

For a sufficiently large universal $C_2$, the physical dimension
condition
\begin{equation}\label{eq:log-dimension-physical}
 d\ge C_2\left[
       \frac{L_0^2}{\epsilon^2}
          \log\left(2+\frac{L_0^2}{\epsilon^2}\right)
       +\log(2K)\right]
\end{equation}
implies \eqref{eq:log-dimension-normalized} and
$m\ge t+\log(2K)$.  Indeed, $t\le L_0^2/\epsilon^2$ and
$\log(2+t\sigma^2)\le2\log(2+L_0^2/\epsilon^2)$ after taking
$\kappa\le1$.  It follows from \eqref{eq:log-moment} that
\begin{equation}\label{eq:log-physical-promises}
 \E\Lip(F_\xi)^2
 \le C\,2^{56}\epsilon^2(t+\sigma^2)\le\frac{L_0^2}{4},
 \qquad
 f(0)-\inf f\le C_0\frac{L_0^2}{H}B_0^{K-1},
\end{equation}
where the first inequality is enforced by $\kappa$, and $C_0\ge1$
is then fixed.  Set $C_1=B_0$ and, for $\Delta H\ge C_0L_0^2$, take
\begin{equation}\label{eq:log-level-count}
 K=1+\left\lfloor
      \frac{\log(\Delta H/(C_0L_0^2))}{\log C_1}
      \right\rfloor.
\end{equation}
Then the global gap in \eqref{eq:log-physical-promises} is at most
$\Delta$.

Take $N=\lfloor Kn_*/8\rfloor$.  Since $n_*\ge16$,
$N+1\le Kn_*/4$ and $N\ge Kn_*/16$.
Draw $M\ge64N^2$ independent full-depth noise samples, include both
signs of each, and assign their $2M$ public labels by one independent
hidden uniform permutation.  The signs are paired across the entire
sample function, with the same sign at every level.  Pairing cancels
the noise exactly, so every finite family has the constructed mean.

The hidden-pair part of \Cref{lem:static-coupling} applies to the
complete level-indexed arrays and gives error $N^2/M\le1/64$.
Its unopened arrays remain hidden; the ancestor-array disclosures
above were only a reduction in the independent-label experiment.

Let $E_{\rm adm}$ be the event that the empirical second moment of
the global sample Lipschitz constants is at most $L_0^2$.
Equation \eqref{eq:log-physical-promises} and Markov's inequality give
$\PP(E_{\rm adm})\ge3/4$. Condition the hard-instance distribution
on this event.  The unconditioned bound
\eqref{eq:log-ideal-success} and the coupling error imply
\begin{equation}\label{eq:log-conditioned-success}
 \PP(\text{successful output by }N\mid E_{\rm adm})
 \le\frac{1/16+1/64}{3/4}=\frac5{48}.
\end{equation}
Thus every family in the conditioned distribution meets the global
gap, smoothness, and moment promises.

Finally, $m\ge d/3$, $n_*\ge ctm\sigma^2$, and the floor in
\eqref{eq:log-physical-choice} give
\[
 N\ge cKdt\sigma^2
       \ge cK\frac{dL_0^4}{\epsilon^4}.
\]
The conditional bound \eqref{eq:log-conditioned-success} is below
$1/3$. Equivalently its unconditioned numerator is at most
$1/16+1/64<11/48$, with admissibility probability at least $3/4$.
\Cref{lem:expected-cost} therefore yields a fixed admissible
finite family with expected cost at least $N/3$.
Since $C_1>1$ is
universal, \eqref{eq:log-level-count} concludes the bound
\[
 \Omega\!\left(
   \frac{dL_0^4}{\epsilon^4}
   \left[1+\log\frac{\Delta H}{C_0L_0^2}\right]
 \right)
\]
under the stated accuracy, gap, and dimension conditions.

\end{proof}

\paragraph{AI Methodology Statement.}
Generative AI played a substantial role in the mathematical development of this work. We used GPT-5.6 Sol to explore joint lower-bound constructions and GPT-Astra extensively to develop and refine the global hard instances, information-theoretic arguments, and proofs. These tools materially contributed to the construction and proof sections of the paper. We further employed independent AI agents to critically examine the arguments and Lean to formally verify selected mathematical statements. The authors reviewed and revised the AI-assisted reasoning and take full responsibility for the correctness, originality, and presentation of the final results.

%% file: references.bib
@article{carmon2019,
  author  = {Carmon, Yair and Duchi, John C. and Hinder, Oliver and Sidford, Aaron},
  title   = {Lower Bounds for Finding Stationary Points {I}},
  journal = {Mathematical Programming},
  volume  = {184},
  pages   = {71--120},
  year    = {2020},
  doi     = {10.1007/s10107-019-01406-y}
}

@article{arjevani2023,
  author  = {Arjevani, Yossi and Carmon, Yair and Duchi, John C. and Foster, Dylan J. and Srebro, Nathan and Woodworth, Blake},
  title   = {Lower Bounds for Non-Convex Stochastic Optimization},
  journal = {Mathematical Programming},
  volume  = {199},
  pages   = {165--214},
  year    = {2023},
  doi     = {10.1007/s10107-022-01822-7}
}

@article{kornowski2024,
  author  = {Kornowski, Guy and Shamir, Ohad},
  title   = {An Algorithm with Optimal Dimension-Dependence for Zero-Order Nonsmooth Nonconvex Stochastic Optimization},
  journal = {Journal of Machine Learning Research},
  volume  = {25},
  number  = {122},
  pages   = {1--14},
  year    = {2024},
  url     = {https://jmlr.org/papers/v25/23-1159.html}
}

@misc{zhang2026,
  author = {Zhang, Haihan and Wu, Wendao and Zhang, Chenheng and Li, Yanyi and Zheng, Chunyuan and Fang, Cong and Li, Haoxuan and Lin, Zhouchen},
  title  = {Joint Lower Bounds for Zeroth-Order Nonconvex Optimization on Euclidean Balls},
  year   = {2026},
  eprint = {2610.00275},
  archivePrefix = {arXiv},
  url    = {https://arxiv.org/abs/2610.00275}
}

@book{gpml,
  author    = {Rasmussen, Carl Edward and Williams, Christopher K. I.},
  title     = {Gaussian Processes for Machine Learning},
  publisher = {MIT Press},
  year      = {2006}
}

@article{beckner1993,
  author  = {Beckner, William},
  title   = {Sharp Sobolev Inequalities on the Sphere and the {Moser--Trudinger} Inequality},
  journal = {Annals of Mathematics},
  volume  = {138},
  number  = {1},
  year    = {1993}
}

@misc{duchi2013,
  author = {Duchi, John C. and Wainwright, Martin J.},
  title  = {Distance-Based and Continuum {Fano} Inequalities with Applications to Statistical Estimation},
  year   = {2013},
  eprint = {1311.2669},
  archivePrefix = {arXiv}
}

@article{dudley1967,
  author  = {Dudley, Richard M.},
  title   = {The Sizes of Compact Subsets of {Hilbert} Space and Continuity of Gaussian Processes},
  journal = {Journal of Functional Analysis},
  volume  = {1},
  number  = {3},
  year    = {1967}
}

@article{duchi2015,
  author  = {Duchi, John C. and Jordan, Michael I. and Wainwright, Martin J. and Wibisono, Andre},
  title   = {Optimal Rates for Zero-Order Convex Optimization: The Power of Two Function Evaluations},
  journal = {IEEE Transactions on Information Theory},
  year    = {2015}
}

@article{goldstein1977,
  author = {Goldstein, A. A.},
  title = {Optimization of Lipschitz Continuous Functions},
  journal = {Mathematical Programming}, volume={13}, pages={14--22}, year={1977}
}

@book{clarke1990,
  author={Clarke, Frank H.}, title={Optimization and Nonsmooth Analysis},
  publisher={SIAM}, address={Philadelphia}, year={1990}
}

@inproceedings{lin2022,
  author={Lin, Tianyi and Zheng, Zeyu and Jordan, Michael I.},
  title={Gradient-Free Methods for Deterministic and Stochastic Nonsmooth Nonconvex Optimization},
  booktitle={Advances in Neural Information Processing Systems}, volume={35},
  pages={26160--26175}, year={2022}
}

@inproceedings{chen2023,
  author={Chen, Lesi and Xu, Jing and Luo, Luo},
  title={Faster Gradient-Free Algorithms for Nonsmooth Nonconvex Stochastic Optimization},
  booktitle={Proceedings of the 40th International Conference on Machine Learning},
  series={Proceedings of Machine Learning Research}, volume={202},
  pages={5219--5233}, year={2023}
}

@inproceedings{zhang2020,
  author={Zhang, Jingzhao and Lin, Hongzhou and Jegelka, Stefanie and Sra, Suvrit and Jadbabaie, Ali},
  title={Complexity of Finding Stationary Points of Nonconvex Nonsmooth Functions},
  booktitle={Proceedings of the 37th International Conference on Machine Learning},
  series={Proceedings of Machine Learning Research}, volume={119},
  pages={11173--11182}, year={2020}
}

@inproceedings{cutkosky2023,
  author={Cutkosky, Ashok and Mehta, Harsh and Orabona, Francesco},
  title={Optimal Stochastic Non-smooth Non-convex Optimization through Online-to-Non-convex Conversion},
  booktitle={Proceedings of the 40th International Conference on Machine Learning},
  series={Proceedings of Machine Learning Research}, volume={202},
  pages={6643--6670}, year={2023}
}

@inproceedings{tian2022,
  author={Tian, Lai and Zhou, Kaiwen and So, Anthony Man-Cho},
  title={On the Finite-Time Complexity and Practical Computation of Approximate Stationarity Concepts of Lipschitz Functions},
  booktitle={Proceedings of the 39th International Conference on Machine Learning},
  series={Proceedings of Machine Learning Research}, volume={162},
  pages={21360--21379}, year={2022}
}

@inproceedings{davis2022,
  author={Davis, Damek and Drusvyatskiy, Dmitriy and Lee, Yin Tat and Padmanabhan, Swati and Ye, Guanghao},
  title={A Gradient Sampling Method with Complexity Guarantees for Lipschitz Functions in High and Low Dimensions},
  booktitle={Advances in Neural Information Processing Systems}, volume={35},
  pages={6692--6703}, year={2022}
}

@book{vershynin2018,
  author={Vershynin, Roman}, title={High-Dimensional Probability: An Introduction with Applications in Data Science},
  publisher={Cambridge University Press}, year={2018}, doi={10.1017/9781108231596}
}

@article{ghadimi2013,
  author  = {Ghadimi, Saeed and Lan, Guanghui},
  title   = {Stochastic First- and Zeroth-Order Methods for Nonconvex Stochastic Programming},
  journal = {SIAM Journal on Optimization},
  volume  = {23},
  number  = {4},
  pages   = {2341--2368},
  year    = {2013},
  doi     = {10.1137/120880811}
}

@article{cartis2012,
  author  = {Cartis, Coralia and Gould, Nicholas I. M. and Toint, Philippe L.},
  title   = {On the Oracle Complexity of First-Order and Derivative-Free Algorithms for Smooth Nonconvex Minimization},
  journal = {SIAM Journal on Optimization},
  volume  = {22},
  number  = {1},
  pages   = {66--86},
  year    = {2012},
  doi     = {10.1137/100812276}
}

@misc{dingyuan2026,
  author        = {Ding, Yixi and Yuan, Ya-Xiang},
  title         = {Exponential Deterministic Query Complexity of Goldstein Stationarity},
  year          = {2026},
  eprint        = {2609.38082},
  archivePrefix = {arXiv},
  primaryClass  = {math.OC},
  url           = {https://arxiv.org/abs/2609.38082}
}

@misc{kornowski2026stationarity,
  author={Kornowski, Guy},
  title={The Complexity of Finding Stationary Points in Nonsmooth Nonconvex Optimization},
  year={2026}, eprint={2609.17780}, archivePrefix={arXiv},
  url={https://arxiv.org/abs/2609.17780}
}

@book{kallenberg2021,
  author={Kallenberg, Olav},
  title={Foundations of Modern Probability}, edition={3},
  publisher={Springer}, year={2021},
  series={Probability Theory and Stochastic Modelling}, volume={99},
  doi={10.1007/978-3-030-61871-1}
}

@inproceedings{chewi2023complexity,
  title     = {On the Complexity of Finding Stationary Points
               of Smooth Functions in One Dimension},
  author    = {Chewi, Sinho and Bubeck, S{\'e}bastien
               and Salim, Adil},
  booktitle = {Proceedings of the 34th International
               Conference on Algorithmic Learning Theory},
  series    = {Proceedings of Machine Learning Research},
  volume    = {201},
  pages     = {358--374},
  year      = {2023},
  publisher = {PMLR},
  url       = {https://proceedings.mlr.press/v201/chewi23a.html}
}

@article{hollender2025computational,
  title   = {The Computational Complexity of Finding
             Stationary Points in Non-Convex Optimization},
  author  = {Hollender, Alexandros and Zampetakis, Manolis},
  journal = {Mathematical Programming},
  volume  = {213},
  pages   = {281--341},
  year    = {2025},
  doi     = {10.1007/s10107-024-02139-3}
}
